\documentclass[a4paper, twoside]{report} 
\usepackage[a4paper, margin=40mm]{geometry} 
\usepackage[onehalfspacing]{setspace}

\usepackage{lmodern}

 \usepackage[T1]{fontenc}

\usepackage{graphicx} 
\usepackage[utf8]{inputenc}
\usepackage{hyperref}

\usepackage{bm}
\usepackage{xspace}
\usepackage{amssymb}
\usepackage{mathtools}
\usepackage{amsthm}
\usepackage{dsfont}
\usepackage{amsmath}
\usepackage[svgnames]{xcolor}
\usepackage{enumitem} 
\usepackage{booktabs} 

\usepackage[round]{natbib}

\usepackage{listings}
\usepackage{algorithm}
\usepackage{algorithmic}
\usepackage[font=footnotesize]{caption}

\usepackage{tikz}
\usetikzlibrary{arrows, positioning}

\theoremstyle{plain}
\newtheorem{theorem}{Theorem}[chapter]
\newtheorem{proposition}[theorem]{Proposition}

\newtheorem{lemma}[theorem]{Lemma}

\theoremstyle{definition}
\newtheorem{definition}[theorem]{Definition}

\theoremstyle{remark}
\newtheorem{remark}[theorem]{Remark}

\definecolor{tabblue}{RGB}{31, 119, 180}
\definecolor{taborange}{RGB}{255, 127, 14}
\hypersetup{
    colorlinks,
    linkcolor=SteelBlue,
    citecolor=DodgerBlue,
    urlcolor=DodgerBlue
}

\newcommand{\data}{\mathcal{D}}
\newcommand{\dwarmup}{{\cal D}_\text{warmup}}
\newcommand*{\vertbar}{\rule[-1ex]{0.5pt}{2.5ex}}
\newcommand*{\horzbar}{\rule[.5ex]{2.5ex}{0.5pt}}

\newcommand{\cond}{\,\vert\,}
\newcommand{\KL}[2]{{{\bf D}_{\rm KL}}\left(#1\,\vert\vert\, #2\right)}
\DeclareMathOperator*{\argmax}{arg\,max}
\DeclareMathOperator*{\argmin}{arg\,min}
\DeclareMathOperator{\Tr}{Tr}

\renewcommand{\d}{\mathrm{d}}
\newcommand{\cov}{{\rm Cov}}
\newcommand{\var}{{\rm Var}}

\newcommand{\dimstate}{D}
\newcommand{\dimstatesub}{d}
\newcommand{\dimstatelast}{{d_{\rm last}}}
\newcommand{\dimstatehiddensub}{{d_{\rm hidden}}}
\newcommand{\dimobs}{o}
\newcommand{\dimin}{M}

\newcommand{\normdist}[3]{{\cal N}\left(#1\,\vert\, #2,\,#3\right)}

\newcommand{\real}{\mathbb{R}}

\newcommand{\myvec}[1]{\mathbf{#1}}
\newcommand{\myvecsym}[1]{\boldsymbol{#1}} 

\newcommand{\vzero}{\myvecsym{0}}

\newcommand{\namemethod}[1]{\texorpdfstring{\texttt{\color{black}{#1}}\xspace}{}}
\newcommand{\newmethod}[1]{\texorpdfstring{\texttt{\color{black}{#1}}\xspace}{}}

\newcommand{\auxv}{\psi}
\newcommand{\cModelraw}{(M.1)\xspace}
\newcommand{\cAuxraw}{(M.2)\xspace}
\newcommand{\cPriorraw}{(M.3)\xspace}
\newcommand{\cPosteriorraw}{(A.1)\xspace} 
\newcommand{\cWeightraw}{(A.2)\xspace}

\newcommand{\cModel}{(M.1: likelihood)\xspace}
\newcommand{\cAux}{(M.2: auxvar)\xspace}
\newcommand{\cPrior}{(M.3: prior)\xspace}
\newcommand{\cPosterior}{(A.1: posterior)\xspace} 
\newcommand{\cWeight}{(A.2: weighting)\xspace}

\newcommand{\cAuxname}{M.2: auxvar\xspace}
\newcommand{\cPriorname}{M.3: prior\xspace}
\newcommand{\cPosteriorname}{A.1: posterior\xspace} 
\newcommand{\cWeightname}{A.2: weight\xspace}

\newcommand{\RLSPR}[1][]{\newmethod{RL[1]-OUPR*#1}}

\newcommand{\RLPRWoLF}[1][inf]{\newmethod{WoLF+RL[#1]-PR*}}
\newcommand{\RLPRKF}[1][inf]{\namemethod{LG+RL[#1]-PR}}
\newcommand{\staticKF}{\namemethod{LG+C-Static}}

\newcommand{\CACI}[1][]{\namemethod{C-ACI#1}} 
\newcommand{\RLPR}[1][1]{\namemethod{RL[#1]-PR}}
\newcommand{\RLPRK}{\namemethod{RL[K]-PR}}
\newcommand{\CPPD}[1][]{\namemethod{CPP-OU#1}}

\newcommand{\plast}{{\bf w}}
\newcommand{\phidden}{{\boldsymbol\psi}}
\newcommand{\phiddensub}{\vz}
\newcommand{\vdlast}[1][t]{{\phi_{#1}\left(\plast\right)}}
\newcommand{\vdhidden}[1][t]{{\varphi_{#1}\left(\phiddensub\right)}}
\newcommand{\covlast}{{\boldsymbol{\Gamma}}}
\newcommand{\covhidden}{{\boldsymbol{\Sigma}}}
\newcommand{\mhidden}{{\boldsymbol{\mu}}}
\newcommand{\mlast}{{\boldsymbol{\nu}}}
\newcommand{\normdistlast}{{\normdist{\plast}{\mlast_t}{\covlast_t}}}
\newcommand{\normdisthidden}{{\normdist{\phiddensub}{\mhidden_t}{\covhidden_t}}}

\newcommand{\mWlfMd}{{\texttt{\textbf{WoLF-TMD}}}\xspace} 
\newcommand{\mWlfImq}{{\texttt{\textbf{WoLF-IMQ}}}\xspace}
\newcommand{\mAgamenoni}{{\texttt{\textbf{KF-IW}}}\xspace}
\newcommand{\mWang}{{\texttt{\textbf{KF-B}}}\xspace}

\newcommand{\mkf}{{\texttt{\textbf{KF}}}\xspace}
\newcommand{\mkfExtended}{{\texttt{\textbf{EKF}}}\xspace}
\newcommand{\ogd}{{\texttt{\textbf{OGD}}}\xspace}
\newcommand{\mAgamenoniExtended}{{\texttt{\textbf{EKF-IW}}}\xspace}
\newcommand{\mWangExtended}{{\texttt{\textbf{EKF-B}}}\xspace}

\newcommand{\valpha}{\myvecsym{\alpha}}

\newcommand{\vepsilon}{\myvecsym{\epsilon}}
\newcommand{\vzeta}{\myvecsym{\zeta}}

\newcommand{\vmu}{\myvecsym{\mu}}

\newcommand{\vkappa}{\myvecsym{\kappa}}

\newcommand{\vPsi}{\myvecsym{\Psi}}

\newcommand{\vtheta}{\myvecsym{\theta}}

\newcommand{\vTheta}{\myvecsym{\Theta}}

\newcommand{\vSigma}{\myvecsym{\Sigma}}

\newcommand{\vUpsilon}{\myvecsym{\Upsilon}}

\newcommand{\va}{\myvec{a}}
\newcommand{\vb}{\myvec{b}}

\newcommand{\ve}{\myvec{e}}

\newcommand{\vh}{\myvec{h}}

\newcommand{\vm}{\myvec{m}}

\newcommand{\vu}{\myvec{u}}

\newcommand{\vw}{\myvec{w}}

\newcommand{\vx}{\myvec{x}}

\newcommand{\vy}{\myvec{y}}

\newcommand{\vz}{\myvec{z}}

\def\vtheta{{\bm{\theta}}}
\def\va{{\bm{a}}}
\def\vb{{\bm{b}}}

\def\ve{{\bm{e}}}

\def\vh{{\bm{h}}}

\def\vm{{\bm{m}}}

\def\vu{{\bm{u}}}

\def\vw{{\bm{w}}}
\def\vx{{\bm{x}}}
\def\vy{{\bm{y}}}
\def\vz{{\bm{z}}}

\newcommand{\vA}{\myvec{A}}
\newcommand{\vB}{\myvec{B}}
\newcommand{\vC}{\myvec{C}}
\newcommand{\vD}{\myvec{D}}

\newcommand{\vF}{\myvec{F}}

\newcommand{\vH}{\myvec{H}}
\newcommand{\vI}{\myvec{I}}

\newcommand{\vK}{\myvec{K}}

\newcommand{\vP}{\myvec{P}}
\newcommand{\vQ}{\myvec{Q}}
\newcommand{\vR}{\myvec{R}}
\newcommand{\vS}{\myvec{S}}

\newcommand{\vU}{\myvec{U}}
\newcommand{\vV}{\myvec{V}}
\newcommand{\vW}{\myvec{W}}
\newcommand{\vX}{\myvec{X}}

\newcommand{\vY}{\myvec{Y}}
\newcommand{\vZ}{\myvec{Z}}

\begin{document}

\begin{titlepage} 
	\begin{center}
	{\LARGE Adaptive, Robust and Scalable Bayesian Filtering for Online Learning \par} 
	\vspace{2cm}
	{\Large Gerardo Duran-Martin \\[2cm]}
     {\large Supervised by Alexander Shestopaloff and Kevin Murphy \par}
	
    \vfill
	{\Large School of Mathematical Sciences \par}
    \vspace{1cm}
    {\includegraphics[width=0.5\linewidth]{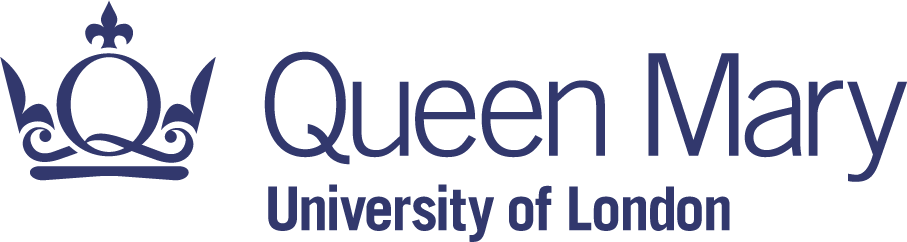} \par}
	\vspace{1cm}
    {\large April 2025 \par}
    \vspace{1cm}
    {Submitted in partial fulfillment of the requirements \\ of the Degree of Doctor of Philosophy \par}
    \vspace{1cm}
    {This work was supported by
    UKRI — Engineering and Physical Sciences Research Council}
    \thispagestyle{plain} 
    \end{center}
\end{titlepage}
\newpage
\chapter*{Statement of Originality}

I, Gerardo Duran Martin, confirm that the research included within this thesis is my own work or that where it has been carried out in collaboration with, or supported by other, that this is duly acknowledged below and my contribution indicated. Previously published material is also acknowledged below. \\ \\
\noindent I attest that I have exercised reasonable care to ensure that the work is original and does not to the best of my knowledge break any UK law, infringe any third party’s copyright or other Intellectual Property Right, or contain any confidential material. \\ \\    
\noindent  I accept that Queen Mary University of London has the right to use plagiarism detection software to check the electronic version of the thesis. \\ \\    
\noindent I confirm that this thesis has not been previously submitted for the award of a degree by this or any other university. \\ \\
\noindent The copyright of this thesis rests with the author and no quotation from it or information derived from it may be published without the prior written consent of the author. \\ \\
\noindent Signature: Gerardo Duran-Martin\\
\noindent Date: April 2025.\\ \\    
\noindent Details of collaboration and publications:\\
\begin{enumerate}
    \item The contents in Chapter 3 are based on the paper \cite{duranmartin2024-bocl}.
    \item The contents in Chapter 4 are based on the paper \cite{duranmartin2024-wlf}.
    \item The contents in Chapter 5 are based on the papers \cite{duranmartin2022-subspace-bandits,cartea2023sharpbayes,chang2023lofi}.
\end{enumerate}

\newpage
\chapter*{Abstract}

In this thesis, we introduce Bayesian filtering as a principled framework for  
tackling diverse sequential machine learning problems,
including 
online (continual) learning,
prequential (one-step-ahead) forecasting, and
contextual bandits.
To this end, this thesis addresses key challenges in applying Bayesian filtering to these problems:  
adaptivity to non-stationary environments,
robustness to model misspecification and outliers, and
scalability to the high-dimensional parameter space of deep neural networks.
We develop novel tools within the Bayesian filtering framework to address each of these challenges, including:
(i) a modular framework that enables the development adaptive approaches for online learning;
(ii) a novel, provably robust filter with similar computational cost to standard filters,
that employs Generalised Bayes; and
(iii) a set of tools for sequentially updating model parameters using approximate second-order optimisation methods that
exploit the overparametrisation of high-dimensional parametric models such as neural networks.
Theoretical analysis and empirical results demonstrate the improved performance of our methods in dynamic,
high-dimensional, and misspecified models.

\clearpage
\vspace*{\fill}
\thispagestyle{empty} 
\begin{quotation}
\em 
``You know, allowing awareness for something that's unlikely is not a disease,'' she said.
``If you're talking about a filter, you should understand how they work.
Optimal filters will still block a few things that you actually wanted to go through---and
will still allow for some things that you wanted to block to instead go through.
That's for an optimal filter.'' [...]
A species like ours, with survival so clearly based on intelligence and information,
should not block the risk of blocking and throwing away potentially valuable ideas.

\medskip
\raggedleft
--Professor Karl Deisseroth, ``Connections''
\end{quotation}
\vspace*{\fill}

\newpage
\chapter*{Acknowledgements}

As a good friend of mine once wrote, this thesis would not have been possible without a combination of luck and hard work. After arriving in the UK to pursue my MSc studies during a global pandemic, I was fortunate to meet Alex Shestopaloff, who, despite my lack of research experience at the time, kindly welcomed me as one of his PhD students. Around the same time, I had the pleasure of working with Kevin Murphy during the Google Summer of Code programme in 2021. It was there that Kevin posed a question which has since become central to this thesis: “Did you know you can train a neural network with a Kalman filter?”

To Alex and Kevin, I owe my deepest gratitude for your trust, guidance, and patience throughout this process.
Your mentorship has been invaluable, and this thesis would not exist without your support.

I thank Eftychia Solea and Arnaud Doucet
for their insightful comments and constructive feedback,
which have enhanced the quality and clarity of this thesis.

I extend my heartfelt thanks to Álvaro Cartea, the director of the Oxford-Man Institute of Quantitative Finance (OMI), for opening the institute's doors to me and for demonstrating the value of precision and care in academic research. 
I am equally grateful to members of the OMI, with whom I had the privilege of collaborating.
In particular, I would like to acknowledge Alvaro Arroyo, Patrick Chang, Sergio Calvo-Ordoñez, and Fayçal Drissi for their friendship and for sharing their knowledge, from which I have greatly benefited.

To Leandro Sánchez-Betancourt, from whom I learnt to appreciate the writing process and to take care of the smallest mathematical details.
It has been a true privilege to work alongside you and to call you a friend and a co-author.

I am also grateful to AHL Man Group for providing me with a summer home to learn and explore research ideas in the real-world.
During my time there, I learnt the importance of prioritising mathematical simplicity. I would like to give special thanks to Anthony Ledford for the fascinating and thought-provoking discussions on mathematics, statistics, and their applications in finance.

The Google TPU Research Cloud's computational resources were crucial for the execution of the Jax code, and I am thankful for their support.

To my parents, I owe so much. From my late father, I learned the value of taking pride in one’s work and the joy of craftsmanship. From my mother, I learned the importance of being consistent.

Finally, to Andi Flores---none of this would have been possible without your unwavering love, support, and trust.
You have reminded me to cherish life’s simple pleasures and to embrace the adventures we have shared across the world.
For that, and for so much more, I thank you.

\tableofcontents
\listoffigures
\listoftables

\newpage
\chapter{Introduction}
\label{ch:introduction}


Sequential problems are central to machine learning and arise in a variety of settings, including
test-time adaptation \citep{schirmer2024ssmtta},
prequential forecasting \citep{gama2008streamlearn},
neural bandits \citep{riquelme2018deepbanditshowdown},
online continual learning \citep{dohare2024continualbackprop}, and
deep reinforcement learning \citep{asadi2024td}.
These problems can be formulated as online learning tasks,
where an agent observes a time-indexed sequence of data and predicts the next unobserved value in the sequence \citep[Chapter 1]{zhang2023ltbook}.
Typically, predictions are informed by past observations and may depend on additional exogenous variables, referred to as \textit{features}.

Filtering methods
provide a principled framework for tackling such challenges in sequential decision-making.
These methods focus on inferring the unknown state of a dynamic system from noisy observations,
making them well-suited for parametric online learning tasks where the unknown state
are taken to be the model parameters.
Among filtering methods, the Kalman filter (KF) has been particularly influential, serving as a foundational algorithm that has inspired numerous extensions and applications \citep{leondes1970-kfapollo, grewal2010kfhist}.
Figure \ref{fig:dynamical-system-intro} presents an example of a two-dimensional
dynamical system (top panel), which seeks  to recover from a one-dimensional projection (bottom panel).
\begin{figure}[htb]
    \centering
    \includegraphics[scale=0.7]{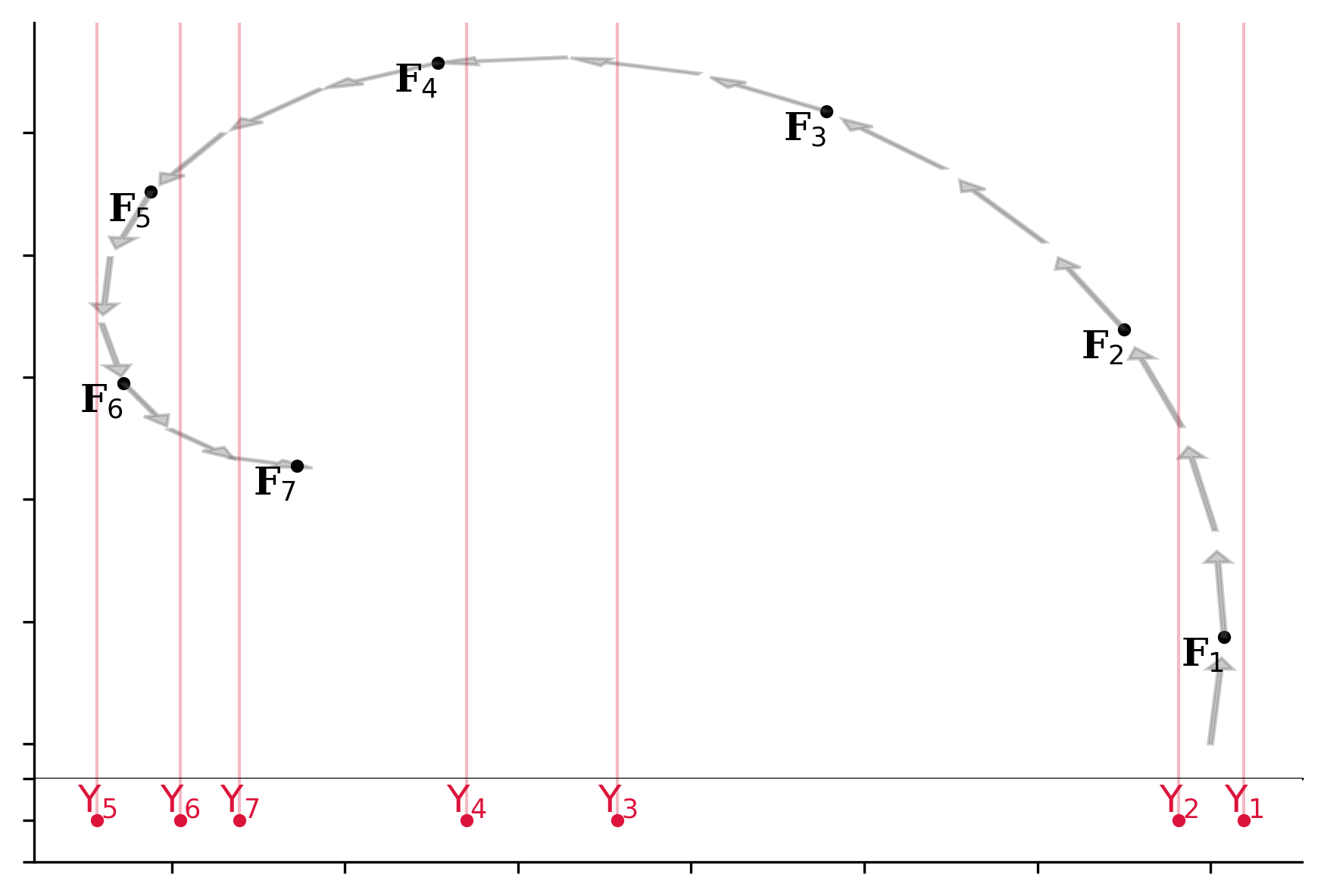}
    \caption{
    One-dimensional projection of a noisy-two dimensional dynamical system.
    In the top panel,
    the gray
    arrows represent the underlying evolution of the system and
    the black dots show the sampled (but unknown) locations of the system.
    In the bottom panel,
    the red dots show the observed measurements projected onto a one-dimensional plane.
    The red vertical lines denote the projection from latent space to observation space.
    }
    \label{fig:dynamical-system-intro}
\end{figure}

Arguably, the widespread appeal of filtering methods stems from three key properties:
extensibility, a Bayesian foundation, and broad applicability.
First, the extensibility of filtering algorithms has enabled researchers to develop numerous variants of the KF.
For instance, extensions have made the KF (i) \textit{robust} to outlier measurements \citep{west1981robustkf},
(ii) \textit{adaptive} to non-stationary environments \citep{mehra1972adaptivefiltering}, and
(iii) \textit{scalable} to high-dimensional state and observation spaces \citep{evensen1994enkf}.
Variants of the KF remain an active area of research, with efforts to further improve robustness, adaptability, and scalability
\citep[see e.g.,][]{tao2023robustkf, zhu2022slidingkf, vilmarest2024viking, chen2022enkfdnn, schmidt2023rankreducedkf, greenberg2024okf}.

Second, the Bayesian interpretation of filtering methods,
commonly referred to as Bayesian filters \citep[][Chapter 1]{sarkka2023filtering},
provides a probabilistic framework for sequential decision-making.
This perspective has inspired algorithms
such as particle filters \citep{doucet2009tutorialpf}
and variational-Bayes filters \citep{sarkka2009kfvb},
which extend the applicability of Bayesian filters to non-linear and non-Gaussian systems.

Third, filtering methods are widely applicable beyond their traditional use in physical systems.
Applications include financial modelling \citep{wells2013kalmanfinance},
signal processing \citep[Chapter 3]{basseville1993detectionchanges},
and other domains where dynamic systems must be monitored and predicted.

Motivated by these three properties, in this thesis, we advocate for a more prominent role of filtering methods in machine learning.
Specifically, we leverage filtering techniques to design novel online learning algorithms that are
(i) robust, (ii) adaptive, and (iii) scalable.
Our approach treats the Bayesian framework as an algorithmic tool for
``rationalising and formalising experience accumulation'' \citep{peterka1981bayesianidentification, breiman2001twocultures}.
In contrast to the classical filtering perspective, we do not assume expert knowledge of the data-generating process.
Instead, the resulting methods aim to maximise predictive performance by dynamically learning parametric
(and potentially non-linear) models based solely on available observations and features.

The methods and frameworks developed in this thesis build upon the foundational principles discussed above,
extending them to address challenges in robustness, adaptability, and scalability within online learning.
The rest of this thesis is organised as follows:

Chapter \ref{ch:recursive-bayes} introduces basic concepts used throughout the thesis,
including prequential forecasting and Bayes' rule as a mechanism for sequentially updating model parameters.
This chapter revisits classical results in recursive statistical learning,
extends these results to non-linear observation models such as neural networks, and
examines recent advances in recursive estimation.

Chapter \ref{ch:adaptivity} develops a framework for adaptive online learning,
building on hierarchical Bayesian models and filtering methods, as outlined in \cite{duranmartin2024-bocl}.

Chapter \ref{ch:robustness} focuses on the robustness of the methods introduced in Chapters \ref{ch:recursive-bayes} and \ref{ch:adaptivity} against outliers and misspecified observation models.
After reviewing prior work in this area, we present a lightweight approach proposed in \cite{duranmartin2024-wlf},
which leverages the generalised-Bayes principle by replacing the traditional log-likelihood in posterior computation with a more flexible loss function.

Chapter \ref{ch:scalability} addresses the scalability challenges of Bayesian filtering methods for online learning.
Drawing from the work in \cite{duranmartin2022-subspace-bandits, cartea2023sharpbayes, chang2023lofi},
we propose three novel strategies: training a linear subspace of model parameters, using low-rank posterior covariance matrices,
and employing last-layer methods to separate feature transformation from observation approximation.

Finally, Chapter \ref{ch:conclusion} summarises the key findings of this thesis and outlines directions for future research.

\chapter{Recursive Bayesian online learning}
\label{ch:recursive-bayes}

In this chapter, we approach the problem of parametric online learning from a Bayesian filtering perspective.
Here, the focus is the recursive estimation of the posterior density over model parameters.
We review how the Kalman filter can be viewed as a form of Bayesian online learning for linear models,
and explore methods to incorporate non-linear measurement models, such as those used in neural networks.

The remainder of this chapter establishes the foundations for both sequential closed-form and fixed-point methods
to compute or approximate the posterior density.

Section \ref{sec:prequential-inference} introduces the problem of prequential inference and
the recursive estimation of model parameters.  
Section \ref{sec:the-multivariate-gaussian} reviews the multivariate Gaussian and its properties, which we use
throughout the thesis.  
Section \ref{sec:linear-gaussian-recursive} presents a recursive approach for updating model parameters
assuming that both, the parameter dynamics model and the measurement model, are Gaussian and linear.
This includes a recursive variant of the Ridge regression algorithm.  
Section \ref{sec:nonlinear-nongaussian-recursive} relaxes the assumptions of linear and Gaussian measurement models and presents various methods to obtain closed-form approximations of the posterior density under non-linear measurement models.
Section \ref{sec:ssm} introduces state-space models in the context of sequential learning and demonstrates how filtering
generalises probabilistic sequential learning of model parameters.  
Finally, Section \ref{sec:filtering-as-learning} concludes the chapter with an outlook for the remainder of this thesis.

\section{Recursive Bayesian inference}
\label{sec:prequential-inference}
Consider a sequence of measurements $\vy_{1:t} = (\vy_1, \ldots, \vy_{t})$ with $ \vy_i \in {\cal Y} \subseteq \real^\dimobs$ and
features $\vx_{1:t} = (\vx_1, \ldots, \vx_t)$ with $ \vx_i \in\real^{\dimin}$,
where $\dimobs,\dimin\in\mathbb{N}$.
Let $\data_t = (\vx_t, \vy_t)$ be a  datapoint and 
$\data_{1:t} = (\data_1, \ldots, \data_t)$ be the dataset at time $t$.
We are interested in the one-step-ahead (prequential) forecast \citep{gama2008streamlearn} for $\vy_{t+1}$
conditioned on the feature $\vx_{t+1}$ and the data up to time $t$, i.e., $\data_{1:t}$.
In our setting,
one observes $\vx_{t+1}$ just before observing $\vy_{t+1}$;
thus, to make a prediction about $\vy_{t+1}$, we have both $\data_{1:t}$ and $\vx_{t+1}$ at our disposal.\footnote{
The features $\vx_{t+1}$ and measurements $\vy_{t+1}$ can span different time-frames.
For example, $\vx_{t+1}$ can be the state of the stock market at a fixed date and
$\vy_{t+1}$ is the return on a stock some days into the future.
}

To establish a link between the features $\vx_{t+1}$ and the measurement $\vy_{t+1}$,
we consider a probability density of the measurement $p(\vy_{t+1} \cond \vtheta, \vx_{t+1})$
such that
\begin{equation}
\label{eq:predictive-mean}
    \int \vy_t \, p(\vy_t \cond \vtheta_t, \vx_t)\d\vy_t = h(\vtheta, \vx_t).
\end{equation}
Here
$h: \real^\dimstate\times\real^\dimin \to \real^\dimobs$
is called the measurement function.
Following a probabilistic approach, an estimate for $\vy_{t+1}$, having
data $\data_{1:t}$,
features $\vx_{t+1}$, and
the measurement function $h$
is given by the
posterior predictive mean
\begin{equation}\label{eq:prequential-forecast-measure}
    \hat{\vy}_{t+1} :=
    \mathbb{E}_{p}[h(\vtheta, \vx_{t+1})\cond\data_{1:t}] =
    \int h(\vtheta, \vx_{t+1})\,p(\vtheta\cond\data_{1:t})\d\vtheta,
\end{equation}
where $p(\vtheta \cond \data_{1:t})$
is the posterior density over model parameters.

Throughout this work, the notation $p(\vy_t \cond \vtheta, \vx_t)$ represents the probability density
of the measurement $\vy_t$, given the latent (unknown) model parameters $\vtheta \in \real^\dimstate$ and the features $\vx_t$.
Similarly, $p(\vtheta \cond \data_{1:t})$ represents the probability density of the model parameters
$\vtheta$ given the data up to time $t$---it reflects the belief over the model parameters after having 
seen $t$ datapoints.

A natural approach to construct the posterior density $p(\vtheta \cond \data_{1:t})$,
whenever $\data_{1:T}$ arrives in a stream, i.e., one datapoint $\data_t$ at a time,
is through Bayes' rule---suppose
we have access to $p(\vtheta \cond \data_{1:t-1})$, with $t < T$, and we are presented with $\data_t = (\vx_t, \vy_t)$,
which we model through the likelihood $p(\vy_t \cond \vtheta, \vx_t)$
Then,
\begin{equation}\label{eq:recursive-bayes}
\begin{aligned}
    p(\vtheta \cond \data_{1:t})
    &\propto  p(\vtheta \cond \data_{1:t-1})\,p(\data_t \cond \vtheta)\\
    &=
    \underbrace{p(\vtheta \cond \data_{1:t-1})}_\text{prior density}\,
    \underbrace{p(\vy_{t} \cond \vtheta, \vx_t)}_\text{likelihood},
\end{aligned} 
\end{equation}
where the second line in \eqref{eq:recursive-bayes} follows from the assumption of an exogenous $\vx_t$.

Given the initial condition $p(\vtheta) = p(\vtheta \cond \data_{1:0})$, with $\data_{1:0} = \{\}$,
recursive and closed-form estimation of \eqref{eq:recursive-bayes}
is obtained whenever $p(\vtheta \cond \data_{1:t-1})$ and $p(\vy_t \cond \vtheta, \vx_t)$ are conjugate, i.e.,
the functional form of $p(\vtheta \cond \data_{1:t})$ is that of $p(\vtheta \cond \data_{1:t-1})$ \citep[Section 3.3]{robert2007bayesianbook}.
As a consequence, the parameterisation over model parameters is characterised in such a way
that it only depends on a set of parameters, which are recursively updated.

\subsection{Example tasks}
\label{sec:tasks}
Here,
we give
some examples of machine learning tasks which can be tackled using recursive inference.
We group these examples into unsupervised tasks and supervised tasks.

\subsubsection{Unsupervised tasks}

Unsupervised tasks 
involve estimating unobservable quantities of interest from the data $\data_{1:t}$.
Below, we present three common tasks in this category.

\paragraph{\underline{\texttt{Segmentation}}} 

Segmentation involves partitioning the data stream into contiguous subsequences or ``blocks'' $\{\data_{1:t_1},
\data_{t_1+1:t_2}, \ldots\}$, where  the DGP
for each block
is governed by a sequence of unknown functions \citep{barry1992ppm}.
The goal is to determine the points in time when a new block begins,
known as changepoints.
This is useful in many applications, such as finance,
where detecting changes in market trends is critical (see e.g., \cite{arroyo2022dynamic}).
In this setting, non-stationarity is assumed to be abrupt and occurring at unknown points in time.
We study an example in Section \ref{experiment:segmentation}.
For a survey of segmentation methods, see e.g,
\cite{aminikhanghahi2017changepointsurvey,gupta2024changepointsurvey}.

\paragraph{\underline{\texttt{Filtering using state-space models (SSM})}}
Filtering estimates an underlying latent state $\vtheta_t$
that evolves over time (often representing a meaningful concept). 
The posterior estimate of $\vtheta_t$
is computed by applying Bayesian
inference to the corresponding
state space model (SSM),
which determines the choice of \cModel,
and how the state changes over time,
through the choice of \cPrior.
Examples include estimating the state of the atmosphere \citep{evensen1994enkf},
tracking the position of a moving object \citep{battin1982apollo},
or recovering a signal from a noisy system \citep{basseville1993detectionchanges}.
In this setting, non-stationarity is usually
assumed to be continuous and occurring at possible time-varying rates.
For a survey of filtering methods, see e.g., \cite{chen2003bayesianfiltersurvey}.

\paragraph{\underline{\texttt{Segmentation using Switching state-space models (SSSM)}}}
In this task, the modeller extends the standard SSM with a set of discrete latent variables $\psi_t \in \{1,\ldots,K\}$,
which may change value at each time step according to a state transition matrix.
The parameters of the rest of the DGP
depend on the discrete state $\psi_t$.
The objective is to infer the sequence of underlying discrete states that best ``explains'' the observed data
\citep{ostendorf1996hmm,ghahramani2000vsssm,beal2001infinite,fox2007sticky,van2008beam,linderman2017rslds}.
In this context, non-stationarity arises from the switching behaviour of the underlying discrete process.

\subsubsection{Supervised tasks}

Supervised tasks involve predicting a measurable outcome $\vy_t$. 
Unlike unsupervised tasks,
this  allows the performance of the model to be assessed objectively, since we can compare the prediction
to the actual observation.
We present three common tasks in this category below.

\paragraph{\underline{\texttt{Prequential forecasting}}}  
Prequential (or one-step-ahead) forecasting \citep{gama2008streamlearn} seeks to predict the value $\vy_{t+1}$ given $\data_{1:t}$ and $\vx_{t+1}$. 
This is distinct from  time-series forecasting, which typically does not consider exogenous variables $\vx_t$,
and thus can forecast (or  ``roll out'')
many steps
into the future.
We study an example in Section \ref{experiment:prequential}.
For a survey on prequential forecasting under non-stationarity, see e.g., \cite{lu2018preqnonstationarysurvey}.

\paragraph{\underline{\texttt{Online continual learning (OCL)}}}  
OCL is 
a broad term used for learning regression or classification models online, typically with neural networks. These methods usually assume that the underlying data generating mechanism could shift.
The objective of OCL methods 
is to train a  model that performs consistently across both past and future data,
rather than just focusing on future forecasting \citep{cai2021ocl}.
The changepoints (corresponding to different ``tasks'') may or may not be known.
This setting addresses the stability-plasticity dilemma,
focusing on retaining previously learned knowledge while adapting to new tasks.
We study an example of OCL for classification, when the task boundaries are not known, in Section \ref{sec:periodic-drifts}.
For a survey on recent methods for OCL, see e.g., \cite{gunasekara2023surveyocl}.

\paragraph{\underline{\texttt{Contextual bandits}}}  
In contextual bandit problems, the agent is presented with features $\vx_{t+1}$,
and must choose an action (arm) that yields the highest expected reward \citep{li2010contextualbandits}. We let $\vy_{t+1} \in \mathbb{R}^A$ 
where $A > 2$ is the number of possible actions; this is a vector where the $a$-th entry contains the reward one would have obtained had one chosen arm $a$.
Let $\vy_{t}^{(a)}$ be the observed reward at time $t$ after choosing arm $a$, i.e., the $a$-th entry of $\vy_t$.
A popular approach for choosing the optimal action (while tackling the exploration-exploitation tradeoff) at each step
is  Thomson sampling (TS) \citep{thompson1933sampling}, which in our setting works as follows:
first, sample a parameter
vector from the posterior,
$\tilde\vtheta_{t}$ from $p(\vtheta_{t} \cond \data_{1:t})$; then, 
greedily choose the best arm
(the one with the highest expected payoff),
$a_{t+1} = \argmax_{a} \hat{\vy}_{t+1}^{(a)}$,
 where $\hat\vy_{t+1} = h(\tilde\vtheta_{t}; \vx_{t+1})$;
 and $\hat{\vy}_{t+1}^{(a)}$ is the $a$-th entry of $\hat{\vy}_{t+1}$;
finally,  receive a reward $\vy_{t+1}^{(a_{t+1})}$.
The goal is to select a sequence of arms $\{a_1, \ldots, a_T\}$ that maximises the cumulative reward $\sum_{t=1}^T \vy_{t}^{(a_t)}$.
TS for contextual bandits has been used in a number of papers, see e.g.,
\cite{mellor2013changepointthompsonsampling,duranmartin2022-subspace-bandits,cartea2023bandits,alami2023banditnonstationary, liu2023nonstationarybandits}.

\section{The multivariate Gaussian}
\label{sec:the-multivariate-gaussian}
An important concept that we use throughout this thesis is that of the multivariate Gaussian density.
\begin{definition}
    Let $\vx \in \real^\dimobs$, $\vmu \in \real^\dimobs$, and $\vSigma$ an $(\dimobs\times\dimobs)$ positive-definite matrix.
    The density function of a multivariate Gaussian with mean $\vmu$ and covariance matrix $\vSigma$ is
    \begin{equation}
        {\cal N}(\vx \cond \vmu, \vSigma) =
        (2\,\pi)^{-\dimobs/2}\,\vert\vSigma\vert^{-1/2}\,\exp\left(-\frac{1}{2}(\vx - \vmu)^\intercal \vSigma^{-1} (\vx - \vmu)\right).
    \end{equation}
\end{definition}

Gaussian densities are, in many cases, used for computational reasons.
This is because it has mathematical properties that allows us to work with equations of the form
\eqref{eq:prequential-forecast-measure}
and retain a Gaussian structure.
In particular, the following two propositions will be extensively used through this chapter to derive recursive updates.
\begin{proposition}\label{prop:joint-conditional-gauss}
    Let $\vx\in\real^\dimin$ and $\vy\in\real^\dimobs$ be two random vectors such that
    $p(\vx) = {\cal N}(\vx \cond \vm, \vP)$ and
    $p(\vy \cond \vx) = {\cal N}(\vy \cond \vH\,\vx + \vb, \vS)$. The joint density for $(\vx, \vy)$
    is also multivariate Gaussian $p(\vx, \vy) = {\cal N}\left((\vx, \vy) \cond \vmu_{\vx, \vy}, \vSigma_{\vx, \vy}\right)$
    with
    \begin{align}
        \vmu_{\vx, \vy} &=
        \begin{bmatrix}
            \vm & \vH\,\vm + \vb
        \end{bmatrix},\\
        \vSigma_{\vx, \vy} &=
        \begin{bmatrix}
            \vP & \vP\,\vH^\intercal\\
            \vH\,\vP & \vH\,\vP\vH^\intercal + \vS
        \end{bmatrix}.
    \end{align}
\end{proposition}
\begin{proof}
    See Appendix A in \citet{sarkka2023filtering}.
\end{proof}

\begin{proposition}\label{prop:marginals-cond-gauss}
    Let $\vx\in\real^\dimin$ and $\vy\in\real^\dimobs$ be two random vectors with joint density function
    \begin{equation}
        p(\vx, \vy) = {\cal N}\left(
        \begin{bmatrix}
            \vx \\ \vy
        \end{bmatrix}
        \,\Big\vert\,
        \begin{bmatrix}
            \va \\
            \vb
        \end{bmatrix},
        \begin{bmatrix}
            \vA & \vC^\intercal \\
            \vC & \vB
        \end{bmatrix}
        \right).
    \end{equation}
    Then
    \begin{align}
        p(\vx) &= {\cal N}(\vx \cond \va, \vA),\\
        p(\vy) &= {\cal N}(\vy \cond \vb, \vB),\\
        p(\vx \cond \vy) &= {\cal N}(\vx \cond \va + \vC\vB^{-1}(\vy - \vb), \vA - \vC\vB^{-1}\vC^\intercal),
        \label{eq:mvn-x-cond-y}\\
        p(\vy \cond \vx) &= {\cal N}(\vy \cond \vb + \vC^\intercal\vA^{-1}(\vx - \va), \vB - \vC^\intercal\vA\vC).
        \label{eq:mvn-y-cond-x}
    \end{align}
\end{proposition}
\begin{proof}
    See Appendix A in \citet{sarkka2023filtering}.
\end{proof}

\section{Linear and Gaussian measurement models}
\label{sec:linear-gaussian-recursive}
In this section, we present an algorithm to compute recursive updates whenever
the prior density at time $t$ is Gaussian, i.e., $p(\vtheta \cond \data_{1:t-1}) = {\cal N}(\vtheta \cond \vmu_{t-1}, \vSigma_{t-1})$,
and the measurement model is linear and univariate Gaussian, i.e.,
$p(y_t \cond \vtheta, \vx_t) = {\cal N}(y_t \cond \vx_t^\intercal\,\vtheta, \beta^2)$
with known variance $\beta^2$.
Here $h(\vtheta, \vx_t) = \vx_t^\intercal~\vtheta$ and, as we will show,
$\mathbb{E}_p[h(\vtheta, \vx_t) \cond \data_{1:t}] = h(\vmu_t, \vx_t) = \vx_t^\intercal\,\vmu_t$.
The assumption of Gaussianity is convenient because it
(i) provides a prior over each model parameter that spans the real line,
(ii) preserves closed-form updates to the posterior $p(\vtheta \cond \data_{1:t})$, and
(iii) its first and second moments fully characterise the density.

A classical statistical model that can be interpreted 
having a Gaussian prior for the model parameters $p(\vtheta)$ and
Gaussian measurement model $p(y_t \cond \vtheta, \vx_t)$ is Ridge regression.
We recall the Ridge regression in the proposition below.
\begin{proposition}[Ridge regression]\label{prop:ridge-regression}
    Consider the dataset $\data_{1:T}$ such that
    $\data_t = (\vx_t \in \real^\dimin, y_t\in\real)$,
    with measurement model
    $p(y_t \cond \vtheta, \vx_t) = {\cal N}(y_t \cond \vx_t^\intercal\,\vtheta, \beta^2)$,
    with $\beta >0$, and and prior $p(\vtheta) = \normdist{\vtheta}{{\bf 0}}{\alpha^2\,\vI}$, for $\alpha > 0$.
    Suppose $\cov\left(y_i, y_j \cond \vtheta, \vx_i, \vx_j\right) = 0$ for all $i \neq j$.
    Then,
    \begin{equation}\label{eq:ridge-regression}
        \hat{\vmu}_T := \mathbb{E}[\vtheta \cond \data_{1:T}] = (\vX^\intercal\,\vX + \lambda\,\vI)^{-1}\vX^\intercal\,\vY
    \end{equation}
    with
    $\vX = [\vx_1^\intercal, \ldots, \vx_T^\intercal]^\intercal$,
    $\vY = [y_1, \ldots, y_T]^\intercal$, and
    $\lambda = \beta^2 / \alpha^2$.
\end{proposition}
\begin{proof}
    See Section 3.4.1 in \cite{hastie2009esl}. 
\end{proof}

If the data $\data_{1:T}$ arrives in a stream,
we can estimate the posterior density $p(\vtheta \cond \data_{1:t})$ for $t=1,\ldots, T$ using Bayes' rule \eqref{eq:recursive-bayes}.


\begin{proposition}[Multivariate recursive Bayesian linear regression]\label{prop:recursive-linear-regression}
    Consider the dataset $\data_{1:T}$
    with $\data_t = (\vx_t \in \real^\dimin, \vy_t\in\real^\dimobs)$,
    and the measurement model
    $p(\vy_t \cond \vtheta, \vx_t) = {\cal N}(\vy_t \cond \vx_t^\intercal\,\vtheta, \vR_t)$ with known covariance matrix $\vR_t$.
    Suppose $\cov\left(\vy_i, \vy_j \cond \vtheta\right) = 0$ for $i \neq q$.
    Let $p(\vtheta) = {\cal N}(\vtheta \cond \vmu_0, \vSigma_0)$ be the initial prior density over the model parameters.
    Then, the posterior density at time $t$ takes the form $p(\vtheta \cond \data_{1:t}) = {\cal N}(\vtheta \cond \vmu_t, \vSigma_t)$ with
    \begin{equation}
    \begin{aligned}\label{eq:recursive-linreg}
        \vS_t &= \vx_t^\intercal\vSigma_{t-1}\vx_t + \vR_t,\\
        \vK_t &= \vSigma_{t-1}\vx_t\vS_t^{-1},\\
        \vmu_t &= \vmu_{t-1} + \vK_t\,(\vy_t - \vx_t^\intercal\,\vmu_{t-1}),\\
        \vSigma_t &= \left(\vI - \vK_t\vx_t\right)\vSigma_{t-1}.
    \end{aligned}
    \end{equation}
\end{proposition}
\begin{proof}
    Following an induction argument,
    suppose $p(\vtheta \cond \data_{1:t-1}) = {\cal N}(\vtheta \cond \vmu_{t-1}, \vSigma_{t-1})$.
    From proposition \ref{prop:joint-conditional-gauss}, the joint density for $(\vtheta, \vy_t)$,
    conditioned on $\data_{1:t-1}$ and $\vx_t$ takes the form of a Gaussian density with mean
    \begin{equation}
    \begin{bmatrix}
        \vmu_{t-1}^\intercal & \vx_t^\intercal\,\vmu_{t-1}
    \end{bmatrix}^\intercal
    \end{equation}
    and covariance matrix
    \begin{equation}
        \begin{bmatrix}
            \vSigma_{t-1} & \vSigma_{t-1}\,\vx_t\\
            \vx_t^\intercal\,\vSigma_{t-1} & \vx_t^\intercal\,\vSigma_{t-1}\vx_t + \vR_t
        \end{bmatrix}.
    \end{equation}
    Then, by Proposition \ref{prop:marginals-cond-gauss}, the density for $\vtheta$ conditioned on
    $\data_{1:t-1}$, $\vx_t$ and $\vy_t$ is Gaussian with mean and covariance matrix given by
    \begin{equation}\label{eq:recursive-linreg-mu-sigma-update}
    \begin{aligned}
        \vS_t &= \vx_t^\intercal\,\vSigma_{t-1}\vx_t + \vR_t,\\
        \vmu_t &= \vmu_{t-1} + \vSigma_{t-1}\vx_t\vS_{t}^{-1}(\vy_t - \vmu_{t-1}^\intercal\vx_t),\\
        \vSigma_{t-1} &= \vSigma_{t-1} - \vSigma_{t-1}\vx^\intercal\vSigma^{-1}_t\vx^\intercal\vSigma_{t-1}.
    \end{aligned}
    \end{equation}
\end{proof}
\begin{proposition}\label{prop:recursive-and-static-ridge}
    The estimate of the mean $\vmu_T$ in \eqref{eq:recursive-linreg}
    is the \textit{exact} posterior mean, i.e., $\vmu_T = {\mathbb E}[\vtheta \cond \data_{1:T}]$,
    and matches the Ridge regression estimate
    \eqref{eq:ridge-regression} whenever
    $\dimobs = 1$, 
    $\vmu_0 = {\bf 0}$, 
    $\vSigma_0 = \alpha^2\,\vI$, and
    $\vR_t = \beta^2$ for all $t \in \{1,\ldots, T\}$.
\end{proposition}
\begin{proof}
    See \cite{kelly1990recursiveridge, ismail1996equivalence}.
\end{proof}
Proposition \ref{prop:recursive-and-static-ridge} shows that, in the linear and Gaussian case,
the final estimate of the recursive update of model parameters 
matches the \textit{offline} batch estimate of model parameters.
We summarise the recursive Bayesian linear regression method in Algorithm \ref{algo:recursive-bayesian-linreg}.
\begin{algorithm}[htb]
\begin{algorithmic}[1]
    \REQUIRE $\data_{1:T}$, $p(\vtheta) = {\cal N}(\vtheta \cond \vmu_{0}, \vSigma_{0})$
    \FOR{$t=1,\ldots,T$}
        \STATE $\vS_t = \vx_t^\intercal\,\vSigma_{t-1}\vx_t + \vR_t$
        \STATE $\vmu_t \gets \vmu_{t-1} + \vSigma_{t-1}\vx_t\vS_{t}^{-1}(\vy_t - \vmu_{t-1}^\intercal\vx_t)$
        \STATE $\vSigma_t \gets \vSigma_{t-1} - \vSigma_{t-1}\vx_t^\intercal\vSigma^{-1}_t\vx_t^\intercal\vSigma_{t-1}$
        \STATE $p(\vtheta \cond \data_{1:t}) = \normdist{\vtheta}{\vmu_t}{\vSigma_t}$
    \ENDFOR
\end{algorithmic}
\caption{
    Pseudocode for the recursive Bayesian linear regression.
}
\label{algo:recursive-bayesian-linreg}
\end{algorithm}

The following example shows a numerical analysis of the results
in Proposition \ref{prop:ridge-regression} and Proposition \ref{prop:recursive-linear-regression}
\subsection{Example: recursive and batch Ridge regression}
Consider the dataset ${\cal D}_{1:T}$ with ${\cal D}_t = (\vx_t \in \real^2, y_t\in\real)$
such that $y_t = \vx_t^\intercal\,\vtheta + \ve_t$,
with $\vtheta \in \real^2$ the true model parameters and $p(\ve_t) = {\cal N}(\ve_t \cond 0, 1)$.
The left panel in Figure \ref{fig:recursive-ridge-params-and-rmse} shows the estimate for $\vmu_t \in \real^2$ along with the
Ridge estimate $\hat{\vmu}_t$.
The right panel in Figure \ref{fig:recursive-ridge-params-and-rmse}
shows the root mean squared error (RMSE) evaluated over a
\textit{held-out} test set $\hat{\data}_{1:t}$ for the Ridge estimate and the recursive-Bayes estimate $\vmu_t$ as a
function of the number of datapoints processed.
\begin{figure}[htb]
    \centering
    \includegraphics[width=0.9\linewidth]{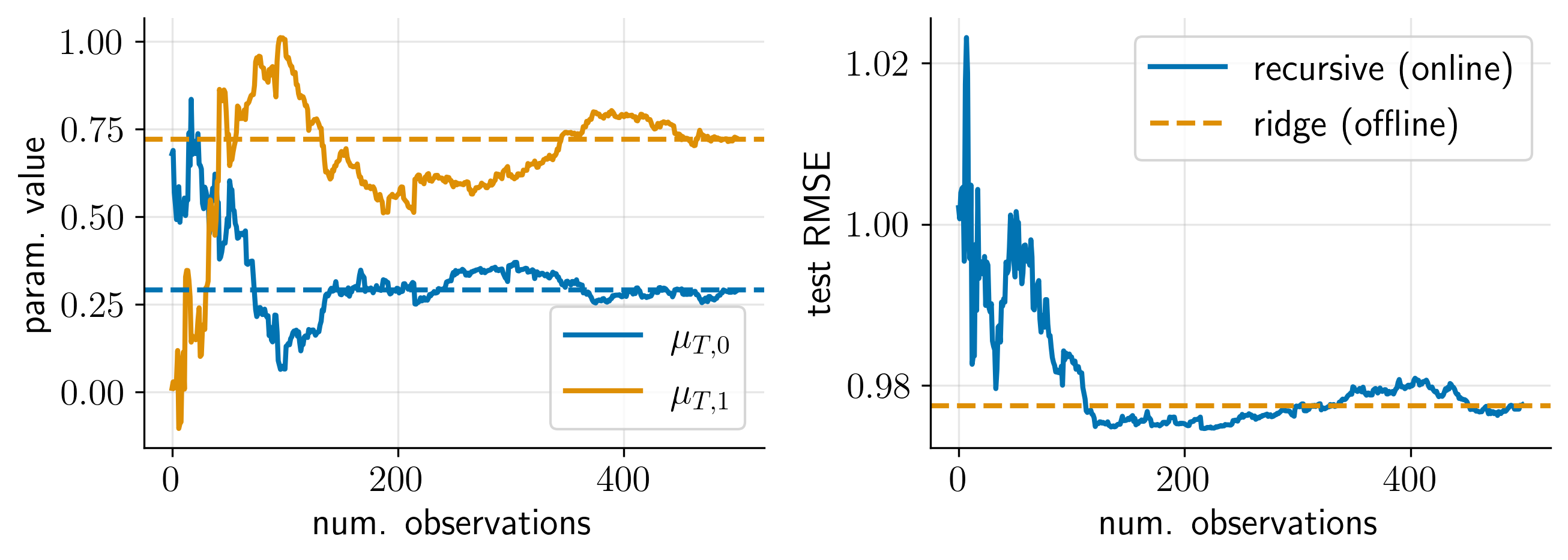}
    \caption{
    \textbf{(Left panel)}
    Mean estimate of the parameters.
    The solid lines correspond to the recursive-Bayes estimate of the mean.
    The dashed lines correspond to the offline estimate of the mean using Ridge regression.
    \textbf{(Right panel)}
    RMSE on a held-out test set.
    }
    \label{fig:recursive-ridge-params-and-rmse}
\end{figure}
We observe that the parameters of the recursive-Bayes estimate $\vmu_t$ tend to, and eventually match, the Ridge estimate of model parameters.
As a consequence, the RMSE on the held-out set for the recursive-Bayes estimate match that of the Ridge estimate
at $T$.

\section{Non-Gaussian and non-linear measurement models}
\label{sec:nonlinear-nongaussian-recursive}

In some scenarios, the assumption of a linear Gaussian measurement model can be overly restrictive.
For instance, when the measurement model $h_t$ is parametrised by a neural network, $h_t$ is a non-linear function of $\vtheta$.
Alternatively, in linear classification problems, the measurement model is best represented by a Bernoulli mass function,
with mean $\sigma(\vtheta^\intercal\,\vx_t)$, where $\sigma: \real \to [0,1]$ is the sigmoid function.

In these scenarios--- when the measurement model is either non-Gaussian or non-linear---the
likelihood is no longer conjugate to a Gaussian prior over the model parameters.
As a result, the posterior density, such as those that use \eqref{eq:recursive-linreg},
are no longer available.
To address this challenge, we rely on approximations to the posterior density or Monte Carlo (MC) methods.

MC methods are commonly used to sample from \eqref{eq:recursive-bayes} when an analytical form is either unknown or intractable.
In these cases, \eqref{eq:prequential-forecast-measure} is approximated by the samples from the posterior.
A wealth of literature exists on sample-based approaches for posterior inference.
For detailed discussions, see \cite{barbu2020montecarlo} for Markov Chain Monte Carlo methods,
\cite{doucet2009tutorialpf} for online particle filtering,
and \cite{naesseth2019elementssmc} for batch sequential Monte Carlo (SMC) techniques.

Functional approximations to the posterior, which is the primary focus in this thesis,
include techniques such as variational Bayes or linearisation schemes.
These computations are more tractable in many practical scenarios.
We will nonetheless provide examples where sample-based methods are employed.

\subsection{Variational Bayes}\label{sec:variational-bayes}
A popular approach to approximate the posterior density is through the use of variational Bayes (VB).
In VB, the posterior density is typically written as an optimisation problem
over a set of \textit{candidate} densities living in a set $\cal{Q}$ that most
closely resembles the posterior density $p(\vtheta \cond \data_{1:t})$
up to a normalisation constant.
In this text, our notion of closeness between the density $q\in{\cal Q}$
and the posterior density
is based on the Kullback-Leibler (KL) divergence \citep{kullback1951kld} which we recall below.
\begin{definition}
    \textbf{(Kullback-Leibler divergence)}
    Let $p$ and $q$ be probability densities defined over the same domain.
    Then, the KL divergence of $q$ from $p$ is defined as
    \begin{equation}\label{eq:def-kl-divergence}
       \KL{q(\vtheta)}{p(\vtheta)}
       := \mathbb{E}_{q}\left[\log\left(\frac{q(\vtheta)}{p(\vtheta)}\right)\right]
       = \int q(\vtheta)\,\left[\log\left(\frac{q(\vtheta)}{p(\vtheta)}\right)\right]\d\vtheta.
    \end{equation}
\end{definition}
See \cite{soch2020statsbook}.

A useful result that we will use throughout the thesis is that of the form
of the KL divergence between two multivariate Gaussians.
\begin{proposition}
    \label{prop:KL-divergence-gaussians}
    Let $p_1(\vx) = \normdist{\vx}{{\vm}_1}{{\vS}_1}$
    and
    $p_2(\vx) = \normdist{\vx}{{\vm}_2}{{\vS}_2}$
    be two $M$-dimensional multivariate Gaussian densities.
    Then, the KL divergence between $p_1$ and $p_2$ takes the form
    \begin{equation}
    \begin{aligned}
        &\KL{p_1(\vx)}{p_2(\vx)}\\
        &= \KL{\normdist{\vx}{{\vm}_1}{{\vS}_1}}{\normdist{\vx}{{\vm}_2}{{\vS}_2}}\\
        &= \frac{1}{2}\left[
            {\rm Tr}\left( {\vS}_2^{-1}{\vS}_1 \right)
            + ({\vm}_2 - {\vm}_1)^\intercal {\vS}_2^{-1} ({\vm}_2 - {\vm}_1)
            - M
            + \log\left( |{\vS}_2|/|{\vS}_1|\right)
        \right].
    \end{aligned}
    \end{equation}
\end{proposition}
\begin{proof}
 See Section 6.2.3 in \cite{muprhy2022-pml1Book}.
\end{proof}

Armed with a notion of closeness, we now define the objective of a VB method.
\begin{definition}
    \textbf{(variational approximation)}
    Let $p$ be a density with support $\rm{dom}(p)$ and
    ${\cal Q}$ a collection of candidate densities with ${\rm dom}(q) = \rm{dom}(p)$ for all $q \in {\cal Q}$.
    The variational approximation of $p$, under a KL divergence, selects $q_* \in {\cal Q}$ according to the criterion
    \begin{equation}\label{eq:variational-bayes-target}
        q_*(\vtheta) = \arg\min_{q \in {\cal Q}} \KL{q(\vtheta)}{p(\vtheta)}.
    \end{equation}
\end{definition}
Criterion \ref{eq:variational-bayes-target} recovers exact Bayesian updating \eqref{eq:recursive-bayes}
if ${\cal Q}$ includes the Bayesian posterior within its family, i.e.,
$q_*(\vtheta) = p(\vtheta)$ if $p \in {\cal Q}$ \citep{knoblauch2022gvi}.

Examples of VB for estimation of model parameters include
the Bayes-by-backpropagation method (BBB) of \cite{blundell2015bbb},
which assumes a diagonal posterior covariance
(more expressive forms are also possible).
 \cite{nguyen2017vcl} extended BBB to
 non-stationary settings.
More recent approaches involve recursive estimation,
such as
the recursive variational Gaussian approximation (R-VGA) method of \cite{lambert2022rvga}
which uses a full rank Gaussian variational approximation (see Section \ref{sec:rvga});
the limited-memory RVGA (L-RVGA) method of  \cite{lambert2023lrvga},
which uses a diagonal plus low-rank  (DLR) Gaussian variational approximation;
the Bayesian online natural gradient (BONG) method of  \cite{jones2024bong},
which combines the DLR approximation with EKF-style linearisation for additional speedups; and
the natural gradient Gaussian approximation (NANO) method of \cite{Cao2024}, which uses a diagonal Gaussian approximation
similar to VD-EKF in \cite{chang2022diagonal}.

The following section outlines an important VB method for recursive estimation of a Gaussian density
under non-linear measurement functions and non-Gaussian measurement models.

\subsection{The recursive variational Gaussian approximation}\label{sec:rvga}
The recursive variational Gaussian approximation (R-VGA), introduced in \cite{lambert2022rvga},
constructs a sequence of variational Gaussian approximations.
We outline this method below
\begin{definition}[R-VGA]
    Consider the dataset $\data_{1:T}$ with $\data_t = (\vx_t\in\real^\dimin, \vy_t\in{\cal B}\subseteq \real^\dimobs)$.
    Let $q_0(\vtheta) = {\cal N}(\vtheta \cond \vmu_0, \vSigma_0)$ with prior mean $\vmu_0\in\real^\dimstate$
    and $(\dimstate\times\dimstate)$ covariance matrix $\vSigma_0$.
    The R-VGA method estimates a VB density for the $\vy_t$ measurement according to
    \begin{equation}\label{eq:rvga-target-orignal}
        q_t(\vtheta) = \arg\min_{q \in {\cal Q}} \KL{q(\vtheta)}{p(\vy_t \cond \vtheta)\,q_{t-1}(\vtheta)\,/\,Z_t},
    \end{equation}
    where ${\cal Q}$ is the family of multivariate Gaussian densities and
    $Z_t = \int p(\vy_t \cond \vtheta)\,q_{t-1}(\vtheta)\d\vtheta$.
    Because the Gaussian density is characterised by its mean and covariance matrix,
    the VB criteria \eqref{eq:rvga-target-orignal} can be written as
\begin{equation}\label{eq:rvga-target-gauss}
    \vmu_t, \vSigma_t =
    \arg\min_{\vmu, \vSigma}\,
    \KL{{\cal N}(\vtheta \cond \vmu, \vSigma)}{p(\vy_t \cond \vtheta)\,{\cal N}(\vtheta \cond \vmu_{t-1}, \vSigma_{t-1})\,/\,Z_t}.
\end{equation}
\end{definition}
If the measurement model is absolutely continuous with respect to the model parameters $\vtheta$
and the observations $\vy_{1:T}$ are independent conditionally on $\vtheta$,
 then the R-VGA is show to have updates for $q_t$ (and hence $\vmu_t$ and $\vSigma_t$)
as
\begin{equation}\label{eq:rvga-mu-sigma-update}
\begin{aligned}
    \vmu_t &= \vmu_{t-1} + \vSigma_{t-1}\,\mathbb{E}_{q_{t}}[\nabla_\vtheta \log p(\vy_t \cond \vtheta, \vx_t)],\\
    \vSigma_t^{-1} &= \vSigma_t^{-1} - \mathbb{E}_{q_{t}}\left[\nabla^2_\vtheta \log p(\vy_t \cond \vtheta, \vx_t)\right].
\end{aligned}
\end{equation}
See Theorem 1 in \cite{lambert2022rvga} for a proof.

The right hand side of \eqref{eq:rvga-mu-sigma-update} requires the computation of the expected
gradient and hessian of the log-likelihood according to the variational approximation at time $t$.
In this sense,  \eqref{eq:rvga-mu-sigma-update} correspond to an implicit update, i.e.,
the estimate of $q_t$ is obtained after multiple inner iterations of \eqref{eq:rvga-mu-sigma-update}.
We summarise the R-VGA method in Algorithm \ref{algo:rvga-update}.
\begin{algorithm}[htb]
\begin{algorithmic}[1]
    \REQUIRE dataset $\data_{1:T}$ with $\data_t = (\vx_t, \vy_t)$
    \REQUIRE initial state of model parameters $q_{0}(\vtheta) = {\cal N}(\vtheta \cond \vmu_{0}, \vSigma_{0})$
    \REQUIRE number of iterations $I\geq 1$
    \FOR{$t=1,\ldots,T$}
    \STATE // Initialise $\vmu_t$ and $\vSigma_t$
    \STATE $\vmu_t \gets \vmu_{t-1}$
    \STATE $\vSigma_t \gets \vSigma_{t-1}$
    \STATE $q_t(\vtheta) = {\cal N}(\vtheta \cond \vmu_t, \vSigma_t)$
    \FOR{$i=1,\ldots,I$}
        \STATE $\vmu_t \gets \vmu_{t-1} + \vSigma_{t-1}\,\mathbb{E}_{q_{t}}[\nabla_\vtheta \log p(\vy_t \cond \vtheta, \vx_t)]$
        \STATE $\vSigma_t^{-1} \gets \vSigma_t^{-1} - \mathbb{E}_{q_{t}}[\nabla^2_\vtheta \log p(\vy_t \cond \vtheta, \vx_t)]$
        \STATE $q_t(\vtheta) = {\cal N}(\vtheta \cond \vmu_t, \vSigma_t)$
    \ENDFOR
    \ENDFOR
\end{algorithmic}
\caption{
    Pseudocode for the R-VGA.
} \label{algo:rvga-update}
\end{algorithm}
In Algorithm \ref{algo:rvga-update}, the terms
$\mathbb{E}_{q_{t}}[\nabla_\vtheta \log p(\vy_t \cond \vtheta, \vx_t)]$
and
$\mathbb{E}_{q_{t}}[\nabla^2_\vtheta \log p(\vy_t \cond \vtheta, \vx_t)]$
can be estimated through sampling---because $q_t(\vtheta) = \normdist{\vtheta}{\mu_t}{\vSigma_t}$,
then
\begin{equation}\label{eq:sample-expectation}
    \mathbb{E}_{q_{t}}[g(\vtheta)]
    \approx \frac{1}{S}\sum_{s=1}^S g\left(\vtheta^{(s)}\right),
\end{equation}
where $g: \real^\dimstate \to \real^K$, $K\geq 1$,
$S \geq 1$ is the number of samples, and
$\vtheta^{(s)}$ is a sample from a multivariate Gaussian with mean $\vmu_t$ and covariance matrix $\vSigma_t$.

The terms in \eqref{eq:rvga-mu-sigma-update} resemble those in \eqref{eq:recursive-linreg-mu-sigma-update}.
In fact, the next proposition shows that R-VGA has the Bayesian recursive linear regression as a special case.
\begin{proposition}\label{prop:rvga-matches-linreg}
    Consider a linear Gaussian model for $\vy_t\in\real^\dimobs$, so that
    $p(\vy_t \cond \vtheta, \vx_t) = {\cal N}(\vy_t \cond \vx_t^\intercal\,\vtheta, \vR_t)$.
    Then the R-VGA update \eqref{eq:rvga-mu-sigma-update} matches that of \eqref{eq:recursive-linreg}.
    As a consequence, it also matches the \textit{true} posterior density
    $p(\vtheta_t \cond \data_{1:t})$.
\end{proposition}
\begin{proof}
    See Theorem 2 in \cite{lambert2022rvga}.
\end{proof}

As we have seen, the R-VGA method provides a recursive approach to approximate
the posterior density whenever the first- and second-order derivates of the
log-measurement density are defined.
In the following two experiments, we evaluate the R-VGA methods.
Example \ref{example:rvga-logistic-regression} sequentially estimates the parameters of a logistic regression model and
Example \ref{example:rvga-neural-network-training} sequentially estimates the parameters of a neural network
to tackle a non-linear classification problem.

\subsection{Experiment: R-VGA for logistic regression}\label{example:rvga-logistic-regression}
In this experiment, we consider the problem of sequential estimation of model parameters for a binary classification problem.
Assume that we observe a sequence of datapoints $\data_{1:T}$
with $\data_t = (\vx_t \in \real^\dimin, \vy_t \in \{0,1\})$.
We model the measurements $\vy_t$ as Bernoulli with mean $\sigma(h(\vtheta, \vx_t))$, i.e.,
\begin{equation}\label{eq:likelihood-logistic-regression-measurement-model}
    p(\vy \cond \vtheta, \vx) = {\rm Bern}\left(\vy \cond \sigma(h(\vtheta, \vx)\right).
\end{equation}
Here, $\sigma(z) = (1 + \exp(-z))^{-1}$ is the sigmoid function and
${\rm Bern}(y \cond p) = p^y\,(1 - p)^{1 - y}$ is the Bernoulli probability mass function.
For the logistic regression, we have $h(\vtheta, \vx) = \vtheta^\intercal\,\vx$.
As a consequence, the log-likelihood is given by
\begin{equation}\label{eq:log-likelihood-logistic-regression-model}
    \log\,p(\vy_t \cond \vtheta, \vx_t) = \vy_t\,\log \sigma(\vtheta^\intercal\vx_t) + (1 - \vy_t)\,\log(1 - \sigma(\vtheta^\intercal\,\vx_t)).
\end{equation}
To apply the R-VGA to the logistic regression model \eqref{eq:likelihood-logistic-regression-measurement-model},
we require to compute the first and second-order derivatives of \eqref{eq:log-likelihood-logistic-regression-model},
which can be readily done either explicitly or implicitly with the use of autodifferentiation computer libraries
such as Jax \citep{jax2018github},
as well as sampling from the posterior density, as shown in \eqref{eq:sample-expectation}.
For this experiment, we take $p(\vtheta_0) = \normdist{\vtheta}{{\bf 0}}{\vI}$.

Figure \ref{fig:recursive-logreg} shows the decision boundaries for the logistic regression model
as a function of the number of processed measurements.
\begin{figure}[htb]
    \centering
    \includegraphics[width=0.9\linewidth]{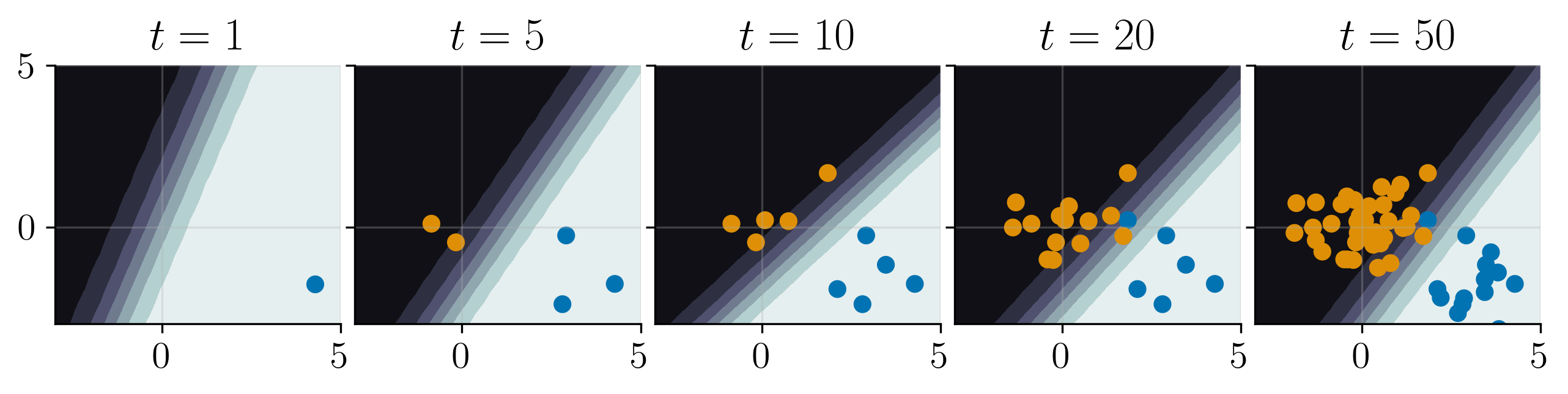}
    \caption{
    Decision boundaries for the logistic regression model as
    a function of the number of processed measurements.
    }
    \label{fig:recursive-logreg}
\end{figure}
Next, Figure \ref{fig:recursive-logreg-params} shows the evolution of the mean estimate of the model parameters
for the logistic regression and two standard deviations.
\begin{figure}[htb]
    \centering
    \includegraphics[width=0.9\linewidth]{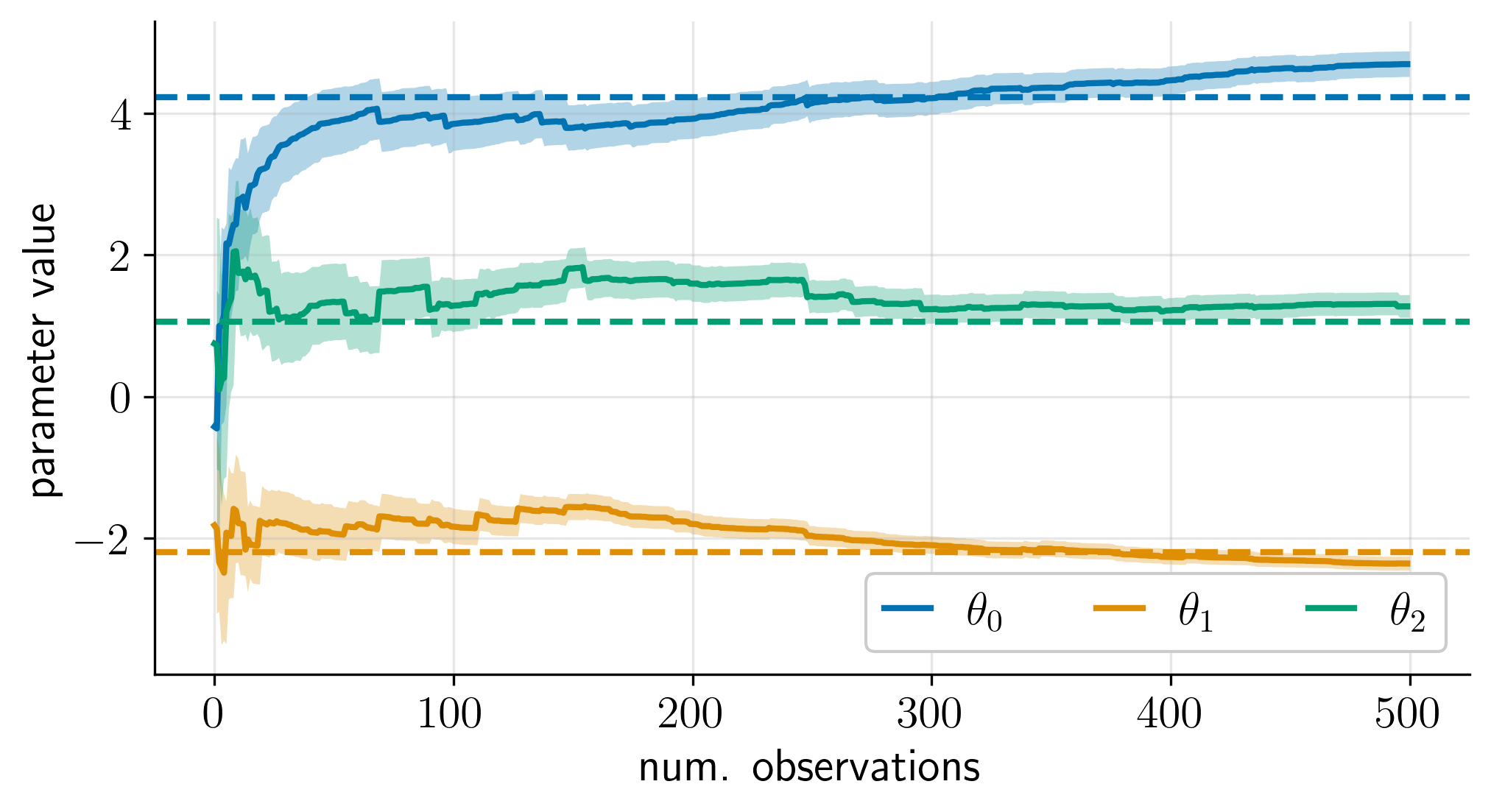}
    \caption{
    The solid lines show
    the posterior mean estimate of model parameters and two standard deviations.
    The dashed lines show the batch estimate of the logistic regression parameters.
    }
    \label{fig:recursive-logreg-params}
\end{figure}
We observe that the model parameters gradually convergence to the batch (offline) version of
logistic regression.
For details on the offline version, see Section 10.2.7 in \cite{muprhy2022-pml1Book}.

\subsection{Experiment: recursive learning of neural networks}\label{example:rvga-neural-network-training}
Because the R-VGA update equations \eqref{eq:rvga-mu-sigma-update}
require only the computation of the Jacobian and the hessian of the log-likelihood,
we can apply the R-VGA methodology to any probabilistic model
whose Jacobian and hessian of the log-likelihood are defined.

In this example, we consider the problem of non-linear binary classification problem.
We consider a Bernoulli measurement model, whose log-likelihood is given by
\begin{equation}
    \log p(\vy_t\cond \vtheta, \vx) = \vy_t\,\log\sigma(h(\vtheta;\vx)) + (1 - \vy_t)\,\log\left(1 - \sigma(h(\vtheta;\vx))\right),
\end{equation}
with $h: \real^\dimstate\times\real^\dimin \to \real$ a neural network.
In particular, we consider a three-hidden-layer multi-layered perceptron with $5$ units per layer and ${\rm leakyReLU}$ activation function.
We run the R-VGA algorithm with $I=4$ inner iterations and $S=1,000$ samples following \eqref{eq:sample-expectation}.
To make a prediction at time $t$, we take the previous mean, i.e.,
$\vmu_{t-1}$ and use this value to make a forecast with the features $\vx_t$, i.e.,
$\hat{\vy}_t = h(\vmu_{t-1}, \vx_t)$.

Figure \ref{fig:recursive-nnet-clf} shows the predictions made by the decision boundary
$\sigma(h(\vmu_{t}, \vx))$ for $\vx \in [-2,2]^2$ as a function of $t$.
The orange and blue dots represent the past datapoints and the cyan-coloured datapoints represents the datapoint observed at $t$.
\begin{figure}[htb]
    \centering
    \includegraphics[width=0.9\linewidth]{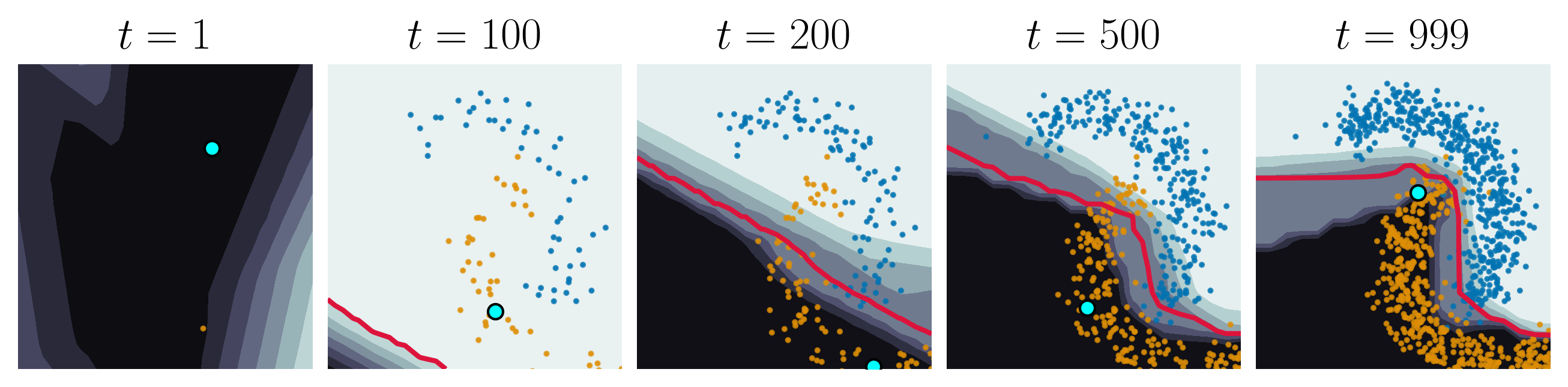}
    \caption{
    Decision boundaries for the classification problem using a neural network
    trained using the R-VGA.
    The crimson line shows the decision boundary at $0.5$.
    The dot coloured in cyan shows the observation seen at time $t$.
    }
    \label{fig:recursive-nnet-clf}
\end{figure}

Next, Figure \ref{fig:recursive-nnet-clf-accuracy} shows the expanding prequential accuracy as a function of
the number of processed observations.
We define the prequential accuracy as $1.0$ if $\hat{\vy}_t = \vy_t$ and $0.0$ otherwise.
\begin{figure}[htb]
    \centering
    \includegraphics[width=0.9\linewidth]{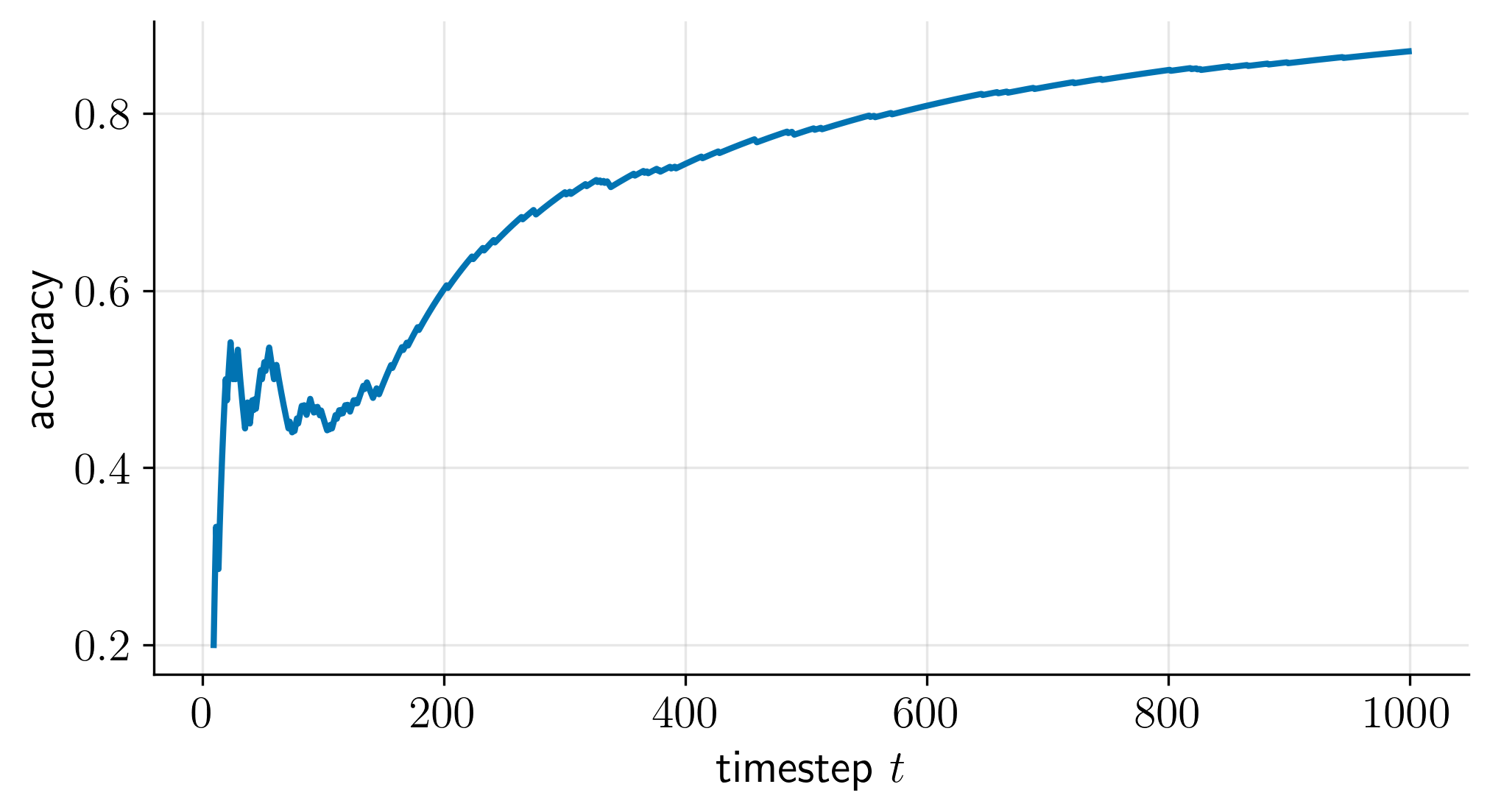}
    \caption{
    Cumulative prequential accuracy for the non-linear classification problem trained using the R-VGA.
    }
    \label{fig:recursive-nnet-clf-accuracy}
\end{figure}
For this experiment, we observe that accuracy of the R-VGA rapidly increases,
and then plateaus for $100$ steps before its predictive power starts to increase.

\section{State-space-models for sequential supervised learning}
\label{sec:ssm}
The methods introduced in Sections \ref{sec:linear-gaussian-recursive} and \ref{sec:nonlinear-nongaussian-recursive} assume that the measurement model accurately represents the true data-generating process.
This assumption allows us, given access to $\vx$ and $p(\vtheta \cond \data_{1:t})$,
to sample random variables $\hat{\vy}$ whose density is $p(\vy \cond \vtheta, \vx)$.
As a consequence,
and under certain conditions,
the Bayesian posterior density can be shown to converge to a point estimate \cite[Sec. 10.2]{van2000asymptoticstats}. 

However, the assumption of a well-specified likelihood may not always hold in practice.
For instance, in neural network training (as discussed in Example \ref{example:rvga-neural-network-training}),
the choice of measurement model $h$ is often driven by practical considerations rather than precise knowledge of the true data-generating process.
In these scenarios, we must adapt to the changing environment or adjust our misspecified model to make accurate predictions.

One such adaptation scheme involves the assumption of time-varying parameters.
Instead of assuming a random vector $\vtheta$, whose posterior density we wish to estimate,
we now assume that the model parameters follow a stochastic process
that evolve over time according to a \textit{state-transition} function
$f:\real^\dimstate\to\real^\dimstate$,
perturbed by zero-mean dynamic noise $\vu_t \in \real^\dimstate$ with ${\rm Var}(\vu_t) = \vQ_t$.
This is commonly referred to as the state-space model assumption.
We define a state-space model below.

\begin{definition}[state-space model]\label{def:ssm}
    A state-space model (SSM) is a signal plus noise model of the form
    \begin{equation}\label{eq:ssm}
    \begin{aligned}
        \vtheta_t &= f(\vtheta_{t-1}) + \vu_t,\\
        \vy_t &= h(\vtheta_t, \vx_t) + \ve_t,
    \end{aligned}
    \end{equation}
    with
    $f:\real^\dimstate \to \real^\dimstate$ the state-transition function,
    $h:\real^\dimstate \times \real^\dimin \to \real^\dimobs$ the measurement function,
    $\var(\vu_t) = \vQ_t$ a $\dimstate\times\dimstate$ positive semidefinite matrix,
    $\var(\ve_t) = \vR_t$ a $\dimobs\times\dimobs$ positive definite matrix, and
    $\vx_t \in \real^\dimin$ exogenous features.
    
    If one assumes zero-mean Gaussian priors with known covariance matrices for $\vu_t$ and $\ve_t$,
    the terms \eqref{eq:ssm} can be represented as
    \begin{equation}
    \begin{aligned}
        p(\vtheta_t \cond \vtheta_{t-1}) &= {\cal N}(\vtheta_t \cond f(\vtheta_{t-1}), \vQ_t),\\
        p(\vy_t \cond \vtheta_t) &= {\cal N}(\vy_t \cond h(\vtheta_t, \vx_t), \vR_t).
    \end{aligned}
    \end{equation}
\end{definition}

\subsection{Linear SSMs and the Kalman filter}
A well-known choice of SSM is that of linear SSMs with known
dynamics covariance $\vQ_t$,
measurement covariance $\vR_t$.
In this scenario, the SSM takes the form
\begin{equation}\label{eq:ssm-linear}
\begin{aligned}
    \vtheta_t &= \vF_t\,\vtheta_{t-1} + \vu_t,\\
    \vy_t &= \vH_t\,\vtheta_t + \ve_t,
\end{aligned}
\end{equation}
where $\vF_t$ is the $(\dimstate\times\dimstate)$ transition matrix and
$\vH_t$ is the $(\dimobs\times\dimstate)$ projection matrix.

The next proposition shows that estimation of the posterior density $p(\vtheta_t \cond \data_{1:t})$
assuming \eqref{eq:ssm-linear} can be done recursively following the so-called Kalman filter (KF) equations.
\begin{proposition}\label{prop:kf-equations}
    Assume an initial Gaussian prior density for model parameters $p(\vtheta_0) = \normdist{\vtheta_0}{\vmu_0}{\vSigma_0}$
    and known $\vQ_t$, $\vR_t$ for all $t$.
    Then, the posterior predictive density for model parameters conditioned on $\data_{1:t-1}$ is Gaussian of the form
    \begin{equation}
        p(\vtheta_t \cond \data_{1:t-1}) = \normdist{\vtheta_t}{\vmu_{t|t-1}}{\vSigma_{t|t-1}},
    \end{equation}
    with predictive mean and predictive covariance
    \begin{equation}\label{eq:kf-predict-step}
    \begin{aligned}
        \vmu_{t|t-1} &= \vF_t\,\vmu_{t-1},\\
        \vSigma_{t|t-1} &= \vF_t\,\vSigma_{t-1}\,\vF_{t}^\intercal + \vQ_t.
    \end{aligned} 
    \end{equation}
    Furthermore, the density for model parameters, conditioned on $\data_{1:t}$ is Gaussian of the form
    \begin{align}
        p(\vtheta_t \cond \data_{1:t}) &= \normdist{\vtheta_t}{\vmu_t}{\vSigma_t},
    \end{align}
    with
    \begin{equation}
    \begin{aligned}\label{eq:kf-update-step}
        \vS_t &= \vH_t\,\vSigma_{t|t-1}\,\vH_t^\intercal + \vR_t,\\
        \vK_t &= \vSigma_{t|t-1}\vH_t^\intercal\,\vS_t^{-1},\\
        \vmu_t &= \vmu_{t|t-1} + \vK_t\,(\vy_t - \vH_t\,\vmu_{t|t-1}),\\
        \vSigma_t &= \left(\vI - \vK_t\vH_t^\intercal\right)\vSigma_{t|t-1}.
    \end{aligned}
    \end{equation}
    Here, $\vmu_t$ and $\vSigma_t$ are the posterior mean and covariance matrix respectively
    and $\vK_t$ is the \textit{gain} matrix.
\end{proposition}
\begin{proof}
    The result follows as a direct consequence of Proposition \ref{prop:joint-conditional-gauss}
    and Proposition \ref{prop:marginals-cond-gauss}.
    For details, see Theorem 6.6 in \cite{sarkka2023filtering}.
\end{proof}

\begin{remark}
Proposition \ref{prop:kf-equations} with $\vH_t = \vx_t^\intercal$ and $\vQ_t = 0\,\vI$
recovers the recursive Bayesian linear regression shown in Algorithm \ref{algo:recursive-bayesian-linreg}.
\end{remark}

\begin{algorithm}[htb]
\begin{algorithmic}[1]
    \REQUIRE $\data_t = (\vx_t, \vy_t)$ // datapoint
    \REQUIRE $(\vmu_{t-1}, \vSigma_{t-1})$ // previous mean and covariance
    \REQUIRE $\vH_t$ // projection matrix
    \REQUIRE $\vQ_t, \vR_t$ // dynamics covariance and measurement-noise covariance
    \STATE  // predict step
    \STATE $\vmu_{t|t-1} \gets \vF_t\vmu_{t-1}$
    \STATE $\vSigma_{t|t-1} \gets \vF_t\vSigma_{t-1}\vF_{t}^\intercal + \vQ_t$
    \STATE //update step
    \STATE $\vS_t = \vH_t\,\vSigma_{t-1}\vH_t^\intercal + \vR_t$
    \STATE $\vK_t = \vSigma_{t|t-1}\vH_t^\intercal\,\vS_t^{-1}$
    \STATE $\vmu_t \gets \vmu_{t-1} + \vK_t(\vy_t - \vH_t\vmu_{t|t-1})$
    \STATE $\vSigma_t \gets \vSigma_{t|t-1} - \vK_t\vH_t^\intercal\vSigma_{t|t-1}$
    \RETURN $(\vmu_t, \vSigma_t)$
\end{algorithmic}
\caption{
    Predict and update steps for the Kalman filter
    for $t \geq 1$ and given prior mean $\vmu_{t-1}$ and covariance $\vSigma_{t-1}$.
}
\label{algo:linear-kalman-filter}
\end{algorithm}

An alternative formulation of the KF update equations is expressed through a rule
that updates the posterior prediction matrix $\vSigma_t^{-1}$.
We formalise this result below.

\begin{proposition}[Kalman filter with precision matrix updates]
    \label{prop:kf-update-step-precision}
    Assume an initial Gaussian prior density for model parameters $p(\vtheta_0) = \normdist{\vtheta_0}{\vmu_0}{\vSigma_0}$
    and known $\vQ_t$, $\vR_t$ for all $t$.
    Then, the density for model parameters, conditioned on $\data_{1:t}$ is Gaussian of the form
    \begin{align}
        p(\vtheta_t \cond \data_{1:t}) &= \normdist{\vtheta_t}{\vmu_t}{\vSigma_t},
    \end{align}
    with
    \begin{equation}
    \begin{aligned}\label{eq:kf-prec-update-step}
        \vSigma_t^{-1} &= \vSigma_{t|t-1}^{-1} + \vH_t^\intercal\,\vR_t^{-1}\,\vH_t,\\
        \vK_t &= \vSigma_t\,\vH_t^\intercal\,\vR_t^{-1},\\
        \vmu_t &= \vmu_{t|t-1} + \vK_t\,(\vy_t - \vH_t\,\vmu_{t|t-1}),\\
    \end{aligned}
    \end{equation}
\end{proposition}

\begin{proof}

Let $\vSigma_{t|t-1}^{-1}$ be a precision matrix, $\vH_t \in \real^{\dimobs\times\dimstate}$, and
$\vR_t$ the covariance of the measurement process.
First, we show that
$ \vSigma_t^{-1} = \vSigma_{t|t-1}^{-1} + \vH_t^\intercal\,\vR_t^{-1}\,\vH_t$.

By algebraic manipulation of the
of the posterior covariance \eqref{eq:kf-update-step}, the Woodbury identity, and the definition of $\vS_t$,
it follows that
\begin{equation*}
\begin{aligned}
    \vSigma_t^{-1}
    &= \left((\vI_\dimstate - \vK_t\,\vH_t^\intercal)\,\vSigma_{t|t-1}\right)^{-1}\\
    &= \vSigma_{t|t-1}^{-1}\left(\vI_\dimstate - \vSigma_{t|t-1}\,\vH_t^\intercal\vS_{t}^{-1}\,\vH_t\right)^{-1}\\
    &= \vSigma_{t|t-1}^{-1}\left[\vI_\dimstate  + \vSigma_{t|t-1}\,\vH_t^\intercal(\vS_t - \vH_t\,\vSigma_{t|t-1}\,\vH_t^\intercal)^{-1}\vH_t\right]\\
    &= \vSigma_{t|t-1}^{-1}\left[\vI_\dimstate  + \vSigma_{t|t-1}\,\vH_t^\intercal(\vH_t\,\vSigma_{t|t-1}^{-1}\,\vH_t^\intercal + \vR_t - \vH_t\,\vSigma_{t|t-1}\,\vH_t^\intercal)^{-1}\vH_t\right]\\
    &= \vSigma_{t|t-1}^{-1}\left[\vI_\dimstate + \vSigma_{t|t-1}\,\vH_t^\intercal\vR_t^{-1}\vH_t\right]\\
    &= \vSigma_{t|t-1}^{-1} + \vH_t^\intercal\vR_t^{-1}\vH_t.
\end{aligned}
\end{equation*}

Next, let $\vSigma_t^{-1}$ be the posterior precision matrix and $\vR_t^{-1}$ be
the precision matrix of the observation at time $t$. 
The precision matrix of the innovations take the form
\begin{equation}
    \vS_t^{-1} = \vR_t^{-1} - \vR_t^{-1}\,\vH_t\,\vSigma_{t}^{-1}\,\vH_t^\intercal\,\vR_t^{-1}.
\end{equation}
To see this, observe that
$\vSigma_{t|t-1}^{-1} = \vSigma_t^{-1} - \vH_t^\intercal\,\vR_t^{-1}\,\vH_t$,
and thus
\begin{equation*}
\begin{aligned}
    \vS_t^{-1}
    &= \left(\vH_t\,\vSigma_{t|t-1}\,\vH_t^\intercal + \vR_t\right)^{-1}\\
    &= \vR_t^{-1} -
    \vR_t^{-1}\vH_t\left(\vSigma_{t|t-1}^{-1} + \vH_t^\intercal\,\vR_t^{-1}\,\vH_t \right)^{-1}\,\vH_t^\intercal\,\vR_t^{-1}\\
    &= \vR_t^{-1} -
    \vR_t^{-1}\vH_t\left(\vSigma_{t}^{-1} - \vH_t^\intercal\,\vR_t^{-1}\,\vH_t + \vH_t^\intercal\,\vR_t^{-1}\,\vH_t \right)^{-1}\,\vH_t^\intercal\,\vR_t^{-1}\\
    &= \vR_t^{-1} - \vR_t^{-1}\vH_t\vSigma_{t}^{-1}\,\vH_t^\intercal\,\vR_t^{-1}.
\end{aligned}
\end{equation*}

Finally, the Kalman gain matrix takes the form
\begin{equation}
    \vK_t = \vSigma_{t}\,\vH_t^\intercal\,\vR_t^{-1}.
\end{equation}
This is because
\begin{equation}\label{eq:part-posterior-covariance-woodbury}
    \begin{aligned}
    \vSigma_{t|t-1}
    &= \left(\vSigma_{t}^{-1} - \vH_t^\intercal\,\vR_t^{-1}\,\vH_t\right)^{-1}\\
    &= \vSigma_t + \vSigma_t\,\vH_t^\intercal\,\left(\vR_t - \vH_t\,\vSigma_t\,\vH_t^\intercal\right)^{-1}\vH_t\vSigma_t,
    \end{aligned}
\end{equation}
and, 
following Proposition \ref{prop:kf-equations}, the Kalman gain matrix takes the form
\begin{equation}\label{eq:part-kalman-gain-prec-raw}
\begin{aligned}
    &\vK_t\\
    &= \vSigma_{t|t-1}\,\vH_t^\intercal\,\vS_t^{-1}\\
    &= \left(\vSigma_t + \vSigma_t\,\vH_t^\intercal\,\left(
    \vR_t - \vH_t\,\vSigma_t\,\vH_t^\intercal
    \right)^{-1}\vH_t\vSigma_t\right)\,\vH_t^\intercal\,
    \vS_t^{-1}\\
    &= \vSigma_t\,\vH_t^\intercal\,\vR_t^{-1}
    + \vSigma_t\,\vH_t^{-1}\,\vR_t^{-1}\left[
    -\vH_t\,\vSigma_t\,\vH_t^\intercal\,\vR_t^{-1} +
    (\vI_\dimobs - \vH_t\vSigma_t\,\vH_t^\intercal\,\vR_t^{-1})^{-1}\,\vH_t\vSigma_t\,\vH_t^\intercal\vS_t^{-1}
    \right].
\end{aligned}
\end{equation}
We need to show that the second term is a zero-valued matrix.
This amounts to showing that
$-\vH_t\,\vSigma_t\,\vH_t^\intercal\,\vR_t^{-1} +
(\vI_\dimobs - \vH_t\vSigma_t\,\vH_t^\intercal\,\vR_t^{-1})^{-1}\,\vH_t\vSigma_t\vH_t^\intercal\vS_t^{-1} = {\bf 0}$.
To show this, note that
\begin{equation}
\begin{aligned}
&-\vH_t\,\vSigma_t\,\vH_t^\intercal\,\vR_t^{-1} + (\vI_\dimobs - \vH_t\vSigma_t\,\vH_t^\intercal\,\vR_t^{-1})^{-1}\,\vH_t\vSigma_t\,\vH_t^\intercal\,\vS_t^{-1}\\
&= -\vH_t\,\vSigma_t\,\vH_t^\intercal\,\vR_t^{-1} + (\vI_\dimobs - \vH_t\vSigma_t\,\vH_t^\intercal\,\vR_t^{-1})^{-1}\,\vH_t\vSigma_t\,\vH_t^\intercal\,\vR_t^{-1}\,\left(\vI_o - \vH_t\vSigma_{t}^{-1}\,\vH_t^\intercal\,\vR_t^{-1}\right)\\
&= -\vI_\dimobs + (\vI_\dimobs -\vH_t\,\vSigma_t\,\vH_t^\intercal\,\vR_t^{-1}) + (\vI_\dimobs - \vH_t\vSigma_t\,\vH_t^\intercal\,\vR_t^{-1})^{-1}\,\vH_t\vSigma_t\,\vH_t^\intercal\,\vR_t^{-1}\,\left(\vI_o - \vH_t\vSigma_{t}^{-1}\,\vH_t^\intercal\,\vR_t^{-1}\right)\\
&= -\vI_\dimobs + \left[\vI_o + (\vI_\dimobs - \vH_t\vSigma_t\,\vH_t^\intercal\,\vR_t^{-1})^{-1}\,\vH_t\vSigma_t\vH_t^\intercal\,\vR_t^{-1}\right]\,(\vI_\dimobs -\vH_t\,\vSigma_t\,\vH_t^\intercal\,\vR_t^{-1})\\
&= -\vI_\dimobs + (\vI_\dimobs - \vH_t\vSigma_t\,\vH_t^\intercal\,\vR_t^{-1})^{-1}\,\left[\vI_\dimobs - \vH_t\vSigma_t\,\vH_t^\intercal\,\vR_t^{-1} + \,\vH_t\vSigma_t\vH_t^\intercal\,\vR_t^{-1}\right]\,(\vI_\dimobs -\vH_t\,\vSigma_t\,\vH_t^\intercal\,\vR_t^{-1})\\
&= -\vI_\dimobs + (\vI_\dimobs - \vH_t\vSigma_t\,\vH_t^\intercal\,\vR_t^{-1})^{-1}\,(\vI_\dimobs -\vH_t\,\vSigma_t\,\vH_t^\intercal\,\vR_t^{-1})\\
&= {\bf 0},
\end{aligned}
\end{equation}
which concludes the proof.
\end{proof}

\begin{algorithm}[htb]
\begin{algorithmic}[1]
    \REQUIRE $\data_t = (\vx_t, \vy_t)$ // datapoint
    \REQUIRE $(\vmu_{t-1}, \vSigma_{t-1}^{-1})$ // previous mean and precision matrix
    \REQUIRE $\vH_t$ // projection matrix
    \REQUIRE $\vQ_t, \vR_t$ // dynamics covariance and measurement-noise covariance
    \STATE  // predict step
    \STATE $\vmu_{t|t-1} \gets \vF_t\vmu_{t-1}$
    \STATE $\vSigma_{t|t-1} \gets \vF_t\vSigma_{t-1}\vF_{t}^\intercal + \vQ_t$
    \STATE //update step
    \STATE $\vSigma_{t}^{-1} \gets \vSigma_{t|t-1}^{-1} + \vH_t^\intercal\,\vR_t^{-1}\,\vH_t$
    \STATE $\vK_t = \vSigma_{t}\vH_t^\intercal\,\vR_t^{-1}$
    \STATE $\vmu_t \gets \vmu_{t-1} + \vK_t(\vy_t - \vH_t\vmu_{t|t-1})$
    \RETURN $(\vmu_t, \vSigma_t^{-1})$
\end{algorithmic}
\caption{
    predict and update steps at time $t$ for the Kalman filter under precision updates.
}
\label{algo:linear-kalman-filter-precision}
\end{algorithm}

\subsection{Experiment: the Kalman filter for non-stationary linear regression}
In the  classical filtering literature, the term $\vQ_t$ is typically used to model system dynamics.
However, it has long been known that inflating $\vQ_t$ can compensate for unmodelled errors \citep{kelly1990recursiveridge,kuhl1990ridge}.
We study this result in the context of online learning in the following experiment.

We evaluate the performance of the Kalman Filter (KF) in a linear regression problem with varying levels
of $\vQ_t = q\,\vI$,
where $q \geq 0$ is the dynamics covariance \textit{inflation} factor.
This allows us to explore how inflating $\vQ_t$ can handle unmodelled errors in non-stationary data streams.
Recall that $\vQ_t = 0\,\vI$ corresponds to the static online regression problem discussed in Section \ref{sec:linear-gaussian-recursive},
while increasing $\vQ_t$ introduces dynamics in the model parameters.

We consider a piece-wise linear regression model with standard-Gaussian errors, i.e., $p(\ve_t) = {\cal N}(\ve_t \cond 0, 1)$.
The features are sampled according to
$\vx_t \sim {\cal U}[-2, 2]$, and the measurements are sampled according to
$\vy_t \sim {\cal N}\big( \phi(\vx_t)^\intercal\vtheta_t, 1\big)$ with
$\phi(x) = (1,\,x,\,x^2)$.
At every timestep, the parameters take the value
\begin{equation}
\vtheta_t =
\begin{cases}
\vtheta_{t-1} & \text{w.p. } 1 - p_\epsilon,\\
{\cal U}[-3, 3]^3 & \text{w.p. } p_\epsilon,
\end{cases}
\end{equation}
with $p_\epsilon = 0.001$, and $\vtheta_0 \sim {\cal U}[-3, 3]^3$.
Figure \ref{fig:segements-lr} shows a sample run of this process.
\begin{figure}[htb]
    \centering
    \includegraphics[width=0.9\linewidth]{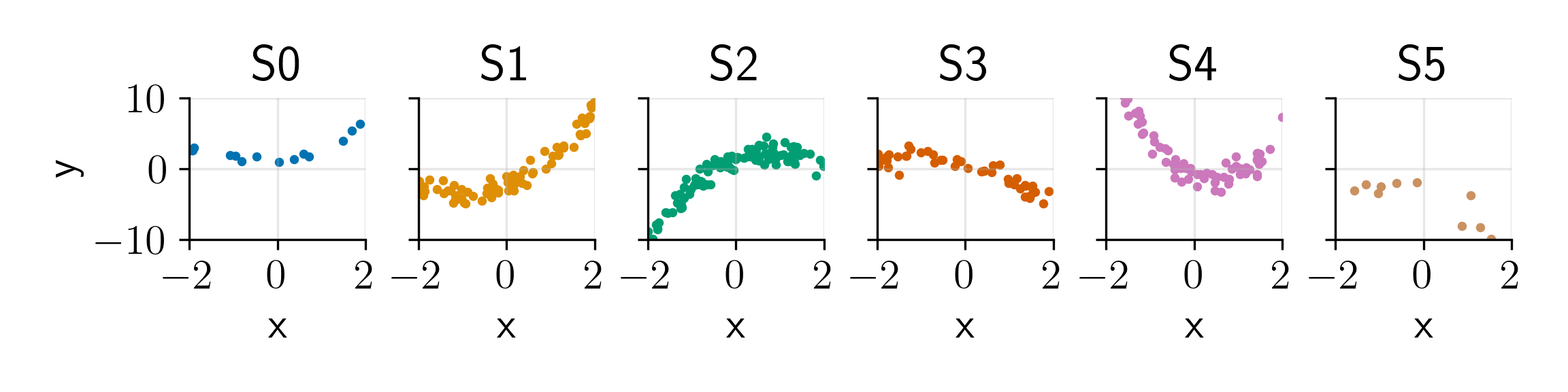}
    \caption{
        Sample run of the piecewise regression process.
        Each box titled $Si$ represents the samples that belong to the $i$-th regime.
    }
    \label{fig:segements-lr}
\end{figure}

Given a sample run of the process $\data_{1:T}$, with $T=300$,
we make use of Algorithm \ref{algo:linear-kalman-filter} with varying levels of a fixed $\vQ_t = q\,\vI$
with $q \geq 0$.
Figure \ref{fig:segments-lr-results} shows the rolling prequential RMSE and the total prequential RMSE.
\begin{figure}[htb]
    \centering
    \includegraphics[width=0.9\linewidth]{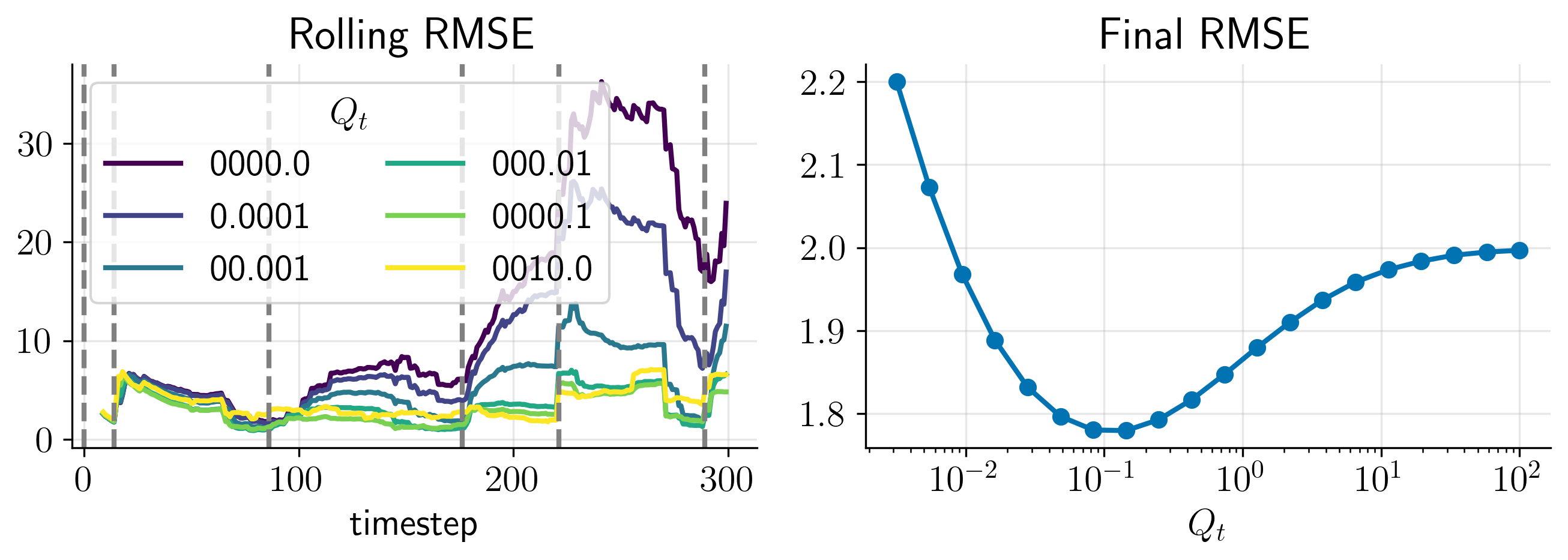}
    \caption{
    \textbf{(Left panel)} Rolling prequential RMSE of the linear model trained using various levels of $\vQ_t$.
    \textbf{(Right panel)} Total prequential RMSE of the linear model
    }
    \label{fig:segments-lr-results}
\end{figure}
We observe that different values of $\vQ_t$ lead to varying levels of prediction error.
When $\vQ_t$ is close to zero, the model exhibits limited adaptability, retaining more of the previous parameter estimates.
Conversely, when $\vQ_t$ is significantly larger,
the model quickly forgets past information, leading to high adaptability but potentially overreacting to noise.
An ideal adaptive method would be able to adjust $\vQ_t$ dynamically in an online fashion.
We revisit this idea in Chapter \ref{ch:adaptivity}.

\section{Extensions to the Kalman filter}
In this section, we introduce two variants of the KF that we use throughout the text;
namely, the extended Kalman filter (EKF) and
the ensemble Kalman filter (EnKF).

\subsection{The extended Kalman filter}\label{sec:extended-kalman-filter}
The extended Kalman filter (EKF) is a modification of the KF
whenever the measurement and state functions are non-linear but differentiable w.r.t. the model parameters $\vtheta$.
Broadly speaking, the EKF linearises the measurement function $h$ and
the state-transition function $f$ via a first-order Taylor approximation,
which yields KF-like update equations.

The EKF has its origins in the Apollo guidance computer program
to estimate the space trajectories of a spacecraft based on a
``sequence of measurements of angles between selected pairs of celestial bodies, together
with the measurement of the angular diameter of a nearby plane''  \citep{battin1982apollo, grewal2010kfhist}.
For details; see Chapter 3 in \citet{leondes1970-kfapollo}.
Because of the mild requirements of the EKF, it has been succesfully applied in multiple settings.
Of particular interest to this thesis is the use of the EKF to train neural networks,
which dates back to \cite{singhal1988ekfmlp}, where it was proposed as an alternative
to the backpropagation algorithm of \cite{rumelhart1986backprop}
used  to train multilayered perceptrons (MLPs).

The next proposition introduces the EKF algorithm.
\begin{proposition}\label{prop:ekf}
    Consider the SSM in definition \ref{def:ssm}
    with transition function $f_t$ and measurement function $h_t$,
    both differentiable w.r.t. $\vtheta$.
    Define the linearised SSM
    \begin{equation}\label{eq:ssm-linearised}
    \begin{aligned}
        \vtheta_t &= \bar{f}_t(\vtheta_{t-1}) + u_t,\\
        \vy_t &= \bar{h}_t(\vtheta_t) + e_t,
    \end{aligned}
    \end{equation}
    with
    $\bar{f}(\vtheta) = \vF_t\,(\vtheta - \vmu_{t-1}) + f_t(\vmu_{t-1})$,
    $\bar{h}_t(\vtheta_t) = \vH_t\,(\vtheta_t - \vmu_{t|t-1}) + h(\vmu_{t|t-1}, \vx_{t})$,
    $\vF_t = \nabla_\vtheta f(\vmu_{t-1})$, and
    $\vH_t = \nabla_\vtheta h(\vmu_{t | t-1}, \vx_t)$.
    Consider the priors
    $p(\vu_t) = \mathcal{N}(\vu_t \cond {\bf 0}, \vQ_t)$,
    $p(\ve_t) = \mathcal{N}(\ve_t \cond {\bf 0}, \vR_t)$.

    Then, the predict and update steps are Gaussian with form
    \begin{equation}
    \begin{aligned}
        p(\vtheta_t \cond \data_{1:t-1}) &= \normdist{\vtheta_t}{\vmu_{t|t-1}}{\vSigma_{t|t-1}},\\
        p(\vtheta_t \cond \data_{1:t}) &= \normdist{\vtheta_t}{\vmu_t}{\vSigma_t}.
    \end{aligned}
    \end{equation}
    With predictive mean and variance given by \eqref{eq:kf-predict-step} and
    posterior mean and variance given by \eqref{eq:kf-update-step}.
\end{proposition}

\begin{proof}
    See Section 7.2 in \cite{sarkka2023filtering}.
\end{proof}

\begin{algorithm}[htb]
\begin{algorithmic}[1]
    \REQUIRE $\data_t = (\vx_t, \vy_t)$ // datapoint
    \REQUIRE $(\vmu_{t-1}, \vSigma_{t-1})$ // previous mean and covariance
    \STATE  // predict step
    \STATE $\vF_t \gets \nabla_\vtheta f(\vmu_{t-1})$
    \STATE $\vmu_{t|t-1} \gets \vF_t\vmu_{t-1}$
    \STATE $\vSigma_{t|t-1} \gets \vF_t\vSigma_{t-1}\vF_{t}^\intercal + \vQ_t$
    \STATE //update step
    \STATE $\vH_t \gets \nabla_\vtheta h(\vmu_{t|t-1}, \vx_t)$
    \STATE $\vS_t = \vH_t\,\vSigma_{t-1}\vH_t^\intercal + \vR_t$
    \STATE $\vK_t = \vSigma_{t|t-1}\vH_t^\intercal\,\vS_t^{-1}$
    \STATE $\vmu_t \gets \vmu_{t-1} + \vK_t(\vy_t - \vH_t\vmu_{t|t-1})$
    \STATE $\vSigma_t \gets \vSigma_{t|t-1} - \vK_t\vH_t^\intercal\vSigma_{t|t-1}$
    \RETURN $(\vmu_t, \vSigma_t)$
\end{algorithmic}
\caption{
    predict and update steps for the extended Kalman filter
    for $t \geq 1$ and given prior mean and covariance
    $(\vmu_0, \vSigma_0)$.
}
\label{algo:extended-kalman-filter}
\end{algorithm}

\subsection{The ensemble Kalman filter}
The ensemble Kalman filter (EnKF) is a popular alternative to the EKF,
whenever the state-space is high-dimensional.
The EnKF was originally  introduced  as a sample-based approximation to the equations that
define the Kalman filter \citep{evensen1994enkf}.

Recall that the Kalman filter update estimates the posterior mean $\vmu_{t}$ according to
\begin{equation}
    \vmu_t = \vmu_{t|t-1} + \vK_t\,(\vy - \hat{\vy}_t),
\end{equation}
with $\vK_t$ the Kalman gain matrix.
It can be shown that under the linear SSM assumptions \eqref{eq:ssm-linear},
the Kalman gain matrix $\vK_t$ takes the form
\begin{equation}
    \begin{aligned}
    \vK_t
    &= {\rm Cov}(\vtheta_t, \vy_t - \hat{\vy}_t)\,{\rm Var}(\vy_t - \hat{\vy}_t)^{-1}\\
    &= {\rm Cov}(\vtheta_t, \vy_t)\,{\rm Var}(\vy_t)^{-1}.\\
    \end{aligned}
\end{equation}
See Ch.4 in \cite{eubank2005kalmanprimer}.

The EnKF propagates a bank of candidate parameters $\left\{\vtheta_t^{(s)}\right\}_{s=1}^S$
following \eqref{eq:ssm} and then updates each
value in the ensemble according to a sample-based gain matrix.
The \textit{predict} terms are obtained by propagating
\begin{equation}
\begin{aligned}
    \vtheta_{t|t-1}^{(s)} &= f\left(\vtheta_{t-1}^{(s)}\right) + \vu_t^{(s)},\\
    \vy_{t|t-1}^{(s)} &= h\left(\vtheta_{t|t-1}^{(s)}\right) + \ve_t^{(s)}.
\end{aligned}
\end{equation}
The EnKF then updates each member in the ensemble according to
\begin{equation}\label{eq:ensemble-kf-update}
    \vtheta_t^{(s)} = \vtheta_{t|t-1}^{(s)} + \hat{\vK}_t\,\left(\vy_t - \hat{\vy}_{t|t-1}^{(s)}\right),
\end{equation}
with $\hat{\vK}_t$ a sample-based estimate of the Kalman gain matrix that takes the form
\begin{equation}
    \hat{\vK}_t = \vC_{t|t-1}\,\vV_{t|t-1}^{-1},
\end{equation}
with
\begin{equation}
    {\rm Cov}(\vtheta_t, \vy_t) \approx \vC_{t|t-1} =
    \frac{1}{S}\sum_{s=1}^S\left(\vtheta_{t|t-1}^{(s)} - \bar{\vtheta}_{t|t-1}\right)\left(\vy_{t|t-1}^{(s)} - \bar{\vy}_{t|t-1}\right)^\intercal,
\end{equation}
and
\begin{equation}
    {\rm Var}(\vy_t) \approx \vV_{t|t-1} = \frac{1}{S}\sum_{s=1}^S \left(\vy_{t|t-1}^{(s)} - \bar{\vy}_{t|t-1}\right)\,\left(\vy_{t|t-1}^{(s)} - \bar{\vy}_{t|t-1}\right)^\intercal
\end{equation}
Here,
$\bar{\vy}_{t|t-1} = \frac{1}{S}\sum_{s=1}^S \vy_{t|t-1}^{(s)}$ and
$\bar{\vtheta}_{t|t-1} = \frac{1}{S}\sum_{s=1}^S \vtheta_{t|t-1}^{(s)}$. Algorithm \ref{algo:ensemble-kalman-filter} shows the predict and update steps for the EnKF.

\begin{algorithm}[htb]
\begin{algorithmic}[1]
    \REQUIRE $\data_{1:T}$ with $\data_{t} = (\vx_t, \vy_t)$
    \REQUIRE $\{\vtheta_0^{(s)}\}_{s=1}^S$ with $\vtheta_0^{(s)} \sim {\cal N}(\cdot \cond \vmu_0, \vSigma_0)$
    \FOR{$t=1,\ldots,T$}
    \STATE // predict step
    \FOR{$s=1,\ldots, S$}
        \STATE $\vtheta_{t|t-1}^{(s)} \gets f\left(\vtheta_{t-1}^{(s)}\right) + \vu_t^{(s)}$
        \STATE $\vy_{t|t-1}^{(s)} \gets h\left(\vtheta_{t|t-1}^{(s)}, \vx_t\right) + \ve_t^{(s)}$
    \ENDFOR
        \STATE // build sample gain matrix
        \STATE $\vC_{t|t-1} = \frac{1}{S}\sum_{s=1}^S\left(\vtheta_{t|t-1}^{(s)} - \bar{\vtheta}_{t|t-1}\right)\left(\vy_{t|t-1}^{(s)} - \bar{\vy}_{t|t-1}\right)^\intercal$
        \STATE $\vV_{t|t-1} = \frac{1}{S}\sum_{s=1}^S \left(\vy_{t|t-1}^{(s)} - \bar{\vy}_{t|t-1}\right)\left(\vy_{t|t-1}^{(s)} - \bar{\vy}_{t|t-1}\right)^\intercal$
        \STATE $\bar{\vK}_t = \vC_{t|t-1}\vV_{t|t-1}^{-1}$
    \STATE // update step
    \FOR{$t=1,\ldots,T$}
        \STATE $\vtheta_{t}^{(s)} \gets \vtheta_{t|t-1}^{(s)} + \bar{\vK}_t(\vy_t - \vy_{t|t-1}^{(s)})$
    \ENDFOR
    \ENDFOR
\end{algorithmic}
\caption{
    Predict and update steps for the ensemble Kalman filter
}
\label{algo:ensemble-kalman-filter}
\end{algorithm}

It can be shown that under a linear SSM, the EnKF matches the KF whenever $S\to\infty$ \citep{evensen2003ensemble}.

\subsection{Exponential-family EKF}
\label{sec:exponential-family-ekf}
Here, we consider the modified EKF method introduced in \cite{ollivier2019extended}.
This method modifies the update equations in Algorithm \ref{algo:extended-kalman-filter}
to make use of any member of the exponential family.
This generalisation is helpful in scenarios where the target variables $\vy_t$ cannot be
reasonably modelled as Gaussians.
For example
when dealing with binary classification problems,
in which the observations are modelled as Bernoulli,
or in multi-class classification problems,
where the observations are modelled as Multinomial.

Given the measurement model $p(\vy_t \cond \vtheta, \vx_t)$
parametrised as an exponential family with mean $h(\vtheta, \vx_t)$,
the exponential-family extended Kalman filter (ExpfamEKF)
replaces the likelihood model at time $t$ with a Gaussian whose mean and covariance
are found by matching the first two moments of the linearised log-likelihood
with respect to the previous mean $\vmu_{t-1}$.

More precisely, the mean at time $t$ is approximated by a first-order approximation around
the predicted mean $\vmu_{t|t-1}$. This takes the form
\begin{equation}
    \bar{h}_t(\vtheta) := h(\vmu_{t|t-1}, \vx_t) + \vH_t\,(\vtheta_t - \vmu_{t|t-1}).
\end{equation}
Next, the measurement variance is taken as the covariance of the linearised model. 
For example, if $\vy_t$ is modelled as a Bernoulli with mean $\sigma(h(\vmu_{t|t-1}, \vx))$,
then $\bar{\vR}_t =\sigma(h(\vmu_{t|t-1}, \vx) (1 - \sigma(h(\vmu_{t|t-1}, \vx))$.
Conversely, if $\vy_t$ is modelled as a multivariate Gaussian with known observation variance,
then $\bar{\vR}_t = \vR_t$.

\subsection{Example: Recursive Learning of Neural Networks II}  
\label{example:rvga-vs-expfamekf-neural-network-training} 
In this experiment we compare the performance of the ExpfamEKF method
on the moons dataset presented in Section \ref{example:rvga-neural-network-training}.  
We evaluate two configurations of the ExpfamEKF over $100$ initialisations.
The first configuration assumes static dynamics ($Q_t = 0\,\vI$),  
which corresponds to the R-VGA under a linearised measurement model.  
The second configuration uses $Q_t = 10^{-5}\,\vI$, accounting for unmodelled errors.

Figure \ref{fig:recursive-nnet-rvga-vs-eekf} illustrates the median rolling prequential accuracy
and interquartile range using a window of $50$ steps and $100$ different initial states.
comparing the predicted and actual class labels over time.
\begin{figure}[htb]  
    \centering  
    \includegraphics[width=0.9\linewidth]{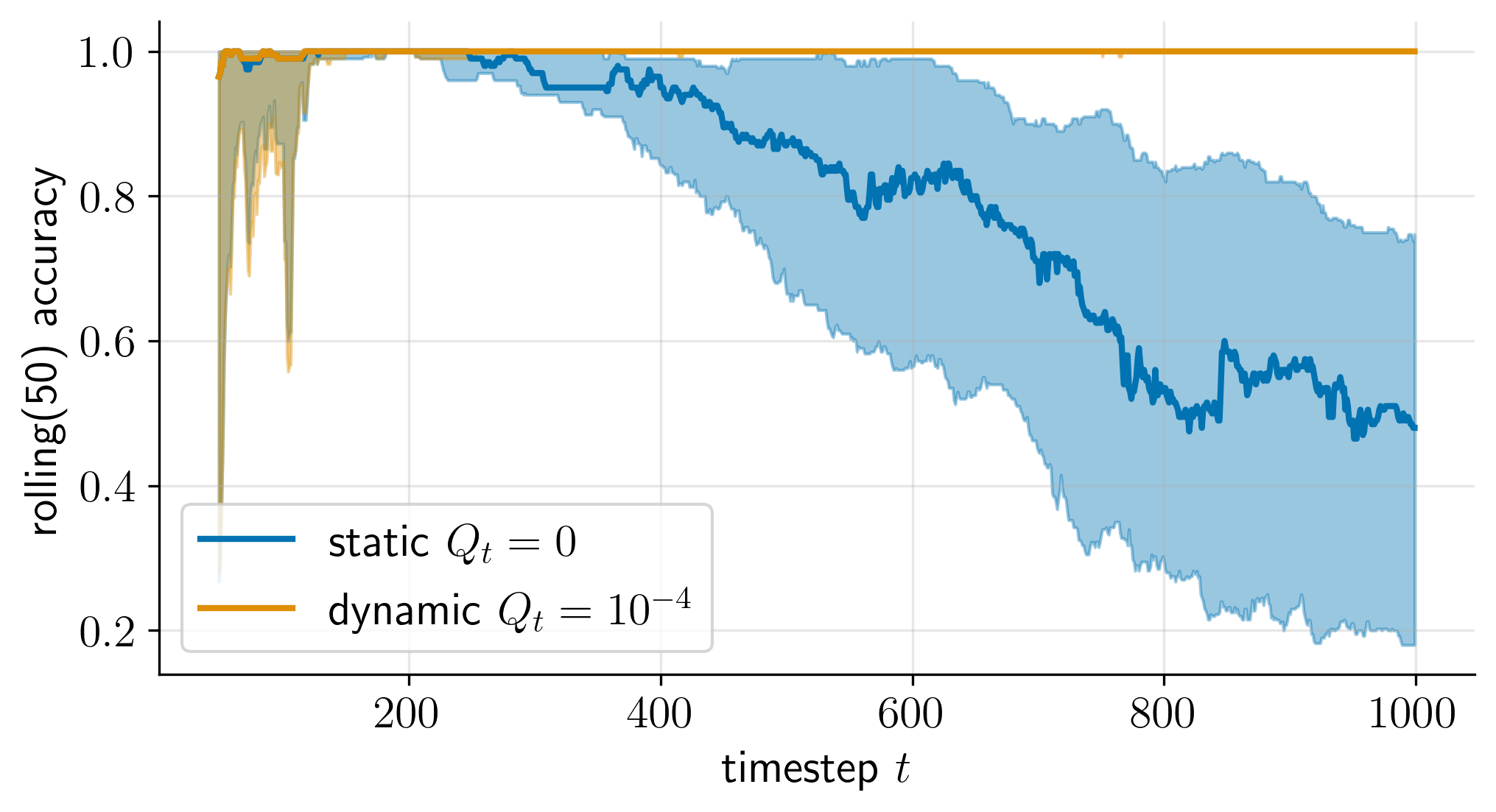}  
    \caption{  
    Rolling prequential accuracy for the non-linear classification problem,  
    trained using R-VGA and ExpfamEKF.  
    }  
    \label{fig:recursive-nnet-rvga-vs-eekf}  
\end{figure}  

We observe that the dynamic ExpfamEKF commences with high noise and then stabilises at a median accuracy of approximately $100\%$.
Similarly, the static configuration exhibits high variability initially and stabilises its accuracy around $100\%$ for $200$
steps before its predictive performance begins to decline and its variability increases.
This behaviour is attributed to numerical instability.
This experiment highlights the importance of small random-walk noise in parameter space for
maintaining numerical stability when deploying these methods online.

\section{Alternative update methods}
Alternative approaches for handling nonlinear 
or nonconjugate measurements have been proposed.

For instance sequential Monte Carlo (SMC) methods have been used to train neural networks \citep{freitas2000smcneuralnets}.
These sample-based methods are particularly advantageous when the state-transition function $f$ is highly non-linear
or when a more exact posterior approximation is required.
Next, Generalised Bayesian methods, such as \cite{mishkin2018slang,knoblauch2022gvi}, generalise the
VB target \eqref{eq:rvga-target-orignal}
to allow the likelihood to be a loss function;
see Chapter \ref{ch:robustness} for further details.
Alternatively,
online gradient descent  methods like \cite{bencomo2023implicit} emulate state-space modelling via gradient-based optimisation,
and gradient-free methods like \cite{goulet2021tagi} estimate the weights of the neural network assuming a diagonal posterior covariance matrix.

\section{Conclusion}
\label{sec:filtering-as-learning}

In this chapter, we introduced the problem of Bayesian sequential estimation of model parameters from a stream of data.
We laid the groundwork for the remainder of this thesis by establishing the necessary foundational concepts,
and presented an array of methods that address this problem by computing or approximating the posterior
density over model parameters recursively.
For computational efficiency, we focused on methods that maintain a
multivariate Gaussian posterior density.

We demonstrated that the problem of sequential Bayesian online linear regression
can be viewed as an \textit{online} version of the Ridge regression,
as a special case of the R-VGA algorithm, and
as a specific instance of the Kalman filter,
where there is an absence of noise in the system dynamics.

Furthermore,
we discussed how the linear assumption underlying both the Kalman filter and online linear regression can be extended to accommodate nonlinear measurement functions.
Specifically, we presented three different approaches: the Extended Kalman Filter (EKF), the Ensemble Kalman Filter (EnKF), and the Recursive Variational Gaussian Approximation (R-VGA).



\chapter{Adaptivity}
\label{ch:adaptivity}

In this chapter, we propose a unifying framework
for methods
that perform probabilistic online learning in non-stationary environments. 
We call the framework BONE, which stands for generalised (B)ayesian (O)nline learning in (N)on-stationary (E)nvironments. 
BONE provides a common structure to tackle a variety of problems,
including online continual learning, prequential forecasting, 
and contextual bandits.

The motivation for BONE arises from a key challenge in sequential online learning: non-stationarity in the data-generating process.
While a number of methods have been proposed to address this issue, the literature remains relatively sparse.
Different communities often tackle similar problems using overlapping ideas, but without a shared methodological foundation.
As a result insights from one area are not easily transferable across domains.

In this chapter, we show that many of these methods,
despite their apparent differences, can all be understood
within a common framework that can be seen as a form of generalised Bayes posterior predictive.
This insight enables a fair and principled comparison across method and lays the groundwork for developing new approaches.
The framework requires specifying
 three modelling choices:
(i) a model for measurements (e.g., a neural network),
(ii) an auxiliary process to model non-stationarity (e.g., the time since the last changepoint), and
(iii) a conditional prior over model parameters  (e.g., a multivariate Gaussian).
The framework also requires two algorithmic choices, which we use to carry out approximate inference under this framework:
(i)
an algorithm to estimate beliefs (posterior distribution) about the model parameters given the auxiliary variable,
and
(ii) an algorithm to estimate beliefs about the auxiliary variable.

A key insight provided by this framework is that most existing methods assume either abrupt changes or gradual drift
in the underlying process.
However, real-world scenarios often involve a combination of both.
To address this, we introduce a novel method that accounts for both types of non-stationarity. Abrupt changes are modelled via the time since the last parameter reset, while slow drift is captured through an Ornstein–Uhlenbeck process over the model parameters. We evaluate this method in a range of experiments and demonstrate its performance across diverse settings

\section{The framework}
In Chapter \ref{ch:recursive-bayes} we introduced the notion of the Bayesian posterior density
to estimate, or approximate, the model parameters of the measurement function $h$ recursively.
The methods presented work well when the data-generating process is well-specified,
the variational approximation family is big enough to accommodate the Bayesian posterior density, or
in the case of the Kalman filter, a \textit{good} choice of
dynamics covariance $\vQ_t$ is determined.

In practice, however, this is not often the case.
Thus, to adapt to regime changes and other forms of non-stationarity,
we introduce an auxiliary random variable $\auxv_t \in \vPsi_t$,
that evolves following the dynamics $p(\auxv_t \cond \auxv_{t-1})$ and encodes information about the non-stationarity of the sequence at time $t$.
Here, $\vPsi_t$ is the set of possible values of the auxiliary variable $\auxv_t$.
The purpose of this variable is, for instance,
to determine which past datapoints $\vy_{1:t-1}$ most closely align with the most recent measurement $\vy_t$.
We describe the auxiliary variable in detail in Section \ref{sec:auxiliary-variable}.
Finally, the model parameters $\vtheta_t$ evolve following the dynamics $p(\vtheta_t \cond \vtheta_{t-1}, \auxv_t)$.
This represents how much parameters change, given the state of the auxiliary variable.

Figure \ref{fig:diagram-bone-ssm}
shows the probabilistic graphical model that motivates our formulation; this resembles the one in \cite{Doucet2000}
with an additional optional dependence between the auxiliary variable and the measurements; in what follows we omit this dependence for brevity.
\begin{figure}[htb]
    \centering
        \begin{tikzpicture}[
        node distance=0.5cm and 1.0cm,
        every node/.style={draw, circle, minimum size=1.2cm,
        text width=0.5cm, align=center},
        every path/.style={thick, ->, >=stealth}
    ]
        	
        \node (theta1) {$ \vtheta_{t-1} $};
        \node[right=of theta1] (theta2) {$ \vtheta_t $};
        \node[right=of theta2] (theta3) {$ \vtheta_{t+1} $};

        \node[above=of theta1] (auxv1) {$\auxv_{t-1}$};
        \node[above=of theta2] (auxv2) {$\auxv_t$};
        \node[above=of theta3] (auxv3) {$\auxv_{t+1}$};

        \node[fill=lightgray, below=of theta1] (y1) {$ \vy_{t-1} $};
        \node[fill=lightgray, below=of theta2] (y2) {$ \vy_t $};
        \node[fill=lightgray, below=of theta3] (y3) {$ \vy_{t+1} $};

        \node[rectangle, fill=lightgray, below=of y1, xshift=-3mm] (x1) {$ \vx_{t-1} $};
        \node[rectangle, fill=lightgray, below=of y2, xshift=-3mm] (x2) {$ \vx_t $};
        \node[rectangle, fill=lightgray, below=of y3, xshift=-3mm] (x3) {$ \vx_{t+1} $};
                
        \path (theta1) edge (theta2);
        \path (theta2) edge (theta3);
        
        \path (theta1) edge (y1);
        \path (theta2) edge (y2);
        \path (theta3) edge (y3);
        
        \path (x1) edge (y1);
        \path (x2) edge (y2);
        \path (x3) edge (y3);
        
        \path (auxv1) edge (theta1);
        \path (auxv2) edge (theta2);
        \path (auxv3) edge (theta3);

        \path (auxv1) edge (auxv2);
        \path (auxv2) edge (auxv3);
        
        \path (auxv1) edge[bend right=40, dotted] (y1);
        \path (auxv2) edge[bend right=40, dotted] (y2);
        \path (auxv3) edge[bend right=40, dotted] (y3);
    \end{tikzpicture}
    \caption{
        Two-levelled hierarchical state-space model (SSM) with known dynamics, motivating our BONE framework
        Solid arrows indicate required dependencies, while dashed arrows represent optional dependencies.
        Rectangles denote exogenous variables, and circles represent random variables. Observed elements are shaded in gray.
        The left shift in $\vx_t$ represents that features are observed before observing $\vy_t$.
    }
    \label{fig:diagram-bone-ssm}
\end{figure}
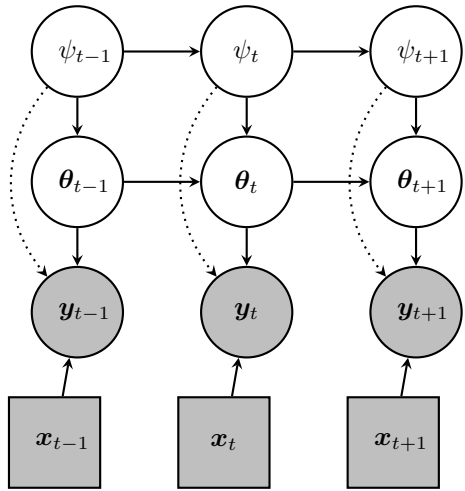

For an experiment of length $T\in\mathbb{N}$, the joint conditional density over the model parameters,
induced by the graphical model shown in Figure \ref{fig:diagram-bone-ssm}, is given by
\begin{equation}\label{eq:bone-log-joint}
    p(\vy_{1:T}, \vtheta_{0:T}, \auxv_{0:T}  \cond \vx_{1:T})
    = p(\vtheta_0)\,p(\auxv_0)\,\prod_{t=1}^T
    p(\vy_t \cond \vtheta_t, \vx_t)\, p(\vtheta_t \cond \vtheta_{t-1}, \auxv_t)\,p(\auxv_t \cond \auxv_{t-1}).
\end{equation}

In this chapter,
we are interested in methods
that efficiently compute the so-called \textit{expected} posterior predictive
$\hat{\vy}_{t+1} := \mathbb{E}_{p(\vtheta_t, \auxv_t \cond \data_{1:t})}[h(\vtheta_t,\vx_{t+1})]$ in an online and recursive manner.
In our setting, one observes $\vx_{t+1}$ just before observing $\vy_{t+1}$;
thus, to make a prediction about $\vy_{t+1}$, we have $\vx_{t+1}$ and $\data_{1:t}$ at our disposal.\footnote{
The  input features $\vx_{t+1}$ and
output measurements $\vy_{t+1}$ can correspond to different time steps.
For example, $\vx_{t+1}$ can be the state of the stock market at a fixed date and
$\vy_{t+1}$ is the return on a stock some days into the future.
}
For the case of a  discrete auxiliary variable $\auxv_t \in \Psi_t$,
the form of the expected posterior predictive for $\vy_{t+1}$, induced by \eqref{eq:bone-log-joint}, is
\begin{equation}\label{eq:bone-posterior-predictive}
\begin{aligned}
    \hat{\vy}_{t+1}
    &= \mathbb{E}_{p(\vtheta_t, \auxv_t \cond \data_{1:t})}[h(\vtheta_t,\vx_{t+1})]\\
    &= \sum_{\auxv_t \in \Psi_t}\int h(\vtheta_t, \vx_{t+1})\,p(\vtheta_t, \auxv_t \cond \data_{1:t}) \d\vtheta_t\\
    &= \sum_{\auxv_t \in \Psi_t}p(\auxv_t \cond \data_{1:t})
    \int h(\vtheta_t, \vx_{t+1}) p(\vtheta_t \cond \auxv_t, \data_{1:t})\d\vtheta_t,
\end{aligned}
\end{equation}

where
\begin{align}
    p(\vtheta_t \cond \auxv_t, \data_{1:t})
    &\propto p(\vy_t \cond \vtheta_t, \vx_t)\,p(\vtheta_t \cond \auxv_t, \data_{1:t-1}),\label{eq:bone-extra-i} \\
    p(\vtheta_t \cond \auxv_t, \data_{1:t-1})
   &=\int p(\vtheta_t \cond \vtheta_{t-1}, \auxv_t)\,p(\vtheta_{t-1} \cond\data_{1:t-1})\,\d\vtheta_{t-1}, \label{eq:bone-extra-ii}\\
    p(\auxv_t \cond \data_{1:t})
    &=  p(\vy_t \cond \vx_t, \auxv_t, \data_{1:t-1})\,
    \sum_{\auxv_{t-1}\in\vPsi_{t-1}}
    p(\auxv_{t-1} \cond \data_{1:t-1})\,
    p(\auxv_t \cond \auxv_{t-1}, \data_{1:t-1}) \label{eq:bone-extra-iii}.
\end{align}

From \eqref{eq:bone-posterior-predictive}, \eqref{eq:bone-extra-i}, \eqref{eq:bone-extra-ii}, and \eqref{eq:bone-extra-iii}
we argue that there are three key modelling choices and two algorithmic choices.  
Specifically, the three key modelling choices are:
\cModelraw the conditional mean $h(\vtheta, \vx)$ together with the likelihood $p(\vy \cond \vtheta, \vx)$;
\cAuxraw the auxiliary variable $\auxv_t$; and
\cPriorraw the conditional prior $p(\vtheta_t \cond \auxv_t, \data_{1:t-1})$.
Additionally, the two algorithmic choices are:
\cPosteriorraw the algorithm to compute (or approximate)  the conditional posterior over model parameters $p(\vtheta_t \cond \auxv_t, \data_{1:t})$, and
\cWeightraw the algorithm that computes (or approximates) the posterior over weights $p(\auxv_t, \cond \data_{1:t})$.

The BONE framework generalises these choices,
allowing for greater flexibility while maintaining the motivating probabilistic structure.
Instead of the likelihood model $p(\vy_t \cond \vtheta_t, \vx_t)$ with conditional mean $h(\vtheta_t, \vx_t)$,
we consider a general function $\exp(-\ell(\vy_t;\, \vtheta_t, \vx_t))$,
where $\ell(\vy_t; \,\vtheta_t,  \vx_t)$ could be either a loss function or a log-likelihood.
Next, instead of the conditional prior $p(\vtheta_t \cond \auxv_t, \data_{1:t-1})$,
we introduce a more general modelling function, $\pi(\vtheta_t; \auxv_t, \data_{1:t-1})$
that governs the prior over model parameters.\footnote{
This function adopts an ad hoc approach to parameter evolution instead of explicitly solving the integration step \eqref{eq:bone-extra-ii}.
}
Similarly,
instead of the posterior density $p(\vtheta_t \cond \auxv_t, \data_{1:t})$,
we employ the  function $q(\vtheta_t;\, \auxv_t, \data_{1:t})$; e.g., an approximation of the posterior,
or a generalised posterior \citep{bissiri2016generalbayes}.
Finally, instead of the posterior over weights $p(\auxv_t \cond \data_{1:t})$, we 
consider a weighting function $\nu(\auxv_t;\, \data_{1:t})$,
which can be the Bayesian posterior or an ad-hoc time-dependent weighting function.

This generalisation is important because it unifies a wide range of existing methods under a common framework.
Many well-known approaches in the literature can be written as elements of BONE by appropriately selecting the model for measurements, the conditional prior, and the posterior approximations.
BONE highlights connections between different methods,
it also enables systematic comparisons under a common umbrella, 
and it allows us to develop novel algorithms.
Table \ref{tab:modelling-choices-BONE} explicitly contrasts the choices in BONE with those in the
classical Bayesian formalism.
\begin{table}[htb]
    \centering
    \begin{tabular}{l|l|l}
         component & BONE & Bayes \\
         \hline
         \cModel & $h(\vtheta_t, \vx_t)$\,\,\&\,\,$\exp(-\ell(\vy_t; \vtheta_t, \vx_t))$ &  $h(\vtheta_t, \vx_t)$\,\,\&\,\,$p(\vy_t \cond \vtheta_t, \vx_t)$ \\
         \cAux & $\auxv_t$ & $\auxv_t$ \\
         \cPrior & $\pi_t(\vtheta_t;\, \auxv_t) := \pi(\vtheta_t;\, \auxv_t, \data_{1:t-1})$ & $p(\vtheta_t \cond \auxv_t, \data_{1:t-1})$\\
         \cPosterior & $q_t(\vtheta_t; \auxv_t) := q(\vtheta_t;\, \auxv_t, \data_{1:t})$ & $p(\vtheta_t \cond \auxv_t, \data_{1:t})$\\
         \cWeight & $\nu_t(\auxv_t) := \nu(\auxv_t;\, \data_{1:t})$ & $p(\auxv_t \cond \data_{1:t})$\\
    \end{tabular}
    \caption{
        Components of the BONE framework.
    }
    \label{tab:modelling-choices-BONE}
\end{table}

With these modifications, the expected posterior predictive under BONE is
\begin{equation}\label{eq:bone-prediction}
\begin{aligned}
    \hat{\vy}_t
    &:= \sum_{\auxv_t \in \Psi_t}
    \underbrace{\nu(\auxv_t \cond \data_{1:t})}_{\text{\cWeight}}\,
    \int
    \underbrace{h(\vtheta_t, \vx_{t+1})}_{\text{\cModel}}\,
    \underbrace{q(\vtheta_t;\, \auxv_t, \data_{1:t})}_{\text{\cPosterior}}
    \d\vtheta_t,
\end{aligned}
\end{equation}
where
\begin{equation}\label{eq:conditional-posterior}
   q(\vtheta_t;\,\auxv_t, \data_{1:t})
    \propto \underbrace{\pi(\vtheta_t;\, \auxv_t, \data_{1:t-1})}_\text{\cPrior}\,
    \underbrace{\exp(-\ell(\vy_t; \vtheta_t, \vx_t))}_\text{\cModel}
\end{equation}
takes the form of a generalised posterior \citep{bissiri2016generalbayes}.
In classical Bayesian setting, the loss function takes the form of the negative log-likelihood, i.e.,
\begin{equation}\label{eq:loss-function-neg-logl}
    \ell(\vy_t; \vtheta_t, \vx_t) = -\log p(\vy_t \cond \vtheta_t, \vx_t).
\end{equation}
Unless stated otherwise, we work with the negative log-likelihood in \eqref{eq:loss-function-neg-logl}
found in the classical Bayesian setting.

A prediction for $\vy_{t+1}$ given $\data_{1:t}$, $\vx_{t+1}$, and $\auxv_t$ is
\begin{equation}\label{eq:forecast-bone}
    \hat{\vy}_{t+1}^{(\auxv_t)}
    = \mathbb{E}_{q_t}[h(\vtheta_t;\, \vx_{t+1}) \cond \auxv_t]
    := \int h(\vtheta_t;\, \vx_{t+1})\,q(\vtheta_t;\,\auxv_t, \data_{1:t})\d \vtheta_t\,.
\end{equation}  
Here we use the shorthand notation $q_t = q(\vtheta_t;\,\auxv_t, \data_{1:t})$ and
$\mathbb{E}_{q_t}[\cdot \cond \auxv_t]$ to highlight dependence on $\auxv_t$.  

Algorithm \ref{algo:bone-step} provides pseudocode for the prediction and update steps in the BONE framework.
Notably, these components can be broadly divided into two categories: modelling and algorithmic.
The modelling components determine the inductive biases in the model, and correspond to $h$, $\ell$, $\auxv_t$, and $\pi$.
The algorithmic components dictate how operations are carried out to produce a final prediction --- this corresponds to $q_t$ and $\nu_t$.
\begin{algorithm}[H]
\small
\begin{algorithmic}[1]
    \REQUIRE $\data_{1:t}$  // past data
      \REQUIRE $\vx_{t+1}$ // optional inputs
    \REQUIRE $h(\vtheta, \vx_{t})$ // Choice of \cModel
   \REQUIRE $\vPsi_t$ // Choice of \cAux
    \FOR{$\auxv_t\in\vPsi_t$}
        \STATE $\pi_t(\vtheta_t;\, \auxv_t) \gets \pi(\vtheta_t ;\, \auxv_t, \data_{1:t-1})$ // choice of \cPrior
        \STATE $q_t(\vtheta_t;\,\auxv_t) \gets q(\vtheta_t ;\, \auxv_t, \data_{1:t}) \propto \pi_{t}(\vtheta_t\;\auxv_t)\,\exp(-\ell(\vy_t; \vtheta_t, \vx_t))$// choice of \cPosterior
        \STATE $\nu_t(\auxv_t) \gets \nu(\auxv_t ;\, \data_{1:t})$  // choice of \cWeight
        \STATE $\hat{\vy}_{t+1}^{(\auxv_t)} \gets
        \mathbb{E}_{q_t}[h(\vtheta_t, \vx_{t+1}) ;\, \auxv_t]$ // conditional prequential prediction
    \ENDFOR
    \STATE $\hat{\vy}_{t+1} \gets \sum_{\auxv_t}\nu_t(\auxv_t)\,\hat{\vy}_{t+1}^{(\auxv_t)}$ // weighted prequential prediction
\end{algorithmic}
\caption{
    Generic predict and update step of BONE with discrete $\auxv_t$ at time $t$.
}
\label{algo:bone-step}
\end{algorithm}

\subsection{Details of BONE}
\label{sec:details}

In the following subsections,
we describe each component of the BONE framework in detail,
provide illustrative examples,
and reference relevant literature for further reading.

\subsection{The  measurement model \cModelraw}
\label{sec:the-task}

Recall that $h(\vtheta, \vx)$ is a parametric model
that encodes the conditional mean for $\vy$, given $\vtheta$ and $\vx$.
For linear measurement models, $h(\vtheta,  \vx)$  is given by:
\begin{equation}
    h(\vtheta, \vx) =
    \begin{cases}
        \vtheta^\intercal \vx & \text{(regression)}, y \in \real \\
        \sigma(\vtheta^\intercal \vx) & \text{(binary classification)}, y \in \{0,1\} \\
        {\rm Softmax}(\vtheta^\intercal \vx) & \text{(multi-class classification)},
        \vy \in \{0,1\}^C
    \end{cases}
\end{equation}
where $\sigma(z) = (1 + \exp(-z))^{-1}$ is the sigmoid function, $C\in\mathbb{N}$ is the number of classes,
${\rm Softmax}(\vz)_k = {\exp(\vz_k)}/{\sum_i \exp(\vz_i)}$ represents the softmax function with
$\vz\in\real^\dimobs$
and $\vz_i$  the $i$-th element of $\vz$.
In the machine learning literature,
the vector $\vz$ is called the logits of
the classifier. 
For non-linear measurement models, such as neural networks, $h(\vtheta, \vx)$ represents the output of the network parameterised by $\vtheta$.
The best choice of $h$ will depend on the nature of the data,
as well as the nature of the task,
in particular, whether it is supervised or unsupervised.
We give some examples in 
Section \ref{section:experiments}.

\subsection{The auxiliary variable \cAuxraw }
\label{sec:auxiliary-variable}

The choice of auxiliary variable $\auxv_t$
is crucial to identify changes in the data-generating process,
allowing our framework to track non-stationarity.
Below, we give a list of the common auxiliary variables used in the literature.

\noindent \underline{\texttt{RL}} (runlength):  
    $\auxv_{t} = r_t \in \{0, \ldots, t\}$ is a scalar representing a \textit{lookback window}, defined as the number of steps since the last regime change.  
    The value $r_t = 0$ indicates the start of a new regime at time $t$,
    while $r_t \geq 1$ denotes the continuation of a regime with a lookback window of length $r_t$.  
    This choice of auxiliary variable is common in the changepoint detection literature.
    See e.g., \cite{adams2007bocd,
    knoblauch2018doublyrobust-bocd, alami2020restartedbocd, agudelo2020bocdprediction, altamirano2023robust, alami2023banditnonstationary}.
    This auxiliary variable is useful for non-stationary data with non-repeating temporal segments,
    provided we know the intensity with which new segments appear.

\noindent \underline{\texttt{RLCC}} (runlength and changepoint count):  
    $\auxv_{t} = (r_t, c_t) \in \{0, \ldots, t\} \times \{0, \ldots, t\}$ is a vector that represents both the runlength and the total number of changepoints,
    as proposed in \cite{wilson2010-bocd-hazard-rate}.
    When $r_t = t$, this implies $c_t = 0$, meaning no changepoints have occurred.  
    Conversely, $r_t = 0$ indicates the start of a new regime and implies $c_t \in \{1, \ldots, t\}$, accounting for at least one changepoint.  
    For a given $r_t\geq 0$, the changepoint count $c_t$ belongs to the range $\{1, \ldots, t - r_t\}$.  
    As with \texttt{RL}, this auxiliary variable assumes consecutive time blocks, but additionally allows us to estimate the likelihood of entering a new regime by tracking the number of changepoints seen so far. 
    This auxiliary variable  is useful for non-stationary data with non-repeating temporal segments when the intensity with which new segments appear is unknown.

\noindent \underline{\texttt{CPT}} (changepoint timestep):
    $\auxv_t = \vzeta_t$, with $\vzeta_t = \{\zeta_{1,t}, \ldots, \zeta_{\ell,t}\}$,
    is a set of size $\ell \in \{0, \ldots t\}$ containing the $\ell$ times at which there was a changepoint,
    with the convention that  $0 \leq \zeta_{1,t} < \zeta_{2,t} < \ldots < \zeta_{\ell,t} \leq t$.
    This choice of auxiliary variable was introduced in \cite{fearnhead2007line} and has been studied in
    \cite{fearnhead2011adaptivecp, fearnhead2019robustchangepoint}.
    Under mild assumptions, it can be shown that \texttt{CPT} is equivalent to \texttt{RL},
    see e.g., \cite{knoblauch2018varbocd}.
    This auxiliary variable  is useful for non-stationary data with non-repeating temporal segments
    when the probability of a new segment appearing is unknown and knowledge of the changepoint location is required.

\noindent \underline{\texttt{CPL}} (changepoint location):  
    $\auxv_t=s_{1:t}\in\{0,1\}^t$ is a binary vector. In one interpretation,
    $s_i=1$ indicates the occurrence of a changepoint at time $i$,
    as in \cite{li2021onlinelearning},
    while in another, it means that $\data_t$ belongs to the current regime, as in \cite{nassar2022bam}.
    This auxiliary variable  is useful for non-stationary data with repeating temporal segments.
    It is useful when the segments are formed of non-consecutive datapoints.

\noindent \underline{\texttt{CPV}} (changepoint probability vector):  
    $\auxv_{t}=v_{1:t}\in(0,1)^t$ is a $t$-dimensional random vector representing the probability of each element in the history belonging to the current regime.  
    This generalises \texttt{CPL} and was introduced in \cite{nassar2022bam} for online continual learning, allowing for a more fine-grained representation of changepoints over time.
    This auxiliary variable is useful for  non-stationary data with repeating temporal segments.
    Unlike \texttt{CPL},
    it takes a vector of weights in $(0,1)$
    which allows for higher flexibilty when compared to \texttt{CPL}.

\noindent \underline{\texttt{CPP}} (changepoint probability):  
    $\auxv_t = \upsilon_t\in(0,1)$ represents the probability of a changepoint.  
    This is a special case of \texttt{CPV} that tracks only the most recent changepoint probability;
    this choice
    was used in \cite{titsias2023kalman} for online continual learning.

\noindent \underline{\texttt{ME}} (mixture of experts):  
    $\auxv_{t} =\alpha_t \in \{1, \ldots, K\}$ represents one of  $K$ experts.
    Each expert corresponds to either a choice of model
    or one of $K$ possible hyperparameters.  
    This approach has been applied to filtering \citep{chaer1997mixturekf} and
    prequential forecasting \citep{liu2023bdemm, abeles2024adaptive}.
    This auxiliary variable facilitates the weighting of predictions made by models
    when one has a fixed number of competing models.

\noindent \underline{\texttt{C}}:  
    $\auxv_t = c$ represents a constant auxiliary variable, where $c$ is just a placeholder or dummy value.
    This is equivalent to not having an auxiliary variable,
    or alternatively, to having a single expert that encodes all available information.

\paragraph{Space-time complexity}
There is a tradeoff between the complexity that $\auxv$ is able to encode and the computation power needed to perform updates. Loosely speaking, this can be seen in the cardinality of the set of possible values of $\auxv$ through time.
Let  $\vPsi_t$ be the space of possible values for $\auxv_t$.
Depending on the choice of $\auxv_t$,
the cardinality of $\vPsi_t$  either stay constant or increase over time, i.e., 
 $\vPsi_{t-1} \subseteq \vPsi_t$ for all $t=1,\ldots,T$.
For instance,
the possible values for \texttt{RL} increase by one at each timestep;
the possible values of $\texttt{CPL}$ double at each  timestep; finally,
the possible values for $\texttt{ME}$ do not increase. 
Table \ref{tab:auxv-time-complexity}
shows the space of values and cardinality that $\vPsi_t$ takes as a function of the choice of auxiliary variable.

\begin{table}[htb]
    \scriptsize
    \centering
    \begin{tabular}{c|cccccccc}
    name & \texttt{C} & \texttt{CPT} & \texttt{CPP} & \texttt{CPL} & \texttt{CPV} & \texttt{ME} & \texttt{RL} & \texttt{RLCC} \\
    values & $\{c\}$ & $2^{\{0, 1, \ldots, t\}}$ & $[0,1]$ & $\{0,1\}^t$ & $(0,1)^t$ & $\{1, \ldots, K\}$ & $\{0, 1, \ldots, t\}$ & $\{\{0, t\}, \ldots, \{t, 0\}\}$\\
    cardinality & $1$ & $2^t$ & $\infty$ & $2^t$ & $\inf$ & $K$ & $t$ & $2+ t(t+1)/2$
    \end{tabular}
    \caption{
    Design space for the auxiliary random variables $\auxv_t$.
    Here,
    $T$ denotes the total number of timesteps and
    $K$ denotes a fixed number of candidates.
    }
    \label{tab:auxv-time-complexity}
\end{table}

\subsection{Conditional prior \cPriorraw}
\label{sec:conditional-prior}

This component defines the prior predictive distribution over model parameters
conditioned on the choice of \cAux $\auxv_t$ and the dataset $\data_{1:t-1}$.
In some cases, explicit access to past data is not needed.

For example, a common assumption  is to have a Gaussian conditional prior over model parameters.
In this case, we assume that, given data $\data_{1:t-1}$ and the auxiliary variable $\auxv_t$,
the conditional prior takes the form
\begin{equation}\label{eq:conditional-prior-gaussian}
    \pi(\vtheta_t;\, \auxv_t,\, \data_{1:t-1}) =
    {\cal N}\big(\vtheta_t\cond g_{t}(\auxv_t, \data_{1:t-1}), G_{t}(\auxv_t, \data_{1:t-1})\big),
\end{equation}
with
$g_{t}:\vPsi_t\times \real^{(\dimstate+\dimobs)(t-1)} \to \real^\dimstate$ a function that returns the mean vector of model parameters,
$\mathbb{E}[\vtheta_t \cond \auxv_t, \data_{1:t-1}]$,
and
$G_t:\vPsi_t\times \real^{(\dimstate+\dimobs)(t-1)}  \to \real^{\dimstate\times\dimstate}$ 
a function that returns a $\dimstate$-dimensional covariance matrix,
${\rm Cov}[\vtheta_t \cond \auxv_t, \data_{1:t-1}]$.
In what follows, we let $(\vmu_0, \vSigma_0)$ be the pre-defined initial prior mean and covariance.
Furthermore, we denote $(\vmu_{t-1}, \vSigma_{t-1})$ be the posterior mean and covariance found at time $t-1$,
which is used as a prior at time $t$.

Below, we provide a non-exhaustive list of possible combinations of
choices for \cAux and \cPrior of the form \eqref{eq:conditional-prior-gaussian}
that can be found in the literature, and we also introduce a new combination.

\noindent \underline{\texttt{C-LSSM}}  (constant linear with affine state-space model).
We assume the parameter dynamics can be modeled by a linear-Gaussian state space model (LSSM),
i.e., 
$\mathbb{E}[\vtheta_t \cond \vtheta_{t-1}] = \vF_t\,\vtheta_{t-1}
+ \vb_t$
and
${\rm Cov}[\vtheta_t \cond \vtheta_{t-1}] = \vQ_t$,
for given $(\dimstate\times\dimstate)$ dynamics matrix $\vF_t$,
$(\dimstate\times1)$ bias vector $\vb_t$,
and $(\dimstate\times\dimstate)$ positive semi-definite matrix $\vQ_t$.
We also assume
$\auxv_t = c$ is a fixed (dummy) constant,
which is equivalent to not having an auxiliary variable.
The characterisation of the conditional prior takes the form
\begin{equation}
\begin{aligned}
    g_t(c, \data_{1:t-1}) &= \vF_t\,\vmu_{t-1}
    + \vb_t,\\
    G_t(c, \data_{1:t-1}) &= \vF_t\,\vSigma_{t-1} \vF_t^\intercal + \vQ_t\,,
\end{aligned}
    \label{eqn:LSSM}
\end{equation}
This is a common baseline model that we will specialise below.

\noindent \underline{\texttt{C-OU}}  (constant with Ornstein-Uhlenbeck process).
This is a special case of the \texttt{C-LSSM} model
where
$\vF_t = \gamma \vI$, 
$\vb_t = (1-\gamma) \vmu_0$,
$\vQ_t = (1 - \gamma^2) \vSigma_0$,
$\vSigma_0=\sigma_0^2 \vI$,
$\gamma \in [0,1]$ is the fixed rate, and
$\sigma_0 \geq 0$.
The conditional prior mean
and covariance are
a convex combination of the form
\begin{equation}
\label{eq:cprior-c-ou}
\begin{aligned}
    g(c, \data_{1:t-1}) &=  \gamma \vmu_{t-1} + (1 - \gamma) \vmu_0, \\
    G(c, \data_{1:t-1}) &=  \gamma^2 \vSigma_{t-1} + (1 - \gamma^2) \vSigma_0.
\end{aligned}
\end{equation}
This combination is used in
\cite{kurle2019continual}.
Smaller values of the rate parameter $\gamma$
correspond to a faster resetting,
i.e., the distribution of model parameters
revert more quickly to the prior belief
$(\vmu_0,\vSigma_0)$,
which means the past data will be forgotten.

\noindent \underline{\texttt{CPP-OU}}  (changepoint probability with Ornstein-Uhlenbeck process).
Here 
$\auxv_t=\upsilon_t \in [0,1]$
is the changepoint probability that we use as
the rate of an Ornstein-Uhlenbeck (OU) process,
as proposed in
\cite{titsias2023kalman,Galashov2024}.
The characterisation of the conditional prior takes the form
\begin{equation}
\label{eq:cprior-cpp-d}
\begin{aligned}
    g(\upsilon_t, \data_{1:t-1}) &= \upsilon_t \vmu_{t-1} + (1 - \upsilon_t)  \vmu_0 \,,\\
    G(\upsilon_t, \data_{1:t-1}) &=  \upsilon_t^2 \vSigma_{t-1} + (1 - \upsilon_t^2)  \vSigma_0\,.
\end{aligned}
\end{equation}
An example on how to compute $\upsilon_t$
using an empirical Bayes procedure
is given in
\eqref{eqn:nut}.

\noindent \underline{\texttt{C-ACI}} (constant with additive covariance inflation).
This corresponds to a special case of \texttt{C-LSSM} in which $\vF=\vI$,
$\vb=\vzero$, and $\vQ=\alpha \vI$
for $\alpha>0$ is the amount of noise
added at each step.
This combination is used in
\cite{kuhl1990ridge,duranmartin2022-subspace-bandits,chang2022diagonal,chang2023lofi} .
The characterisation of the conditional prior takes the form
\begin{equation}\label{eq:cprior-c-aci}
\begin{aligned}
    g(c,  \data_{1:t-1}) &= \vmu_{t-1},\\
    G(c, \data_{1:t-1}) &= \vSigma_{t-1} + \vQ_t.
\end{aligned}
\end{equation}
This is similar to \texttt{C-OU} with $\gamma=1$,
however, here we inject new noise at each step. 
Another variant of this scheme,
known as \textit{shrink-and-perturb}
\citep{Ash2020},
takes 
$g(c,  \data_{1:t-1}) = q\,\vmu_{t-1}$
and
$G(c, \data_{1:t-1}) = \vSigma_{t-1} + \vQ_t$, where $0 < q < 1$
is the shrinkage parameters,
and $\vQ_t=\sigma_0^2\,\vI$.

\noindent \underline{\texttt{C-Static}}  (constant with static parameters).
Here  $\auxv_t = c$ (with $c$ a dummy variable).
This is a special case of the \texttt{C-ACI} configuration in which $\alpha=0$.
The conditional prior is characterised by
\begin{equation}\label{eq:cprior-c-f}
\begin{aligned}
    g_t(c, \data_{1:t-1}) &= \vmu_{t-1},\\
    G_t(c, \data_{1:t-1}) &= \vSigma_{t-1}.
\end{aligned}
\end{equation}

\noindent \underline{\texttt{ME-LSSM}}  (mixture of experts with LSSM).
Here $\auxv_t = \alpha_t \in \{1,\ldots, K\}$,
and we have a bank of $K$ independent
LSSM models; the auxiliary variable specifies which model to use at each step.
The characterisation of the conditional prior takes the form
\begin{equation}
\begin{aligned}
    g_t(\alpha_t, \data_{1:t-1}) &= \vF^{(\alpha_t)}_t\,\vmu_{t-1}^{(\alpha_t)}
    + \vb_t^{(\alpha_t)}\,,
    \\
    G_t(\alpha_t, \data_{1:t-1}) &= \vF^{(\alpha_t)}_t\,\vSigma_{t-1}^{(\alpha_t)}\vF_t^\intercal + \vQ_t^{(\alpha_t)}\,.
\end{aligned}
    \label{eqn:MELSSM}
\end{equation}
The superscript $(\alpha_t)$ denotes the conditional prior for the $k$-th expert. More precisely, $\vmu_{t-1}^{(\alpha_t)},\vSigma_{t-1}^{(\alpha_t)}$ are the posterior at time $t-1$ using $\vF^{(\alpha_t)}_{t-1}$ and $\vQ_{t-1}^{(\alpha_t)}$ from the $k$-th expert.
This combination was introduced in \cite{chaer1997mixturekf}.

\noindent \underline{\texttt{RL-PR}}  (runlength with prior reset):
for $\auxv_t = r_t$, this choice of auxiliary variable constructs a new mean and covariance
considering the past $t - r_t$ observations. We have
\begin{equation}\label{eq:cprior-rl-pr}
\begin{aligned}
    g_t(r_t, \data_{1:t-1}) &= \vmu_0\,\mathds{1}(r_t  = 0) + \vmu_{(r_{t-1})}\mathds{1}(r_t > 0),\\
    G_t(r_t, \data_{1:t-1}) &= \vSigma_0\,\mathds{1}(r_t  = 0) + \vSigma_{(r_{t-1})}\mathds{1}(r_t > 0),\\
\end{aligned}
\end{equation}
where  $\vmu_{(r_{t-1})}, \vSigma_{(r_{t-1})}$ denotes the posterior belief computed using observations
from indices $t - r_t$ to $t - 1$.
The case $r_t = 0$ corresponds to choosing the initial pre-defined prior mean and covariance $\vmu_0$ and $\vSigma_0$.
This combination assumes that data from a single regime arrives in sequential \textit{blocks} of time
of length $r_t$.
This choice of \cPrior was first studied in \cite{adams2007bocd}.

\noindent  \underline{\texttt{RL[1]-OUPR*}}  (greedy runlength with OU and prior reset):
\label{sec:SPR}
This is a new combination we consider in this paper,
which is designed to accommodate both gradual changes and sudden changes. More precisely, we assume
$\auxv_t = r_t$, and
we choose the conditional prior as 
either a hard
reset to the prior,
if $\nu_t(r_t) > \varepsilon$,
or a convex combination
of the prior and the previous belief state (using an \texttt{OU} process),
if $\nu_t(r_t) \leq \varepsilon$.
That is,
we define the conditional prior as
\begin{equation}
\label{eq:SPR-equation-gt}
    g_t(r_t, \data_{1:t-1}) =
    \begin{cases}
        \vmu_0\,(1 - \nu_t(r_t)) + \vmu_{(r_t)}\,\nu_t(r_t) & \nu_t(r_t) > \varepsilon,\\
        \vmu_0 & \nu_t(r_t) \leq \varepsilon,
    \end{cases}
\end{equation}
\begin{equation}
\label{eq:SPR-equation-Gt}
   G_t(r_t, \data_{1:t-1}) =
    \begin{cases}
        \vSigma_0\,(1 - \nu_t(r_t)^2) + \vSigma_{(r_t)}\,\nu_t(r_t)^2 & \nu_t(r_t) > \varepsilon,\\
        \vSigma_0 & \nu_t(r_t) \leq \varepsilon.
    \end{cases}
\end{equation}
Here 
$\nu_t(r_t)=p(r_t \cond \data_{1:t})$, with $r_t=r_{t-1}+1$,
is the probability we are continuing a segment, and $\nu_t(r_t)$
with $r_t = 0$ is the probability of a changepoint.
For details on how to compute $\nu_t(r_t)$,
see  \eqref{eq:rl-oupr-weight}.
The value of the threshold parameter $\varepsilon$ controls whether an abrupt change or a gradual change should take place.
In the limit when $\varepsilon =1$, this new combination does not learn, 
since it is always doing a hard
reset to the initial beliefs at time $t=0$.
Conversely, when $\varepsilon=0$, we obtain an OU-type update weighted by $\nu_t$.
When $\varepsilon=0.5$, we revert back to prior beliefs when the most likely hypothesis is that a changepoint has just occurred. 
Finally, we remark that the above combination allows us to make use of non-Markovian choices for \cModel, as we see in Section \ref{experiment:KPM}.
This is, to the best of our knowledge, a new combination that has not been proposed
in the previous literature;
for further details see Appendix \ref{sec:rl-spr-implementation}.

\noindent \underline{\texttt{CPL-Sub}} (changepoint location with subset of data):
for $\auxv_t = s_{1:t}$,
this conditional prior constructs the mean and covariance  as
\begin{equation}
\begin{aligned}
    g_t(s_{1:t}, \data_{1:t-1}) &= \vmu_{(s_{1:t-1})},\\
    G_t(s_{1:t}, \data_{1:t-1}) &= \vSigma_{(s_{1:t-1})},
\end{aligned}
\end{equation}
where  $\vmu_{(s_{1:t-1})}, \vSigma_{(s_{1:t-1})}$  denotes the posterior belief computed using
the observations for entries where $s_{1:t-1}$ have value of $1$.
This combination assumes that data from the current regime is scattered from the past history.
That is, it assumes that data from a past regime could become relevant again at a later date.
This combination has been studied in \cite{nguyen2017vcl}.

\noindent \underline{\texttt{CPL-MCI}}  (changepoint location with multiplicative covariance matrix):
for  $\auxv_t = s_{1:t}$, 
this choice of conditional prior maintains the prior mean, but increases
the norm of the prior covariance by a constant term $\beta \in (0,1)$. More precisely, we have that
\begin{equation}
\begin{aligned}
    g_t(s_{1:t}, \data_{1:t-1}) &= \vmu_{(s_{1:t-1)}},\\
    G_t(s_{1:t}, \data_{1:t-1}) &=  
    \begin{cases}
        \beta^{-1}\vSigma_{(s_{1:t-1})}  & s_t = 1\,,\\
        \vSigma_{(s_{1:t-1})}  & s_t = 0\,.
    \end{cases}
\end{aligned}
\end{equation}
This combination was first proposed in \cite{li2021onlinelearning}.

\noindent \underline{\texttt{CPT-MMPR}}  (changepoint timestep using moment-matched prior reset):
for  $\auxv_t = \vzeta_{t}$, with $\vzeta_t = \{\zeta_{1,t}, \ldots, \zeta_{\ell,t}\}$,
and $\zeta_{\ell,t}$ the position of the last changepoint,
the work of \cite{fearnhead2011adaptivecp}
assumes a dependence structure between changepoints.
That is, to build the conditional prior mean and covariance,  they consider the past $\data_{\zeta_{\ell,t}:t-1}$
datapoints whenever $\zeta_{\ell,t} \leq t-1$
and a moment-matched approximation to the mixture density over all possible subset densities since the last changepoint
whenever $\zeta_{\ell,t} = t$.
Mathematically, if 
$\zeta_{\ell,t} \leq t-1$, the prior mean and the prior covariance take the form
\begin{equation}
\begin{aligned}
    g(\vzeta_{t}, \data_{1:t-1}) &= \vmu_{(\zeta_{\ell,t}:t-1)},\\
    G(\vzeta_{t}, \data_{1:t-1}) &= \vSigma_{(\zeta_{\ell,t}:t-1)}.
\end{aligned}
\end{equation}
If $\zeta_{\ell,t} = t$, the conditional prior mean and conditional prior covariance are built using a moment-matching approach.
For the Gaussian case,
moment-matching is equivalent to minimising a KL divergence \citep{minka2013expectationpropagation}.
This takes the form
\begin{equation}\label{eq:cprior-cpt-mmpr}
\begin{aligned}
    g(\vzeta_t, \data_{1:t-1}) = \vmu_t\,,\\
    G(\vzeta_t, \data_{1:t-1}) = \vSigma_t,
\end{aligned}
\end{equation}
where
\begin{equation}
    \vmu_t, \vSigma_t = \argmin_{\vmu, \vSigma}
    {\boldsymbol{\rm D}}_{\rm KL}\left(
        \sum_{\zeta_{\ell,t-1}=1}^{t-1}\tilde{q}(\vtheta_t \cond \zeta_{\ell,t-1}, \zeta_{\ell,t}, \data_{1:t-1})
        \,||\,
        {\cal N}(\vtheta_t \cond \vmu, \vSigma)
    \right).
\end{equation}
Here, ${\boldsymbol{\rm D}}_{\rm KL}$ denotes the KL divergence
and $\tilde{q}(\vtheta_t \cond \zeta_{\ell,t-1}, \zeta_{\ell,t}, \data_{1:t-1} )$ is the unnormalised density
\begin{equation*}
\begin{aligned}
    &\tilde{q}(\vtheta_t \cond \zeta_{\ell,t-1}, \zeta_{\ell,t}, \data_{1:t-1})\\
    &= p(\zeta_{\ell,t-1} \cond \data_{1:t-1})\,p(\zeta_{\ell,t} \cond \zeta_{\ell,t-1}, \data_{1:t-1})\,{\cal N}(\vtheta_t \cond \vmu_{(\zeta_{\ell,t-1}:t-1)}, \vSigma_{(\zeta_{\ell,t-1}:t-1)}).
\end{aligned}
\end{equation*}
Choosing \texttt{MMPR} couples the choice of \cPrior with the algorithmic choice \cWeight
through $p(\zeta_{\ell,t-1} \cond \data_{1:t-1})$ and the choice $p(\zeta_{\ell,t} \cond \zeta_{\ell,t-1}, \data_{1:t-1})$.
For an example of \texttt{MMPR} with \texttt{RL} choice of \cAux, see Appendix \ref{sec:rl-mmpr-implementation}.

\subsection{Algorithm to compute the posterior over model parameters \cPosteriorraw}
\label{sec:A1}

This section presents algorithms for estimating the density $q(\vtheta_{t};\, \auxv_{t}, \data_{1:t})$;
we focus on methods that yield Gaussian posterior densities.
Specifically, we are interested in practical approaches for approximating the conditional Bayesian posterior,
as given in \eqref{eq:conditional-posterior}.

There is a vast body of literature on methods for estimating the posterior over model parameters,
many of which were introduced in Chapter \ref{ch:recursive-bayes}.
Here, we focus on three common approaches for computing the Gaussian posterior:
conjugate updates (\texttt{Cj}), 
linear-Gaussian approximations (\texttt{LG}),
and variational Bayes (\texttt{VB}).
For an overview of choices of \cPosterior
and a comparison in terms of their computational complexity, see
Table 3 in \cite{jones2024bong}.
 
\subsubsection{Conjugate updates (\texttt{Cj})}

Conjugate updates (\texttt{Cj}) provide a classical approach for computing the posterior by leveraging conjugate prior distributions.
Conjugate updates occur when the functional form of the conditional prior $\pi(\vtheta_t; \auxv_t, \data_{1:t-1})$
matches that of the measurement model $p(\vy_t \cond \vtheta, \vx_t)$ \citep[Section 3.3]{robert2007bayesianbook}.
This property allows the posterior to remain within the same family as the prior,
which leads to analytically tractable updates and facilitates efficient recursive estimation.

A common example is the conjugate Gaussian model, where the measurement model is Gaussian with known variance and the prior is a multivariate Gaussian. This results in closed-form updates for both the mean and covariance. Another example is the Beta-Bernoulli pair, where the measurement model follows a Bernoulli distribution with an unknown probability, and the prior is a Beta distribution.
See e.g., \cite{Bernardo94,West97} for details.

The recursive nature of conjugate updates makes them particularly useful for real-time or sequential learning scenarios, where fast and efficient updates are crucial.

\subsubsection{Linear-Gaussian approximation (\texttt{LG}) }

The linear-Gaussian (\texttt{LG}) method builds on the conjugate updates (\texttt{Cj}) above.
More precisely, the prior is Gaussian and the measurement model is approximated by a linear Gaussian model,
which simplifies computations.

The prior over model parameters is taken as:
\begin{equation}
\pi(\vtheta_t;\, \auxv_t, \data_{1:t-1}) = \mathcal{N}\left(\vtheta_t \cond \vmu_{t-1}^{(\auxv_t)}, \vSigma_{t-1}^{(\auxv_t)}\right),
\end{equation}
where $\vmu_{t-1}^{(\auxv_t)}$ and $\vSigma_{t-1}^{(\auxv_t)}$ are the mean and covariance, respectively.
We use the measurement function $h$ to define a first-order approximation $\bar{h}_t$ around the prior mean which is given by
\begin{equation}
    \bar{h}_t(\vtheta_t, \vx_t) = h\left(\vmu_{t-1}^{(\auxv_t)}, \vx_t\right) +
    \vH_t\,\left(\vtheta_t - \vmu_{t-1}^{(\auxv_t)}\right)\,.
\end{equation}
Here, $\vH_t$ is the Jacobian of $h(\vtheta,\vx_t)$ with respect to $\vtheta$, evaluated at $\vmu_{t-1}^{(\auxv_t)}$.
The approximate posterior measure is given by
\begin{equation}
\begin{aligned}
    q(\vtheta_t; \auxv_t, \data_{1:t})
    &\propto \mathcal{N}(\vy_t \cond \bar{h}_t(\vtheta_t, \vx_t), \vR_t)\,\pi(\vtheta_t;\, \auxv_t,\data_{1:t-1})
    \\
    &= \mathcal{N}(\vy_t \cond \bar{h}_t(\vtheta_t, \vx_t), \vR_t)\,\mathcal{N}\left(\vtheta_t \cond \vmu_{t-1}^{(\auxv_t)}, \vSigma_{t-1}^{(\auxv_t)}\right)
    \\
    &\propto \mathcal{N}(\vtheta_t \cond \vmu_t^{(\auxv_t)}, \vSigma_t^{(\auxv_t)}),
\end{aligned}
\end{equation}
where  $\vR_t$ is a known noise covariance matrix of the measurement $\vy_t$.
Under the \texttt{LG} algorithmic choice, the updated equations are
\begin{equation}\label{eq:ekf-update-step}
\begin{aligned}
    \hat{\vy}_t^{(\auxv_t)} &= h\left(\vmu_{t-1}^{(\auxv_t)}, \vx_t\right),\\
    \vS_t^{(\auxv_t)} &= \vH_t\,\vSigma_{t-1}^{(\auxv_t)}\,\vH_t^\intercal + \vR_t,\\
    \vK_t^{(\auxv_t)} &= \vSigma_{t-1}^{(\auxv_t)}\,\vH_t^\intercal\,\left(\vS_t^{(\auxv_t)}\right)^{-1},\\
    \vmu_t^{(\auxv_t)} &= \vmu_{t-1}^{(\auxv_t)} + \vK_t^{(\auxv_t)}\left(\vy_t - \hat{\vy}_t^{(\auxv_t)}\right),\\
    \vSigma_t^{(\auxv_t)} &=
    \vSigma_{t-1}^{(\auxv_t)} - \left(\vK_t^{(\auxv_t)}\right)\,\left(\vS_t^{(\auxv_t)}\right)\,\left(\vK_t^{(\auxv_t)}\right)^\intercal.
\end{aligned}
\end{equation}

\subsubsection{Variational Bayes (\texttt{VB})}
Here, we consider an extension of the variational Bayes method introduced in Section \ref{sec:variational-bayes},
where we condition on the auxiliary variable $\auxv_t$.
We have the following optimisation problem
 for the posterior variational parameters:
\begin{equation}\label{eq:variational-bayes-objective}
    (\vmu_t, \vSigma_t) =
    \argmin_{\vmu, \vSigma}
    {\boldsymbol{\rm D}}_{\rm KL}
    \left(\mathcal{N}(\vtheta_t \cond \vmu, \vSigma) \,\|\, p(\vy_t \cond \vtheta_t, \vx_t)\,\pi(\vtheta_t;\,\auxv_t, \data_{1:t-1})\right),
\end{equation}
where
$\pi_t(\vtheta_t;\,\auxv_t)$ is the chosen prior distribution \cPrior.

\subsubsection{Alternative methods}

Alternative approaches for handling nonlinear 
or nonconjugate measurements have been proposed,
such as sequential Monte Carlo (SMC) methods \citep{freitas2000smcneuralnets},
and ensemble Kalman filters (EnKF) \citep{Roth2017enkf}.
These sample-based methods are particularly advantageous when the dimensionality of $\vtheta$
is large,
or when a more exact posterior approximation is required,
providing greater flexibility in non-linear and non-Gaussian environments.

Generalised Bayesian methods, such as \cite{mishkin2018slang,knoblauch2022gvi}, generalise the VB update
of \eqref{eq:variational-bayes-objective}
by allowing the right-hand side to be a loss function.
Alternatively, online gradient descent  methods like \cite{bencomo2023implicit} emulate state-space modelling via gradient-based optimisation.


\subsection{Weighting function for auxiliary variable \cWeightraw}
\label{sec:choice-weighting-function}

The  term
$\nu_{t}(\auxv_{t})$
 defines the weights over possible values of the auxiliary variable
  \cAux.
We compute it as the marginal posterior distribution
$\nu_{t}(\auxv_{t})=p(\auxv_t\cond \data_{1:t})$ (see e.g.,  \cite{adams2007bocd, fearnhead2007line, fearnhead2011adaptivecp, li2021onlinelearning})
or with  \textit{ad-hoc} rules (see e.g., \cite{nassar2022bam, abeles2024adaptive,titsias2023kalman}).
In the former case,
the weighting function takes the form
\begin{equation}
\label{eq:weighting-marginal-posterior}
\begin{aligned}
    \nu_t(\auxv_t)
    &= p(\auxv_t \cond \data_{1:t})\\
    &= 
    p(\vy_t \cond \vx_t, \auxv_t, \data_{1:t-1})\,
    \int_{\auxv_{t-1}\in\vPsi_{t-1}}
    p(\auxv_{t-1} \cond \data_{1:t-1})\,
    p(\auxv_t \cond \auxv_{t-1}, \data_{1:t-1}) \d \auxv_{t-1},
\end{aligned}
\end{equation}
where one assumes that $\vy_t$ is conditionally independent of $\auxv_{t-1}$,
given $\auxv_t$, and $\vx_t$ is an exogenous vector.
The first term on the right hand side of
\eqref{eq:weighting-marginal-posterior}
is known as the conditional posterior predictive,
and is given by
\begin{equation}
    p(\vy_t \cond \vx_t, \auxv_t, \data_{1:t-1}) =
    \int p(\vy_t \cond \vtheta_t, \vx_t)\,\pi(\vtheta_t;\, \auxv_t, \data_{1:t-1}) \d\vtheta_t.
\end{equation}
This integral over $\vtheta_t$ 
may require approximations,
as we discussed in Section~\ref{sec:A1}.
Furthermore,
the integral over $ \auxv_{t-1}$
in \eqref{eq:weighting-marginal-posterior}
may also require approximations,
depending on the nature of the auxiliary variable $\auxv_t$, and the modelling assumptions for
$p(\auxv_t \cond \auxv_{t-1}, \data_{1:t-1})$.
We provide some examples below.


\subsubsection{Discrete auxiliary variable (DA)}

Here we assume the 
auxiliary variable takes values in a discrete space $\auxv_t \in \vPsi_t$.
The weights for the discrete auxiliary variable (\texttt{DA}) can be computed with a fixed number of hypotheses $K\geq 1$  or with a growing number of hypotheses
if the cardinality of $\vPsi_t$ increases through time;
we denote these cases by \texttt{DA[K]} and \texttt{DA[inf]} respectively.
Below, we provide three examples that estimate the weights under \texttt{DA[inf]} recursively.

\noindent \underline{\texttt{RL}}  (runlength with Markovian assumption): for $\auxv_t = r_t$, the work of \cite{adams2007bocd}
takes
\begin{equation}\label{eq:transition-rl-markov}
    p(r_t \cond r_{t-1}, \data_{1:t-1}) =
    \begin{cases}
        1 - H(r_{t-1}) & \text{if } r_t = r_{t-1} + 1,\\
        H(r_{t-1}) & \text{if } r_t = 0, \\
        0 & \text{otherwise},
    \end{cases}
\end{equation}
where $H: \mathbb{N}_0 \to (0,1)$ is the hazard function. A popular choice is to take $H(r)=\kappa\in(0,1)$ to be a fixed
constant hyperparameter known as the hazard rate.
The choice \RLPR[inf] 
is known as the Bayesian online changepoint detection model (BOCD).

\noindent \underline{\texttt{CPL}}  (changepoint location): for $\auxv_t=s_{1:t}$, the work of \cite{li2021onlinelearning} takes
\begin{equation}\label{eq:transition-cpl-independent}
    p(\tilde{s}_{1:t} \cond s_{1:t-1}, \data_{1:t-1}) = 
    \begin{cases}
        \kappa & \text{if } \left([\tilde{s}_{1:t} \setminus \tilde{s}_t] = s_{1:t-1}\right) \text{ and } \tilde{s}_t = 1,\\
        1 - \kappa & \text{if } \left([\tilde{s}_{1:t} \setminus \tilde{s}_t] = s_{1:t-1}\right) \text{ and } \tilde{s}_t = 0,\\
        0 & \text{otherwise},
    \end{cases}
\end{equation}
i.e., the sequence of changepoints at time $t$ correspond to the sequence of changepoints up to time $t-1$,
plus a newly sampled value for $t$.
See Appendix \ref{c-aux:CPL} for details
 on how to compute $\nu_t(s_{1:t})$.

\noindent \underline{\texttt{CPT}} (changepoint timestep with Markovian assumption): for $\auxv_t = \vzeta_t$, the work of \cite{fearnhead2007line}
takes
\begin{equation}
    p(\vzeta_t \cond \vzeta_{t-1}, \data_{1:t-1}) = p(\zeta_{\ell,t} \cond \zeta_{\ell,t-1}) = J(\zeta_{\ell,t} - \zeta_{\ell,t-1}),
\end{equation}
with $J:\mathbb{N}_0 \to (0,1)$ a probability mass function.
Note that $\zeta_{\ell,t} - \zeta_{\ell,t-1}$ is the distance between two changepoints, i.e., a runlength.
In this sense, $\zeta_{\ell,t} - \zeta_{\ell,t-1} = r_t$,
which relates \texttt{CPT} to \texttt{RL}.
See their paper for details
on how to compute $\nu_t(\vzeta_{t})$.

\paragraph{Low-memory variants --- from \texttt{DA[inf]} to \texttt{DA[K]}} 
In the examples above, the number of computations to obtain 
$\sum_{\auxv_{t}} \nu_{t}(\auxv_{t})$ grows in time.
To fix the computational cost, one can restrict the sum
to be over a subset ${\cal A}_t$ of the space of $\auxv_t$ with cardinality $|{\cal A}_t| = K \geq 1$.
Each element in the set ${\cal A}_t$ is called a hypothesis and given $K\geq 1$, we 
keep the $K$ most likely elements
---according to $\nu_{t}(\auxv_{t})$--- in ${\cal A}_t$.
We then define the normalised weighting function
\begin{equation}
    \hat{\nu}_{t}(\auxv_{t}) = 
    \frac{\nu_{t}(\auxv_{t})}{\sum_{\auxv'_{t} \in {\cal A}_t} \nu_{t}(\auxv'_{t})},
\end{equation}
which we use instead of $\nu_{t}(\auxv_{t})$.
For example, in \texttt{RL} above, ${\cal A}_{t-1} = \{r_{t-1}^{(k)} : k = 1, \ldots, K\}$
are the unique $K$ most likely runlengths where the superscript represents the ranking according to $\nu_{t-1}(\cdot)$.
Then, at time $t$,
the augmented set $\bar{\cal A}_t$ becomes $({\cal A}_{t-1}+1)\cup \{0\}$, where the sum is element-wise,
and we then compute the $K$ most likely elements of $\bar{\cal A}_t$ to define ${\cal A}_t$.
In \texttt{CPL}, ${\cal A}_{t-1} $
contains the $K$ most likely sequences of changepoints,    $\bar{\cal A}_{t} $ is defined as the collection of the $2\,K$ sequences where each sequence of ${\cal A}_{t-1}$ has a zero or one concatenated at the end.
Finally, the $K$ most likely elements in $\bar{\cal A}_{t}$ define ${\cal A}_{t}$.
This style of pruning is common in segmentation methods; see, e.g., \cite{saatcci2010GPBOCD},
but other pruning are also possible, such as those proposed by \cite{li2021onlinelearning},
or sampling-based approaches; see e.g. \cite{Doucet2000}.

\paragraph{Other choices for \texttt{DA[K]}}
Finally, some choices of weighting functions are derived using ad-hoc rules,
meaning that explicit or approximate solutions to the Bayesian posterior are not needed.
One of the most popular choices of ad-hoc weighting functions are mixture of experts,
which weight different models according to a given criterion.

\noindent \underline{\texttt{ME}}  (mixture of experts with algorithmic weighting):
Consider $\auxv_t = \valpha_t$.
Let $\valpha_{t,k}=k$ denote the $k$-th \textit{configuration} over \cPrior.
Next, denote by $\vw_t = \{\vw_{t,1}, \ldots, \vw_{t,K}\}$ a set of weights,
where $\vw_{t,k}$ corresponds to the weight for the $k$-th expert at time $t$.
The work of \cite{chaer1997mixturekf}
considers the weighting function
\begin{equation}
    \nu_t(\vw_{t})_k = \frac{\exp(\vw_{t,k}^\intercal\,\vy_t)}{\sum_{j=1}^K \exp(\vw_{t,j}^\intercal\,\vy_t)},
\end{equation}
for $k=1,\ldots,K$.
The set of weights $\vw_t$ are determined by maximising the surrogate gain function
\begin{equation}
    {\cal G}_t(\vw_t)
    = p(\vy_t \cond \vx_t, \data_{1:t-1})
    = \sum_{k=1}^K p(\vy_t \cond \vx_t, \valpha_{t,k}, \data_{1:t-1})\,\nu_t(\vw_{t})_k,
\end{equation}
with respect to $\vw_{t,k}$ for all $k = 1, \ldots, K$ at every timestep $t$.

We write \texttt{DA[K]}, where \texttt{K} is the number hypothesis, for methods that use \texttt{K} hypotheses at most. 
On the other hand, we write \texttt{DA[inf]} when we do not 
impose a bound on the number of hypotheses used.
Note that even when the choice of \cWeight is built using \texttt{DA[inf]},
one can modify it to make it \texttt{DA[K]}.

\paragraph{Discrete auxiliary variable with greedy hypothesis selection (\texttt{DA[1]})}
A special case of the above is \texttt{DA[1]}, where we employ a single hypothesis.
In these scenarios, 
we set $\nu(\auxv_t) = 1$ where $\auxv_t$ is the most likely hypothesis.

\noindent \underline{\texttt{RL}} (Greedy runlength):
For $\auxv_t = r_t$ and $\texttt{DA[1]}$,
we take
\begin{equation}\label{eq:transition-rl-spr}
    p(r_t \cond r_{t-1}, \data_{1:t-1}) =
    \begin{cases}
        1 - \kappa & \text{if } r_t = r_{t-1} + 1,\\
        \kappa & \text{if } r_t = 0, \\
        0 & \text{otherwise}.
    \end{cases}
\end{equation}
Our choice of \cWeight 
is based on the marginal predictive likelihood ratio,
which is derived from the computation of $p(r_t \cond \data_{1:t})$
under either either
an increase in the runlength ($r_t^{(1)} = r_{t-1} + 1)$
or a reset of the runlength ($r_t^{(0)} = 0$).
Under these assumptions, the form of $\nu_t(r_t^{(1)})$
is
\begin{equation}\label{eq:rl-oupr-weight}
    \nu_t(r_t^{(1)}) = \frac{p(\vy_t \cond r_t^{(1)}, \vx_t, \data_{1:t-1})\,(1 - \kappa)}{p(\vy_t \cond r_t^{(0)}, \vx_t, \data_{1:t-1})\,\kappa + p(\vy_t \cond r_t^{(1)}, \vx_t, \data_{1:t-1})\,(1-\kappa)}.
\end{equation}
For details on the computation of \cWeight, see Appendix \ref{sec:rl-spr-implementation}.
For a detailed implementation of
\cAux \texttt{RL},
\cPrior \texttt{OUPR},
\cWeight \texttt{DA[1]}, and
\cPosterior \texttt{LG}, 
see Algorithm \ref{algo:rl-spr-step} in the Appendix.

For example, \texttt{RL[1]} is a runlength $r_t$ with a single hypothesis.
We provide another example next.

\noindent \underline{\texttt{CPL}} (changepoint location with retrospective membership):
for $\auxv_t=s_{1:t}$, the work of \cite{nassar2022bam}
evaluates the probability of past datapoints belonging in the current regime.
In this scenario,
\begin{equation}
    p(s_{1:t} \cond s_{1:t-1}, \data_{1:t-1}) = p(s_{1:t} \cond \data_{1:t-1}),
\end{equation}
so that
\begin{equation}
    p(s_{1:t} \cond \data_{1:t}) \propto
    p(s_{1:t} \cond \data_{1:t-1})\,p(\vy_t \cond \vx_t, s_{1:t}, \data_{1:t-1}).
\end{equation}

This method allows for exact computation by summing over all possible $2^t$ elements.
However, to reduce the computational cost, they propose a discrete optimisation over possible values
$\{\nu_t(s_{1:t})\,:\, s_{1:t} \in \{0,1\}^t\}$, where
$\nu_t(s_{1:t}) = p(s_{1:t} \cond \data_{1:t})$.
Then, the hypothesis with highest probability is stored and gets assigned a weight of one.

\subsubsection{Continuous auxiliary variable (\texttt{CA})}
\label{sec:CA}

Here, we briefly discuss continuous auxiliary variables \texttt{(CA)}.
For some choices of $\auxv_t$ and transition densities $p(\auxv_t \cond \auxv_{t-1}, \data_{1:t-1})$,
computation of \eqref{eq:weighting-marginal-posterior} becomes infeasible.
In these scenarios, 
we use simpler approximations.
We give an example below.

\noindent \underline{\texttt{CPP}} (Changepoint probability with empirical Bayes estimate):
for $\auxv_t=\upsilon_t$, consider
\begin{equation}
    p(\upsilon_t \cond \upsilon_{t-1}, \data_{1:t-1}) = p(\upsilon_t),
\end{equation}
so that
\begin{equation}
    p(\upsilon_t \cond \data_{1:t}) \propto
    p(\upsilon_t)\,p(\vy_t \cond \vx_t, \upsilon_t).
\end{equation}
The work of \cite{titsias2023kalman} 
takes $\nu_t(\upsilon_t)=\delta(\upsilon_t - \upsilon_t^*)$, where $\delta$ is the Dirac delta function
and $\upsilon_t^*$ is a point estimate centred at
the maximum of the marginal
posterior predictive likelihood:
\begin{equation}
    \upsilon_t^* = \argmax_{\upsilon \in [0,1]} p(\vy_t \cond \vx_t, \upsilon, \data_{1:t-1}).
    \label{eqn:nut}
\end{equation}
In practice, \eqref{eqn:nut} is approximated by taking gradient steps towards the minimum.
This is a form of empirical Bayes approximation,
since we compute the most likely value of the prior 
after marginalizing out $\vtheta_t$.
The work of \cite{Galashov2024} considers a modified configuration with choice of \cAux
${\bm \upsilon}_t \in (0,1)^\dimstate$.

\section{Unified view of examples in the literature}
\label{sec:literature}

Table \ref{tab:related-bone-methods} shows that many existing methods can be written as instances
of BONE.
Rather than specifying the choice of 
(M.1), we instead
write the task for which it was designed,
as discussed in Section \ref{sec:tasks}.
We will experimentally compare a subset of these methods in Section \ref{section:experiments}.

The methods presented in Table \ref{tab:related-bone-methods} can be directly applied
to tackle any of the problems mentioned in Section \ref{sec:tasks}.
However, as choice of \cModel, we specify the task under which the configuration was introduced.\footnote{
In general, the components of BONE can be thought as 
the building blocks for new methods. Some of these combinations 
would not be useful, but they can be employed nonetheless.
}
\begin{table}[htb]
    \centering
    \scriptsize
    \renewcommand{\arraystretch}{1} 
    \begin{tabular}{l|c|c|c|c|c|c}
         \hline
         \textbf{Reference}  & \textbf{Task} & \textbf{\cAuxname} & \textbf{\cPriorname} & \textbf{\cPosteriorname} & \textbf{\cWeightname} \\
         \hline
         \cite{kalman1960} &  filtering & \texttt{C} & \texttt{LSSM} & \texttt{LG} & \texttt{DA[1]} \\
         \cite{magill1965optimaladaptivefilter} &  filtering & \texttt{ME} & \texttt{LSSM} & \texttt{LG} & \texttt{DA[K]}\\
         \cite{chang1978switchingkf} &  filtering & \texttt{ME} & \texttt{LSSM} & \texttt{LG} & \texttt{CA} \\
         \cite{chaer1997mixturekf} &  filtering & \texttt{ME} & \texttt{LSSM} & \texttt{LG} & \texttt{DA[K]}\\
         \cite{ghahramani2000vsssm} &  SSSM & \texttt{ME} & \texttt{Static} & \texttt{VB} & \texttt{CA} \\
         \cite{adams2007bocd} &  seg. & \texttt{RL} & \texttt{PR} & \texttt{Cj} & \texttt{DA[inf]}\\
         \cite{fearnhead2007line} &  seg. \& preq. & \texttt{CPT}/\texttt{ME} & \texttt{PR} & \texttt{Any} & \texttt{DA[inf]}\\
         \cite{wilson2010-bocd-hazard-rate} &  seg. & \texttt{RLCC} & \texttt{PR} & \texttt{Cj} & \texttt{DA[inf]}\\
         \cite{fearnhead2011adaptivecp} &  seg. & \texttt{CPT}/\texttt{ME} & \texttt{MMPR} & \texttt{Any} & \texttt{DA[inf]}\\
         \cite{mellor2013changepointthompsonsampling} &  bandits & \texttt{RL}  & \texttt{PR} & \texttt{Cj} & \texttt{DA[inf]}\\ 
         \cite{nguyen2017vcl} &  OCL & \texttt{CPL} & \texttt{Sub} & \texttt{VB} & \texttt{DA[1]}\\
         \cite{knoblauch2018varbocd} &  seg. & \texttt{RL}/\texttt{ME} & \texttt{PR} & \texttt{Cj} & \texttt{DA[inf]}\\
         \cite{kurle2019continual} &  OCL & \texttt{CPV} & \texttt{Sub} & \texttt{VB} & \texttt{DA[1]} \\
         \cite{li2021onlinelearning} &  OCL & \texttt{CPL} & \texttt{MCI} & \texttt{VB} & \texttt{DA[inf]} \\
         \cite{nassar2022bam} &  bandits \& OCL & \texttt{CPV} & \texttt{Sub} & \texttt{LG} & \texttt{DA[1]} \\
         \cite{liu2023bdemm} &  preq. & \texttt{ME} & \texttt{C},\texttt{LSSM} & \texttt{Any} & \texttt{DA[K]} \\
         \cite{chang2023lofi} & OCL & \texttt{C} & \texttt{ACI} & \texttt{LG} & \texttt{DA[1]} \\
         \cite{titsias2023kalman} &  OCL & \texttt{CPP} & \texttt{OU} & \texttt{LG} & \texttt{CA} \\
         \cite{Galashov2024} & CL & \texttt{CPP} & \texttt{OU} & \texttt{VB} & \texttt{CA} \\
         \cite{abeles2024adaptive} &  preq. & \texttt{ME} & \texttt{LSSM} & \texttt{LG} & \texttt{DA[K]}\\
         \RLSPR (ours) & any & \texttt{RL} & \texttt{SPR} & \texttt{Any} & \texttt{DA[1]}
    \end{tabular}
    \caption{
    List of methods ordered by publication date.
    The tasks 
    are discussed in Section \ref{sec:tasks}.
    We use the following abbreviations:
    SSSM means switching state space model;
    (O)CL means (online) continual learning;
    seg.~means segmentation;
    preq.~means prequential.
    Methods that consider two choices of \cAux are denoted by `\texttt{X}/\texttt{Y}'.
    This corresponds to a double expectation in \eqref{eq:bone-prediction}---one for each choice of auxiliary variable.
    }
    \label{tab:related-bone-methods}
\end{table}

\section{Worked examples}

\label{sec:caux-weights}
In this section, we provide a detailed calculation of $\nu_{t}(\auxv_t)$ for some choices of $\auxv_t$.
We consider a choice of \cModel to be linear Gaussian with known observation variance $\vR_t$, i.e.,
\begin{equation}\label{eq:linear-gaussian-measurement-model}
    p(\vy_t \cond \vtheta, \vx_t) = {\cal N}(\vy_t \cond \vx_t^\intercal\,\vtheta_t, \vR_t).
\end{equation}

\subsection{Runlength with prior reset (\texttt{RL-PR})}
\label{sec:rl-pr-implementation}

\subsubsection{Unbounded number of hypotheses \texttt{RL[inf]-PR}}
The work in \cite{adams2007bocd} takes
$\auxv_t = r_t$ to be the runlength, with $r_t \in \{0, 1, \ldots, t\}$, that
that counts the number of steps since the last changepoint.
Assume the runlength follows the dynamics \eqref{eq:transition-rl-markov}.
We consider $\nu_{t}(r_{t}) = p(r_t | \data_{1:t})$ such that
\begin{equation}
    p(r_t \cond \data_{1:t}) =
    \frac{p(r_t, \data_{1:t})}{\sum_{\hat{r}_t=0}^t p(\hat{r}_t, \data_{1:t})},
\end{equation}
for $r_t \in \{0, \ldots, t\}$.
The \texttt{RL-PR} method estimates $p(r_t, \data_{1:t})$ for all $r_t \in \{0, \ldots, t\}$ at every timestep.
To estimate this value recursively, we  sum over all possible previous runlengths as follows
\begin{equation}\label{eq:prrl-lj}
\begin{aligned}
    &p(r_t,  \data_{1:t})\\
    &= \sum_{r_{t-1}=0}^{t-1} p(r_t,  r_{t-1}, \data_{1:t-1}, \data_t)\\
    &= \sum_{r_{t-1}=0}^{t-1} p(r_{t-1} , \data_{1:t-1})
    \,p(r_t \cond r_{t-1}, \data_{1:t-1})\, p(\vy_t \cond r_t, r_{t-1},\vx_t, \data_{1:t-1})\\
    &= p(\vy_t \cond r_t , \vx_t, \data_{1:t-1}) \sum_{r_{t-1}=0}^{t-1}
    p(r_{t-1} , \data_{1:t-1})
    p(r_t\cond r_{t-1}).
\end{aligned}
\end{equation}
In the last equality, there are two implicit assumptions,
(i) the runlength at time $t$ is conditionally independent of the data $\data_{1:t-1}$ given the runlength at time $t-1$, and
(ii) the model is Markovian in the runlength, that is, conditioned on $r_t$, the value of $r_{t-1}$ bears no information.
Mathematically, this means that (i)  $p(r_t \cond r_{t-1}, \data_{1:t-1}) = p(r_t \cond r_{t-1})$
and (ii) $p(\vy_t \cond r_t, r_{t-1}, \data_{1:t-1}) = p(\vy_t \cond r_t, \data_{1:t-1})$. 
From \eqref{eq:prrl-lj}, we observe there are only two possible scenarios for the value of $r_t$. Either 
$r_t = 0$ or $r_t = r_{t-1} + 1$
with $r_{t-1} \in \{0, \ldots, t-1\}$.
Thus, 
$p(r_t, \data_{1:t})$ becomes
\begin{equation}\label{eq:bocd-joint}
    \begin{aligned}
        p(r_t, \data_{1:t}) &= p(\vy_t \cond r_t, \vx_t, \data_{1:t-1})\, p(r_{t-1}, \data_{1:t-1})\, p(r_t \cond r_{t-1})
        & \text{ if }r_t \geq 1\\
        p(r_t, \data_{1:t}) &= p(\vy_t \cond  r_t, \vx_t, \data_{1:t-1}) \sum_{r_{t-1}=0}^{t-1} p(r_{t-1}, \data_{1:t-1}) \, p(r_t \cond r_{t-1}) & \text{ if } r_t = 0\,.
    \end{aligned}
\end{equation}
The joint density \eqref{eq:bocd-joint} considers two possible scenarios:
either we stay in a regime considering the past $r_t \geq 1$ observations,
or we are in a new regime, in which $r_t = 0$.
Finally, note that \eqref{eq:bocd-joint} depends on three terms:
(i) the transition probability $p(r_t \cond r_{t-1})$, which it is assumed to be known,
(ii) the previous log-joint $p(r_{t-1}, \data_{1:t-1})$, with $r_{t-1} \in \{0, 1, \ldots, t-1\}$,
which is estimated at the previous timestep, and
(iii) the prior predictive density
\begin{equation}\label{eq:rl-predictive}
    p(\vy_t \cond r_t, \vx_t, \data_{1:t-1}) = \int
    p(\vy_t \cond \vtheta_{t}, \vx_t)
    \, p(\vtheta_{t} \cond r_t, \data_{1:t-1}) \d\vtheta_t.
\end{equation}
For a choice of \cModel given by \eqref{eq:linear-gaussian-measurement-model} and
a choice of \cPrior given by \eqref{eq:cprior-rl-pr},
the posterior predictive \eqref{eq:rl-predictive} takes the form.
\begin{equation}
\begin{aligned}
    p(\vy_t \cond r_t, \vx_t, \data_{1:t-1})
    &= \int {\cal N}\left(\vy_t \cond \vx_t^\intercal\,\vtheta_t, \vR_t\right)\,{\cal N}\left(\vtheta_t \cond \vmu_{t-1}^{(r_t)}, \vSigma_{t-1}^{(r_t)}\right) \d\vtheta_t\\
    &= {\cal N}\left(\vy_t \cond \vx_t^\intercal \vmu_{t-1}^{(r_{t})},\,\vx_t^\intercal\,\vSigma_{t-1}^{(r_{t})}\,\vx_t + \vR_t\right),
\end{aligned}
\end{equation}
with $r_t \in \{0, \ldots, t-1\}$.
Here,
$\left(\vmu_{t-1}^{(r_t)}, \vSigma_{t-1}^{(r_t)}\right)$ are the posterior mean and covariance
at time $t-1$ built using the last $r_t \geq 1$ observations.
If $r_t = 0$, then $(\vmu_{t-1}^{(r_t)}, \vSigma_{t-1}^{(r_{t})}) = \left(\vmu_0, \vSigma_0\right)$.

\subsubsection{Bounded number of hypotheses \texttt{RL[K]-PR}}
If we maintain a set of $K$ possible hypotheses, then $\vPsi_t = \{r_{t-1}^{(1)}, \ldots, r_{t-1}^{(K)}\} \in \{0, \ldots, t-1\}^K$
is a collection of $K$ unique runlengths obtained at time $t-1$.
Next, \eqref{eq:prrl-lj} takes the form
\begin{align}
        p(r_t, \data_{1:t}) &= p(\vy_t \cond r_t, \vx_t, \data_{1:t-1})\, p(r_{t-1}, \data_{1:t-1})\, p(r_t \cond r_{t-1})
        & \text{ if }r_t \geq 1, \label{eq:bocd-joint-fixed-K-up}\\
        p(r_t, \data_{1:t}) &= p(\vy_t \cond  r_t, \vx_t, \data_{1:t-1}) \sum_{r_{t-1} \in \vPsi_{t-1}} p(r_{t-1}, \data_{1:t-1}) \, p(r_t \cond r_{t-1}) & \text{ if } r_t = 0 \label{eq:bocd-joint-fixed-K-down}\,.
\end{align}
Here, we have that either $r_t = r_{t-1} + 1$ when $r_{t-1} \in \vPsi_{t-1}$ or $r_{t} = 0$. 
After computing $p(r_t, \data_{1:t})$ for all $K+1$ possibles values of $r_t$, a choice is made to keep $K$ hypotheses.
For timesteps $t \leq K$, we evaluate all possible hypotheses until $t > K$.

Algorithm \ref{algo:rl-pr-step} shows an update step under this process when we maintain a set of $K$ possible
hypotheses.

\subsection{Runlength with moment-matched prior reset (\texttt{RL-MMPR})}
\label{sec:rl-mmpr-implementation}
Here, we consider a modified version of the method introduced in \cite{fearnhead2011adaptivecp}.
We consider the choice of \texttt{RL} and
adjust the choice of \cPrior for \texttt{RL-PR} introduced in Appendix \ref{sec:rl-pr-implementation} whenever $r_t =0$.
In this combination, for $r_t = 0$, we take
$\tau(\vtheta_t \cond r_t, \data_{1:t-1}) = p(\vtheta_t \cond r_t, \data_{1:t-1})$.
Next
\begin{equation}\label{eq:mmpr-conditional-prior}
\begin{aligned}
    p(\vtheta_t \cond r_t, \data_{1:t-1})
    &= \sum_{r_{t-1}=0}^{t-1} p(\vtheta_t, r_{t-1} \cond r_t, \data_{1:t-1})\\
    &= \sum_{r_{t-1}=0}^{t-1} p(r_{t-1} \cond \data_{1:t-1})\,p(r_t \cond r_{t-1})\,p(\vtheta_t \cond r_t, r_{t-1}, \vy_{1:t-1})\\
    &= \sum_{r_{t-1}=0}^{t-1} p(r_{t-1} \cond \data_{1:t-1})\,p(r_t \cond r_{t-1})\,{\cal N}(\vtheta_t \cond \vmu_{t-1}^{(r_{t-1})}, \vSigma_{t-1}^{(r_{t-1})}).
\end{aligned}
\end{equation}
Because \eqref{eq:mmpr-conditional-prior} is a mixture model, we choose a conditional prior to be Gaussian that approximates the
first two moments.
We obtain
\begin{equation}\label{eq:rl-mmpr-first-moment}
    \mathbb{E}[\vtheta_t \cond r_t, \vy_{1:t-1}] =
    \sum_{r_{t-1}=0}^{t-1} p(r_{t-1} \cond \data_{1:t-1})\,p(r_t \cond r_{t-1})\,\vmu_{t-1}^{(r_{t-1})}
\end{equation}
for the first moment, and
\begin{equation}\label{eq:rl-mmpr-second-moment}
    \mathbb{E}[\vtheta_t\,\vtheta_t^\intercal \cond r_t, \vy_{1:t-1}]
    \sum_{r_{t-1}=0}^{t-1} p(r_{t-1} \cond \data_{1:t-1})\,p(r_t \cond r_{t-1})
    \left(\vSigma_{t-1}^{(r_{t-1})} + \vmu_{t-1}^{(r_{t-1})}\,\vmu_{t-1}^{(r_{t-1})\intercal}\right)
\end{equation}
for the second moment.
The conditional prior mean and prior covariance under $r_t = 0$ take the form
\begin{equation}\label{eq:rl-mmpr-prior-reset}
\begin{aligned}
    \vmu_t^{(0)} &= \mathbb{E}[\vtheta_t \cond r_t, \vy_{1:t-1}],\\
    \vSigma_t^{(0)} &= \mathbb{E}[\vtheta_t\,\vtheta_t^\intercal \cond r_t, \vy_{1:t-1}] - \left(\mathbb{E}[\vtheta_t \cond r_t, \vy_{1:t-1}]\right)\left(\mathbb{E}[\vtheta_t \cond r_t, \vy_{1:t-1}]\right)^\intercal.
\end{aligned}
\end{equation}
Algorithm \ref{algo:rl-mmpr-step} shows an update step under this process when we maintain a set of $K$ possible
hypotheses.

\subsection{Runlength with OU dynamics and prior reset (\texttt{RL[1]-OUPR*})}
\label{sec:rl-spr-implementation}
In this section, we provide pseudocode for the new hybrid method we propose.
Specifically, our choices in BONE are:
\RLSPR for \cAux and \cPrior,
\texttt{LG} for \cPosterior, and
\texttt{DA[1]} for \cWeight.
Because of our choice of \cWeight, \RLSPR considers a single hypothesis (or runlength) which,
at every timestep, is either increased by one or set back to zero, according to the probability of a changepoint
and a threshold $\epsilon \in (0,1)$.

In essence, \RLSPR follows the logic behind \RLPR introduced in Section \ref{sec:rl-pr-implementation}
with $K=1$ hypothesis and different choice of \cPrior.
To derive the algorithm for \RLSPR at time $t > 1$, suppose $r_{t-1}$ is available
(the only hypothesis we track).
Denote by $r_t^{(1)}$ the hypothesis of a runlength increase, i.e., $r_t = r_{t-1} + 1$ and
denote by $r_t^{(0)}$ the hypothesis of a runlenght reset, i.e., $r_t = 0$.
The probability of a runlength increase under a single hypothesis takes the 
form:
\begin{equation}
\begin{aligned}
    \nu_t(r_t^{(1)})
    &= p(r_{t}^{(1)} \cond \data_{1:t})\\
    &= \frac{p(r_t^{(1)}, \data_{1:t})}{p(r_t^{(1)}, \data_{1:t}) + p(r_t^{(0)}, \data_{1:t})}\\
    &=
    \frac
    {p(\vy_t \cond r_t^{(1)}, \vx_t, \data_{1:t-1})\,p(r_{t-1}, \data_{1:t-1})\,(1 - \kappa)}
    {p(\vy_t \cond r_t^{(0)}, \vx_t, \data_{1:t-1})\,p(r_{t-1}, \data_{1:t-1})\,\kappa
    + p(\vy_t \cond r_t^{(1)}, \vx_t, \data_{1:t-1})\,p(r_{t-1}, \data_{1:t-1})\,(1 - \kappa)}
    \\
    &= 
    \frac{p(\vy_t \cond r_t^{(1)}, \vx_t, \data_{1:t-1})\,(1 - \kappa)}{p(\vy_t \cond r_t^{(0)}, \vx_t, \data_{1:t-1})\,\kappa + p(\vy_t \cond r_t^{(1)}, \vx_t, \data_{1:t-1})\,(1-\kappa)}.
    \label{eq:rl-spr-eq-posterior}
\end{aligned}
\end{equation}
where $\kappa=p(r_t \cond r_{t-1})$ with $r_t = 0$ is the prior probability of a changepoint and
and $1-\kappa=p(r_t \cond r_{t-1}$ with $r_t=r_{t-1}+1$
is the probability of continuation of the current segment.

Next, we use $\nu_t(r_t)$ to decide whether to update our parameters 
or reset them according to a prior belief according to some threshold $\epsilon$.
This implements our choice of \cPrior given in \eqref{eq:SPR-equation-gt}  and \eqref{eq:SPR-equation-Gt}.
Because we maintain a single hypothesis, the weight at the end of the update step is set to $1$.
Algorithm \ref{algo:rl-spr-step} shows an update step for \RLSPR under the choice of \cModel given by \eqref{eq:linear-gaussian-measurement-model}.

\subsection{Changepoint location with multiplicative covariance inflation \texttt{CPL-MCI}}
\label{c-aux:CPL}
The work in \cite{li2021onlinelearning}
takes $\auxv_t = s_{1:t}$ to be a $t$-dimensional vector where the
$i$-th element is a binary vector that determines a changepoint at time $t$.
Then, the sum of the entries of $s_{1:t}$ represents the total number of changepoints up to, and including, time $t$.

We take $\nu_{t}(s_{1:t}) = p(s_{1:t} | \data_{1:t})$,
which is recursively expressed as
\begin{equation}\label{eq:bcm-recursive-changepoint}
\begin{aligned}
    p(s_{1:t} \cond \data_{1:t})
    &= p(s_t, s_{1:t-1} \cond \vy_{t}, \vx_t, \data_{1:t-1})\\
    &= p(s_{1:t-1} \cond \data_{1:t-1}) p(s_t \cond s_{1:t-1}, \vx_t, \vy_t, \data_{1:t-1}).
\end{aligned}
\end{equation}
Here, $p(s_{1:t-1} \cond \data_{1:t-1})$ is inferred at the previous timestep $t-1$.
The estimate of a changepoint conditioned on the past changes and the  measurements  is
\begin{equation}
\begin{aligned}
    &p(s_t = 1 \cond s_{1:t-1}, \data_{1:t})\\
    &\qquad = \frac{p(s_t = 1)p(\vy_t \cond \vx_t, s_{1:t-1}, s_t=1, \data_{1:t-1})}
    {p(s_t = 1)p(\vy_t \cond s_t=1, \vx_t, s_{1:t-1}, \vy_{1:t-1}) + p(s_t = 0)p(\vy_t \cond s_t=0, \vx_t, s_{1:t-1}, \data_{1:t-1})}\\
    &\qquad =
    \left(1+ {\exp\left(-\log\left(
    \frac{p(s_t = 1)p(\vy_t \cond s_t=1, \vx_t, s_{1:t-1}, \vy_{1:t-1})}{p(s_t = 0)p(\vy_t \cond s_t=0, \vx_t, s_{1:t-1}, \data_{1:t-1})}
    \right)\right)}\right)^{-1} = \sigma(m_t),
\end{aligned}
\end{equation}
where $\sigma(x)= 1/(1+\exp(-x))$ and 
\begin{equation}
    m_t
    = \log\left(\frac{p(\vy_t \cond s_t=1, \vx_t, s_{1:t-1}, \data_{1:t-1})}{p(\vy_t \cond s_t=0, \vx_t, s_{1:t-1}, \data_{1:t-1})}\right)
    + \log\left(\frac{p(s_t=1)}{p(s_t=0)}\right),
\end{equation}
and similarly,
\begin{equation}
    p(s_t = 0 \cond s_{1:t-1}, \data_{1:t}) = 1 - \sigma(m_t).
\end{equation}
Finally, the transition between states is given by $p(s_{1:t} \cond s_{1:t-1}) = p(s_t)$.

\section{Algorithms}

\begin{algorithm}[H]
    \small
    \begin{algorithmic}[1]
        \REQUIRE $(\vmu_0, \vSigma_0)$ // default prior beliefs
        \REQUIRE $\data_{t} = (\vx_t, \vy_t)$  // current observation
        \REQUIRE $\{r_{t-1}^{(k)}\}_{k=1}^K \in \{0, \ldots, t-1\}^K$ // bank of runlengths at time $t-1$
        \REQUIRE $\{p(r_{t-1}^{(k)}, \data_{1:t})\}_{k=1}^K$ // joint from past hypotheses
        \REQUIRE $\left\{(\vmu_{t-1}^{(k)}, \vSigma_{t-1}^{(k)})\right\}_{k=1}^K$ // beliefs from past hypotheses
        \REQUIRE $\vx_{t+1}$ // next-step observation
        \REQUIRE $p(\vy \cond \vtheta, \vx) = {\cal N}(\vy \cond \vtheta^\intercal\vx, \vR_t)$ // Choice of \cModel
        \STATE // Evaluate hypotheses if there is no changepoint
        \FOR{$k=1,\ldots,K$}
            \STATE $r_t^{(k)} \gets r_{t-1}^{(k)} + 1$
            \STATE $p(\vy_t \cond r_t^{(k)}, \vx_t, \data_{1:t-1}) \gets {\cal N}(\vy_t  \cond \vx_t^\intercal\,\vmu_{t-1}^{(k)},\,\vx_t^\intercal\,\vSigma_{t-1}^{(k)}\,\vx_t + \vR_t)$ // posterior predictive for $k$-th hypothesis
            \STATE $p(r_t^{(k)},\,\data_{1:t}) \gets p(\vy_t \cond r_t^{(k)}, \vx_t, \data_{1:t-1})\,p(r_{t-1}^{(k)}, \data_{1:t-1})\,p(r_t^{(k)} \cond r_{t-1}^{(k)})$ // update joint density
            \STATE $(\bar{\vmu}_t^{(k)}, \bar{\vSigma}_t^{(k)}) \gets (\vmu_{t-1}^{(k)}, \vSigma_{t-1}^{(k)})$
           \STATE $\tau_t(\vtheta_t; r_t^{(k)}) \gets {\cal N}(\vtheta_t \cond \bar{\vmu}_t, \bar{\vSigma}_t)$ // choice of \cPrior
           \STATE $q_t(\vtheta_t;\, r_t^{(k)}) \propto \tau_t(\vtheta_t; r_t^{(k)})\,p(\vy_t \cond \vtheta^\intercal\vx_t, \vR_t) \propto {\cal N}(\vtheta_t \cond \vmu_t^{(k)}, \vSigma_t^{(k)})$ // following \eqref{eq:ekf-update-step}
        \ENDFOR
        \STATE // Evaluate hypothesis under a changepoint
        \STATE $r_t^{(k+1)} \gets 0$
        \STATE $p(\vy_t \cond r_t^{(k+1)}, \vx_t, \data_{1:t-1}) \gets {\cal N}(\vy_t  \cond \vx_t^\intercal\,\vmu_0,\,\vx_t^\intercal\,\vSigma_0\,\vx_t + \vR_t)$ // posterior predictive for $k$-th hypothesis
        \STATE $p(r_t^{(k+1)}, \data_{1:t}) \gets p(\vy_t \cond r_t^{(k+1)},\,\vx_t,\,\data_{1:t-1})\sum_{k=1}^K p(r_t^{(k)}, \data_{1:t})\,p(r_t^{(t+1)} \cond r_t^{(k)} - 1)$
        \STATE // Extend number of hypotheses to $K+1$ and keep top $K$ hypotheses
        \STATE $I_{1:k} = {\rm top.k}(\{p(r_t^{(1)},\,\data_{1:t}), \ldots, p(r_{t}^{(k+1)}, \data_{1:t})\},\,K)$
        \STATE $\{p(r_t^{(k)}, \data_{1:t})\}_{k=1}^K \gets {\rm slice.at}(\{p(r_t^{(k)},\,\data_{1:t})\}_{k=1}^{K+1},\,I_{1:K})$
        \STATE $\{(\vmu_t^{(k)}, \vSigma_t^{(k)})\}_{k=1}^K \gets {\rm slice.at}(\{(\vmu_t^{(k)},\,\vSigma_t^{(k)})\}_{k=1}^{K+1},\,I_{1:K})$
        \STATE // build weight and make prequential prediction
        \STATE $\nu_t(r_t^{(k)}) \gets \frac{p(r_t^{(k)}, \data_{1:t})}{\sum_{j=1}^K p(r_t^{(j)}, \data_{1:t})}$ for $k=1,\ldots, K$
        \STATE $\hat{\vy}_{t+1} \gets \vx_{t+1}^\intercal\,\left(\sum_{k=1}^K \nu_t(r_t^{(k)})\vmu_t^{(k)}\right)$ // prequential prediction under a linear-Gaussian model
       \RETURN $\{(\vmu_{t}^{(k)}, \vSigma_{t}^{(k)}, r_t^{(k)})\}_{k=1}^K$, $\hat{\vy}_{t+1}$
        \end{algorithmic}
    \caption{
        Implementation of \RLPR[K].
        We consider an update at time $t$ and one-step ahead forecasting at time $t+1$
        under a Gaussian linear model with known observation variance.
    }
    \label{algo:rl-pr-step}
\end{algorithm}
In Algorithm \ref{algo:rl-pr-step},
the function ${\rm top.k}(A,K)$ returns the indices of the top $K \geq 1$ elements of $A$ with highest value.
The function  ${\rm slice.at}(A, B)$ returns the elements in $A$ according to the list of indices $B$.
If $|A| \leq |B|$, we return all elements in $A$.

\begin{algorithm}[H]
    \small
    \begin{algorithmic}[1]
        \REQUIRE $\data_{t} = (\vx_t, \vy_t)$  // current observation
        \REQUIRE $\{r_{t-1}^{(k)}\}_{k=1}^K \in \{0, \ldots, t-1\}^K$ // bank of runlengths at time $t-1$
        \REQUIRE $\{p(r_{t-1}^{(k)}, \data_{1:t})\}_{k=1}^K$ // joint from past hypotheses
        \REQUIRE $\left\{(\vmu_{t-1}^{(k)}, \vSigma_{t-1}^{(k)})\right\}_{k=1}^K$ // beliefs from past hypotheses
        \REQUIRE $\vx_{t+1}$ // next-step observation
        \REQUIRE $p(\vy \cond \vtheta, \vx) = {\cal N}(\vy \cond \vtheta^\intercal\vx, \vR_t)$ // Choice of \cModel
        \STATE // Evaluate hypotheses if there is no changepoint
        \FOR{$k=1,\ldots,K$}
            \STATE $r_t^{(k)} \gets r_{t-1}^{(k)} + 1$
            \STATE $p(\vy_t \cond r_t^{(k)}, \vx_t, \data_{1:t-1}) \gets {\cal N}(\vy_t  \cond \vx_t^\intercal\,\vmu_{t-1}^{(k)},\,\vx_t^\intercal\,\vSigma_{t-1}^{(k)}\,\vx_t + \vR_t)$ // posterior predictive for $k$-th hypothesis
            \STATE $p(r_t^{(k)},\,\data_{1:t}) \gets p(\vy_t \cond r_t^{(k)}, \vx_t, \data_{1:t-1})\,p(r_{t-1}^{(k)}, \data_{1:t-1})\,p(r_t^{(k)} \cond r_{t-1}^{(k)})$ // update joint density
            \STATE $(\bar{\vmu}_t^{(k)}, \bar{\vSigma}_t^{(k)}) \gets (\vmu_{t-1}^{(k)}, \vSigma_{t-1}^{(k)})$
           \STATE $\tau_t(\vtheta_t; r_t^{(k)}) \gets {\cal N}(\vtheta_t \cond \bar{\vmu}_t, \bar{\vSigma}_t)$ // choice of \cPrior
           \STATE $q_t(\vtheta_t;\, r_t^{(k)}) \propto \tau_t(\vtheta_t; r_t^{(k)})\,p(\vy_t \cond \vtheta^\intercal\vx_t, \vR_t) \propto {\cal N}(\vtheta_t \cond \vmu_t^{(k)}, \vSigma_t^{(k)})$ // following \eqref{eq:ekf-update-step}
        \ENDFOR
        \STATE // Evaluate hypothesis under a changepoint
        \STATE $r_t^{(k+1)} \gets 0$
        \STATE $\vmu_0 \gets \mathbb{E}[\vtheta_t \cond r_t, \vy_{1:t-1}]$ // following \eqref{eq:rl-mmpr-first-moment}
        \STATE $\vSigma_0 \gets \mathbb{E}[\vtheta_t\,\vtheta_t^\intercal \cond r_t, \vy_{1:t-1}] - \left(\mathbb{E}[\vtheta_t \cond r_t, \vy_{1:t-1}]\right)\left(\mathbb{E}[\vtheta_t \cond r_t, \vy_{1:t-1}]\right)^\intercal$ // following \eqref{eq:rl-mmpr-first-moment} and \eqref{eq:rl-mmpr-second-moment}
        \STATE $p(\vy_t \cond r_t^{(k+1)}, \vx_t, \data_{1:t-1}) \gets {\cal N}(\vy_t  \cond \vx_t^\intercal\,\vmu_0,\,\vx_t^\intercal\,\vSigma_0\,\vx_t + \vR_t)$ // posterior predictive for $k$-th hypothesis
        \STATE $p(r_t^{(k+1)}, \data_{1:t}) \gets p(\vy_t \cond r_t^{(k+1)},\,\vx_t,\,\data_{1:t-1})\sum_{k=1}^K p(r_t^{(k)}, \data_{1:t})\,p(r_t^{(t+1)} \cond r_t^{(k)} - 1)$
        \STATE // Extend number of hypotheses to $K+1$ and keep top $K$ hypotheses
        \STATE $I_{1:k} = {\rm top.k}(\{p(r_t^{(1)},\,\data_{1:t}), \ldots, p(r_{t}^{(k+1)}, \data_{1:t})\},\,K)$
        \STATE $\{p(r_t^{(k)}, \data_{1:t})\}_{k=1}^K \gets {\rm slice.at}(\{p(r_t^{(k)},\,\data_{1:t})\}_{k=1}^{K+1},\,I_{1:K})$
        \STATE $\{(\vmu_t^{(k)}, \vSigma_t^{(k)})\}_{k=1}^K \gets {\rm slice.at}(\{(\vmu_t^{(k)},\,\vSigma_t^{(k)})\}_{k=1}^{K+1},\,I_{1:K})$
        \STATE // build weight and make prequential prediction
        \STATE $\nu_t(r_t^{(k)}) \gets \frac{p(r_t^{(k)}, \data_{1:t})}{\sum_{j=1}^K p(r_t^{(j)}, \data_{1:t})}$ for $k=1,\ldots, K$
        \STATE $\hat{\vy}_{t+1} \gets \vx_{t+1}^\intercal\,\left(\sum_{k=1}^K \nu_t(r_t^{(k)})\vmu_t^{(k)}\right)$ // prequential prediction under a linear-Gaussian model
       \RETURN $\{(\vmu_{t}^{(k)}, \vSigma_{t}^{(k)}, r_t^{(k)})\}_{k=1}^K$, $\hat{\vy}_{t+1}$
        \end{algorithmic}
    \caption{
        Implementation of \texttt{RL[K]-MMPR}.
        We consider an update at time $t$ and one-step ahead forecasting at time $t+1$
        under a Gaussian linear model with known observation variance.
    }
    \label{algo:rl-mmpr-step}
\end{algorithm}

\begin{algorithm}[H]
\small
\begin{algorithmic}[1]
    \REQUIRE $\data_{t} = (\vx_t, \vy_t)$  // current observation
    \REQUIRE $\vx_{t+1}$ // next-step observation
    \REQUIRE $\epsilon \in (0,1)$ // restart threshold
    \REQUIRE $r_{t-1} \in \{0, \ldots, t-1\}$ // runlength at time $t-1$
    \REQUIRE $(\vmu_0, \vSigma_0)$ // default prior beliefs
    \REQUIRE $(\vmu_{t-1}, \vSigma_{t-1})$ // beliefs from prior step
    \REQUIRE $p(\vy \cond \vtheta, \vx) = {\cal N}(\vy \cond \vtheta^\intercal\vx, \vR_t)$ // Choice of \cModel
   \STATE $(r_t^{(0)}, r_t^{(1)}) \gets (0, r_{t-1} + 1)$ // choice of \cAux
   \STATE $p(\vy_t \cond r_t^{(0)}, \vx_t, \data_{1:t-1}) \gets {\cal N}(\vy_t  \cond \vx_t^\intercal\,\vmu_0,\,\vx_t^\intercal\,\vSigma_0\,\vx_t + \vR_t)$ // posterior predictive at changepoint
   \STATE $p(\vy_t \cond r_t^{(1)}, \vx_t, \data_{1:t-1}) \gets {\cal N}(\vy_t  \cond \vx_t^\intercal\,\vmu_{t-1},\,\vx_t^\intercal\,\vSigma_{t-1}\,\vx_t + \vR_t)$ // posterior predictive if no changepoint
   \STATE $\nu_t(r^{(1)}) \gets \frac{p(\vy_t \cond r_t^{(1)}, \vx_t, \data_{1:t-1})(1 - \pi)}
   {p(\vy_t \cond r_t^{(1)}, \vx_t, \data_{1:t-1})\,(1 - \pi) + p(\vy_t \cond r_t^{(0)}, \vx_t, \data_{1:t-1})\,\pi}$ // probability of no-changepoint at timestep $t$
   \STATE 
   \IF{$\nu(r_t^{(1)}) > \epsilon$}
       \STATE $r_t \gets r_t^{(1)}$
       \STATE $\bar\vmu_t^{(r_t)} \gets \vmu_{t-1}^{(r_{t-1})}\,\nu(r_t^{(1)}) + \vmu_0 \, \left(1 - \nu(r_t^{(1)})\right)$
       \STATE $\bar\vSigma_t^{(r_t)} \gets \vSigma_{t-1}^{(r_{t-1})}\,\nu(r_t^{(1)})^2 + \vSigma_0 \, \left(1 - \nu(r_t^{(1)})^2\right)$
    \ELSIF{$\nu(r_t^{(1)}) \leq \epsilon$}
        \STATE $r_t \gets r_t^{(0)}$
        \STATE $\bar\vmu_t^{(r_t)} \gets \vmu_0$
        \STATE $\bar\vSigma_t^{(r_t)} \gets \vSigma_0$
   \ENDIF
   \STATE $\tau_t(\vtheta_t; r_t) \gets {\cal N}(\vtheta_t \cond \bar{\vmu}_t, \bar{\vSigma}_t)$ // choice of \cPrior
   \STATE $q_t(\vtheta_t;\, r_t) \propto {\cal N}(\vtheta_t \cond \bar{\vmu}_t, \bar{\vSigma}_t)\,p(\vy_t \cond \vtheta^\intercal\vx_t, \vR_t) \propto {\cal N}(\vtheta_t \cond \vmu_t, \vSigma_t)$ // choice of \cPosterior --- via \eqref{eq:ekf-update-step}
    \STATE $\hat{\vy}_{t+1} \gets \vx_{t+1}^\intercal\,\vmu_t$ // prequential prediction (given linear-Gaussian model)
   \RETURN $(\vmu_{t}, \vSigma_{t}, r_t)$, $\hat{\vy}_{t+1}$
\end{algorithmic}
\caption{
    Implementation of \RLSPR, with update at time $t$ and for one-step ahead forecasting at time $t+1$,
    under a Gaussian linear model with known observation variance.
}
\label{algo:rl-spr-step}
\end{algorithm}

\section{Experiments}
\label{section:experiments}

In this section we experimentally evaluate 
different algorithms 
within the BONE framework
on a number of tasks.
 
Each experiment consists of a \textit{warmup} period where the hyperparameters are chosen,
and a \textit{deploy} period where sequential predictions and updates are performed.
In each experiment, we fix the choice of measurement model $h$ \cModel and posterior inference method \cPosterior,
and then compare different methods with respect to their choice of \cAux, \cPrior, and \cWeight. 
For \texttt{DA} methods, we append the number of hypotheses in brackets to determine
how many hypotheses are being considered.
For example,
\RLPR[1] denotes one hypothesis,
\RLPR[K] denotes $K$ hypotheses,
and \RLPR[inf] denotes all possible hypotheses.
In all experiments, unless otherwise stated, we consider a single hypothesis for choices of \texttt{DA}.
See Table \ref{tab:rosetta-methods} for the methods we compare.

\begin{table}[htb]
    \centering
    \footnotesize
    \begin{tabular}{c|c|c|p{8cm}|p{1.5cm}}
        \textbf{M.2-M.3} & \textbf{Eq.} & \textbf{A.2} & \textbf{Description} & \textbf{Sections}\\
         \hline
         \hline
        \multicolumn{4}{c}{static}\\
         \hline
         \namemethod{C-Static} & \eqref{eq:cprior-c-f} & - &   
         {\scriptsize
         This corresponds to the static case with a classical Bayesian update.
         This method does not assume changes in the environment.
         }
         & {\scriptsize
         \ref{experiment:KPM}, \ref{experiment:heavy-tail-regression}
         }
         \\
         \hline
        \multicolumn{4}{c}{abrupt changes}\\
         \hline
         \texttt{RL-PR} & \eqref{eq:cprior-rl-pr} & \texttt{DA[inf]} & 
         {\scriptsize
         This approach, commonly referred to as Bayesian online changepoint detection (\textbf{BOCD}),
         assumes that non-stationarity arises from independent  blocks of time, each with stationary data. Estimates are made using data from the current block.
         See Appendix \ref{sec:rl-pr-implementation} for more details.
         }
         & 
         {\scriptsize
         \ref{experiment:hour-ahead-forecasting}, \ref{exp:logistic-reg}, \ref{experiment:bandits}, \ref{experiment:KPM}, \ref{experiment:heavy-tail-regression}, 
         }
         \\
          \newmethod{WoLF+RL-PR*} & \eqref{eq:cprior-rl-pr} & \texttt{DA[inf]} &
         {\scriptsize
         Special case of \namemethod{RL-PR} with explicit choice of \cModel which makes it robust to outliers.
         }
         & {\scriptsize \ref{experiment:heavy-tail-regression} }
         \\
         \hline
        \multicolumn{4}{c}{gradual changes}\\
         \hline
         \CPPD & \eqref{eq:cprior-cpp-d} & \texttt{CA}  & 
         {\scriptsize
         Updates are done using a discounted mean and covariance according to the probability estimate that a change has occurred.
         }
         & {\scriptsize
         \ref{experiment:hour-ahead-forecasting}, \ref{exp:logistic-reg}, \ref{experiment:bandits}
         }
         \\
         \namemethod{C-ACI} & \eqref{eq:cprior-c-aci} &  - &  
         {\scriptsize
         At each timestep,
         this method assumes that the parameters
         evolve according to a linear map $\vF_t$,
         at a rate given by a known positive semidefinite covariance matrix $\vQ_t$.
         } 
         & 
         {\scriptsize
        \ref{experiment:hour-ahead-forecasting}, \ref{exp:logistic-reg}, \ref{experiment:bandits}, 
         }
         \\
        \hline
        \multicolumn{4}{c}{abrupt \& gradual changes}\\
         \hline
          \namemethod{RL-MMPR} & \eqref{eq:rl-mmpr-prior-reset} & \texttt{DA[inf]} & 
         {\scriptsize 
            Modification of \namemethod{CPT-MMPR} that assumes dependence between any two consecutive blocks of time
            and with choice of \texttt{RL}.
            This combination employs a moment-matching approach when evaluating the prior mean and covariance under a changepoint.
            See Appendix \ref{sec:rl-mmpr-implementation} for more details.
         }
         &
         {\scriptsize 
        \ref{experiment:KPM}
         }
         \\
         \texttt{RL-OUPR} & \eqref{eq:SPR-equation-gt}& \texttt{DA[1]} & 
         {\scriptsize
        Depending on the threshold parameter, updates involve either (i) a convex combination of the prior belief with the previous mean and covariance based on the estimated probability of a change (given the run length), or (ii) a hard reset of the mean and covariance, reverting them to prior beliefs.
        See Appendix \ref{sec:rl-spr-implementation} for more details.
        } &
        {\scriptsize 
        \ref{experiment:hour-ahead-forecasting}, \ref{exp:logistic-reg}, 
        \ref{experiment:bandits}, 
        \ref{experiment:KPM}
        }
    \end{tabular}
    \caption{
    List of methods we compare in our experiments.
    The first column, \textbf{M.2--M.3}, is defined by the choices of \cAux and \cPrior.
    The second column, \textbf{Eq.}, references the equation that define M.2--M.3.
    The third column, \textbf{A.2}, determines the  choice of \cWeight.
    The fourth column, \textbf{Description}, provides a brief summary of the method.
    The fifth column, \textbf{Sections}, shows the sections where the method is evaluated.
    The choice of \cModel and \cPosterior are defined on a per-experiment basis. 
    (The only exception being \newmethod{WolF+RL-PR}).
    For \cAux the acronyms are as follows: \texttt{RL} means runlength, \texttt{CPP} means changepoint probability,
    \texttt{C} means constant, and \texttt{CPT} means changepoint timestep.
    For \cPrior the acronyms are as follows:
    \texttt{PR} means prior reset,
    \texttt{OU} means Ornstein–Uhlenbeck,
    \texttt{LSSM} means linear state-space model,
    \texttt{Static} means full Bayesian update,
    \texttt{MMR} means moment-matched prior reset,
    and \texttt{OUPR} means Ornstein–Uhlenbeck and prior reset.
    We use the convention in \cite{huvskova1999gradualabruptchange} for the terminology abrupt/gradual changes.
    }
    \label{tab:rosetta-methods}
\end{table}

\subsection{Prequential prediction}
\label{experiment:prequential}

In this section, we give several examples of non-stationary prequential prediction problems.

\subsubsection{Online regression for hour-ahead electricity forecasting}
\label{experiment:hour-ahead-forecasting}

In this experiment,  we consider the task of predicting the hour-ahead electricity load
before and after the Covid pandemic.
We use the dataset presented in \cite{farrokhabadi2020electricitycovid}, which has 31,912 observations;
each observation
contains 7 features $\vx_t$ and a single target variable $\vy_t$.
The 7 features correspond to
pressure (kPa), cloud cover (\%), humidity (\%), temperature (C) , wind direction (deg), and wind speed (KmH).
The target variable is the hour-ahead electricity load (kW).
To preprocess the data,
we normalise the target variable $\vy_t$ by subtracting an exponentially weighted moving average (EWMA) mean with a half-life of 20 hours,
then dividing the resulting series by an EWMA standard deviation with the same half-life.
To normalise the features $\vx_t$, we divide each by a 20-hour half-life EWMA.
The features are lagged by one hour.

Our choice of measurement model $h$  is a two-hidden layer multilayered perceptron (MLP)
with four units per layer and a ReLU activation function.

For this experiment, we consider
\RLSPR  (our proposed method),
\RLPR (a classical method),
\CACI (a simple benchmark),
and \CPPD (a modern method).
For computational convenience, we plug in a point-estimate (MAP estimate)
of the neural network parameters when making predictions using $h$.
More precisely, given $\auxv_t$, we use $h(\vtheta_t^*, \vx_{t+1})$ to make a (conditional) prediction,
where $\vtheta_t^* = \argmax_\vtheta q(\vtheta;\, \auxv_t, \data_{1:t})$.
For  a fully Bayesian treatment of neural network predictions, see \cite{immer2021improving};
we leave the implementation of these approaches for future work.

The hyperparameters of each method are found using the first 300 observations (around 13 days)
and deployed on the remainder of the dataset.
Specifically, during the warmup period we tune the value of the probability of a changepoint for \RLSPR and \RLPR. For \CACI we tune  $\vQ_t$, and  for \CPPD we tune the learning rate. See the open-source notebooks for more details.

In the top panel of Figure \ref{fig:day-ahead-plot} 
we show the evolution of the target variable $\vy_t$ between March 3 2020 and March 10 2020.
The bottom panel of Figure \ref{fig:day-ahead-plot}
shows the 12-hour rolling mean absolute error (MAE) of predictions made by the methods.
We see that there is a changepoint around March 7 2020 as pointed out in \cite{farrokhabadi2020electricitycovid}. 
This is likely due to the introduction of Covid lockdown rules.
Among the methods considered, \CACI and \RLSPR adapt the quickest after the changepoint and maintain
a low rolling MAE compared to \RLPR and \CPPD.

\begin{figure}[htb]
    \centering
    \includegraphics[width=0.80\linewidth]{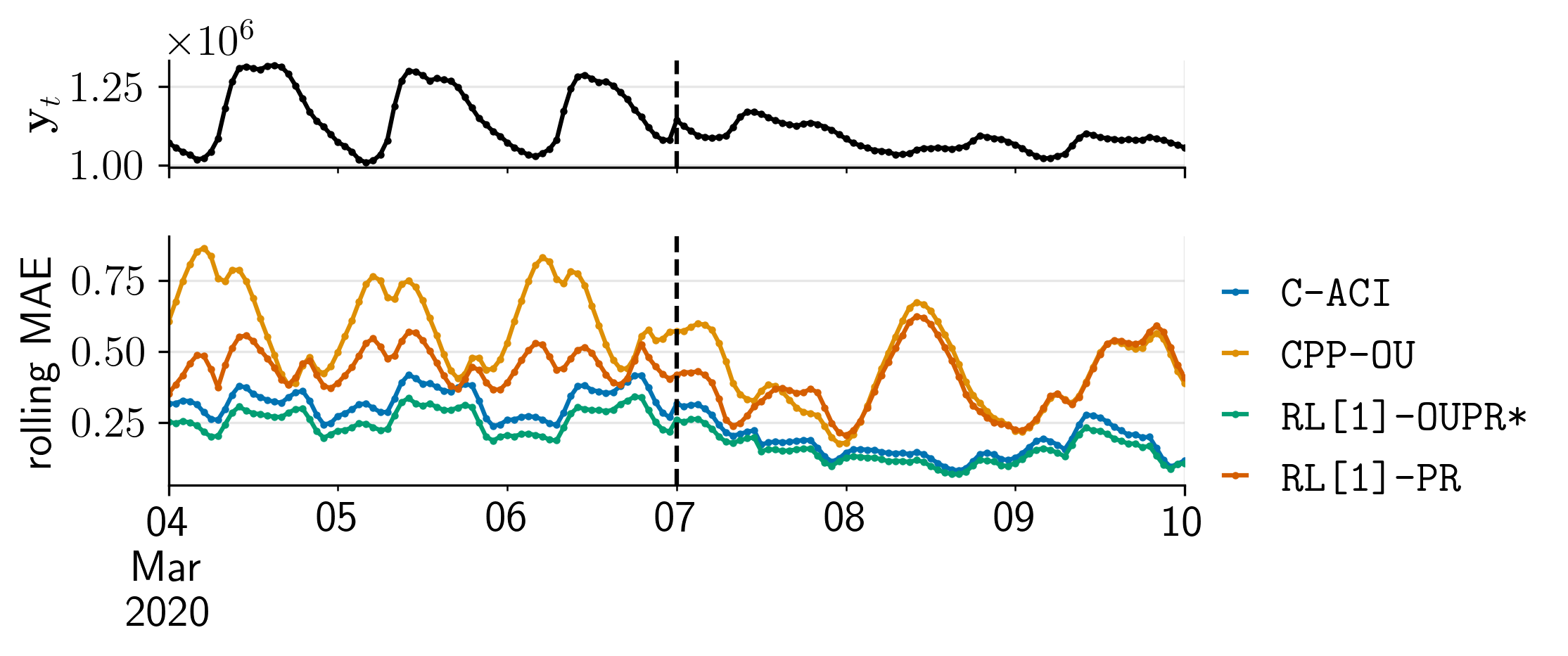}
    \caption{
    The \textbf{top panel} shows the target variable (electricity consumption) 
    from March 1 2020 to March 12 2020.
    The \textbf{bottom panel} shows the twelve-hour rolling relative absolute error of predictions
        for the same time window.
    The dotted black line corresponds to March 7 2020, when Covid lockdown began.
    }
    \label{fig:day-ahead-plot}
\end{figure}

Next, Figure \ref{fig:day-ahead-results-zoomed},
shows the forecasts made by each method between March 4 2020 and March March 8 2020.
We observe a clear cyclical pattern before March 7 2020 but less so afterwards,
indicating a change in daily electricity usage from diurnal to constant.

\begin{figure}[htb]
    \centering
    \includegraphics[width=0.80\linewidth]{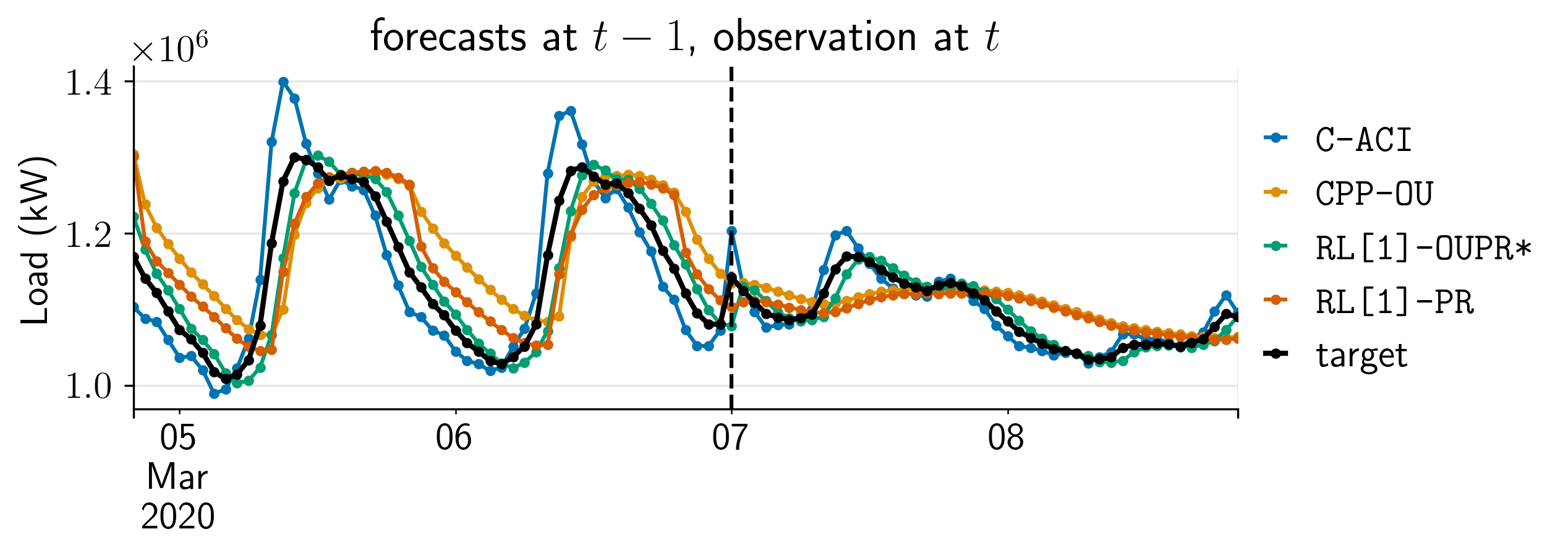}
    \caption{
    One day ahead electricity forecasting results
    for Figure \ref{fig:day-ahead-plot}.
    The dotted black line corresponds to  March 7 2020.
    }
    \label{fig:day-ahead-results-zoomed}
\end{figure}

We also observe that \RLPR and \CPPD slow-down their rate of adaptation.
One possible explanation of this behaviour is that the changes are
not abrupt enough to be captured by the algorithms.
To provide evidence for this hypothesis, Figure \ref{fig:day-ahead-rlpr-predictions}
shows, on the left $y$-axis, the predictions for \RLPR and the target variable $\vy_t$.
On the right $y$-axis, we show the estimated runlength.

\begin{figure}[htb]
    \centering
    \includegraphics[width=0.65\linewidth]{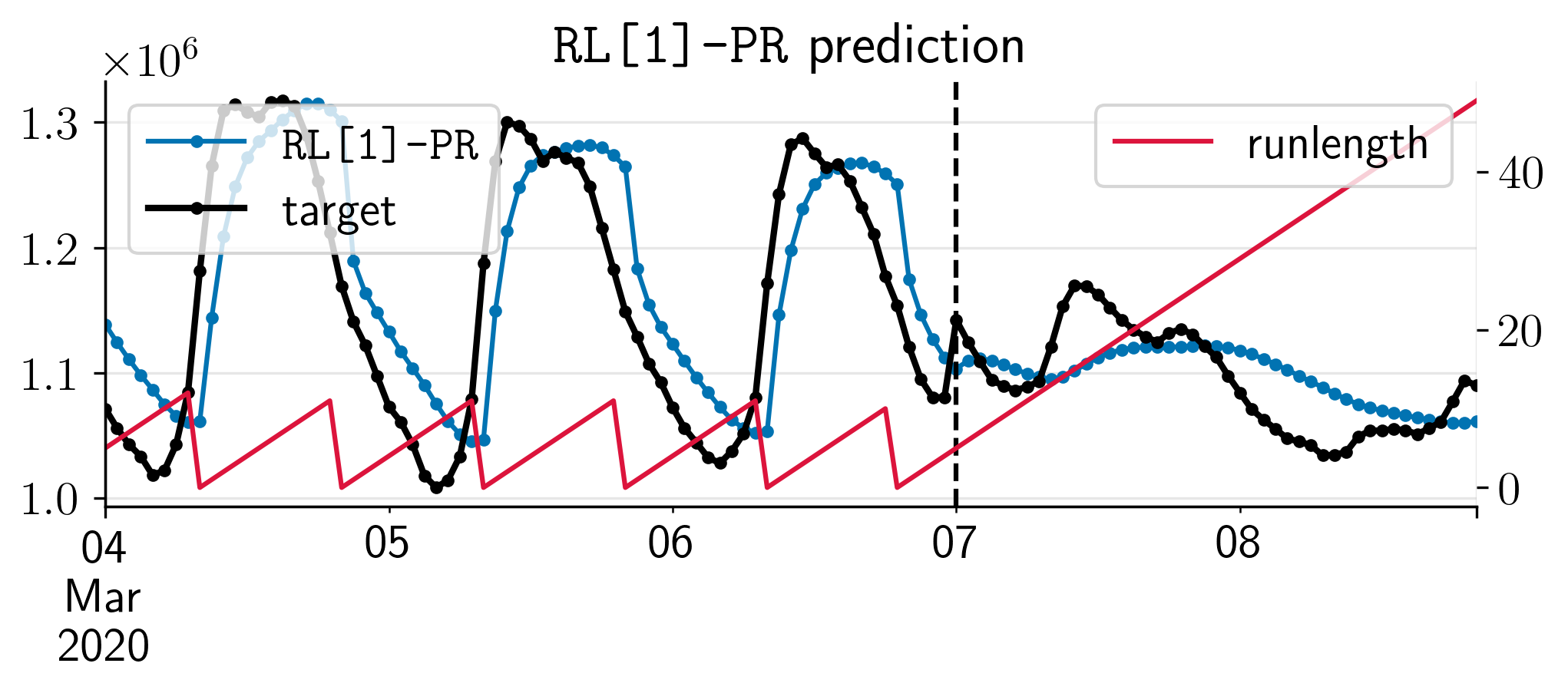}
    \caption{
    One day ahead electricity forecasting results for \RLPR together with the target variable
    on the left y-axis, and the value for runlength (\texttt{RL}) on the right y-axis.
    We see that after the 7 March changepoint, the runlength monotonically increases,
    indicating a stationary regime.
    }
    \label{fig:day-ahead-rlpr-predictions}
\end{figure}

We see that \RLPR resets approximately twice every day until the time of the changepoint.
After that, there is no evidence of a changepoint (as provided by the hyperparameters and the modelling choices),
so \RLPR does not reset which translates to less adaptation for the period to the right of the changepoint.

Finally, we compare the error of predictions made by the competing methods.
This is quantified in 
Figure \ref{fig:day-ahead-results},
which shows a box-plot of the five-day MAE for each of the competing methods over the whole dataset,
from March 2017 to November 2020.
Our new \RLSPR method has the lowest MAE.

\begin{figure}[htb]
    \centering
    \includegraphics[width=0.65\linewidth]{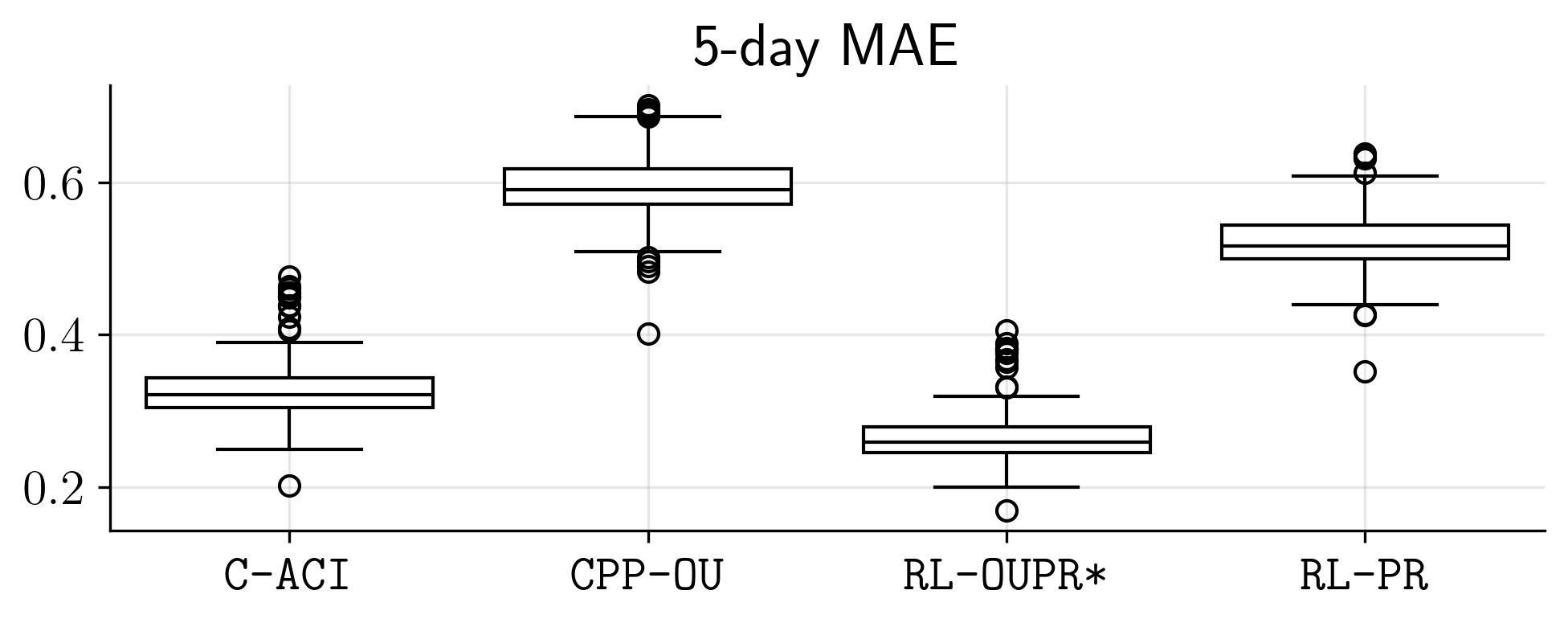}
    \caption{
    Distribution of the 5-day mean absolute error (MAE) for each of the competing methods on electricity forecasting over the entire period. For this calculation we split the  dataset into consecutive buckets containing  five days of data each, and for a given bucket we compute  the average absolute error of the predictions and observations that fall within the bucket.
    }
    \label{fig:day-ahead-results}
\end{figure}

\subsubsection{Online classification with periodic drift}
\label{exp:logistic-reg}
\label{sec:periodic-drifts}

In this section we study the performance of \CACI, \CPPD, 
\RLPR, and \RLSPR for the classification experiment of Section 6.2 in \cite{kurle2019continual}.
More precisely, in this experiment $x_{t,i}\sim \mathrm{Unif}[-3,3]$ for $i \in \{1,2\}$,
$\vx_t = (x_{t,1}, x_{t,2}) \in\mathbb{R}^2$, 
$y_t \sim \mathrm{Bernoulli}(\sigma(\vtheta_t^\intercal\,\vx_t) )$
with $\vtheta^{(1)}_t = 10\,\sin(5^\circ\,t)$ and $\vtheta^{(2)}_t = 10\,\cos(5^\circ\,t)$.
Thus the unknown values of model parameters are slowly drifting deterministically according to sine and cosine functions.
The timesteps go from 0 to 720.


\begin{figure}[H]
    \centering
    \includegraphics[width=0.65\linewidth]{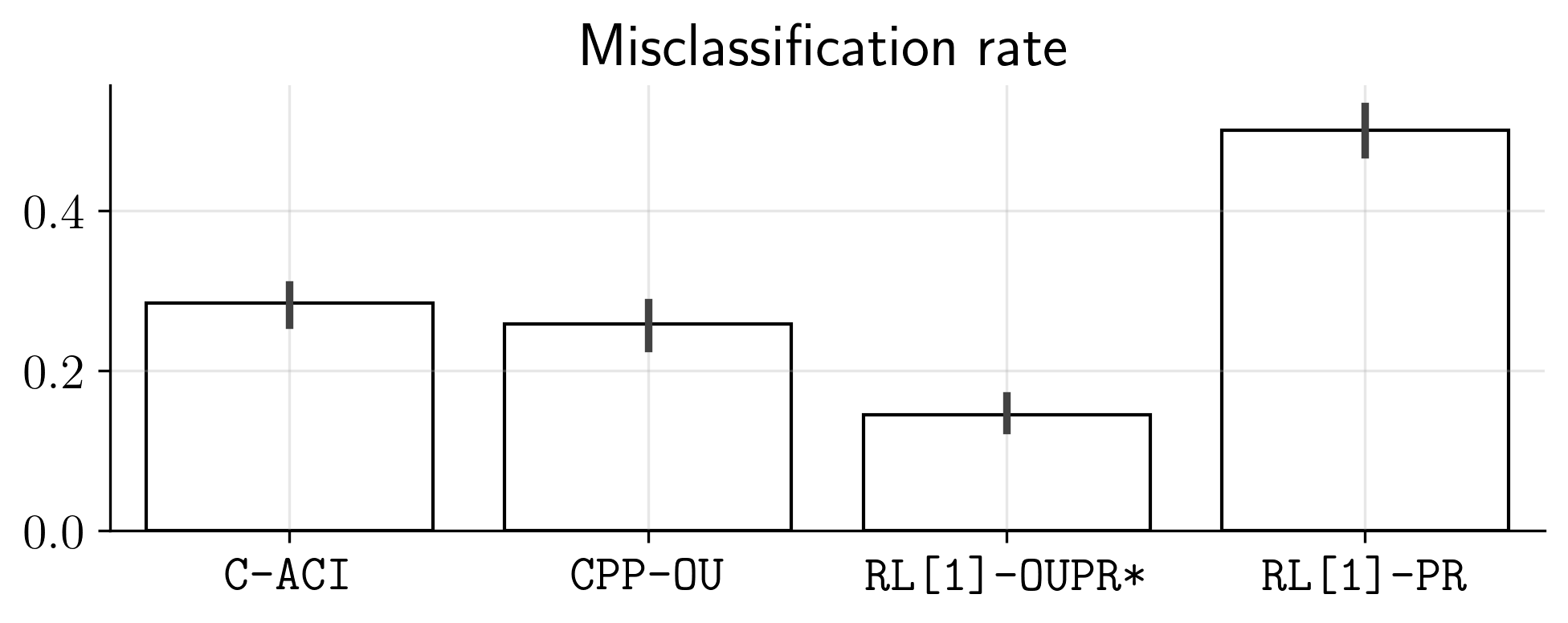}
    \caption{
    Misclassification rate of various methods on the online classification with periodic drift task.
    }
    \label{fig:clf-results}
\end{figure}

Figure \ref{fig:clf-results} summarises the results of the experiment where we show the misclassification rate (which is one minus the accuracy) for the competing methods.
%
Our \RLSPR method works the best, and signifcantly outperforms \RLPR,
since we use an OU drift process with a  soft prior reset
rather than assuming constant parameter
with a hard prior rset.

We can improve the performance of \RLPRK  if the number of hypotheses $K$ increases,
 and if we vary the changepoint probability threshold $\kappa$,
as shown in Figure \ref{fig:clf-rlpr-comparison}.
However, even then the performance of this method  still does not match our method.

\begin{figure}[H]
    \centering
    \includegraphics[width=0.60\linewidth]{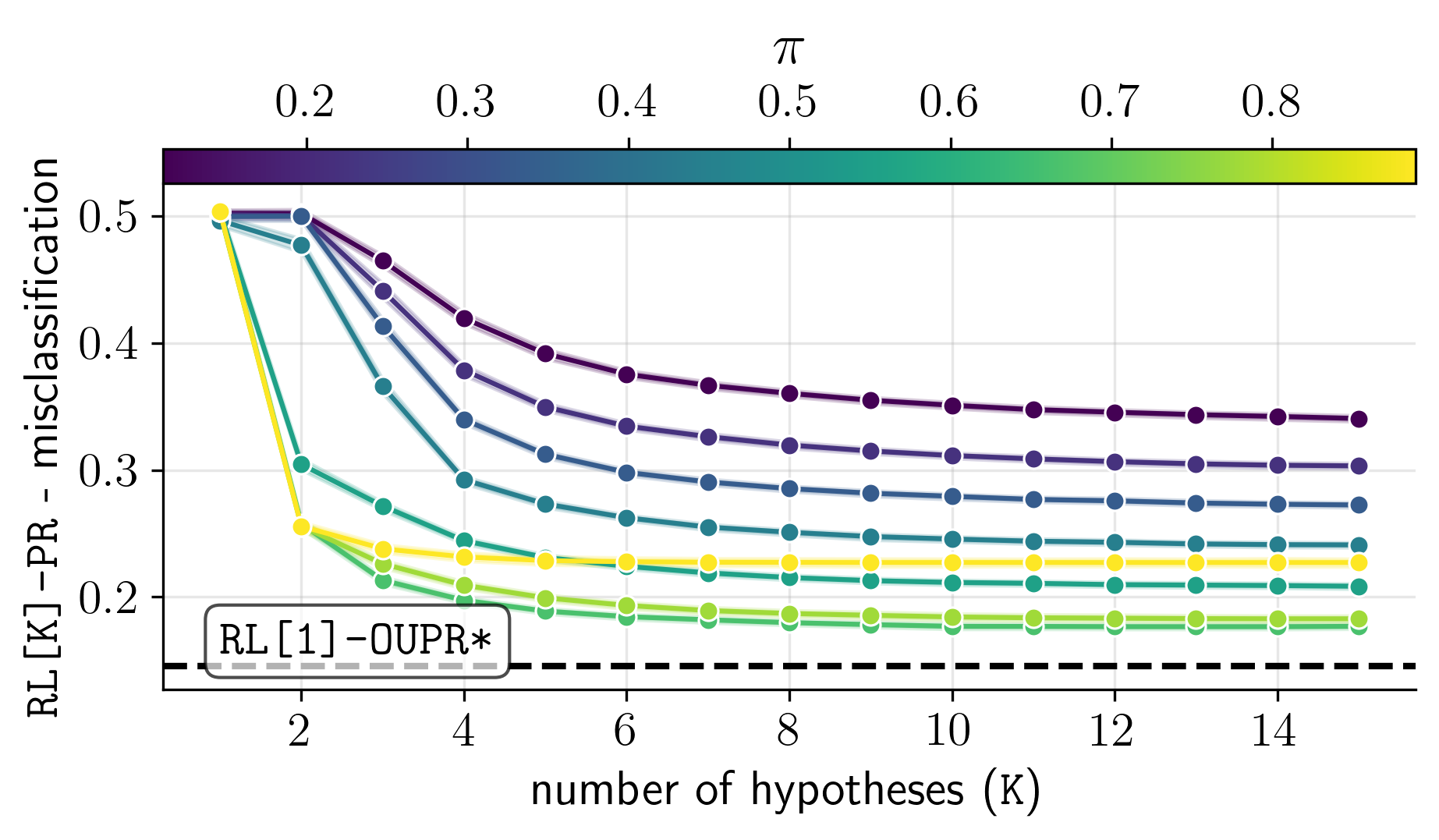}
    \caption{
    Accuracy of predictions for \RLPR as a function of the number of hypothesis and
    the prior probability of a changepoint $\kappa$.
    The black dotted line is the performance of \RLSPR reported in Figure \ref{fig:clf-results}.
    }
    \label{fig:clf-rlpr-comparison}
\end{figure}

\subsubsection{Online classification with drift and jumps}
\label{exp:classification-jumps}

In this section we study the performance of \CACI, \CPPD, 
\RLPR, and \RLSPR for an experiment with drift and sudden changes.
More precisely, we assume that the parameters of a logistic regression problem evolve according to
\begin{equation}
\vtheta_t =
\begin{cases}
\vtheta_{t-1} + \vepsilon_t & \text{w.p. } 1 - p_\epsilon,\\
{\cal U}[-2, 2]^2 & \text{w.p. } p_\epsilon,
\end{cases}
\end{equation}
with $p_\epsilon = 0.01$,
$\vtheta_0 \sim {\cal U}[-2, 2]^2$, and
$\vepsilon_t$ is  a zero-mean distributed random vector with isotropic covariance matrix $(0.01)^2\,\vI_2$
(where $\vI_2$ is a $2\times 2$ identity matrix).
Intuitively, this experiment has model parameters that drift slowly with occasional abrupt changes (at a rate of $0.01$).

\begin{figure}[H]
    \centering
    \includegraphics[width=0.6\linewidth]{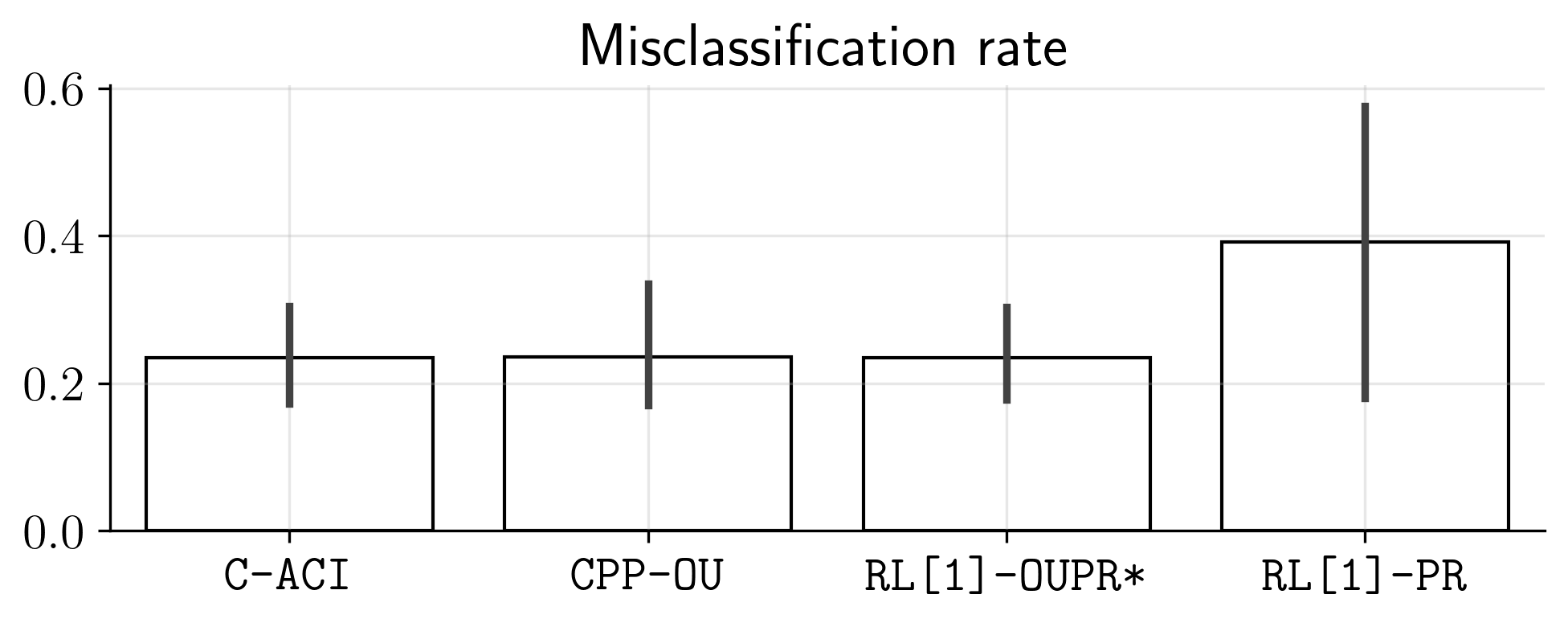}
    \caption{
        Misclassification rate of various methods on the online classification with drift and jumps task. 
    }
    \label{fig:clf-results-abrupt}
\end{figure}

Figure \ref{fig:clf-results-abrupt} shows the
misclassification
rate among the competing methods.
We observe that \CACI, \CPPD, and \RLSPR have comparable performance, whereas \RLPR  is the method with highest misclassification rate. 

\begin{figure}[H]
    \centering
    \includegraphics[width=0.60\linewidth]{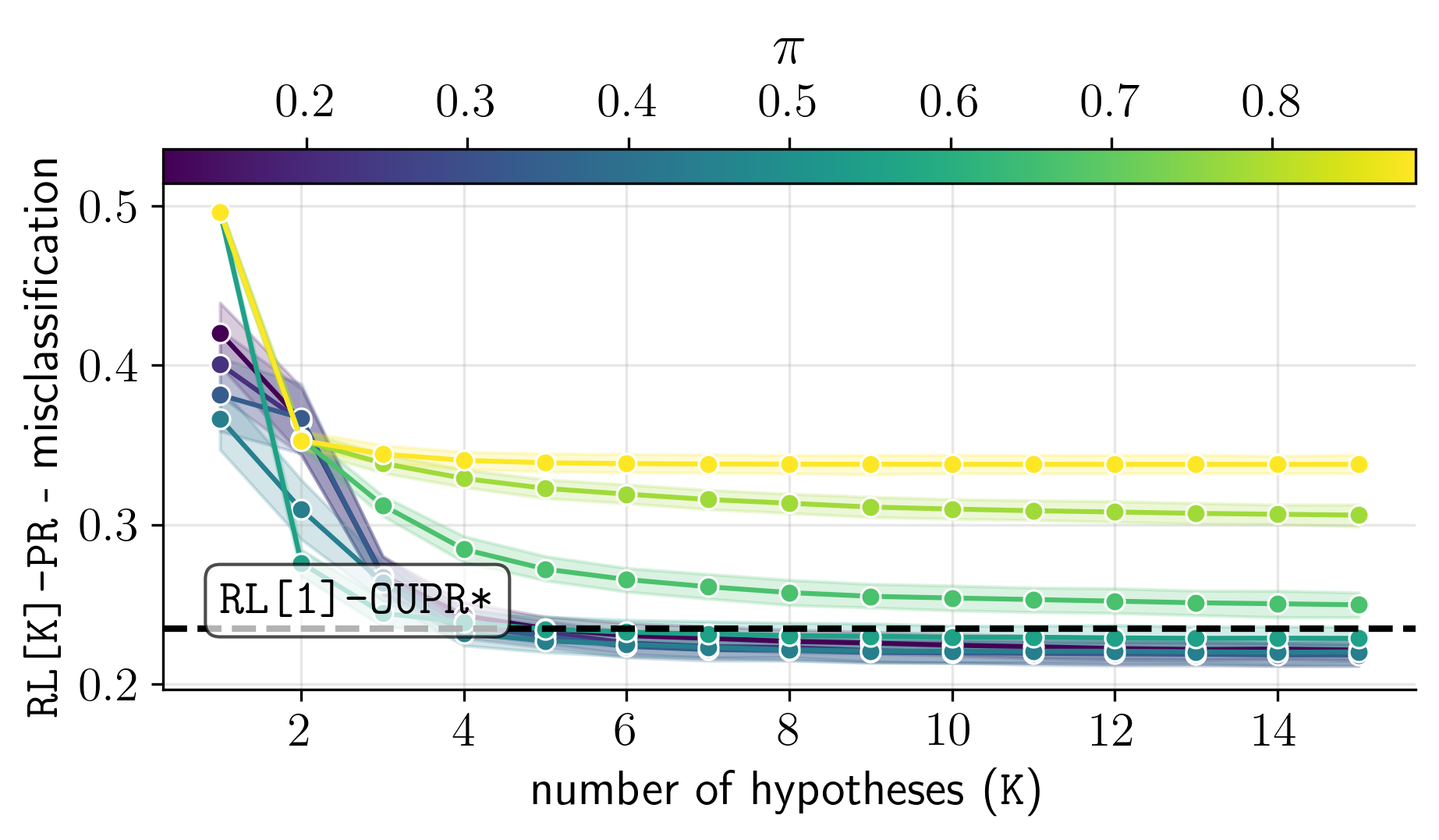}
    \caption{
    Accuracy of predictions for \RLPR[K] as a function of the number of hypotheses ($\texttt{K}$)
    and the probability of a changepoint $\kappa$.
    The black dotted line is the performance of \RLSPR reported in Figure \ref{fig:clf-results-abrupt}.
    }
    \label{fig:clf-rlpr-comparison-abrupt}
\end{figure}
To explain this behaviour, Figure \ref{fig:clf-rlpr-comparison-abrupt}
shows the performance of \RLPR[K] as a function of number of hypotheses and prior
probability of a changepoint $\kappa$.
We observe that up to three hypotheses, the lowest misclassification error of \RLPR[K] is higher than that of \RLSPR,
which only considers one hypothesis.
However, as we increase the number of hypotheses, the best performance for \RLPR[K]  obtains a lower misclassification rate than \RLSPR.
This is in contrast to the results in Figure 5.
Here, we see that with more hypotheses \RLPR[K] outperforms our new method at the expense of being  more memory intensive.

\subsection{Contextual bandits}
\label{experiment:bandits}

In this section, we study the performance of \CACI, \CPPD, \RLPR, and \RLSPR for the 
simple  Bernoulli bandit
from Section 7.3 of \cite{mellor2013changepointthompsonsampling}. 
More precisely, we consider a multi-armed bandit problem with 10 arms, 10,000 steps per simulation, and 100 simulations. The payoff of a given arm is the outcome of a Bernoulli random variable with unknown probability $\vtheta_t = \min\{\max\{\vtheta_{t-1} + 0.03\,Z_t,0\},1\}$ for $\{Z_t\}_{t\in\{1,2,\dots,10,000\}}$ independent and identically distributed standard normal random variables. We take $\vtheta_0\sim \mathrm{Unif}[0,1]$ and use the same formulation for all ten arms with independence across arms. The observations are the rewards and there are no features (non-contextual).

The idea of using \RLPR in multi-armed bandits problems was introduced in \cite{mellor2013changepointthompsonsampling}. With this experiment, we extend the concept to other members of the BONE framework. We use Thompson sampling for each of the competing methods. Figure \ref{fig:bandit-results} shows the regret of using \CACI, \CPPD, \RLPR, and \RLSPR for the above multi-armed bandits  problem.
The results we obtain are similar to those of Section \ref{sec:periodic-drifts}. This is because both problems have a similar drift structure. 

\begin{figure}[H]
    \centering
    \includegraphics[width=0.6\linewidth]{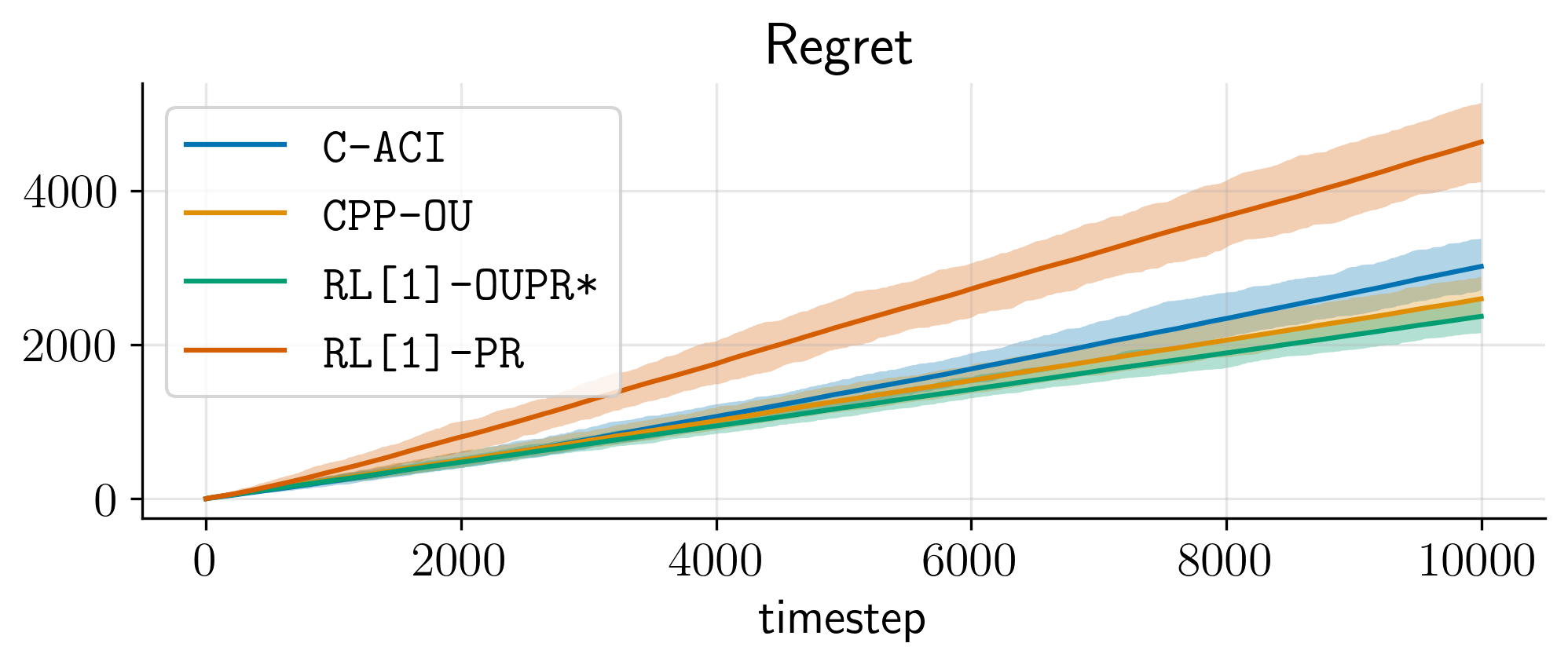}
    \caption{
    Regret of competing methods on the contextual bandits task. Confidence bands are computed with one hundred simulations.
    }
    \label{fig:bandit-results}
\end{figure}

\subsection{Segmentation and prediction}

In this section, we evaluate methods both in terms of their ability to ``correctly'' segment the observed output signal,
and to do one-step-ahead predictions.
Note that by ``correct segmentation'',
we mean one that matches the ground truth data generating process.
This metric can only be applied to synthetic data.

\subsubsection{Autoregression with dependence across the segments}
\label{experiment:KPM}
\label{experiment:segmentation}


In this experiment,
we consider the synthetic autoregressive dataset introduced in Section 2 of
\cite{fearnhead2011adaptivecp},
consisting of a set of one dimensional polynomial curves that are constrained to match up at segmentation boundaries,
as shown in the top left of Figure \ref{fig:sements-dependency-lr}.

We compare the performance of the three methods in the previous subsection.
For this experiment, we employ a probability of a changepoint $\kappa = 0.01$. 
Since this dataset has dependence of the parameters across segments,
 we allow  for the choice of \cModel
to be influenced by the choice of \cAux, i.e., our choice of model is given by $h(\vtheta_t; \auxv_t, \vx_t)$. 
For this experiment, we take \cAux to be \texttt{RL} and our choice of \cModel becomes
\begin{equation}\label{eq:h_fct_aux}
    h(\vtheta_t; r_t, \vx_{1:t}) = \vtheta_t^\intercal\,\vh(\vx_{1:t}, r_t),
\end{equation}
with $\vh(\vx_{1:t}, r_t) = [1, \Delta, \Delta^2]$, $\Delta = (x_t - x_{r_t})$, and $x_{r_t} \geq x_t$.
Intuitively this represents a quadratic curve fit to the beginning $x_{r_t}$ and end points $x_t$  of the current segment.
Given the form of \cModel in  \eqref{eq:h_fct_aux}, here we do not consider \CACI nor \CPPD.
Instead, we use 
runlength with moment-matching prior reset, i.e.,
\namemethod{RL-MMPR} (see Table \ref{tab:rosetta-methods}) which was designed for segmentation with dependence.


\begin{figure}[htb]
    \centering
    \includegraphics[width=0.45\linewidth]{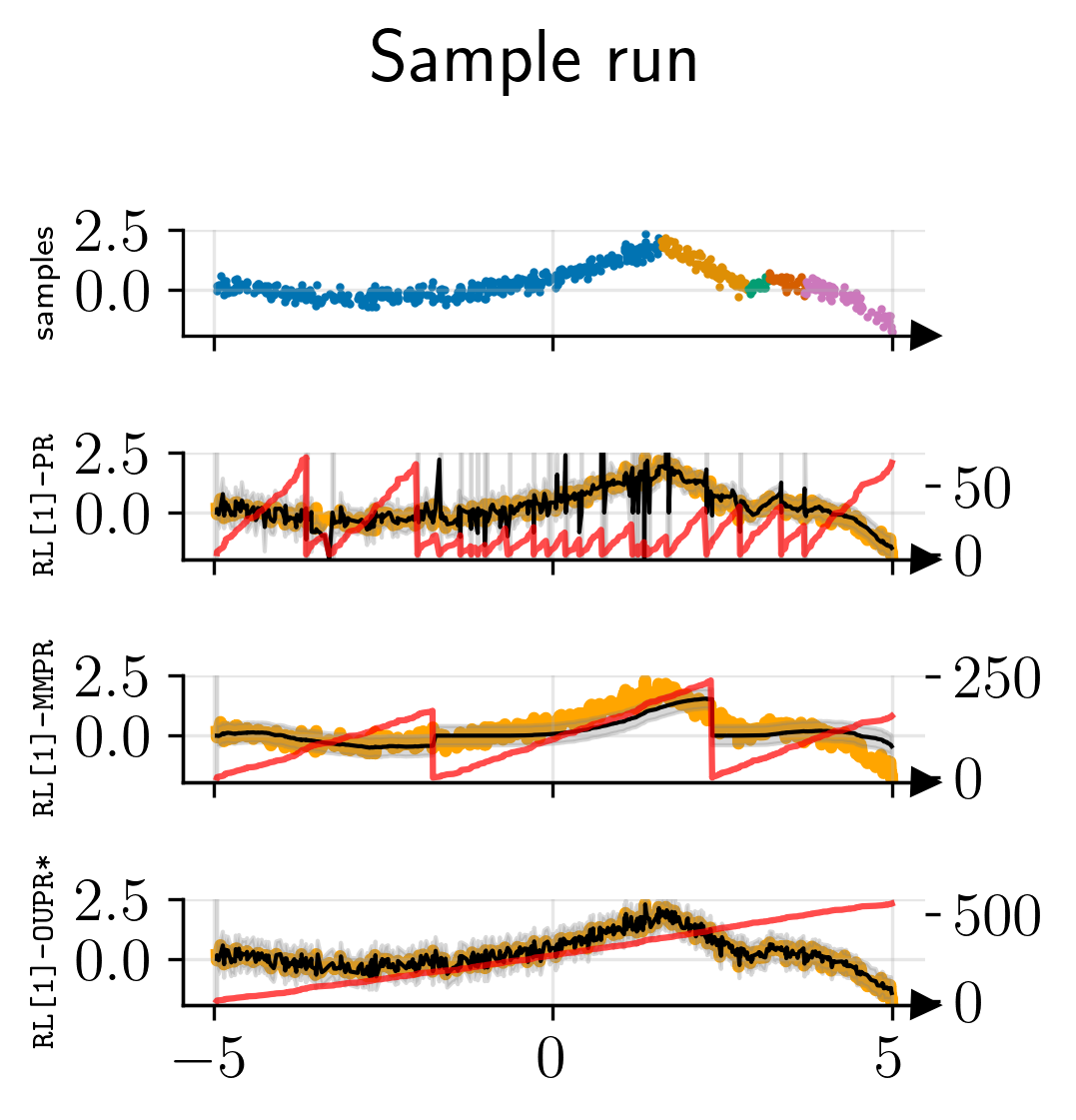}
    \includegraphics[width=0.45\linewidth]{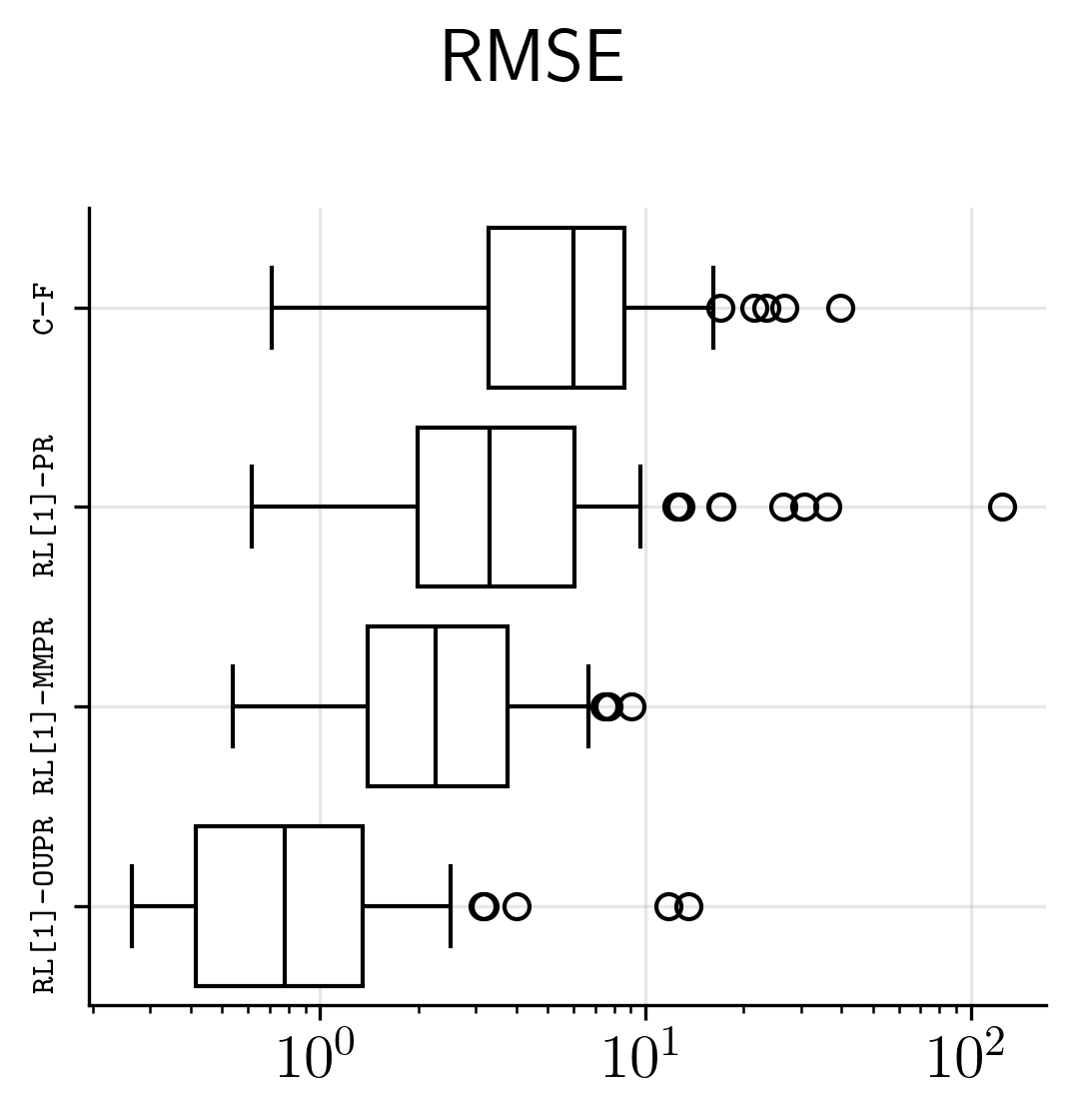}
    \caption{
        The \textbf{left panel} shows
        a sample run of the piecewise polynomial regression with dependence across segments.
        The $x$-axis is for the features,
        the (left) $y$-axis is for measurements together with the  estimations made by \RLPR,
        \namemethod{RL-MMPR}, and \RLSPR,
        the (right) $y$-axis is for the value of $r_t$ under each model.
        The orange line 
        denotes the true data-generating process and the red line denotes the
        value of the hypothesis \texttt{RL}.
        The \textbf{right panel} shows the 
        RMSE of predictions over 100 trials.
    }
    \label{fig:sements-dependency-lr}
\end{figure}

Figure \ref{fig:sements-dependency-lr} shows the results.
On the right, 
we observe that \RLSPR has the lowest RMSE.
On the left, we plot the predictions of each method, so we can visualise the nature of their errors.
For  \RLPR, the spikes occur
because the method has many false positive beliefs in a changepoint occurring,
and this causes breaks in the predictions
due  the explicit dependence of $h$ on $r_t$ and the hard parameter reset upon changepoints.
For  \namemethod{RL-MMPR},
the slow adaptation 
(especially when $x_t\in[1,5]$)
is because
 the method does not adjust beliefs as quickly as it should.
Our \RLSPR  method strikes a good compromise.

\begin{figure}[H]
    \centering
    \includegraphics[width=0.60\linewidth]{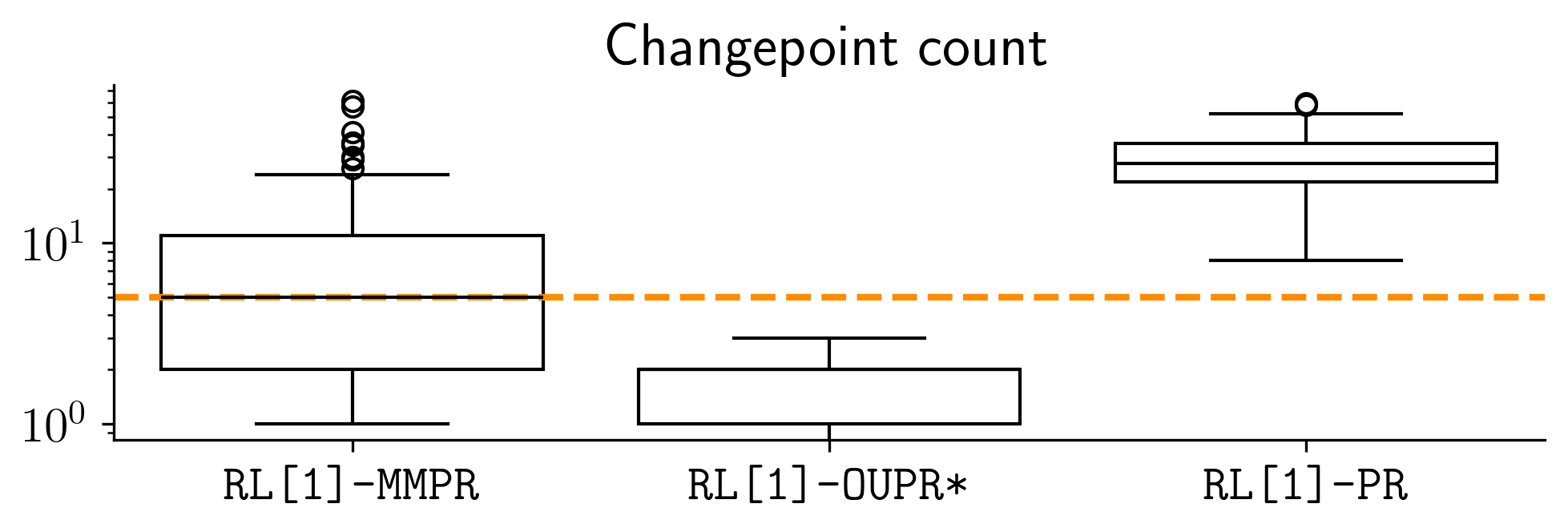}
    \caption{
    Count of changepoints over an experiment for 100 trials.
    The orange line shows the true number of changepoints for all trials.
    }
    \label{fig:segments-count}
\end{figure}

Figure \ref{fig:segments-count} shows the distribution (over 100 simulations) of the number of detected changepoints,
i.e., instances where $\nu_t(r_t)$ with $r_t = 0$ is the highest.
We observe that superior predictive performance 
in Figure \ref{fig:sements-dependency-lr}
does not necessarily translate to a better segmentation capability. For example,  the distribution produced by \namemethod{RL-MMPR} sits around the actual number of changepoints (better at segmenting) whereas \RLSPR, which is detecting far fewer changepoints, is the best performing prediction method.
This reflects the discrepancy between the objectives of segmentation and prediction.
For a more thorough analysis and evaluation of changepoint detection methods on time-series data,
see \cite{van2020evaluation}.

\section{Conclusion}
We introduced a unified Bayesian framework to  perform online predictions in non-stationary environments,
and showed how it covers many prior works.
We also used our framework
to design a new method,  \RLSPR, which is suited to tackle prediction problems
when the observations exhibit both abrupt and gradual changes.
We hope to explore other novel variants and applications in future work.

\chapter{Robustness}
\label{ch:robustness}

In this chapter,
we derive a novel, provably robust, and closed-form Bayesian update rule 
for online learning under a state-space model
in the presence of outliers and misspecified measurement models.
Our method combines generalised Bayesian (GB) inference
with filtering methods such as the extended and ensemble Kalman filter.
We use the former to show robustness and the latter to ensure 
computational efficiency in the case of nonlinear models.

Robustness to outliers is a critical requirement in sequential online learning,
where noisy or corrupted observations can easily degrade performance.
Recent (closed-form) methods addressing this challenge often rely on variational Bayes (VB), which, while effective,
tend to be more computationally expensive than standard Kalman filter (KF) approaches
due to the overhead of inner-loop computations.
(see e.g., \cite{wang2018,tao2023robustkf}).
Although more classical (low-cost) alternatives exist, they are typically restricted to linear systems and the results
may not extend to nonlinear or high-dimensional cases.

The method we propose is simple to implement, computationally efficient,
and robust to both outliers and model misspecification.
It offers performance comparable to, or better than, existing robust filtering approaches, including VB-based methods,
but at a much lower computational cost.
We demonstrate the effectiveness of our approach across a range of online learning tasks with outlier-contaminated observations.

\section{The method}
Our method is based on the GB  approach where one modifies the update equation in \eqref{eq:recursive-bayes} to use a loss function
$\ell_t: \real^\dimstate \to \real$ in place of the likelihood of the measurement process.
This gives the choice of \cPosterior of the form
\begin{equation}\label{eq:generalised-posterior}
    q(\vtheta_t \cond \data_{1:t}) \propto \exp(-\ell_t(\vtheta_t))\,q(\vtheta_t \cond \data_{1:t-1}),
\end{equation}
where $\auxv_t = \{\}$.
We propose to gain robustness to outliers 
in observation space by taking the loss function to be the model's  negative log-likelihood scaled by a data-dependent
weighting term
\begin{equation}
\label{eq:weighted-loglikelihood}
\ell_t(\vtheta_t)
= -W^2(\vy_t, \hat{\vy}_t)\,\log p(\vy_t| \vtheta_t),
\end{equation}
with 
$W: \real^{\dimobs}\times\real^\dimobs\to\real_{++}$ the weighting function---see Section \ref{sec:choice-weighting-function}
for examples--- and 
$p(\vy_t \vert \vtheta_t)$ the choice of likelihood.
We call our method the \emph{weighted observation likelihood filter (WoLF)}.
To specify an instance of our method, ones needs to define
the  likelihood $p(\vy_t \cond \vtheta_t)$
and the weighting function $W$.
In the next subsections, we show the flexibility of  WoLF
and derive weighted-likelihood-based KF and EKF algorithms.
Setting $W(\vy_t, \bar{\vy}_t)=1$ trivially recovers existing methods, but we will instead use non-constant weighting functions inspired by the work of \cite{Barp2019,matsubara2021robust,Altamirano2023_bocpd,Altamirano2023_gp}.

\section{Linear weighted observation likelihood filter}
In this section, we present the WoLF method under a linear SSM \eqref{eq:ssm-linear}.
In particular, the following proposition provides a closed-form solution
for the update step of WoLF under a linear measurement function and a Gaussian likelihood.
\begin{proposition}\label{prop:weighted-kf}
    Consider the linear-Gaussian SSM \eqref{eq:ssm-linear}
    with weighting function $W:\real^\dimobs\times\real^\dimobs \to \real$.
    Then, the update step of WoLF with loss function \eqref{eq:weighted-loglikelihood} is given
    by Proposition \ref{prop:kf-update-step-precision}
    with $\vR_t^{-1}$ replaced by $\bar{\vR}_t^{-1} = W^2(\vy_t, \hat{\vy}_t)\,\vR_t^{-1}$.
\end{proposition}

\label{proof:weighted-kf}
\begin{proof}
    Let $w_t^2 := W^2(\vy_t, \hat{\vy}_t)$.
    The loss function takes the form
    \begin{equation}\label{eq:weighted-gaussian-loglikelihood}
    \begin{aligned}
        \ell_t(\vtheta_t)
        &= -w_t^2 \log \normdist{\vy_t}{\vH_t \vtheta_t}{\vR_t}\\
        &= \frac{1}{2}\left(\vy_t - \vH_t\vtheta_t\right)^\intercal(\vR_t / w_t^2)^{-1}\left(\vy_t - \vH_t\vtheta_t\right)
        -\frac{w_t^2\,\dimobs}{2}\log\pi - \frac{w_t^2}{2}\log|\vR_t|\\
        &= \frac{1}{2}\left(\vy_t - \vH_t\vtheta_t\right)^\intercal\bar\vR_t^{-1}\left(\vy_t - \vH_t\vtheta_t\right)
        + C,
    \end{aligned}
    \end{equation}
    with $\bar\vR_t = \vR_t / w_t^2$, and
    where $C = -\frac{w_t^2\,\dimobs}{2}\log\pi - \frac{w_t^2}{2}\log|\vR_t|$ is a term that does not depend on $\vtheta_t$.
    The remaining follows from the standard KF derivation. 
    Note that the loss function does not correspond to the log-likelihood for a homoskedastic Gaussian model since $\bar{\vR}_t$
    may depend on all data, including $\vy_t$.
\end{proof}
Figure \ref{fig:weighted-log-likelihood-gaussian} shows the weighted log-likelihood \eqref{eq:weighted-gaussian-loglikelihood}
for a univariate ${\cal N}(0, 1)$ Gaussian density
as a function of the weighting term $w_t^2\in(0, 1]$. 
\begin{figure}[htb]
    \centering
    \includegraphics[width=0.5\linewidth]{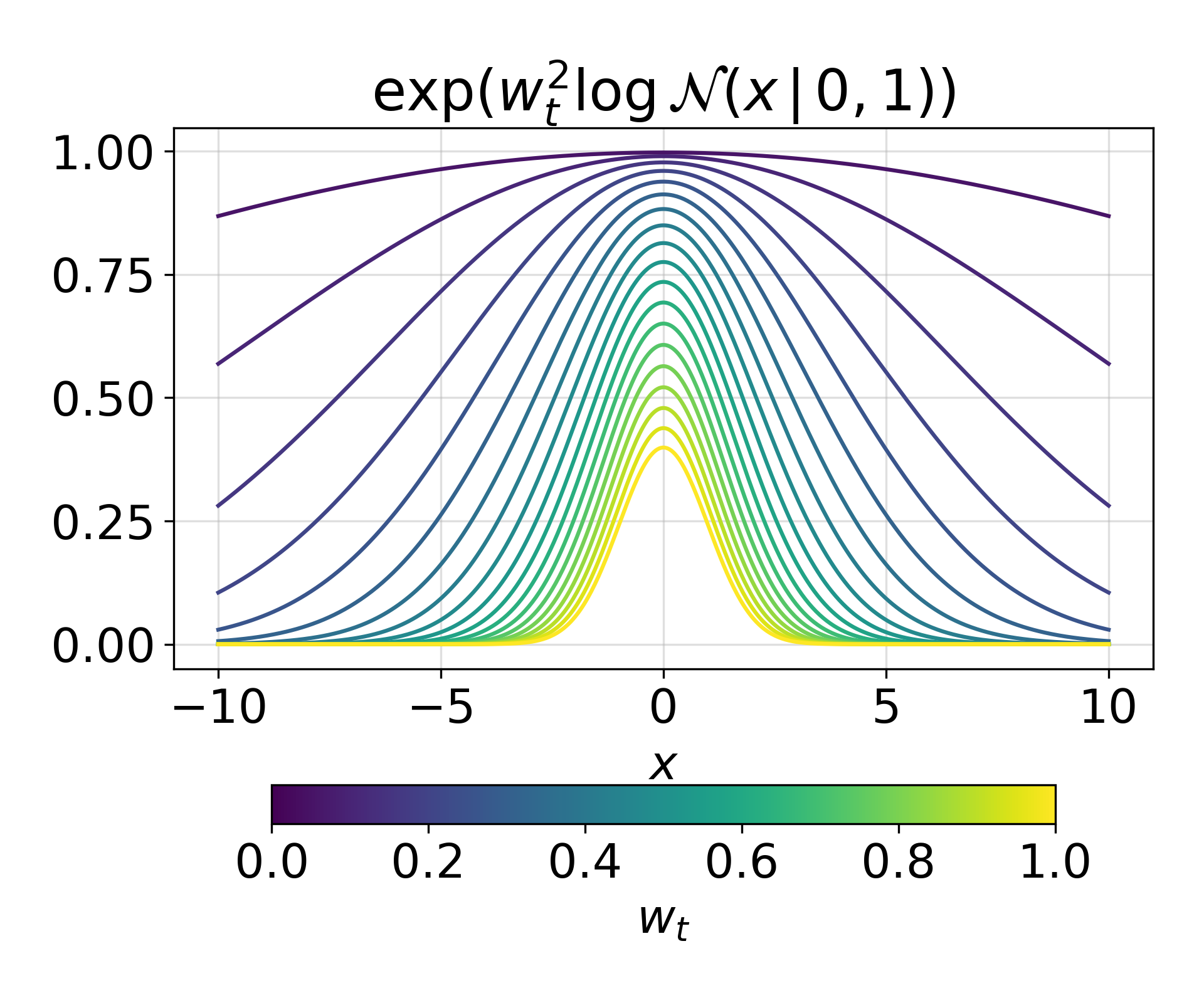}
    \caption{Weighted likelihood (unnormalised) for a standard Gaussian.}
    \label{fig:weighted-log-likelihood-gaussian}
\end{figure}
We observe that a weighted log-likelihood resembles a heavy-tailed likelihood for $w_t < 1$.

The resulting predict and update steps for WoLF under linear dynamics and
zero-mean Gaussians for the state and measurement process
are shown in Algorithm \ref{algo:wlf-step}.
\begin{algorithm}[ht]
\begin{algorithmic}
    \REQUIRE $\vF_t$, $\vQ_t$ // predict step
    \STATE $\vmu_{t|t-1} \gets \vF_t\,\vmu_{t-1}$
    \STATE $\vSigma_{t|t-1} \gets \vF_t\,\vSigma_{t-1}\, \vF_t^\intercal + \vQ_t$
    \REQUIRE $\vy_t$, $\vH_t$, $\vR_t$  // update step
    \STATE $\hat{\vy}_t \gets \vH_t \,\vmu_{t|t-1}$
    \STATE $w_t \gets W(\vy_t, \hat{\vy}_t)$
    \STATE $\vSigma_{t}^{-1}  \gets \vSigma_{t\vert t-1}^{-1} +
    w_t^2\,\vH_{t}^{\intercal}\, \vR_{t}^{-1}\, \vH_{t}$
    \STATE $\vK_t \gets w_t^2\,\vSigma_t\,\vH_t^\intercal \,\vR_t^{-1}$
    \STATE $\vmu_{t}  \gets \vmu_{t\vert t-1}+ \vK_t \left(\vy_{t}-\hat{\vy}_{t}\right)$
\end{algorithmic}
\caption{
    WoLF predict and update step
}
\label{algo:wlf-step}
\end{algorithm}

The computational complexity of WoLF under linear 
dynamics matches that of the KF, i.e., 
$O(\dimstate^3)$. Alternative robust filtering algorithms require
multiple iterations per measurement to achieve robustness and stability, making them significantly slower;
see Table \ref{tab:complexity-linear-model} for the computational complexity for the methods we consider, and
Figure \ref{fig:uci-per-step-time} for empirical comparisons.

\section{Nonlinear weighted observation likelihood filter}
\label{subsec:wlf-nonlinear-extensions}
The WoLF method readily extends to other learning algorithms.
For example, a WoLF version of the EKF, which is obtained by introducing a weighting function to $\eqref{eq:ssm-linearised}$
yielding the approximate choice of \cModel:
\begin{equation}
    \log p(\vy_t \vert \vtheta_t) = W^2(\vy_t, \bar{\vy}_t)\,\log\normdist{\vy_t}{\bar{\vy}_t}{\vR_t}.
\end{equation}
We can also derive a novel outlier-robust exponentially-weighted moving average algorithm
(see Section \ref{experiment:outlier-robust-ewma}).

\section{The choice of weighting function}
\label{subsec:choice-weighting-function}

Weighted likelihoods have a well-established history in Bayesian inference
and have demonstrated their efficacy in improving robustness
\citep{grunwald2012safe,holmes2017assigning,grunwald2017inconsistency,miller2018robust,bhattacharya2019bayesian,alquier2020concentration,dewaskar2023robustifying}. 
In this context, the corresponding posteriors are often referred to as fractional, tempered, or power posteriors. 
In most existing work, the determination of weights relies on heuristics 
and the assigned weights remain constant across all data points
so that $W(\vy_t, \hat{\vy}_t) = w \in \real$ for all $t$. 
In contrast, we dynamically incorporate information from the most recent observations without incurring additional computational costs
by defining the weight as a function of the current observation
$\vy_t$
and its prediction
$\hat{\vy}_t =  h_t(\vmu_{t|t-1})$,
which is based on all of the past observations.

To define the weighting function, 
we take inspiration from previous work 
for dealing with outliers.
In particular,  \citet{wang2018}
proposed classifying robust filtering algorithms
into two main types:
\textit{compensation-based}
algorithms,
which incorporate information from tail events into the model
in a robust way
\citep[see, e.g.,][]{Huang2016, Agamennoni2012},
and
\textit{detect-and-reject} algorithms,
which 
assume that outlier observations bear no useful information 
and thus are ignored
\citep[see, e.g.,][]{wang2018, mu2015}.
Below we show how both of these strategies can be implemented
using our WoLF method by merely changing the weighting function.

\paragraph{Inverse multi-quadratic weighting function:}
As an example of a compensation-based method,
we follow \citet{Altamirano2023_gp}
and use the Inverse Multi-Quadratic (IMQ) weighting,
which in our SSM setting is
\begin{align} \label{eq:w_t}
    W(\vy_t, \hat{\vy}_t) & =\left(1+\frac{||\vy_{t}-\hat{\vy}_{t}||_2^2}{c^{2}}\right)^{-1/2},
\end{align}
where $c > 0$ is the soft threshold and $\|\cdot\|_2$ denotes the $l_2$ norm.
We call WoLF with IMQ weighting ``WoLF-IMQ''.

\paragraph{Mahalanobis-based weighting function:}
The $l_2$ norm in the IMQ can be modified to account for the covariance structure of the measurement
process by replacing it with the Mahalanobis distance between $\vy_t$ and $\hat{\vy}_t$:
\begin{align}\label{eq:mahalanobis-weight}
    W(\vy_t, \hat{\vy}_t) & =\left(1+\frac{\|\vR_t^{-1/2}(\vy_t - \hat{\vy}_t)\|_2^2}{c^{2}}\right)^{-1/2}.
\end{align}
We call  WoLF with this weighting function the WoLF-MD method.
This type of weighted IMQ function has been used extensively in the kernel literature
\citep[see e.g.][]{chen2019stein, detommaso2018stein, riabiz2022optimal}.

\paragraph{Threshold Mahalanobis-based weighting function:}
As an example of a detect-and-reject method,
 we consider
\begin{equation}\label{eq:thresholded-mahalanobis-weight}
    W(\vy_t, \hat{\vy}_t) =
    \begin{cases}
    1 & \text{if } \|\vR_t^{-1/2}(\vy_t - \hat{\vy}_t)\|_2^2 \leq c,\\
    0 & \text{otherwise}.
    \end{cases}
\end{equation}
with $c > 0$ the fixed threshold.
The weighting function \eqref{eq:thresholded-mahalanobis-weight} corresponds
to ignoring information from estimated measurements whose Mahalanobis distance
to the true measurement is larger than
some predefined threshold $c$.
In the  linear setting, this weighting function is related to the benchmark method employed in \citet{ting2007}.
We refer to WoLF with this weighting function as ``WoLF-TMD''.


The proposed weighting functions --- the IMQ, the MD, and the TMD ---
are defined such that $W: \real^\dimobs\times\real^\dimobs \to [0, 1]$
and therefore can only down-weight observations.
This means that our updates are always conservative,
i.e., our posteriors will be wider in the presence of outliers
(see Figure \ref{fig:intro-image} for an example).

\begin{figure}
    \centering
    \includegraphics[width=0.6\columnwidth]{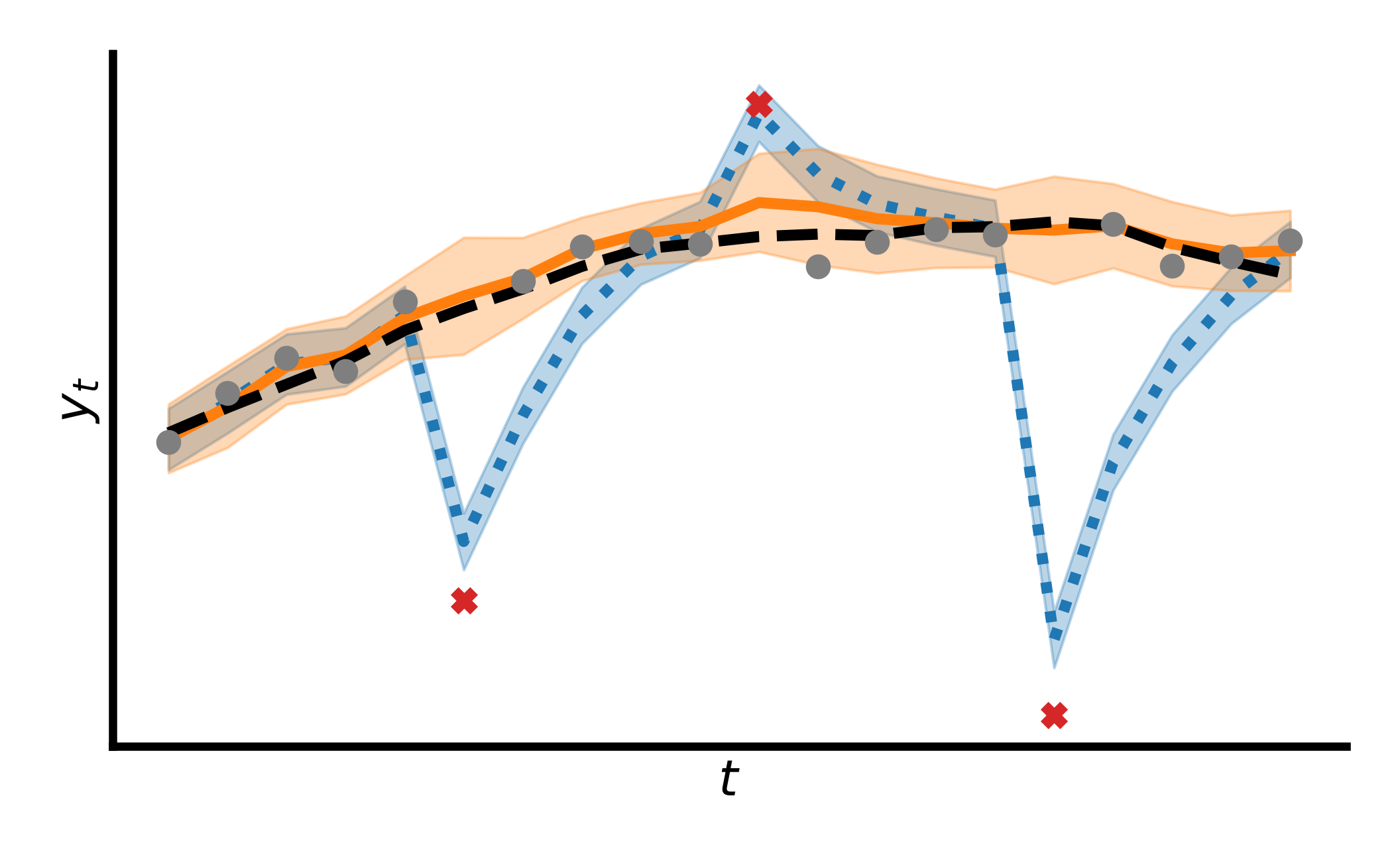}
    \caption{
        First state component of the SSM \eqref{eq:noisy-2d-ssm}.
        The grey dots are measurements sampled from \eqref{eq:noisy-2d-ssm} and
        the red crosses are measurements sampled from an outlier measurement process.
        The \textcolor{tabblue}{dotted blue line} shows the KF posterior mean estimate and
        the \textcolor{taborange}{solid orange line} shows our proposed WoLF posterior mean estimate.
        The regions around the posterior mean cover two standard deviations.
        For comparison, the dashed black line shows the true sampled state process.
    }
    \label{fig:intro-image}
\end{figure}

\section{Robustness properties}
\label{sec:theory}

In this section, we prove
the outlier-robustness for WoLF-type methods.
We use the classical framework of \citet{huber2011robust}.
Consider measurements $\vy_{1:t}$.
We measure the influence of a contamination $\vy^{c}_t$ by examining the divergence
between the posterior with the original observation $\vy_t$
and the posterior with the contamination $\vy_t^{c}$, which is allowed to be arbitrarily large. 
As a function of $\vy_t^{c}$, this divergence is called the \emph{posterior influence function} (PIF) and
was studied in
\citet{matsubara2021robust, Altamirano2023_bocpd, Altamirano2023_gp}.
Following \citet{Altamirano2023_gp}, we consider the Kullback-Leibler (KL) divergence.

\begin{definition}[posterior influence function]
    Let
    $\data_{1:t} = \{(\vx_1, \vy_1), \ldots, (\vx_t, \vy_t)\}$ and
    $\data_{1:t}^c = \{(\vx_1, \vy_1), \ldots, (\vx_t, \vy_t^c)\}$
    be two datasets.
    Consider the choice of \cPosterior $q$.
    The posterior influence function (PIF) of the measurement $\vy_t^c$ under $q$,
    given the dataset $\data_{1:t}$,
    is defined by
    \begin{equation}
        \operatorname{PIF}_q(\vy_t^{c},\data_{1:t}) = 
        \KL
        {q(\vtheta_t \cond \data_{1:t}^c)}
        {q(\vtheta_t \cond \data_{1:t})}.
    \end{equation}
\end{definition}

\begin{definition}[outlier robust posterior]
We denote the choice of \cPosterior $q$ to be outlier robust to measurements (or simply outlier robust)
if, for any given $\vx_t\in\real^{\dimin}$,
the effect of the contamination $\vy_t^c$ is bounded, i.e.,
\begin{equation}
    \sup_{\vy_t^c\in\mathbb{R}^{\dimobs}} \operatorname{PIF}_q(\vy_t^{c}, \data_{1:t}) <\infty.
\end{equation}
\end{definition}

\begin{theorem}\label{theorem:kf-unbounded}
    Let $q_{\rm LG}$ be the linear Gaussian update function shown in
    Section \ref{sec:extended-kalman-filter},
    then, $q_{\rm LG}$ has an unbounded PIF and is not outlier robust.
\end{theorem}
\begin{theorem}\label{theorem:wolf-bounded}
    Let $q_{\rm W-LG}$ be the generalised posterior \eqref{eq:generalised-posterior}
    with loss function \eqref{eq:weighted-gaussian-loglikelihood} and weight $W$ such that
    $\sup_{\vy_t\in\real^d}W(\vy_t, \hat{\vy}_t)<\infty$ and
    $\sup_{\vy_t\in\real^d}W(\vy_t, \hat{\vy}_t)^k\,\|\vy_t\|_{2}<\infty$
    for $k \geq 2$.
    Then,  $q_{\rm W-LG}$ has a bounded PIF and is, therefore, outlier robust.
\end{theorem}
In particular, the conditions are satisfied when $W$ is
\eqref{eq:w_t}, \eqref{eq:mahalanobis-weight}, or \eqref{eq:thresholded-mahalanobis-weight},
which are the focus of this work.


\section{Proof of robustness}
In this section, we prove Theorems \ref{theorem:kf-unbounded} and \ref{theorem:wolf-bounded} stated above.
We begin by stating some remarks that we use to prove the theorems.

\begin{remark}\label{remark:hermitian-matrix-eigvals}
    Let $\vA$ be a $\dimstate \times \dimstate$ Hermitian matrix.  
    By the spectral theorem, the eigenvalues of $\vA$ are real \citep{tao2010eigenvalues}.
    We denote these eigenvalues as $\sigma_{\rm max}(\vA) = \sigma_1(\vA) \geq \ldots \geq \sigma_\dimstate(\vA) = \sigma_{\rm min}(\vA)$,  
    where $\sigma_{\rm max}(\vA)$ and $\sigma_{\rm min}(\vA)$ represent the largest and smallest eigenvalues, respectively.  
\end{remark}

\begin{remark}\label{remark:upper-bound-xAx}
    Let $\vA$ be a $\dimstate\times\dimstate$ positive semidefinite matrix.
    Then, for every $\vx \in \real^\dimstate$,
    \begin{equation}
        \sqrt{\vx^\intercal\,\vA\,\vx} = \|\vA^{1/2}\,\vx\|_2 \leq \|\vA^{1/2}\|_2\,\|\vx\|_2 = \sqrt{\sigma_{\rm max}(\vA)}\,\|\vx\|_2.
    \end{equation}
    Here, $\|\vA\|_2 = \sqrt{\sigma_{\rm max}(\vA\,\vA)}$ is the spectral norm.
    See \cite{coppersmith1997rayleighinequalities} for details.
\end{remark}

\begin{remark}\label{remark:trace-bound}
    Let $\vA$ be a $\dimstate\times\dimstate$ positive semidefinite matrix.
    Then,
    \begin{equation}
        \|\vA\|_F^2 = {\rm Tr}(\vA\,\vA) \leq {\rm Tr}(\vA)^2.
    \end{equation}
    See Lemma 2.3 in \cite{yang2002traceineq} for details.
\end{remark}

\begin{remark}\label{remark:fro-properties}
    The Frobenius norm is a
    submutiplicative matrix norm and compatible with the L2 vector norm,
    i.e., for any two $\dimstate\times\dimstate$ Hermitian matrices $\vA$ and $\vB$, and
    a vector $\vz\in\real^\dimstate$,
    $\|\vA\,\vB\|_F \leq \|\vA\|_F\,\|\vB\|_F$ and
    $\|\vA\vx\|_2 \leq \|\vA\|_F\,\|\vx\|_2$.

    See Section 5.6 in \cite{horn2012matrixanalysis} for a proof on submultiplicativity.
    Next, compatibility follows from Remark \ref{remark:upper-bound-xAx} and
    properties of the trace. In particular:
    $\|\vA\,\vx\|_2 \leq \|\vA\|_2\,\|\vx\|_2
    = \sqrt{\sigma_{\rm max}(\vA\,\vA)}\,\|\vx\|_2
    \leq \sqrt{\sum_{d=1}^\dimstate \sigma_d(\vA\vA)}
    = \|\vA\|_F\,\|\vx\|_2
    $.
\end{remark}

\begin{remark}\label{remark:weyls-inverse-inequality}
    Let $\vA$ and $\vB$ be $\dimstate\times\dimstate$ Hermitian matrices.
    From Weyl's inequality, it follows that 
    \begin{equation}\label{eq:part-weyl-inverse}
        \sigma_{\rm min}(\vA + \vB)^{-1} \leq \left(\sigma_{\rm min}(\vA) + \sigma_{\rm min}(\vB)\right)^{-1}.
    \end{equation}
    See \cite{tao2010eigenvalues} for details.
\end{remark}

\begin{remark}\label{remark:lower-bound-determinants}
    For any two positive semi-definite matrices
    $\vA$ and $\vB$,
    \begin{equation}
    	|\vA| + |\vB| \leq |\vA + \vB|.
    \end{equation}
    The result follows
    from $|\vA + \vB| = |\vA|\,|\vI_\dimstate + \vA^{-1}\vB| \geq |\vA|\,(1 + |\vA^{-1}\vB|) = |\vA| + |\vB|$,
    where we used the fact that, for a matrix $\vC$, $\sigma_d(\vC + c) = c + \sigma_d(\vC)$ for $d \in \{1, \ldots, D\}$.
    So that $|\vA^{-1}\,\vB + 1| = \prod_{d=1}^\dimstate(1 + \sigma_d(\vA^{-1}\,\vB)) \geq 1 + \prod_{d=1}^\dimstate \sigma_d(\vA^{-1}\,\vB)$ = $1 + |\vA^{-1}\,\vB|$.
\end{remark}

\subsection{Proof of Theorem \ref{theorem:kf-unbounded} --- KF is not outlier-robust}
\label{section:proof-kf-unbounded}
Let  
$q_{\rm LG}(\vtheta_t \cond \data_{1:t}) = \mathcal{N}(\vtheta_t \cond \vmu_t, \vSigma_t)$ and  
$q_{\rm LG}(\vtheta_t \cond \data_{1:t}^c) = \mathcal{N}(\vtheta_t \cond \bar{\vmu}_t, \bar{\vSigma}_t)$  
be the posterior densities under $q_{\rm LG}$,  
conditioned on the uncorrupted dataset $\data_{1:t}$ and the  
corrupted dataset $\data_{1:t}^c$, respectively.
Furthermore, to avoid cases where the PIF is bounded through the choice of $\vH_t$,
suppose that $\|\vK\|_F \neq 0 \iff \vH_t \neq {\bf 0}$.

We show that $q_{\rm LG}$ is not outlier robust, i.e.,  
\begin{equation}
    \sup_{\vy_t^c} {\rm PIF}_{q_{\rm LG}}(\vy_t^c, \data_{1:t}) = \infty.  
\end{equation}  
To establish this, we first derive the explicit form of the PIF under the choice of $q_{\rm LG}$.  
The result is presented in the lemma below.  

\begin{lemma}\label{lemma:pif-form-kf}
The PIF between the posterior density $q_{\rm LG}$ under the corrupted and uncorrupted dataset
takes the form
\begin{equation}\label{eq:PIF-kf}
    {\rm PIF}_{q_{\rm LG}}(\vy_t^c,\data_{1:t})
    =
    \tfrac{1}{2}\left[\vK_t\,(\vy_{t} - \vy_{t}^{c})\right]^\intercal\,
    \vSigma_{t}^{-1}
    \left[\vK_t\,(\vy_{t} - \vy_{t}^{c})\right].
\end{equation}
where
$\vK_t = \vSigma_t\,\vH_t^\intercal\,\vR_t^{-1}$
\end{lemma}

\begin{proof}

Following proposition \ref{prop:kf-update-step-precision},
the update equations for the contaminated and uncontaminated densities take the form
\begin{equation}
\label{eq:part-clean-corrupted-posterior-terms-kf}
\begin{aligned}
	\vmu_{t} =\vmu_{t\vert t-1}+\,\vK_t \left(\vy_{t}-\hat{\vy}_{t}\right),&&
	\bar{\vmu_{t}} =\vmu_{t\vert t-1} + \vK_t\,\left(\vy_{t}^{c}-\hat{\vy}_{t}\right),\\
	\vSigma_{t}^{-1} =\vSigma_{t\vert t-1}^{-1} +  \vH_{t}^{\intercal} \vR_{t}^{-1} \vH_{t}, &&
	\bar{\vSigma_{t}}^{-1} =\vSigma_{t|t-1}^{-1} + \vH_{t}^{\intercal}\,\vR_{t}^{-1}\,\vH_{t},
\end{aligned}
\end{equation}
where we observe that $\vSigma_t = \bar\vSigma_t$.

Next, following Proposition \ref{prop:KL-divergence-gaussians},
we write the PIF as
\begin{equation}\label{eq:part-pif-expansion-i}
\begin{aligned}
    \operatorname{PIF}_{q_{\rm LG}}(\vy_t^{c},\vy_{1:t})
    &=\frac{1}{2}\Bigg(
    \Tr\left(\vSigma_{t}^{-1}\vSigma_{t}\right) - \dimstate +
    \left(\vmu_{t} - \vmu_{t}^{c}\right)^\intercal \vSigma_{t}^{-1}\left(\vmu_{t} - \vmu_{t}^{c}\right)+ 
    \ln\left(\frac{\det\vSigma_{t}}{\det\vSigma_{t}}\right)\Bigg)\\
    &=\Tr\left(\vI_\dimstate\right) - \dimstate + \tfrac{1}{2}\left(\vmu_{t} - \vmu_{t}^{c}\right)^\intercal \vSigma_{t}^{-1}\left(\vmu_{t} - \vmu_{t}^{c}\right)\\
    &=\tfrac{1}{2}\left(\vy_{t} - \vy_{t}^{c}\right)^\intercal \vK_t^\intercal\vSigma_{t}^{-1}\vK_t\left(\vy_{t} - \vy_{t}^{c}\right)\\
    &=\tfrac{1}{2}\left[\vK_t\,(\vy_{t} - \vy_{t}^{c})\right]^\intercal\,
    \vSigma_{t}^{-1}
    \left[\vK_t\,(\vy_{t} - \vy_{t}^{c})\right].
\end{aligned}
\end{equation}
\end{proof}

\paragraph{Proof of Theorem \ref{theorem:kf-unbounded}}
\begin{proof}
    First write an upper bound for the PIF
    following Lemma \ref{lemma:pif-form-kf}, Remark \ref{remark:fro-properties}, and Remark \ref{remark:fro-properties}:
    \begin{equation}
    \begin{aligned}
        {\rm PIF}_{q_{\rm LG}}(\vy_t^c,\data_{1:t})
        &=
        \tfrac{1}{2}\left[\vK_t\,(\vy_{t} - \vy_{t}^{c})\right]^\intercal\,
        \vSigma_{t}^{-1}
        \left[\vK_t\,(\vy_{t} - \vy_{t}^{c})\right]\\
        &\leq \tfrac{1}{2}\,\sigma_{\rm max}(\vSigma_t^{-1})\|\vK_t\,(\vy_{t} - \vy_{t}^{c})\|_2^2\\
        &\leq \tfrac{1}{2}\,\sigma_{\rm max}(\vSigma_t^{-1})\|\vK_t\|_F^2\,\|\vy_{t} - \vy_{t}^{c}\|_2^2\\
        &= C_1\,\|\vy_{t} - \vy_{t}^{c}\|_2^2.
    \end{aligned}
    \end{equation}
    Then, since $\|\vK_t\|_F^2 \neq 0$ by assumption,
    \begin{equation}
        \sup_{\vy_t^c} {\rm PIF}_{q_{\rm LG}}(\vy_t^c,\data_{1:t})
        \leq \sup_{\vy_t^c} C_1\,\|\vy_{t} - \vy_{t}^{c}\|_2^2 = \infty.
    \end{equation}
\end{proof}

\subsection{Proof of Theorem \ref{theorem:wolf-bounded} --- WoLF is outlier-robust}
\label{section:proof-wolf-bounded}
Let $W$ be a weighting function such that  
$\sup_{\vy_t \in \real^\dimobs}\,W(\vy_t, \hat{\vy}_t) < \infty$ and  
$\sup_{\vy_t \in \real^\dimobs}\,W(\vy_t, \hat{\vy}_t)^k\,\|\vy_t - \hat{\vy}_t\|_2 < \infty$  
for $k \geq 2$.  
For simplicity, denote  
$w_t := W(\vy_t, \hat{\vy}_t)$ and  
$\bar{w}_t := W(\vy_t^c, \hat{\vy}_t)$.  
Let  
$q_{\rm W-LG}(\vtheta_t \cond \data_{1:t}) = \mathcal{N}(\vtheta_t \cond \vmu_t, \vSigma_t)$ and  
$q_{\rm W-LG}(\vtheta_t \cond \data_{1:t}^c) = \mathcal{N}(\vtheta_t \cond \bar{\vmu}_t, \bar{\vSigma}_t)$  
be the posterior densities under $q_{\rm W-LG}$,  
conditioned on the uncorrupted dataset $\data_{1:t}$ and the  
corrupted dataset $\data_{1:t}^c$, respectively, using the weighting function $W$.

Our goal is to show that $q_{\rm LG}$ is outlier robust, i.e.,  
\begin{equation}
    \sup_{\vy_t^c} {\rm PIF}_{q_{\rm W-LG}}(\vy_t^c, \data_{1:t}) < \infty.  
\end{equation}  
To establish this, we first derive the explicit form of the PIF under the choice of $q_{\rm W-LG}$.  
The result is presented in the lemma below.

\begin{lemma}\label{lemma:pif-form-wolf}
The PIF between the posterior density $q_{\rm W-LG}$ under the corrupted and uncorrupted dataset
takes the form
\begin{equation}\label{eq:PIF-wolf}
    {\rm PIF}_{q_{\rm W-LG}}(\vy_t^c,\data_{1:t}) =
     \frac{1}{2}\Bigg(
    \underbrace{\Tr\left(\bar{\vSigma_{t}}^{-1}\vSigma_{t}\right)}_{\rm (T.1)}
    - \dimstate
    +
    \underbrace{\left(\vmu_{t} - \bar{\vmu_{t}}\right)^\intercal \bar{\vSigma}_{t}^{-1}\left(\vmu_{t} - \bar{\vmu}_{t}\right)}_{\rm (T.2)}
    +
  \underbrace{\log\left(\frac{\det\bar\vSigma_{t}}{\det\vSigma_{t}}\right)}_{\rm (T.3)}\Bigg),
\end{equation}
where
\begin{equation}
\label{eq:part-clean-corrupted-posterior-terms}
\begin{aligned}
	\vmu_{t} =\vmu_{t\vert t-1}+w_t^2\,\vK_t \left(\vy_{t}-\hat{\vy}_{t}\right),&&
	\bar{\vmu_{t}} =\vmu_{t\vert t-1} + \bar{w_t}^2\,\bar{\vK_t}\,\left(\vy_{t}^{c}-\hat{\vy}_{t}\right),\\
	\vSigma_{t}^{-1} =\vSigma_{t\vert t-1}^{-1} +  w_t^2\vH_{t}^{\intercal} \vR_{t}^{-1} \vH_{t}, &&
	\bar{\vSigma_{t}}^{-1} =\vSigma_{t|t-1}^{-1} +\bar{w_t}^2\,\vH_{t}^{\intercal}\,\vR_{t}^{-1}\,\vH_{t},
\end{aligned}
\end{equation}
\end{lemma}

\begin{proof}
    The proof follows directly from Algorithm \ref{algo:wlf-step}
    and Proposition \ref{prop:KL-divergence-gaussians} on
    the KL divergence between two multivariate Gaussian densities.
\end{proof}

Lemma \ref{lemma:pif-form-wolf} establishes that if the PIF under $q_{\rm W-LG}$ is outlier-robust,  
then each of the terms $({\rm T.1})$, $({\rm T.2})$, and $({\rm T.3})$  
must remain finite (bounded) under the supremum over all possible values of $\vy_t^c$.  
Before demonstrating that this is indeed the case,  
we first present two auxiliary lemmas that will assist in proving the boundedness of these terms.

\begin{lemma}\label{lemma:norm-covariance-bound}
    The Frobenius norm of the corrupted posterior covariance $\bar{\vSigma}_t$
    shown in \eqref{eq:part-clean-corrupted-posterior-terms} is bounded above by
\begin{equation}
	\|\bar{\vSigma_{t}}\|_{\rm F} \leq
	\frac{\dimstate}{\sigma_{\rm min}(\vSigma_{t|t-1}^{-1}) + w_t^2\,\sigma_{\rm min}(\vH_t^\intercal\,\vR_t^{-1}\,\vH_t))}.
\end{equation}
\end{lemma}

\begin{proof}
    Let $\bar{\vSigma}_t$ be the posterior covariance under the corrupted dataset $\data_{1:t}^c$.
    Following
    Remark \ref{remark:hermitian-matrix-eigvals},
    Remark \ref{remark:trace-bound}, and
    the properties of the trace, we obtain
    \begin{equation}\label{eq:part-upper-bound-fro}
    \|\bar{\vSigma}_t\|_F
    \leq \Tr\left(\bar{\vSigma}_t\right)\\
    =  \sum_{d=1}^\dimstate \sigma_d(\bar{\vSigma}_t)\\
    \leq \dimstate\,\sigma_{\rm max}(\bar{\vSigma}_t).
    \end{equation}

    Next, an upper bound for
    $\sigma_{\rm max}(\bar{\vSigma}_t)$ is given by
    \begin{equation}\label{eq:part-upper-bound-sigma}
    \begin{aligned}
        \sigma_{\rm max}(\bar{\vSigma}_t)
        &= \sigma_{\rm max}\left(\left(\vSigma_{t|t-1}^{-1} + \bar{w}_t^2\,\vH_t^\intercal\,\vR_t^{-1}\,\vH_t\right)^{-1}\right)\\
        &= \left(\sigma_{\rm min}\left(\vSigma_{t|t-1}^{-1} + \bar{w}_t^2\,\vH_t^\intercal\,\vR_t^{-1}\,\vH_t\right)\right)^{-1}\\
        &\leq \left( \sigma_{\rm min}(\vSigma_{t|t-1}^{-1}) + \sigma_{\rm min}\left(\bar{w}_t^2\,\vH_t^\intercal\,\vR_t^{-1}\,\vH_t\right)\right)^{-1}\\
        &= \left( \sigma_{\rm min}(\vSigma_{t|t-1}^{-1}) + \bar{w}_t^2\,\sigma_{\rm min}\left(\vH_t^\intercal\,\vR_t^{-1}\,\vH_t\right)\right)^{-1}.
    \end{aligned}
    \end{equation}
    In \eqref{eq:part-upper-bound-sigma},
    the second equality is a consequence of the positive definiteness (symmetry) of $\bar\vSigma_t$,
    and the upper bound follows from Remark \ref{remark:weyls-inverse-inequality}.

    The desired result follows as a consequence of
    \eqref{eq:part-upper-bound-fro} and \eqref{eq:part-upper-bound-sigma}.
\end{proof}

\begin{lemma}\label{lemma:norm-diff-means-bound}
	The L2 norm of the difference between the corrupted and uncorrupted posterior mean is bounded above by
	\begin{equation}\label{eq:part-difference-posterior-means}
		\|\vmu_t - \bar{\vmu}_t\|_2^2 \leq C_3\,\left(w_t^2\|\vy_t^c - \hat{\vy}_t\|  \right)^2 + C_1,
	\end{equation}
	where $C_1$ and $C_3$ are real-valued elements
	that do not depend on $\vy_t^c$.
\end{lemma}

\begin{proof}
We begin by expanding \eqref{eq:part-difference-posterior-means}:
\begin{equation}\label{eq:part-mean-diff-pif}
\begin{aligned}
	\|\vmu_t - \bar{\vmu}_t\|_2^2
	&= \|\vK_t\,(\vy_t - \hat{\vy}_t) - \bar{\vK}_t\,(\vy_t^c - \hat{\vy}_t)\|_2^2\\
	&\leq \|\vK_t\,(\vy_t - \hat{\vy}_t)\|_2^2 + \|\bar{\vK}_t\,(\vy_t^c - \hat{\vy}_t)\|_2^2\\
	&= \|\bar{w}_t^2\,\bar{\vSigma}_{t|t-1}\,\vH_t\,\vR_{t}^{-1}\,(\vy_t^c - \hat{\vy}_t)\|_2^2 + C_1\\
	&= \bar{w}_t^4\, \|\bar{\vSigma}_{t|t-1}\,\vH_t\,\vR_{t}^{-1}\,(\vy_t^c - \hat{\vy}_t)\|_2^2 + C_1,
\end{aligned}
\end{equation}
where $C_1 = \|\vK_t\,(\vy_t - \hat{\vy}_t)\|_2^2$ does not depend on $\vy_t^c$.

Next, following Remark \ref{remark:fro-properties}, we bound
$\bar{w}_t^4\, \|\bar{\vSigma}_{t|t-1}\,\vH_t\,\vR_{t}^{-1}\,(\vy_t^c - \hat{\vy}_t)\|_2^2$
as follows
\begin{equation}\label{eq:part-mean-diff-pif-2}
\begin{aligned}
        &\bar{w}_t^4\, \|\bar{\vSigma}_{t|t-1}\,\vH_t\,\vR_{t}^{-1}\,(\vy_t^c - \hat{\vy}_t)\|_2^2\\
	&\leq \bar{w}_t^4\,
	\|\bar{\vSigma}_{t|t-1}\|_F^2\,
	\|\vH_t\,\vR_t^{-1}\|_2^2\,
	\|\vy_t^c - \hat{\vy_t}^2\|_2^2\\
	&= C_2\,\bar{w}_t^4\,
	\|\bar{\vSigma}_{t|t-1}\|_F^2\,
	\|\vy_t^c - \hat{\vy_t}^2\|_2^2,
\end{aligned}
\end{equation}
where $C_2 = \|\vH_t\,\vR_t^{-1}\|_2^2$.

Finally,
using Lemma \ref{lemma:norm-covariance-bound},
an upper bound for \eqref{eq:part-mean-diff-pif-2}
is given by
\begin{equation}
\begin{aligned}
        &C_2\,\bar{w}_t^4\,
	\|\bar{\vSigma}_{t|t-1}\|_F^2\,
	\|\vy_t^c - \hat{\vy_t}^2\|_2^2\\
	&\leq C_2\,\bar{w}_t^4\,
	\left(\frac
	{D}
	{\sigma_{\rm min}(\vSigma_{t|t-1}^{-1}) + w_t^2\,\sigma_{\rm min}(\vH_t^\intercal\,\vR_t^{-1}\,\vH_t))}
	\right)^2\,
	\|\vy_t^c - \hat{\vy}_t\|_2^2\\
	&= C_2\,D^2\,
	\frac
	{\bar{w}_t^4}
	{\left(\sigma_{\rm min}(\vSigma_{t|t-1}^{-1}) + w_t^2\,\sigma_{\rm min}(\vH_t^\intercal\,\vR_t^{-1}\,\vH_t))\right)^2}\,
	\|\vy_t^c - \hat{\vy}_t\|_2^2\\
	&= C_2\,D^2
	\frac
	{1}
	{\left(w_t^{-2}\,\sigma_{\rm min}(\vSigma_{t|t-1}^{-1}) + \sigma_{\rm min}(\vH_t^\intercal\,\vR_t^{-1}\,\vH_t))\right)^2}\,
	\|\vy_t^c - \hat{\vy}_t\|_2^2\\
	&\leq
	\frac{C_2\,D^2}{\sigma_{\rm min}(\vSigma_{t|t-1}^{-1})}\,\bar{w}_t^4\,\|\vy_t^c - \hat{\vy}_t\|_2^2 + C_1\\
	&= C_3\,\left(
	\,\bar{w}_t^2\,\|\vy_t^c - \hat{\vy}_t\|_2
	\right)^2,
\end{aligned}
\end{equation}
where $C_3 = \frac{C_2\,D^2}{\sigma_{\rm min}(\vSigma_{t|t-1}^{-1})}$.
\end{proof}

With the auxiliary lemmas established and following the remarks above,  
we now proceed to demonstrate that each of the terms $({\rm T.1})$, $({\rm T.2})$, and $({\rm T.3})$  
are indeed bounded, as required for outlier robustness.  
Since the key components of the following propositions were already defined in Proposition \ref{lemma:pif-form-wolf},  
we proceed directly by stating each term and showing that it is bounded over  
all possible values of $\vy_t^c$.

\begin{proposition}[Bound for ${\rm T.1}$]
\label{prop:wolf-bound-t1}
\begin{equation}
	\sup_{\vy_t^c}
	\Tr\left(\bar{\vSigma_{t}}^{-1}\vSigma_{t}\right)
	< \infty.
\end{equation}
\end{proposition}

\begin{proof}
Using \eqref{eq:part-clean-corrupted-posterior-terms}, $({\rm T.1})$ in \eqref{eq:PIF-wolf}
can be written as
\begin{equation}
\begin{aligned}	
	\Tr\left(\bar{\vSigma_{t}}^{-1}\vSigma_{t}\right)
	&=\Tr\left(
	\left[ \vSigma_{t|t-1}^{-1} + \bar{w}_t^2\,\vH_t^\intercal\,\vR_t^{-1}\,\vH_t \right]\,
	\vSigma_t\right)\\
	&= 
	\Tr\left(\vSigma_{t|t-1}^{-1}\,\vSigma_{t}\right)
	+ \bar{w}_t^2\,\Tr\left(
	\vH_t^\intercal\,\vR_t^{-1}\,\vH_t\,\vSigma_t
	\right)
\end{aligned}
\end{equation}
Since $\sup_{\vy_t^c} \bar{w}_t$ is bounded,
it follows that
\begin{equation}
	\sup_{\vy_t^c}
	\Tr\left(\bar{\vSigma_{t}}^{-1}\vSigma_{t}\right)
	< \infty.
\end{equation}
\end{proof}

\begin{proposition}[Bound for ${\rm T.2}$]
\label{prop:wolf-bound-t2}
\begin{equation}\label{part:pif-t2-term}
\sup_{\vy_t^c} \left(\vmu_{t} - \bar{\vmu_{t}}\right)^\intercal \bar{\vSigma}_{t}^{-1}\left(\vmu_{t} - \bar{\vmu}_{t}\right)
< \infty
\end{equation}
\end{proposition}

\begin{proof}
	Begin by expanding the left hand side of \eqref{part:pif-t2-term}:
	\begin{equation}\label{eq:part-pif-t2-term-expand}
	\begin{aligned}
	&\left(\vmu_{t} - \bar{\vmu_{t}}\right)^\intercal \bar{\vSigma}_{t}^{-1}\left(\vmu_{t} - \bar{\vmu}_{t}\right)\\
	&=\left(\vmu_{t} - \bar{\vmu_{t}}\right)^\intercal
	\left(\vSigma_{t|t-1} +
	\bar{w}_t^2\,\vH_t^\intercal\,\vR_t^{-1}\vH_t\right)
	\left(\vmu_{t} - \bar{\vmu}_{t}\right)\\
	&=\left(\vmu_{t} - \bar{\vmu_{t}}\right)^\intercal
	\vSigma_{t|t-1}^{-1}
	\left(\vmu_{t} - \bar{\vmu}_{t}\right) + 
	\left(\vmu_{t} - \bar{\vmu_{t}}\right)^\intercal
	\left(\bar{w}_t^2\,\vH_t^\intercal\,\vR_t^{-1}\vH_t\right)
	\left(\vmu_{t} - \bar{\vmu}_{t}\right).
	\end{aligned}
	\end{equation}
	Next, we bound each of the terms in \eqref{eq:part-pif-t2-term-expand}.

        For the first term, we obtain
	\begin{equation}\label{eq:part-pif-t2-term-expand-i}
	\begin{aligned}
	&\left(\vmu_{t} - \bar{\vmu_{t}}\right)^\intercal
	\vSigma_{t|t-1}^{-1}
	\left(\vmu_{t} - \bar{\vmu}_{t}\right)\\
	&\leq \sigma_{\rm max}(\vSigma^{-1}_{t|t-1})
	\|\vmu_t - \bar{\vmu}_t\|\\
	&\leq 
	\sigma_{\rm max}(\vSigma^{-1}_{t|t-1})\,
	\left(
	C_3\,\left(
	\,\bar{w}_t^2\,\|\vy_t^c - \hat{\vy}_t\|_2
	\right)^2 + C_1
	\right)\\
	&= C_5\,\left(
	\,\bar{w}_t^2\,\|\vy_t^c - \hat{\vy}_t\|_2
	\right)^2 + C_6,
	\end{aligned}
	\end{equation}
        where we make use of Remark \ref{remark:upper-bound-xAx} and
        Lemma \ref{lemma:norm-diff-means-bound}.
        
        By assumption,  
        $\sup_{\vy_t^c} \bar{w}_t^2 \|\vy_t^c - \vy_t\| < \infty$,  
        which ensures that the supremum over $\vy_t^c$ for  
        \eqref{eq:part-pif-t2-term-expand-i} is bounded.  

	Similarly, for the second term, suppose
        $\sigma_{\rm max}(\vH_t^\intercal\,\vR_t^{-1}\vH_t) > 0$, then
	\begin{equation}\label{eq:part-pif-t2-term-expand-ii}
	\begin{aligned}
	&\left(\vmu_{t} - \bar{\vmu_{t}}\right)^\intercal
	\left(\bar{w}_t^2\,\vH_t^\intercal\,\vR_t^{-1}\vH_t\right)
	\left(\vmu_{t} - \bar{\vmu}_{t}\right)\\
        &\leq \bar{w}_t^2\,\sigma_{\rm max}(\vH_t^\intercal\,\vR_t^{-1}\vH_t)
	\left(
	C_3\,\left(
	\,\bar{w}_t^2\,\|\vy_t^c - \hat{\vy}_t\|_2
	\right)^2 + C_1
	\right)\\
	&\leq C_7\,\left(
	\,\bar{w}_t^3\,\|\vy_t^c - \hat{\vy}_t\|_2
	\right)^2 + C_8.
	\end{aligned}
	\end{equation}
        Again, by assumption,  
        $\sup_{\vy_t^c} \bar{w}_t^3 \|\vy_t^c - \vy_t\| < \infty$,  
        which ensures that the supremum over $\vy_t^c$ for  
        \eqref{eq:part-pif-t2-term-expand-ii} is bounded.  
\end{proof}

\begin{proposition}[Bound for ${\rm T.3}$]
\label{prop:wolf-bound-t3}
\begin{equation}
	\sup_{\vy_t^c}\log\left(
	\frac{|\bar\vSigma_{t}|}
	{|\vSigma_{t}|}\right)
	< \infty
\end{equation}
\end{proposition}

\begin{proof}
We first expand $({\rm T.3})$ to obtain
\begin{equation}
	\log\left(\frac{|\bar\vSigma_{t}|}{|\vSigma_{t}|}\right)
	= \log |\bar\vSigma_t| - \log |\vSigma_t|,
\end{equation}
which shows that we only to bound $\log|\bar\vSigma_t|$.
For this, consider
\begin{equation}\label{eq:part-pif-t3-bound}
\begin{aligned}
	\log|\bar\vSigma_t|
	&= \log|(\vSigma_{t|t-1} + \bar{w}_t^2\,\vH_{t}^\intercal\,\vR_t^{-1}\,\vH_t)^{-1}|\\
	&=-\log|\vSigma_{t|t-1}^{-1}
	+ \bar{w}_t^2\,\vH_{t}^\intercal\,\vR_t^{-1}\,\vH_t|\\
	&\leq -\log\left(
	|\vSigma_{t|t-1}^{-1}| +
	|\bar{w}_t^2\,\vH_t^\intercal\,\vR_t^{-1}\,\vH_t|
	\right)\\
	&= -\log\left(
	|\vSigma_{t|t-1}^{-1}| +
	\bar{w}_t^{2D}
	|\,\vH_t^\intercal\,\vR_t^{-1}\,\vH_t|
	\right)
\end{aligned}
\end{equation}
Since $\vSigma_{t|t-1}^{-1}$
is positive definite and
$\vH_t^\intercal\,\vR_t^{-1}\,\vH_t$
is positive semidefinite,
we upper bound \eqref{eq:part-pif-t3-bound}
further by taking the minimum of
the two terms, i.e., 
\begin{equation}
	\log|\bar\vSigma_t|
	\leq
	-\log\left(
	\min\{
	|\vSigma_{t|t-1}^{-1}|,\,
	\bar{w}_t^{2D}
	|\,\vH_t^\intercal\,\vR_t^{-1}\,\vH_t|
	\}
	\right).
\end{equation}
If the smallest determinant is given by the
posterior predictive covariance $\vSigma_{t|t-1}$,
then \eqref{eq:part-pif-t3-bound} is bounded.
Conversely, if the smallest determinant
is $\bar{w}_t^{2D} |\,\vH_t^\intercal\,\vR_t^{-1}\,\vH_t|$,
we obtain
\begin{equation}\label{eq:part-pif-t3-case2}
\begin{aligned}
	\log|\bar\vSigma_t|
	&\leq
	-\log\left(\bar{w}_t^{2D}|\vH_t^\intercal\,\vR_t^{-1}\,\vH_t|\right)\\
	&= -2D\log(\bar w_t) - \log|\vH_t^\intercal\,\vR_t^{-1}\,\vH_t|.
\end{aligned}
\end{equation}
Finally,
note that $\sup_{\vy_t^c} \bar{w}_t < \infty$
implies
$\sup_{\vy_t^c} \log\bar{w}_t < \infty$.
So that \eqref{eq:part-pif-t3-case2} is bounded.
\end{proof}

\paragraph{Proof of Theorem \ref{theorem:wolf-bounded}}
\begin{proof}
    The proof follows from Propositions \ref{prop:wolf-bound-t1}, \ref{prop:wolf-bound-t2}, and \ref{prop:wolf-bound-t3}.
\end{proof}

\section{Experiments}
\label{sec:experiments-robustness}
In this section, we study the performance of the WoLF algorithm in multiple experiments.

For our robust baselines,
we make use of three methods that are representative of recent state-of-the-art 
approaches to robust filtering:
the Bernoulli KF of \citet{wang2018} (\mWang),
which is an example of a detect-and-reject strategy; and
the inverse-Wishart filter of \citet{Agamennoni2012} (\mAgamenoni),
which  is an example of a compensation-based strategy.
The \mWang and \mAgamenoni are deterministic and optimise a VB objective to compute a Gaussian approximation to the state posterior.
For the neural network fitting problem,
we also consider  a variant of  online gradient descent (\ogd) based on Adam \citep{kingma2014adam}, which
uses multiple inner iterations per step (measurement).
This method does scale to high-dimensional state spaces, but only gives a maximum a posteriori (MAP) estimate and
is not as sample efficient as a robust Bayesian filter.

For experiments where KF or EKF is used as the baseline,
 we consider the following WoLF variants:
(i) the WoLF version with inverse multi-quadratic weighting function (\mWlfImq),
(ii)  the thresholded WoLF with Mahalanobis-based weighting function (\mWlfMd).

\begin{table}[ht]
    \small
    \centering
    \begin{tabular}{llll}
    \toprule
        Method & Cost & \#HP & Ref \\
    \midrule
     \mkf &  $O(\dimstate^3)$ & 0 & \cite{kalman1960}\\
     \mWang & $O(I\,\dimstate^3)$ & 3 & \cite{wang2018}\\
     \mAgamenoni & $O(I\,\dimstate^3)$ & 2 & \cite{Agamennoni2012} \\
     \ogd & $O(I\, \dimstate^2)$ & 2 & \cite{bencomo2023implicit}\\
     \mWlfImq  &  $O(\dimstate^3)$ & 1 &(Ours)\\
     \mWlfMd  &  $O(\dimstate^3)$ & 1 & (Ours)\\
    \bottomrule
    \end{tabular}
    \caption{
        Computational complexity  of the update step,
        assuming  $d \leq \dimstate$ and assuming linear dynamics.
        Here, $I$ is the number of inner iterations,
        \#HP refers to the number of hyperparameters we tune, and
        ``Cost'' refers to the computational complexity. }
    \label{tab:complexity-linear-model}
\end{table}

\subsection{Robust KF for tracking a 2D object}
\label{experiment:2d-tracking}

We consider the classical problem of estimating the position of an object moving in 2D
with constant velocity,
which is commonly used to benchmark tracking problems
(see e.g., Example 8.2.1.1 in \citet{murphy2023-pmlbook2} or Example 4.5 in \citet{sarkka2023filtering}).
The SSM takes the form
\begin{equation} \label{eq:noisy-2d-ssm}
\begin{aligned}
    p(\vtheta_t \vert \vtheta_{t-1}) &= \normdist{\vtheta_t}{\vF_t\vtheta_{t-1}}{\vQ_t},\\
    p(\vy_t \vert \vtheta_t) &= \normdist{\vy_t}{\vH_t\vtheta_t}{\vR_t},
\end{aligned}
\end{equation}
where $\vQ_t = q\,{\bf I}_4$, $\vR_t = r\,{\bf I}_2$,
$(\vtheta_{0,t}, \vtheta_{1,t})$ is the position, 
$(\vtheta_{2, t}, \vtheta_{3,t})$ is the velocity,
{\small
\begin{align*}
    \vF_t &=
    \begin{pmatrix}
    1 & 0 & \Delta & 0\\
    0 & 1 & 0 & \Delta \\
    0 & 0 & 1 & 0 \\
    0 & 0 & 0 & 1
    \end{pmatrix}, & 
    \vH_t &= \begin{pmatrix}
        1 & 0 & 0 & 0\\
        0 & 1 & 0 & 0
    \end{pmatrix},
\end{align*}}
$\!\!\Delta = 0.1$ is the sampling rate,
$q = 0.10$ is the system noise,
$r = 10$ is the measurement noise,
and ${\bf I}_K$ is a $K\times K$ identity matrix.
We simulate 500 trials, each with 1,000 steps.
For each method, we compute the scaled RMSE metric
$ J_{T,i} = \sqrt{\sum_{t=1}^T (\vtheta_{t,i}- \vmu_{t,i}) ^ 2}$
for $i\in\{0,1,2,3\}$ as well as the total running time (relative to the \mkf).

In our experiments,
the true data generating process is one of two variants of \eqref{eq:noisy-2d-ssm}.
The first variant
(which we call {\bf Student observations})
corresponds to a system whose measurement process
comes from the Student-t likelihood:
\begin{equation}\label{eq:noisy-2d-ssm-outlier-covariance}
    \begin{aligned}
    p(\vy_t \vert \vtheta_t)
    &= {\rm St}({\vy_t\,\vert\,\vH_t\vtheta_t,\,\vR_t,\nu_t})\\
    &= \int_0^\infty {\cal N}\left(\vy_t\,\vert\,\vH_t\vtheta_t, \frac{\vR_t}{\tau}\right){\rm Gam}\left(\tau \vert \frac{\nu_t}{2}, \frac{\nu_t}{2}\right) d\tau, 
    \end{aligned}
\end{equation}
with ${\rm Gam}(\cdot \vert a, b)$ the Gamma density function with shape $a$ and rate $b$, and
$\nu_t=2.01$.
The second variant 
(which we call {\bf mixture observations})
corresponds to a system where the mean of the observations
changes sporadically. Instances of this variant can occur as a form of
human error or a software bug in a data-entry program.
To emulate this scenario, we modify \eqref{eq:noisy-2d-ssm}
by using the following mixture model for the observation
process:
\begin{equation}\label{eq:noisy-2d-ssm-outlier-mean}
\begin{aligned}
p(\vy_t \vert \vtheta_t) &= \normdist{\vy_t}{\vm_t}{\vR_t},\\
\vm_t &= 
    \begin{cases}
        \vH_t\,\vtheta_t & \text{w.p.}\ 1 - p_\epsilon,\\
        2\,\vH_t\,\vtheta_t & \text{w.p.}\ p_\epsilon,
    \end{cases}
\end{aligned}
\end{equation}
where $p_\epsilon = 0.05$. 

\paragraph{Results}

\begin{figure}[htb]
    \centering
    \includegraphics[width=0.45\linewidth]{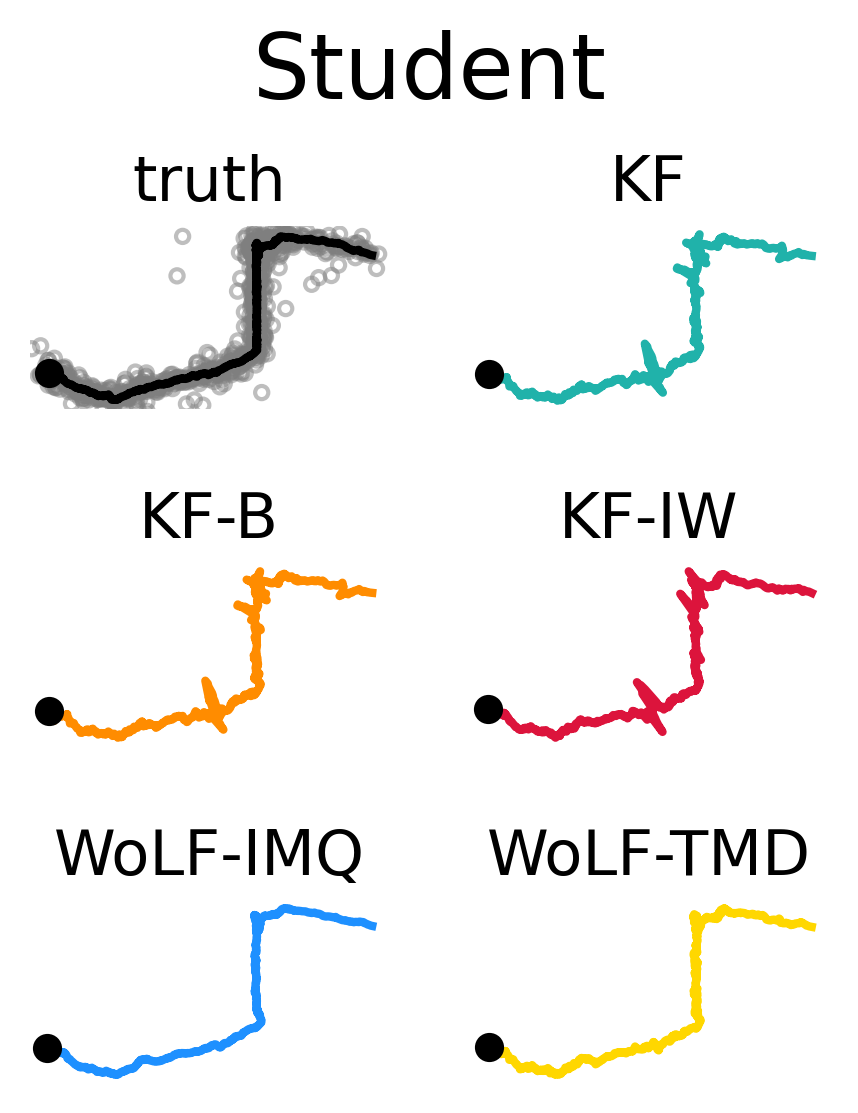}
    \vline
    \includegraphics[width=0.45\linewidth]{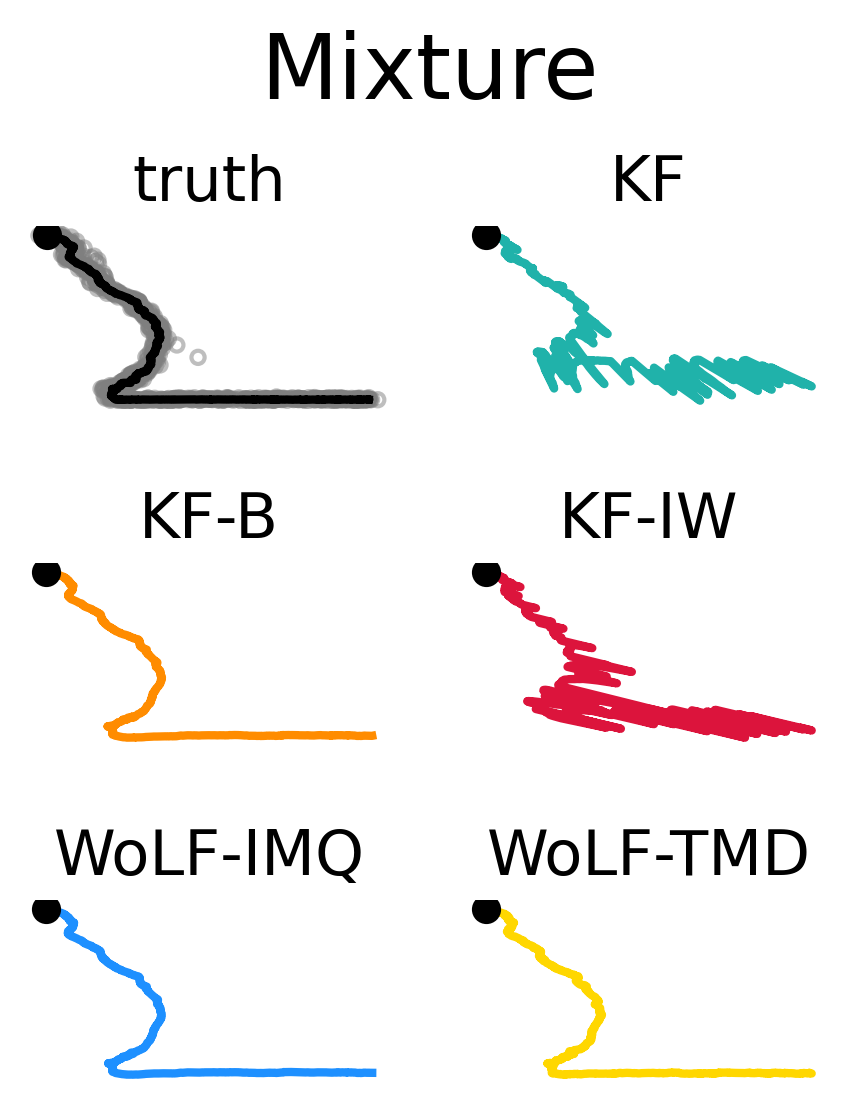}
    \caption{
    The left panel shows a sample path using the Student variant and
    the right panel shows a sample path using the mixture variant.
    The top left figure on each panel shows the true underlying state in black,
    and the measurements as grey dots.
    }
    \label{fig:linear-ssm-sample-runs}
\end{figure}

Figure \ref{fig:linear-ssm-sample-runs} shows a sample of each variant
along with the filtered state for each method.
For the Student variant (left panel),
the \mWlfImq and the \mWlfMd  estimate the true state
more closely than the competing methods.
Both the \mAgamenoni and the \mWang look comparable to the \mkf, which are not robust to outliers.
For the mixture variant (right panel).
the \mWlfImq, the \mWlfMd, and the \mWang filter the true state correctly.
In contrast, the \mAgamenoni and the \mkf are not robust to outliers.

\begin{figure}[ht]
    \centering
    \includegraphics[width=0.7\linewidth]{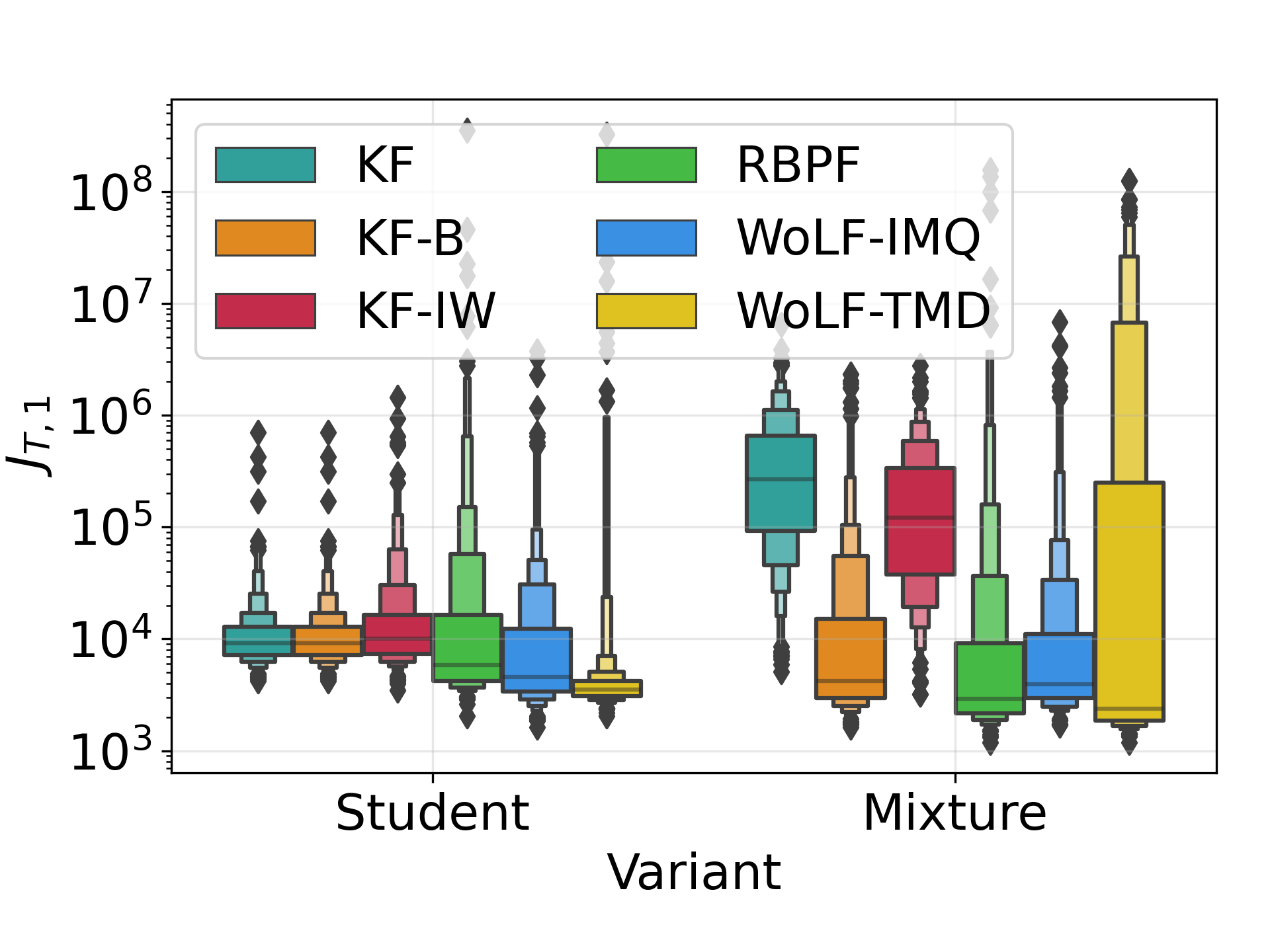}
    \caption{
        Distribution  (across 500 2d tracking trials)
        of RMSE for first component of the state vector, $J_{T,0}$.
        Left panel: Student observation model.
        Right panel: Mixture observation model.
    }
    \label{fig:2d-ssm-sample}
\end{figure}

The results in Figure \ref{fig:linear-ssm-sample-runs}
hold for multiple trials as shown in Figure \ref{fig:2d-ssm-sample},
which plots the distribution
of the errors in the first component of the state vector.
As a benchmark, we include the particle-filter-based method of \cite{boustati2020generalised},
which we denote \texttt{RBPF}.
The \texttt{RBPF} performs comparably to our proposed method;
however, it has much higher computational cost and does not have a closed-form solution.

\begin{table}[ht]
    \centering
    \begin{tabular}{c|cc}
        \toprule
        Method & Student & Mixture \\
        \midrule
         \mWang & 2.0x & 3.7x \\
         \mAgamenoni & 1.2x & 5.3x \\
         \mWlfImq (ours) & 1.0x  & 1.0x \\
         \mWlfMd (ours) &  1.0x & 1.0x \\
         \bottomrule
    \end{tabular}
    \caption{Mean slowdown rate over \mkf.}
    \label{tab:2d-ssm-running-time}
\end{table}

Table \ref{tab:2d-ssm-running-time} shows
the median slowdown (in running time) to process the measurements relative to the \mkf.
The slowdown for method \texttt{X} is obtained 
dividing the running time of method \texttt{X} over the running time of the \mkf.
Under the Student variant, the \mWlfImq, the \mWlfMd, and the \mAgamenoni have similar running time to the \mkf.
In contrast, the \mWang takes twice the amount of time.
Under the mixture variant, the \mWang and the \mAgamenoni are almost four times and five times slower than the \mkf respectively.
The changes in slowdown rate are due the number of inner iterations that were chosen during the first trial.

\subsection{Online learning of a neural network in the presence of outliers}
\label{subsec:uci-corrupted}

In this section, we benchmark the methods using a corrupted version
of the tabular UCI regression datasets.\footnote{
The dataset is available at \url{https://github.com/yaringal/DropoutUncertaintyExps}.
}
Here,
we consider a single-hidden-layer multi-layered perceptron (MLP)
with twenty hidden units and a real-valued output unit.
In this experiment, the state dimension (number of parameters in the MLP)
is $\dimstate =(n_\text{in} \times 20 + 20) + (20 \times 1 + 1)$, where
$n_\text{in}$ is the dimension of the feature $\vx_t$.
In Table \ref{tab:uci-description}, we show the the values that $n_\text{in}$ takes for each dataset.
\begin{table}[htb]
\centering
\scriptsize
\begin{tabular}{lrrr}
    \toprule
     & \#Examples $T$ & \#Features $\dimobs$ & \#Parameters $\dimstate$ \\
    Dataset &  &  &  \\
    \midrule
    Boston & $ 506 $ & $ 14 $ & $ 321 $ \\
    Concrete & $ 1,030 $ & $ 9 $ & $ 221 $ \\
    Energy & $ 768 $ & $ 9 $ & $ 221 $ \\
    Kin8nm & $ 8,192 $ & $ 9 $ & $ 221 $ \\
    Naval & $ 11,934 $ & $ 18 $ & $ 401 $ \\
    Power & $ 9,568 $ & $ 5 $ & $ 141 $ \\
    Protein & $ 45,730 $ & $ 10 $ & $ 241 $ \\
    Wine & $ 1,599 $ & $ 12 $ & $ 281 $ \\
    Yacht & $ 308 $ & $ 7 $ & $ 181 $ \\
\bottomrule
\end{tabular}
\caption{Description of UCI datasets.
Number of parameters refers to the size of the one-layer
MLP.
}
\label{tab:uci-description}
\end{table}

Below, we take the static case $\vQ_t = 0\,\vI_\dimstate$,
so that the prior predictive mean is $\vmu_{t|t-1} = \vmu_{t-1}$.

Each trial is carried out as follows:
first, we randomly shuffle the rows in the dataset;
second, we divide the dataset into a warmup dataset (10\% of rows) and a corrupted dataset (remaining 90\% of rows);
third, we normalise the corrupted dataset using min-max normalisation from the warmup dataset;
fourth, with probability $p_\epsilon=0.1$,
we replace a measurement $\vy_t\in\real$ with a corrupted data point  $\vu_t \sim {\cal U}[-50, 50]$; and
fifth, we run each method on the corrupted dataset.

For each dataset and for each method, we evaluate the prior predictive RMedSE
\begin{equation}
    {\rm RMedSE} = \sqrt{{\rm median}\{(\vy_t - h_t(\vmu_{t | t- 1})) ^ 2\}_{t=1}^T}
\end{equation}
which is the squared root of the median squared error
between the measurement $\vy_t$ and the prior predictive $h_t(\vmu_{t|t-1}) = h(\vmu_{t | t-1}, \vx_t)$.\footnote{We use median instead of mean because we have outliers in measurement space.}
Here, $h$ is the MLP.
We also evaluate the average time step of each method, i.e.,
we run each method and divide the total running time by the number of samples in the corrupted dataset.

Figure \ref{fig:uci-per-step-time}
shows the percentage change of the RMedSE and the percentage change of running time
with respect to those of the \ogd
for all corrupted UCI datsets.
\begin{figure}[ht]
    \centering
    \includegraphics[width=0.8\linewidth]{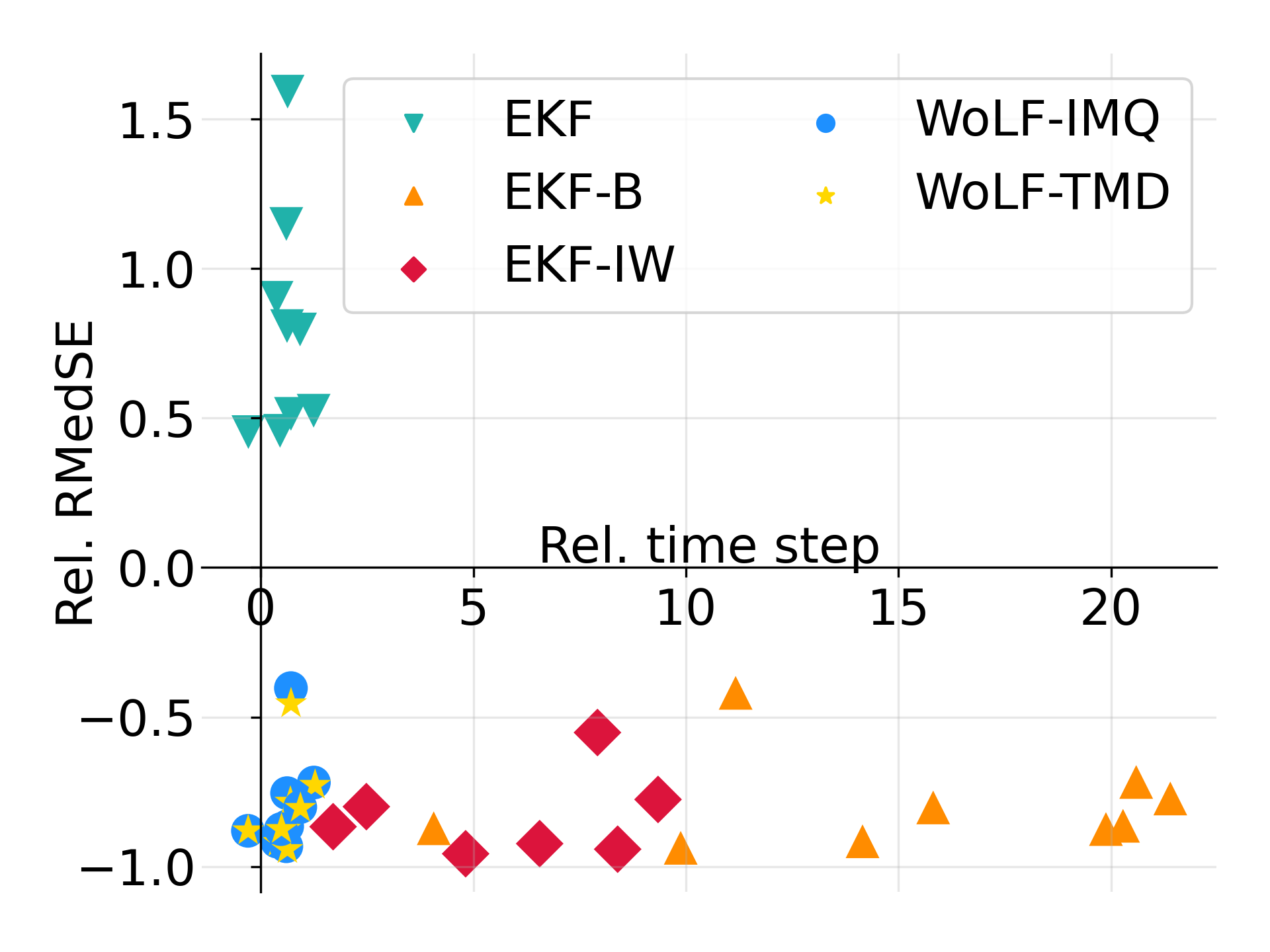}
    \caption{
    RMedSE
    versus time per step
    (relative to the \ogd minus $1$)
    across the corrupted UCI datasets.
    }
    \label{fig:uci-per-step-time}
\end{figure}
Given the computational complexity of the remaining methods, ideally, a robust Bayesian alternative to the \ogd
should be as much to the left as possible on the $x$-axis (rel. time step)
and as low as possible on the $y$-axis (rel. RMedSE).
We observe that the \mWlfImq and the \mWlfMd have both of these traits.
In particular, we observe that the only two points in the third quadrant are those of the \mWlfImq and the \mWlfMd.
Note that the \mAgamenoniExtended and the \mWangExtended have much higher relative running time and
the \mkfExtended has much higher relative RMedSE.

\subsection{Robust EKF for online MLP regression (1d)}
\label{experiment:training-neural-network}

In this section, we consider an online
nonlinear 1d regression, with the training
data coming
either from an i.i.d. source, or a correlated source.
The latter 
corresponds to a non-stationary problem.

We present a stream of observations
${\cal D}^\text{filter} = (y_1, x_1), \ldots (y_T, x_T)$ with
$y_t \in \real$ the measurements,
$x_t \in \real$ the exogenous variables, and
$T = 1500$.
The measurements and exogenous variables are sequentially sampled from the processes
\begin{equation}\label{eq:noisy-sinusoidal}
    y_t = 
    \begin{cases}
         \vtheta^*_{1} x_t  - \vtheta^*_{2}
         \cos(\vtheta^*_{3} x_t\,\pi) + \vtheta^*_{4} x_t^3 
        + V_t & \text{w.p.}\, 1 - p_\epsilon,\\
        U_t & \text{w.p. }\, p_\epsilon,
    \end{cases}
\end{equation}
where the parameters of the observation
model are $\vtheta^*=(0.2, -10, 1.0, 1.0)$,
the inputs are $x_t\sim {\cal U}[-3, 3]$, 
and the noise is
$V_t\sim {\cal N}(0, 3)$, 
$U_t \sim {\cal U}[-40, 40]$, and $p_\epsilon = 0.05$.

We consider four configurations of this experiment. In each experiment
the data is either sorted by $x_t$ value (i.e, 
the exogenous variable satisfies $x_i < x_j$ for all $i < j$,
representing a correlated source) or is
unsorted (representing
an i.i.d. source), and
the measurement function is either 
a clean version of the true data generating process
(i.e., 
\eqref{eq:noisy-sinusoidal} with $p_\epsilon = 0$ and unknown coefficients $\vtheta$), or 
a neural network with unknown parameters $\vtheta$.
Specifically, we use a  multi-layered perceptron (MLP) with two hidden layers and 10 units per layer:
\begin{equation}\label{eq:experiment-mlp}
    h(\vtheta_t, x_t)
    = \vw_t^{(3)}\phi\left(\vw_t^{(2)}\phi\left(\vw_t^{(1)}x_t + \vb_t^{(1)}\right) + \vb_t^{(2)}\right) + \vb_t^{(3)},
\end{equation}
with
activation function $\phi(u) = \max\{0, u\}$
applied elementwise.
Thus the state vector encodes the parameters:
\begin{equation*}
\begin{aligned}
    \vtheta_t = (
    \vw_t^{(1)} \in \real^{10\times 1}, \vw_t^{(2)} \in \real^{10\times 10}, \vw_t^{(3)} \in \real^{1 \times 10},
    \vb_t^{(1)}\in\real^{10}, \vb_t^{(2)} \in \real^{10},  \vb_t^{(3)} \in \real)
\end{aligned}
\end{equation*}
and has size so that $\vtheta\in\real^{141}$.
Note that in this experiment $h_t(\vtheta) = h(\vtheta, x_t)$.
We set $Q_t = 10^{-4}{\bf I}$,
which allows the parameters to slowly drift over time and provides some regularisation.

For each method, we evaluate the RMedSE.
The \mAgamenoniExtended and the \mWangExtended methods
are taken with two inner iterations,
which implies that their computational complexity is twice that of the WoLF methods.

\paragraph{MLP measurement model}

Figure \ref{fig:online-mlp-example-run-sorted} shows results
when the data are presented in sorted order of $x_t$.
We show the  performance on 100 trials.
The left panel
shows the mean prior-predictive $h(\vmu_{t | t-1}, x_t)$ of 
each method, and the underlying true state process,
for a single trial.
The right panel shows the RMedSE after multiple trials.
\begin{figure}[htb]
    \centering
    \includegraphics[width=0.45\linewidth]{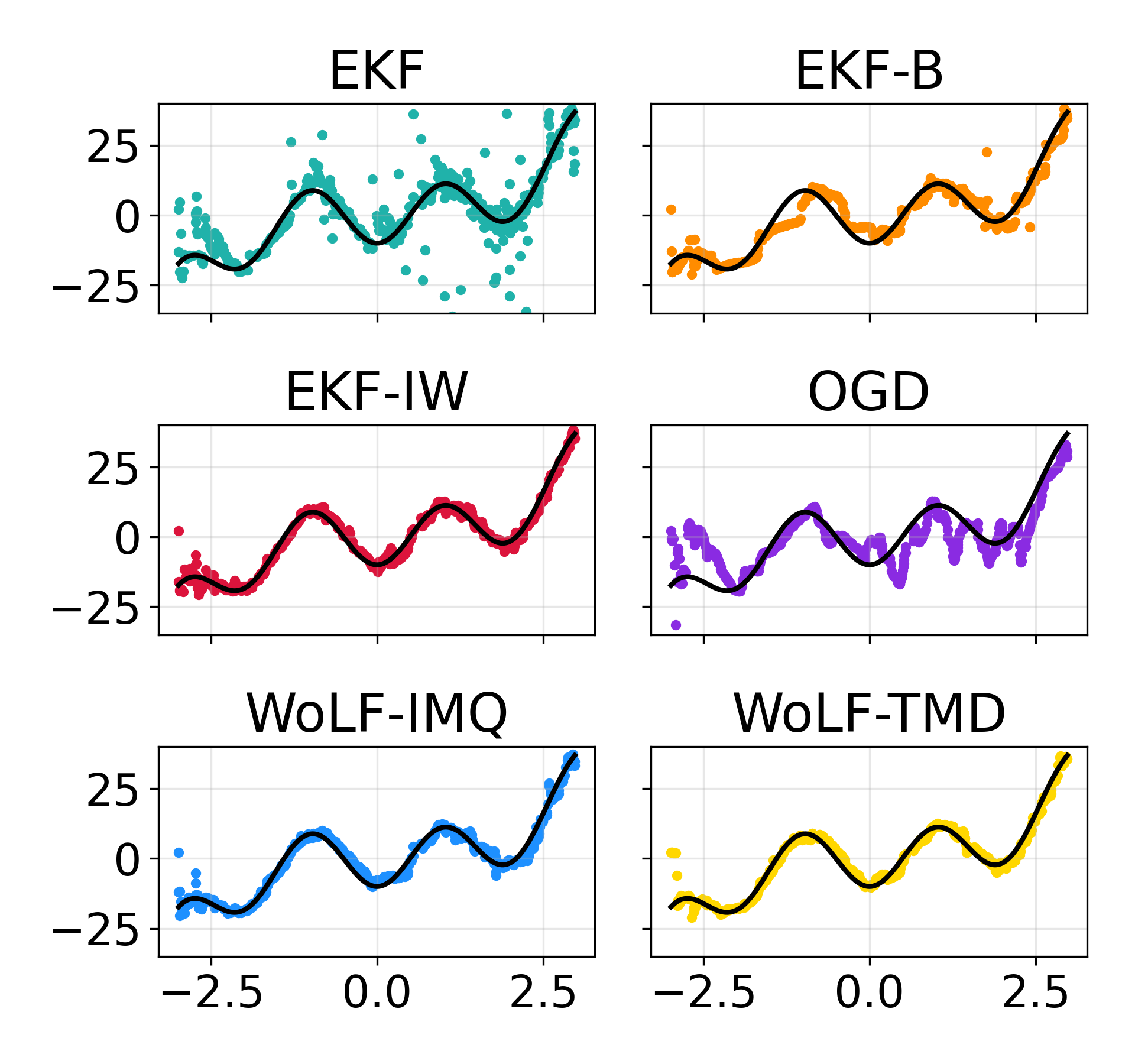}
    \includegraphics[width=0.45\linewidth]{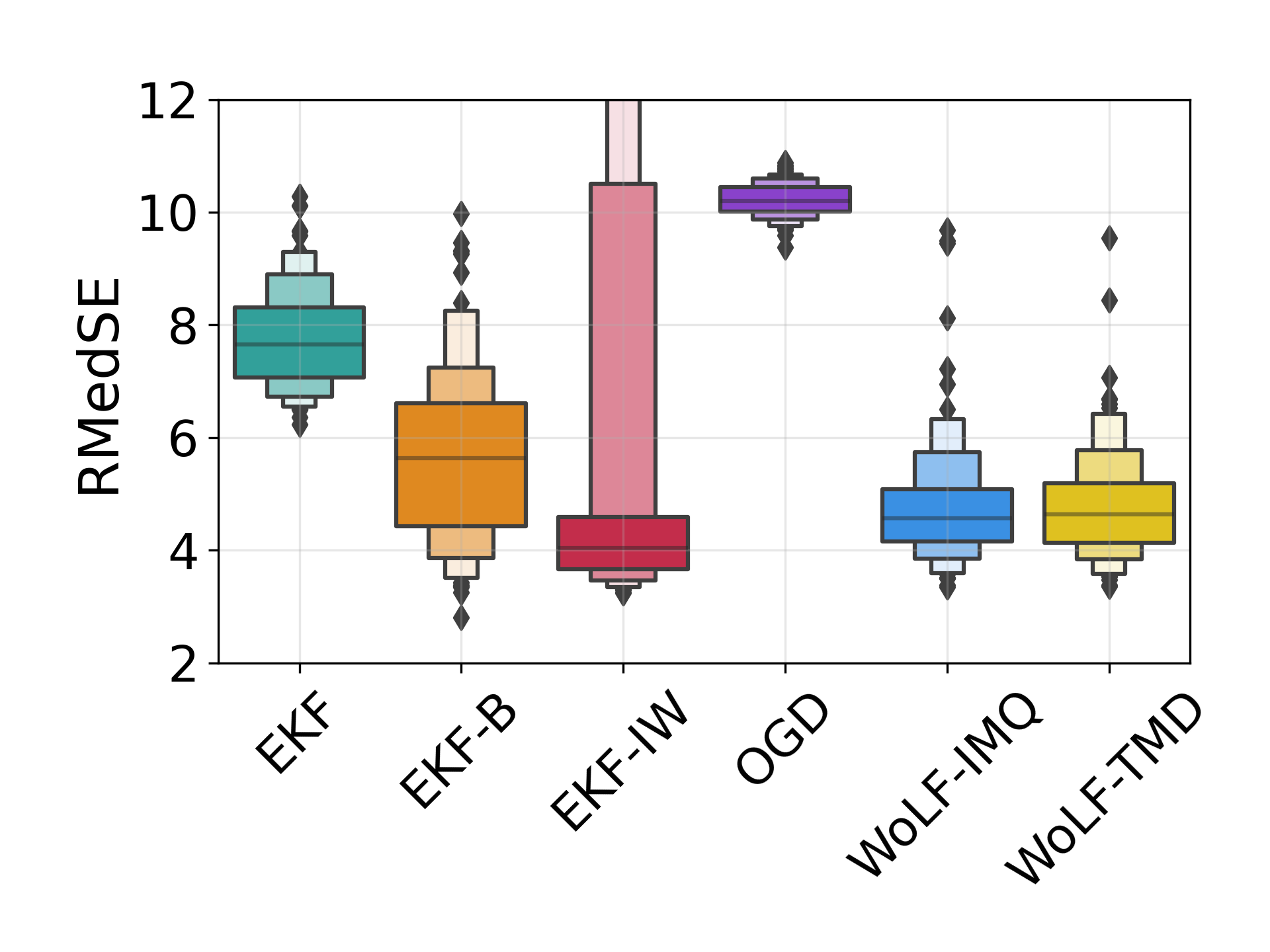}
    \caption{
    Results with sorted data.
        Left panel shows a
        run of each filter on the 1d regression,
        with the true underlying
        data-generating function in solid black line
        and the next-step  predicted observation
        as dots.
        Right panel shows the
        RMedSE distribution over multiple trials.
    }
    \label{fig:online-mlp-example-run-sorted}
\end{figure}
We observe on the right panel that the \mWlfImq and the \mAgamenoniExtended have 
the lowest mean error and lowest standard deviation among the competing methods.
However, the \mAgamenoniExtended takes twice as long to run the experiment.
For all methods, the performance worse on the left-most side of the plot on the left panel,
which is a region with not enough data to determine whether a measurement is an inlier or an outlier.

Figure \ref{fig:online-mlp-example-run-unsorted} shows the results when data are presented in random order of $\vx_t$. 
\begin{figure}[htb]
    \centering
    \includegraphics[width=0.45\linewidth]{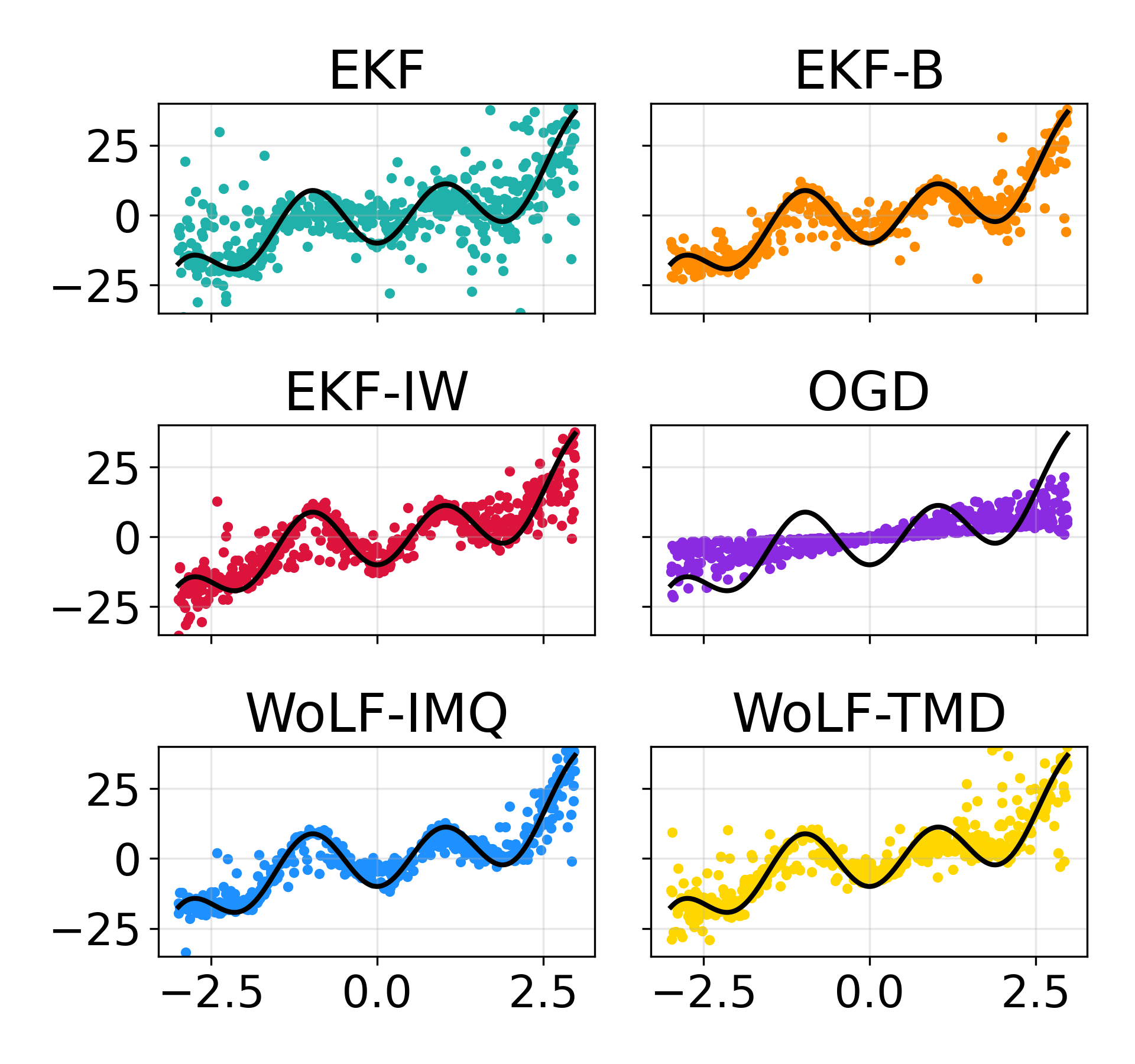}
    \includegraphics[width=0.45\linewidth]{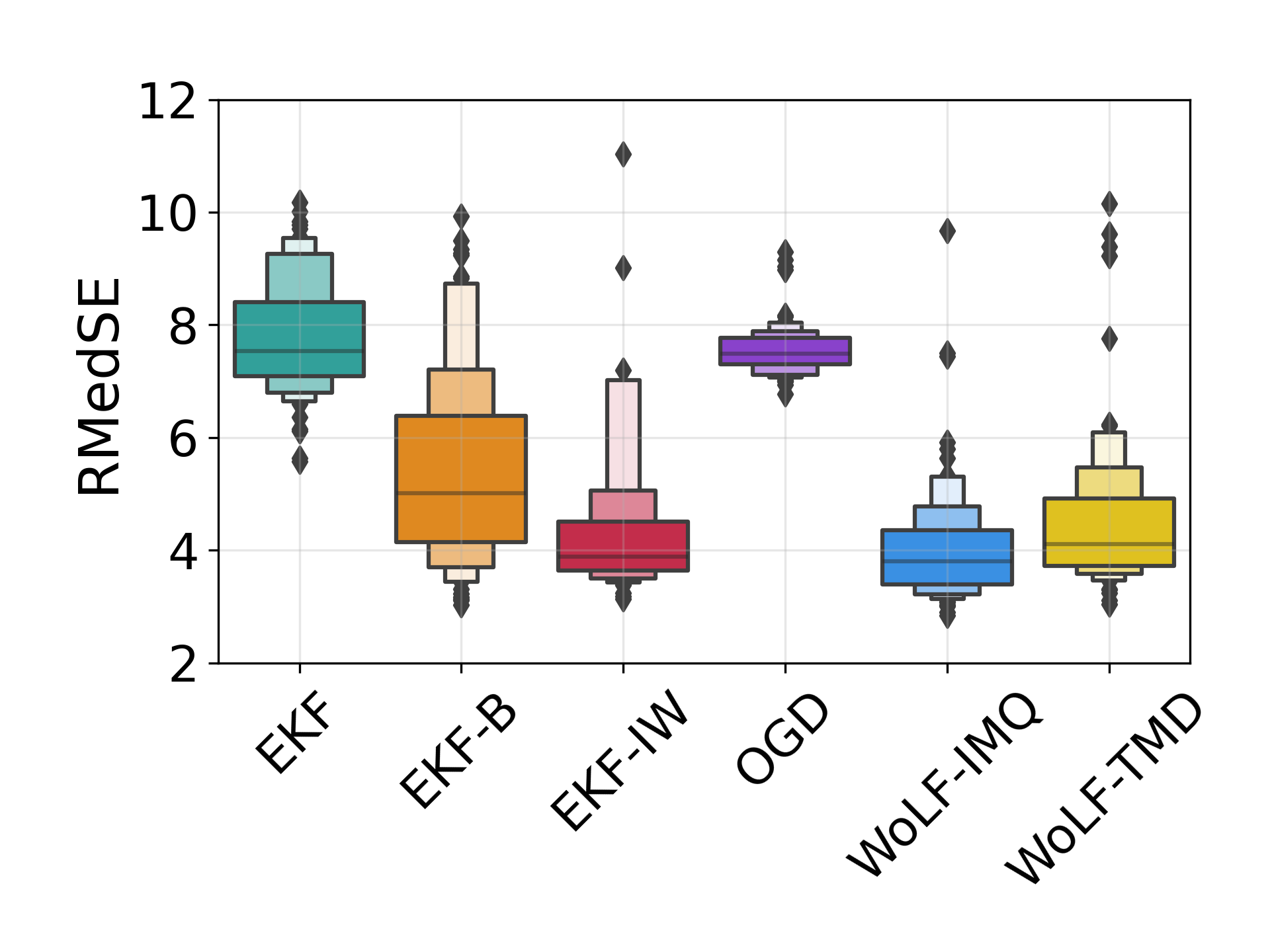}
    \caption{
    Results with unsorted inputs.
        The left panel shows a run of each filter with the underlying
        data-generating function in solid black line
        and the next-step  predicted observation
        as dots.
        The right panel shows the distribution of $G_T$ for multiple runs.
        We remove all values of $G_T$ that have a value larger than 800.
    }
    \label{fig:online-mlp-example-run-unsorted}
\end{figure}
We show results for a single run on the left panel and the the RMedSE after multiple trials on the right panel.
Similar to the sorted configuration, we observe that the \mAgamenoniExtended and the \mWlfImq are the methods with lowest RMedSE.
However, the \mAgamenoniExtended has longer tails than the \mWlfImq.

\paragraph{True measurement model}

We modify the experiment above by taking the measurement function to be
$h_t(\vtheta_t) = h(\vtheta_t, x_t) = \vtheta_{t,1}x_t  - \vtheta_{t,2}\cos(\vtheta_{t,3} x_t\,\pi) + \vtheta_{t,4}x_t^3$,
with state $\vtheta_t \in \real^4$ and $\vtheta_{t,i}$ the $i$-th entry of the state vector $\vtheta_t$.
Figure \ref{fig:sorted-unsorted-clean-measurement} shows a single
run of the filtering process when the data is presented unsorted (left panel)
and sorted (right panel).
\begin{figure}[htb]
    \centering
    \includegraphics[width=0.45\linewidth]{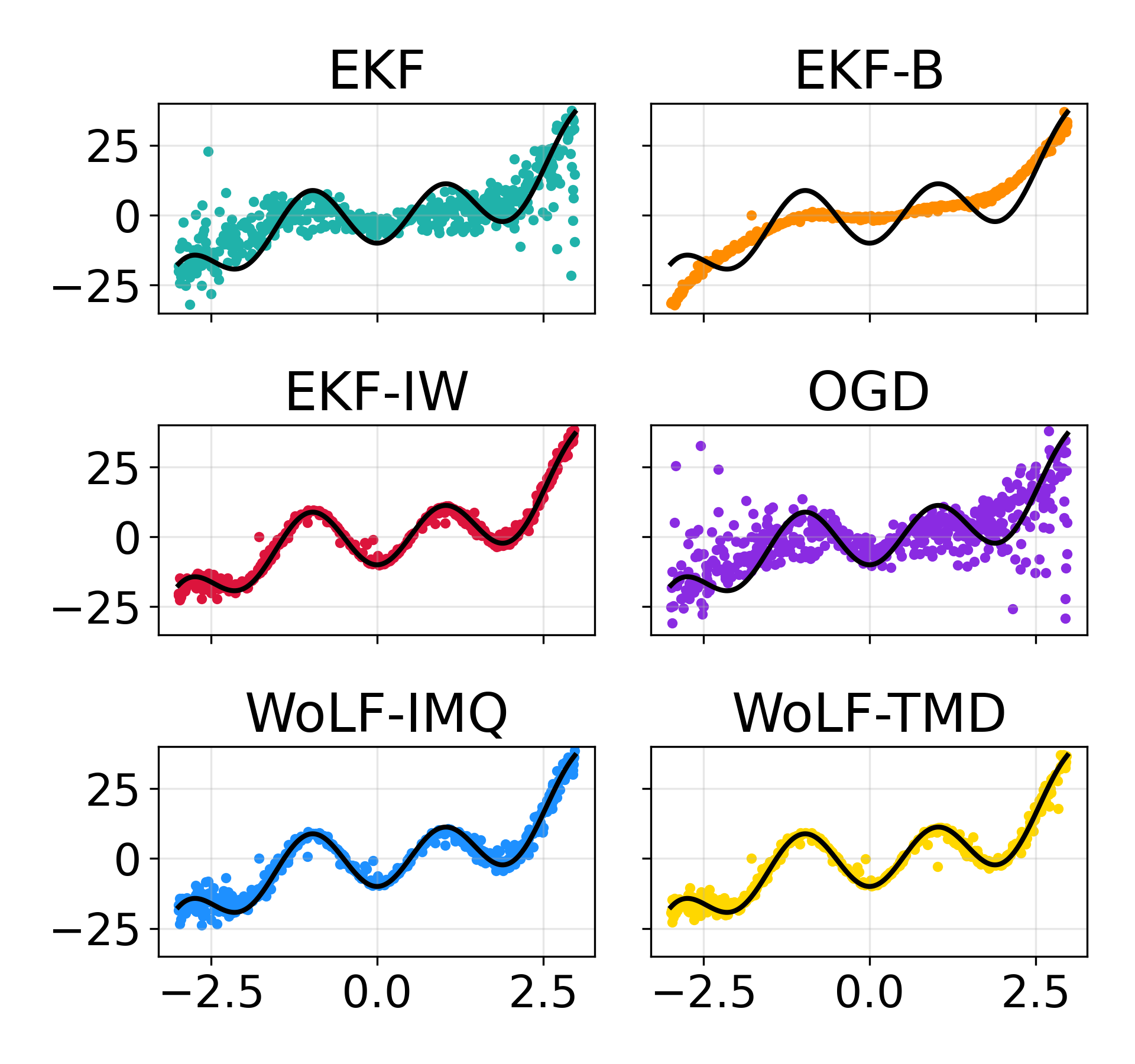}
    \includegraphics[width=0.45\linewidth]{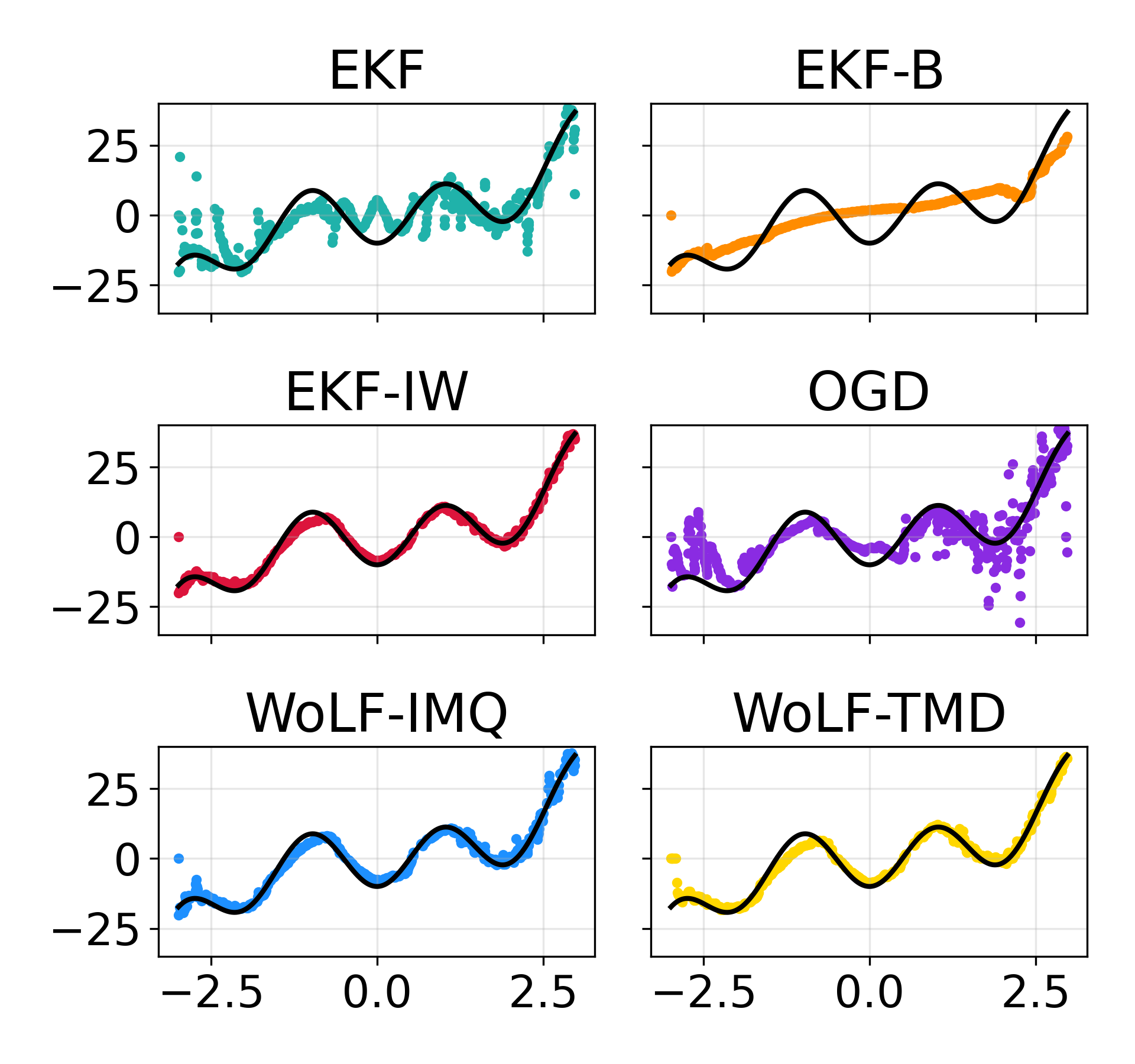}
    \caption{
        The figure shows a run of each filter with the underlying
        data-generating function in solid black line
        and the evaluation of $h(\vmu_{t|t-1}, x_t)$ in points.
        The left panel shows the configuration with unsorted $x_t$ values and
        the right panel shows the configuration with sorted $x_t$ values.
    }
    \label{fig:sorted-unsorted-clean-measurement}
\end{figure}
We observe that the behaviour of the \mWlfImq, the \mWlfMd, and the \mAgamenoniExtended have similar performance.
However, the \mAgamenoniExtended takes twice the amount of time to run.
The \ogd and the \mkfExtended are not able to correctly filter out outlier measurement at the tails.
Finally, the \mWangExtended over-penalises inliers and does not capture the curvature of the measurement process.

\subsection{Non-stationary heavy-tailed regression}
\label{experiment:heavy-tail-regression}
In this experiment, we make use of the BONE framework developed in Chapter \ref{ch:adaptivity}
and the WoLF filter discussed in this chapter
to develop an outlier-robust method for linear regression with heavy-tailed noise
and changing model parameters.

It is well-known that the combination \texttt{RL-PR} is sensitive to outliers if the choice of \cModel is misspecified,
since an observation that is ``unusual'' may trigger a changepoint unnecessarily.
As a consequence,
various works have proposed outlier-robust variants to the \RLPR[inf] for segmentation
\citep{knoblauch2018doublyrobust-bocd,fearnhead2019robustchangepoint, altamirano2023robust,sellier2023robustbocdgp}
and for filtering \citep{reimann2024changingfiltering}.
In what follows, we show how we can easily
accommodate
robust methods into the BONE framework by changing the way we compute the likelihood
and/or posterior.
In particular, we 
consider the WoLF-IMQ method of \cite{duranmartin2024-wlf}.\footnote{
We set the soft threshold value  to 4, representing four standard deviations of tolerance before declaring an outlier.
}
We use  WoLF-IMQ 
because it is a provably robust algorithm and it is a 
straightforward modification of the  linear Gaussian posterior update equations.
We denote
\RLPR[inf] with \cPosterior taken to be \texttt{LG} as \RLPRKF and
\RLPR[inf] with \cPosterior taken to be WoLF-IMQ as \RLPRWoLF.

To demonstrate the utility of a robust method,
we consider a 
piecewise linear regression model with Student-$t$ errors,
where the measurement are sampled according to
$\vx_t \sim {\cal U}[-2, 2]$,
$ \vy_t \sim {\rm St}\big( \phi(\vx_t)^\intercal\vtheta_t, 1,\,\, 2.01 \big)$
a Student-$t$ distribution with location $ \phi(\vx_t)^\intercal\vtheta_t$,
scale $1$,
degrees of freedom $2.01$, and
$\phi(\vx_t) = (1,\,x,\,x^2)$.
At every timestep, the parameters take the value
\begin{equation}
\vtheta_t =
\begin{cases}
\vtheta_{t-1} & \text{w.p. } 1 - p_\epsilon,\\
{\cal U}[-3, 3]^3 & \text{w.p. } p_\epsilon,
\end{cases}
\end{equation}
with $p_\epsilon = 0.01$, and $\vtheta_0 \sim {\cal U}[-3, 3]^3$. 
Intuitively, at each timestep, there is probability $p_\epsilon$  of a changepoint, and conditional on a changepoint occurring, the each of the entries of the new parameters $\vtheta_t$ are sampled from a uniform in $[-3,3]$. 
Figure \ref{fig:segements-tdlist-lr} shows a sample data generated by this process.

\begin{figure}[htb]
    \centering
    \includegraphics[width=0.9\linewidth]{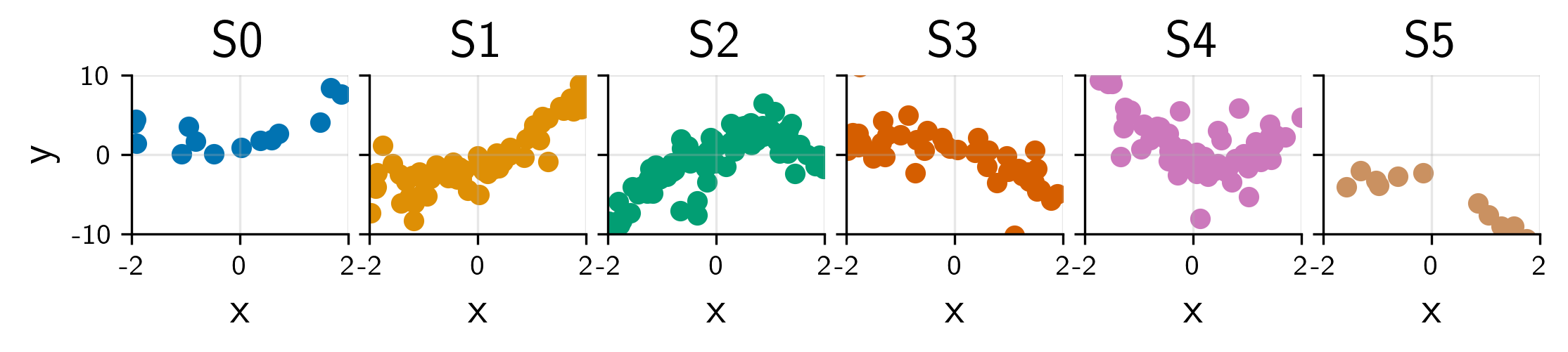}
    \caption{
        Sample run of the heavy-tailed-regression process.
        Each box corresponds to the samples within a segment.
    }
    \label{fig:segements-tdlist-lr}
\end{figure}

To process this data, our choice of \cModel is
$h(\vtheta_t, \vx_t) = \vtheta_t^\intercal\,\phi(\vx_t)$
with 
\begin{equation}
    \ell(\vy_t; \vtheta_t, \vx_t)
    = -W^2\big(\vy_t, h(\vtheta_t, \vx_t)\big)\,\log{\cal N}(\vy_t \cond h(\vtheta_t, \vx_t), 1.0),
\end{equation}
a weighted Gaussian log-likelihood and
$W(u, z) = (1 + \frac{(u - z)^2}{c^2})^{-1/2}$ the inverse multi-quadratic (IMQ) function
with soft threshold value $c=4$, 
representing four standard deviations of tolerance to outliers.
Here $u,z \in \real$.

The left panel in Figure \ref{fig:outliers-lr-res} shows
the rolling mean (with a window of size 10) of the RMSE for
\RLPRKF, \RLPRWoLF, and \staticKF.
The right panel in Figure \ref{fig:outliers-lr-res}
shows the distribution of the RMSE for all methods after 30 trials.
\begin{figure}[htb]
    \centering
    \includegraphics[width=0.48\linewidth]{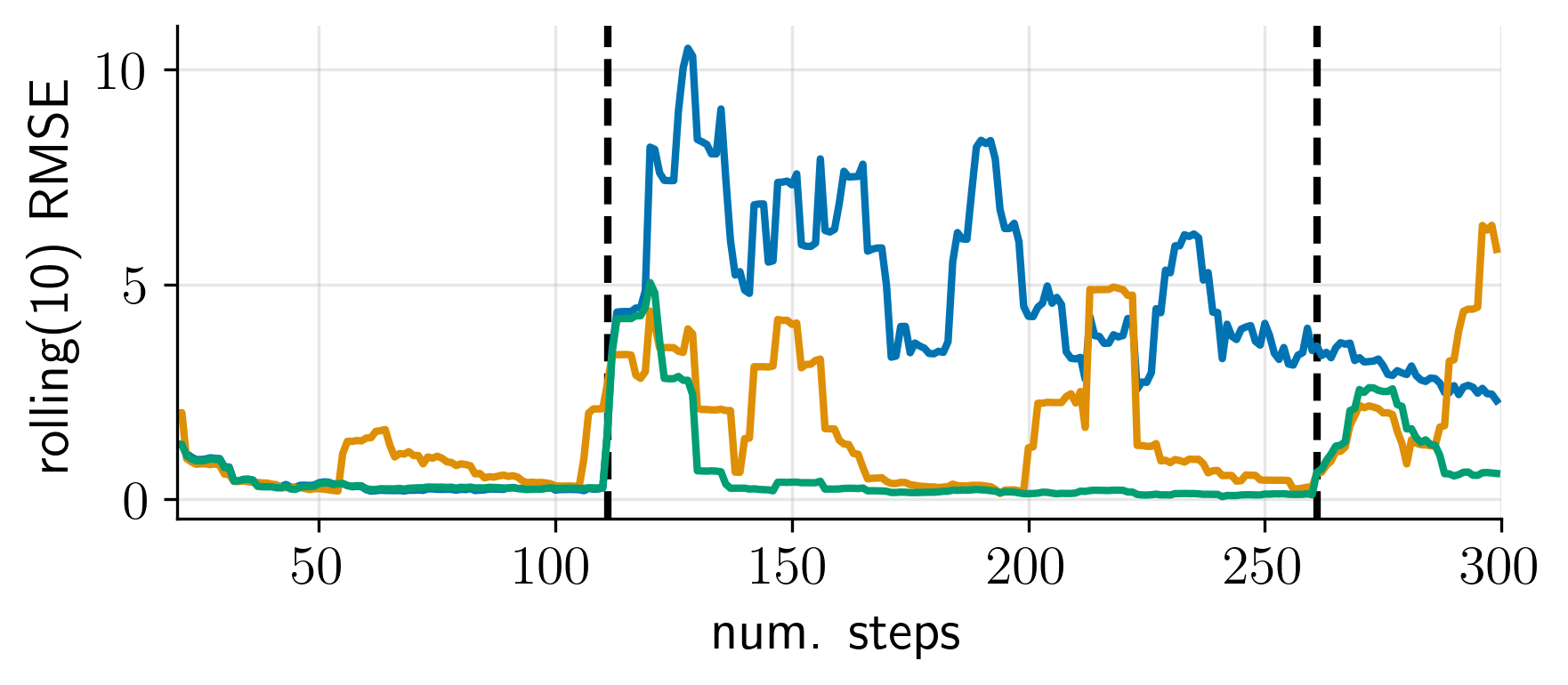}
    \includegraphics[width=0.48\linewidth]{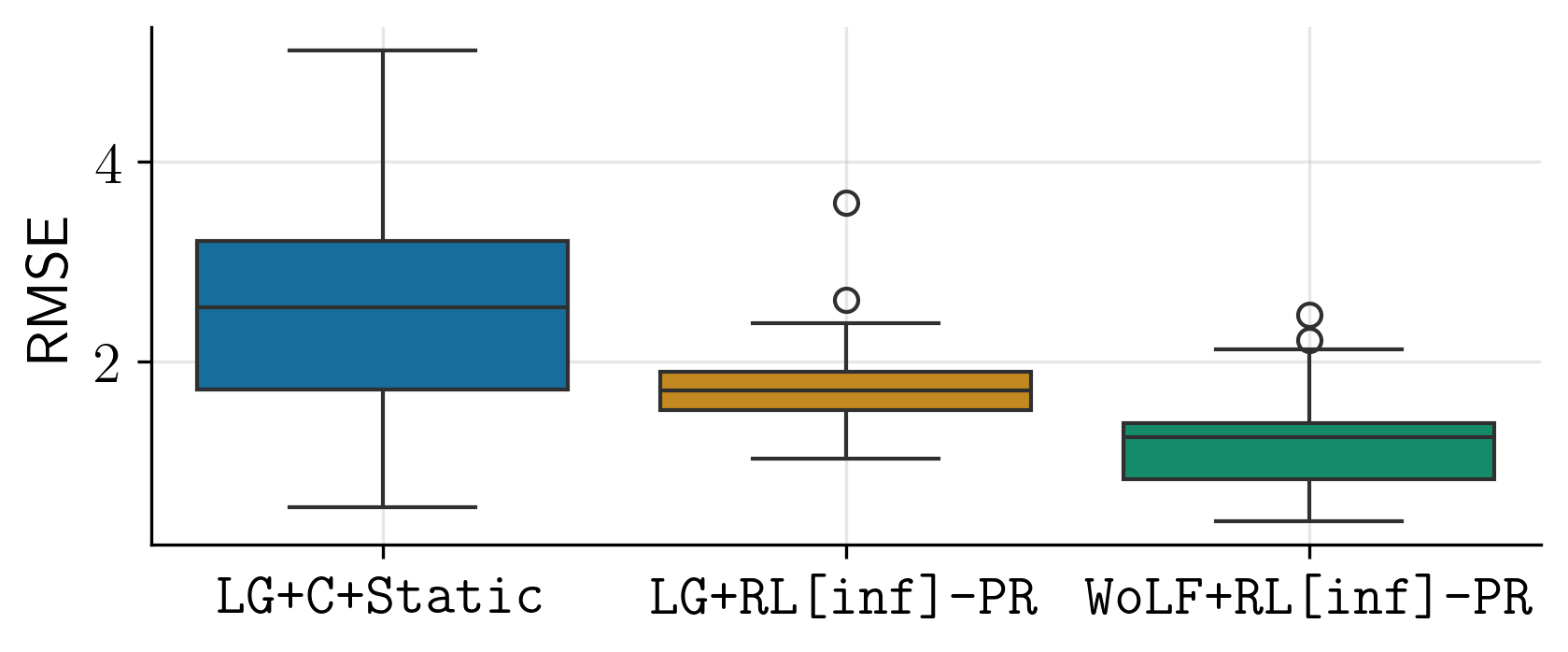}
    \caption{
    The \textbf{left panel} shows the rolling RMSE using a window of the 10 previous observations.
    The \textbf{right panel} shows the distribution of final RMSE over 30 runs.
    The vertical dotted line denotes a change in the true model parameters.
    }
    \label{fig:outliers-lr-res}
\end{figure}

The left panel of Figure \ref{fig:outliers-lr-res} shows that
\staticKF has a lower rolling RMSE error than \RLPRKF up to first changepoint (around 100 steps).
The performance of \staticKF significantly deteriorates afterwards.
Next, \RLPRKF wrongly detects changepoints and resets its parameters frequently.
This results in periods of increased rolling RMSE.
Finally, \RLPRWoLF has the lowest error among the methods.
After the regime change, its error increases at a similar rate to the other methods,
however, it correctly adapts to the regime and its error decreases soon after the changepoint.

\begin{figure}[htb]
    \centering
    \includegraphics[width=0.48\linewidth]{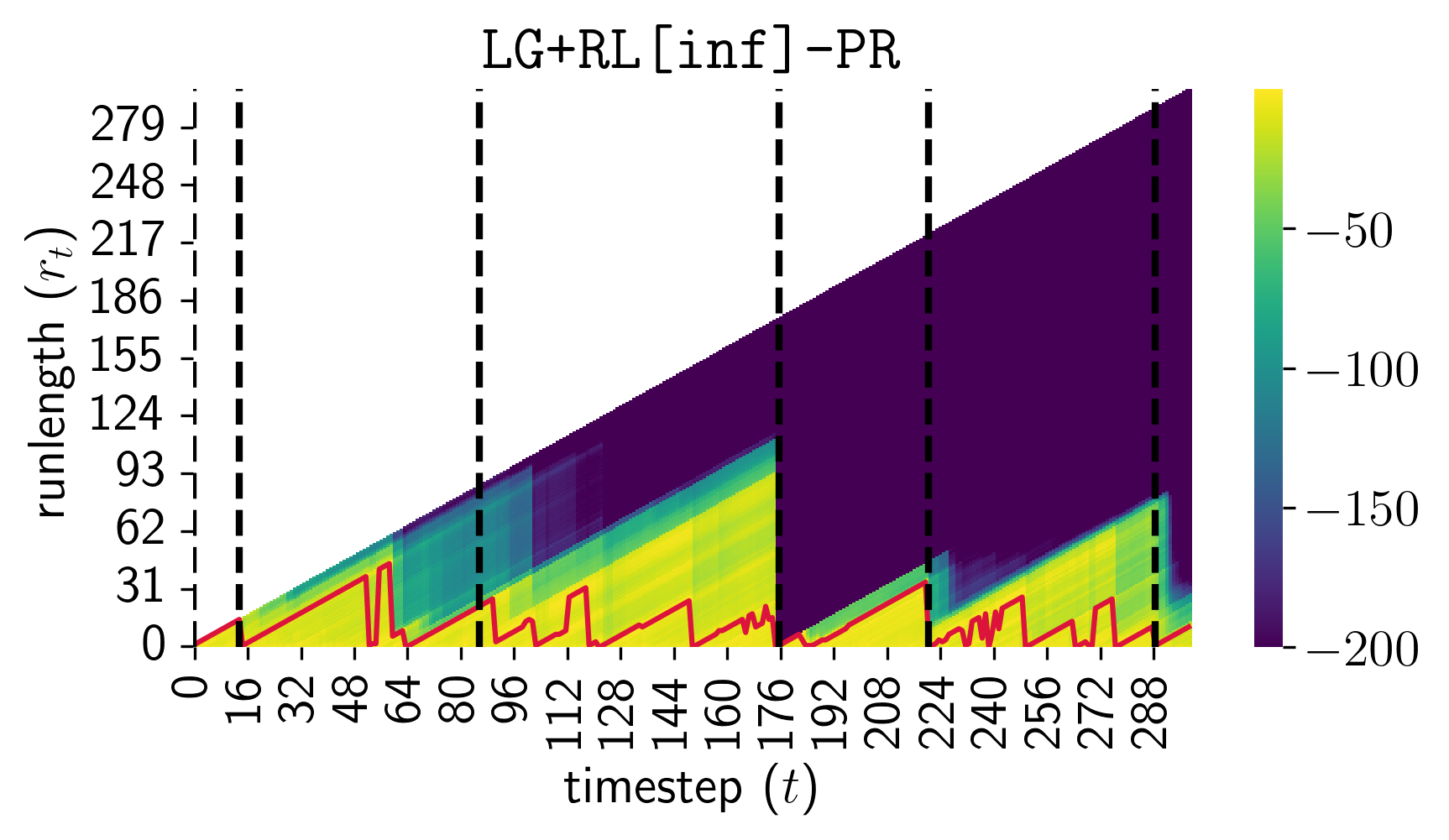}
    \includegraphics[width=0.48\linewidth]{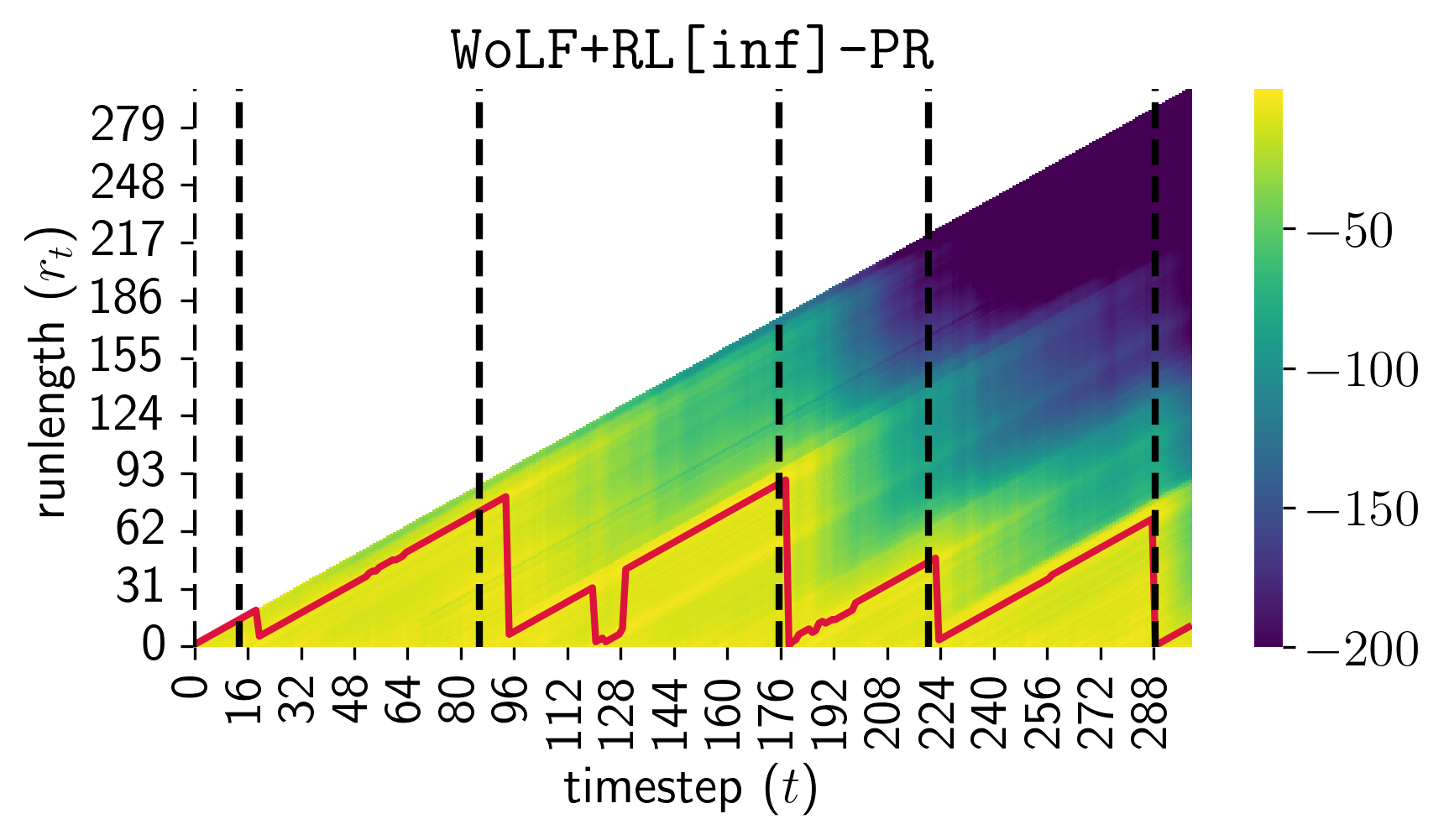}
    \caption{
        Segmentation of the non-stationary linear regression problem.
        The left panel shows the segmentation done by \RLPRKF.
        The right panel shows the segmentation done by \RLPRWoLF.
        The $x$-axis is the timestep $t$, the $y$-axis is the runlength $r_t$ (note that it is always the case that $r_t\leq t$), and the color bar shows the value $\log\,p(r_t \cond \vy_{1:t})$.
        The red line in either plot is the trajectory of the mode, i.e.,
        the set $r_{1:t}^* = \{\argmax_{r_{1}} p(r_1 \cond \data_{1}),  \ldots, \argmax_{r_{t}} p(r_t \cond \data_{1:t})\}$.
        Note that the non-robust method (left) oversegments the signal.
    }
    \label{fig:bocd-lr-stress}
\end{figure}

Figure \ref{fig:bocd-lr-stress} shows
the posterior belief of the value of the runlength  
using \RLPRKF and \RLPRWoLF.
The constant reaction to outliers in the case of \RLPRKF means that the parameters keep reseting back to the initial prior belief.
As a consequence, the RMSE of \RLPRKF deteriorates. 
On the other hand, \RLPRWoLF  resets less often, and accurately adjusts to the regime changes when they do happen. This results in the lowest RMSE among the three methods.

\subsection{Outlier-robust exponentially-weighted moving average}
\label{experiment:outlier-robust-ewma}
In this experiment,
we present an outlier-robust exponentially weighted moving average (EWMA)
as an instance of our WoLF method.\footnote{
This experiment is based on the notes 
\url{https://gerdm.github.io/posts/wolf-ewma/}.
}

This experiment is organised as follows:
first, we recap the EWMA.
Then, we introduce the unidimensional SSM with unit (unobserved) signal and observation coefficients and show that the EWMA is a special case of the KF in this setting.
Next, we derive the WoLF method for an EWMA.
Finally, we show a numerical experiment that
illustrates the robustness of the WoLF method in corrupted one-dimensional financial data.

\paragraph{The exponentially weighted moving average (EWMA)}
Given a sequence of observations (or measurements) $y_{1:t} = (y_1, \ldots, y_t)$,
the EWMA of the observations at time $t$ is given by
\begin{equation}
m_t = \beta\,y_t + (1-\beta)\,m_{t-1},
\end{equation}
where $\beta \in (0,1]$ is the smoothing factor (or learning rate).
Higher levels of $\beta$ give more weight to recent observations.

\paragraph{The Kalman filter in one dimension}
Consider the following one-dimensional SSM:
\begin{equation}\label{eq:ssm-unidimensional}
\begin{aligned}
z_t &= z_{t-1} + w_t,\\
y_t &= z_t + e_t,
\end{aligned}
\end{equation}
where $z_t$ is the (unobserved) signal, $w_t$ is the process noise, and $e_t$ is the observation noise.
We assume that ${\rm var}(w_t) = q_t$ and ${\rm var}(e_t) = r_t$.
Put simply, the SSM model $(2)$ assumes that the observations $y_t$ are generated by a (unobserved) signal $z_t$ plus noise $e_t$.
The (unobserved) signal $z_t$ evolves over time according to a random walk with noise $w_t$
and the observations $y_t$ are generated by the (unobserved) signal $z_t$ plus noise $e_t$.

\begin{proposition}\label{prop:kf-unidimensional-is-ewma}
    Under the initial density $p(z_0) = {\cal N}(z_0 \cond m_0, s_0)$ and
    measurement model $p(y_t \cond z_t) = {\cal N}(y_t \cond z_t, r_t^2)$.
    The posterior density $p(z_t \cond y_{1:t})$ is given by
    \begin{equation}
    \begin{aligned}
    p(z_t \cond y_{1:t}) &\propto p(y_t \cond z_t)\,p(z_t \cond y_{1:t-1})\\
    &= {\cal N}(z_t \cond m_t, s_t^2),
    \end{aligned}
    \end{equation}
    with
    \begin{equation}\label{eq:kf-unidimensional}
    \begin{aligned}
    k_t &= \frac{s_{t-1}^2 + q_t^2}{s_{t-1}^2 + q_t^2 + r_t^2},\\
    s_t^2 &= k_t\,r_t^2,\\
    m_t &= k_t\,y_t + (1-k_t)\,m_{t-1}.
    \end{aligned}
    \end{equation}
\end{proposition}

\begin{proof}
    Suppose $p(z_0) = {\cal N}(z_0 \cond m_0, s_0)$.
    Let
    $p(z_{t-1} \cond y_{1:t-1}) = {\cal N}(z_{t-1} \cond m_{t-1}, s_{t-1})$ and
    $p(y_t \cond z_t) = {\cal N}(y_t \cond z_t, r_t)$.
    Then, the prediction step is given by
    
    $$
    \begin{aligned}
    p(z_t \cond y_{1:t-1}) &= \int p(z_t \cond z_{t-1})\,p(z_{t-1} \cond y_{1:t-1})\,dz_{t-1}\\
    &= \int {\cal N}(z_t \cond z_{t-1}, q_t)\,{\cal N}(z_{t-1} \cond m_{t-1}, s_{t-1})\,dz_{t-1}\\
    &= {\cal N}(z_t \cond m_{t-1}, s_{t-1} + q_t)\\
    &= {\cal N}(z_t \cond m_{t-1}, s_{t|t-1}),
    \end{aligned}
    $$
    with $s_{t|t-1} = s_{t-1} + q_t$.
    Next, the update step is given by
    $$
    \begin{aligned}
    p(z_t \cond y_{1:t})
    &\propto p(y_t \cond z_t)\,p(z_t \cond y_{1:t-1})\\
    &= {\cal N}(y_t \cond z_t, r_t^2)\,{\cal N}(z_t \cond m_{t-1}, s_{t-1} + q_t^2).
    \end{aligned}
    $$
    To compute the posterior, consider the log-posterior density
    \begin{equation*}
    \begin{aligned}
    \log p(z_t \cond y_{1:t})
    &= -\frac{1}{s_{t|t-1}^2}(z_t - m_{t-1})^2 - \frac{1}{r_t^2}(y_t - z_t)^2 + {\rm const}.\\
    &= -\frac{1}{s_{t|t-1}}(z_t ^ 2 - 2z_t m_{t-1} + m_{t-1}^2) - \frac{1}{r_t^2}(y_t^2 - 2y_t z_t + z_t^2) + {\rm const}.\\
    &= -\frac{1}{s_{t|t-1}}(z_t^2 - 2z_tm_{t-1}) - \frac{1}{r_t^2}(z_t^2 - 2z_ty_t) + {\rm const}.\\
    &= -\left((s_{t|t-1}^{-2} + r_{t}^{-2})z_t^2 - 2z_t\left( \frac{m_{t-1}}{s_{t|t-1}^2} + \frac{y_t}{r_t^2} \right)\right) + {\rm const}.\\
    &= -\left(s_{t|t-1}^{-2} + r_t^{-2}\right)\left[z_t^2 - 2z_t\left(s_{t|t-1} + r_{t}^{-2}\right)^{-1}\left( \frac{m_{t-1}}{s_{t|t-1}^2} + \frac{y_t}{r_t^2} \right)\right] + {\rm const}\\
    &= -\left(s_{t|t-1}^{-2} + r_t^{-2}\right)\left[z_t - \left(s_{t|t-1} + r_{t}^{-2}\right)^{-1}\left( \frac{m_{t-1}}{s_{t|t-1}^2} + \frac{y_t}{r_t^2} \right)\right]^2 + {\rm const}
    \end{aligned}
    \end{equation*}
    where ${\rm const}$ denotes a constant that does not depend on $z_t$.
    From the above, we see that the posterior $p(z_t \cond y_{1:t})$ is a Gaussian density with mean $m_t$ and variance $s_t^2$,
    where
    $$
    \begin{aligned}
        m_t &= \left(s_{t|t-1} + r_{t}^{-2}\right)^{-1}\left( \frac{m_{t-1}}{s_{t|t-1}^2} + \frac{y_t}{r_t^2} \right),\\
        s_t^2 &= \left(s_{t|t-1}^{-2} + r_t^{-2}\right)^{-1}.
    \end{aligned}
    $$
    Next, we simplify the above expressions to obtain the Kalman filter equations \eqref{eq:kf-unidimensional}.
    For the posterior mean $m_t$, we have
    $$
    \begin{aligned}
        m_t
        &= \left(s_{t|t-1}^2 + r_{t}^{-2}\right)^{-1}\left( \frac{m_{t-1}}{s_{t|t-1}^2} + \frac{y_t}{r_t^2} \right),\\
        &= \frac{s_{t|t-1}^2r_t^2}{s_{t|t-1} + r_t^2}\left( \frac{r_t^2m_{t-1} + s_{t|t-1}^2y_t}{s_{t|t-1}^2r_t^2} \right),\\
        &= \frac{r_t^2}{s_{t|t-1}^2 + r_t^2}m_{t-1} + \frac{s_{t|t-1}^2}{s_{t|t-1}^2 + r_t^2}y_t,\\
        &= \left(1 - \frac{s_{t|t-1}^2}{s_{t|t-1}^2 +r_t^2}\right)m_{t-1} + \frac{s_{t|t-1}^2}{s_{t|t-1}^2 + r_t^2}y_t,\\
        &= (1 - k_t)m_{t-1} + k_ty_t.
    \end{aligned}
    $$
    
    with $k_t = s_{t|t-1}^2 / (s_{t|t-1}^2 + r_t^2)$.
    Finally, compute the posterior variance $s_t^2$:
    $$
        s_t^2
        = \frac{1}{s_{t|t-1}^{-2} + r_t^{-2}}
        = \frac{s_{t|t-1}^2r_t^2}{s_{t|t-1}^2 + r_t^2}
        = \left(\frac{s_{t|t-1}^2}{s_{t|t-1}^2 + r_t^2}\right)r_t^2
        = k_tr_t^2.
    $$
\end{proof}

Proposition \ref{prop:kf-unidimensional-is-ewma}
shows that the Kalman filter applied to the SSM \eqref{eq:ssm-unidimensional}
is equivalent to the EWMA with $\beta$ replaced by $k_t$,
i.e., the KF is an EWMA with a time-varying smoothing factor.

\paragraph{The WoLF method for the EWMA}
To create a 1D version of WoLF, recall that WoLF replaces the $r_t$ in the KF equations \eqref{eq:kf-unidimensional}
for $r_t^2 = r^2 / w_t^2$ with $w_t: \mathbb{R} \to \mathbb{R}$ a weight function.
Intuitively, the weight function $w_t$ determines degree of certainty that $y_t$ is an outlier.
Here, we consider the IMQ weight function
\begin{equation}
    w_t = \left(1 + \frac{(y_t - m_{t-1})^2}{c^2}\right)^{-1/2},
\end{equation}
where $c > 0$ is the soft threshold.

Next, consider the SSM \eqref{eq:ssm-unidimensional} with
$q_t^2 = q^2$ and $r_t^2 = r^2 / w_t^2$. Here $q \geq 0$ and $r > 0$ are fixed hyperparameters.
With these assumptions, the rate $k_t$ in \eqref{eq:kf-unidimensional} for WoLF becomes
\begin{equation}
k_t = \frac{s_{t-1}^2 + q^2}{s_{t-1}^2 + q^2 + r^2 / w_t^2}.
\end{equation}
As a consequence, we obtain that, as $y_t \to \infty$,
the rate $k_t$ converges to $0$ faster than $y_t$ tends to $\infty$.
We obtain
\begin{equation}
    (m_t \to m_{t-1} \text{ and } s_t^2 \to s_{t-1}^2) \text{ as } y_t \to \infty.
\end{equation}
In other words, with 1D Wolf, large and unexpected errors get discarded.
The larger the error, the less information it provides to the estimate $m_t$.

The WoLF EWMA is computed using
\begin{equation}\label{eq:wolf-1d}
\begin{aligned}
k_t &= \frac{s_{t-1}^2 + q^2}{s_{t-1}^2 + q^2 + r^2 / w_t^2},\\
s_t^2 &= k_t\,r_t^2,\\
m_t &= k_t\,y_t + (1-k_t)\,m_{t-1}.
\end{aligned}
\end{equation}

\paragraph{A robust EWMA for log-returns}
To test the 1d-WoLF,
consider data from the Dow Jones Industrial Average (DJI) from 2019 to the end of 2024.
We want to estimate the EWMA of log-returns for DJI.
Suppose that the data is corrupted with outliers and we do know 
in advance the level of corruption or their occurrence.
Figure \ref{fig:dji-returns} shows the log-returns DJI from  2020 to 2024.
\begin{figure}[htb]
    \centering
    \includegraphics[width=0.9\linewidth]{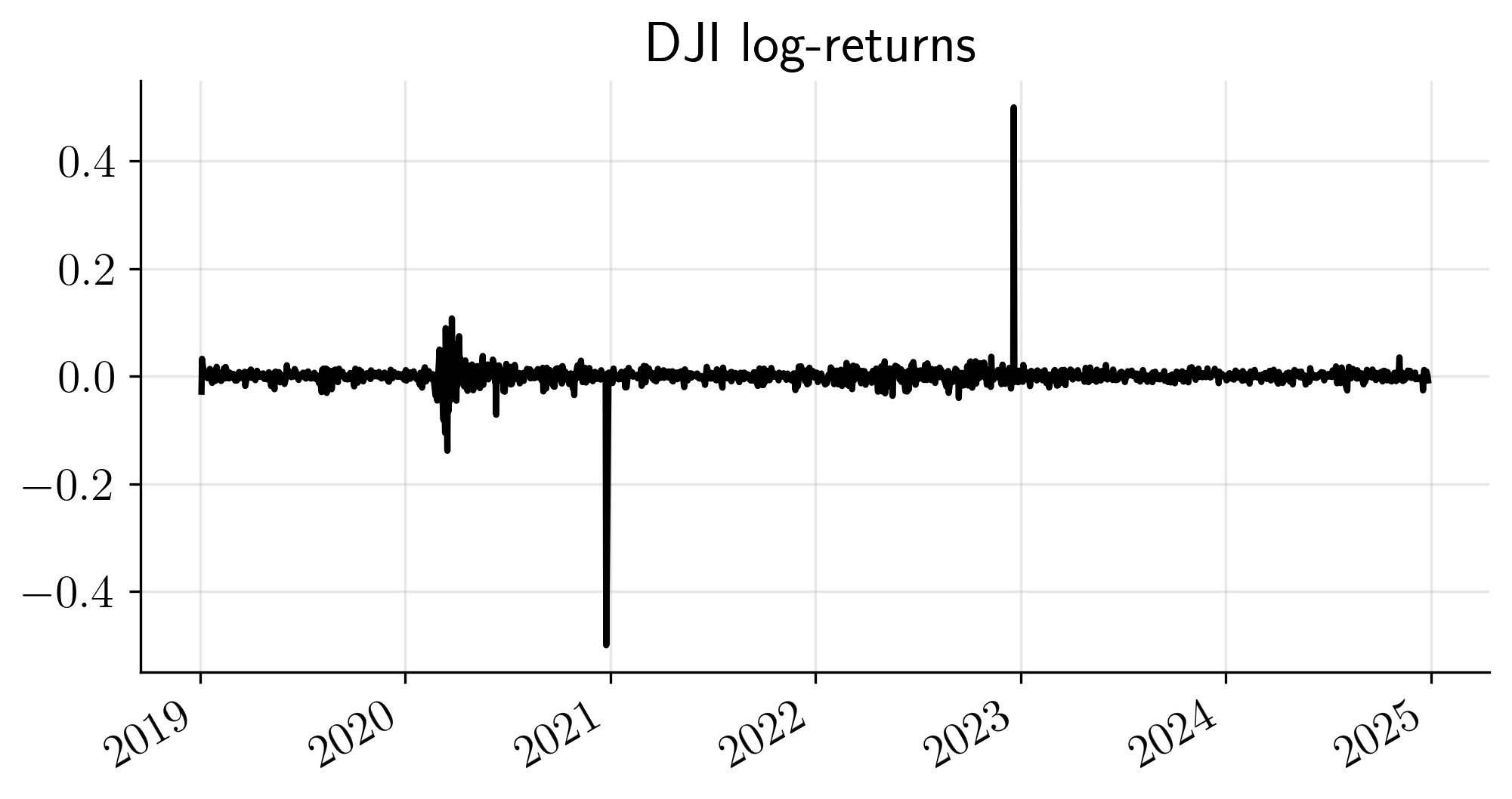}
    \caption{
        Log-returns of DJI from 2019 to 2024.
        The outliers at the beginning of 2021 and at the beginning of 2023
        correspond to erroneous datapoints.
    }
    \label{fig:dji-returns}
\end{figure}

Next, Figure \ref{fig:dji-ewma-methods} shows the EWMA estimate and the WoLF-EWMA estimate
with hyperparameters $\beta=0.095$ for the EWMA and
$m_0=0$, $s_0=1$, $q=0.01$, $r=1.0$, and $c=0.05$ for WoLF-EWMA.
\begin{figure}[htb]
    \centering
    \includegraphics[width=0.9\linewidth]{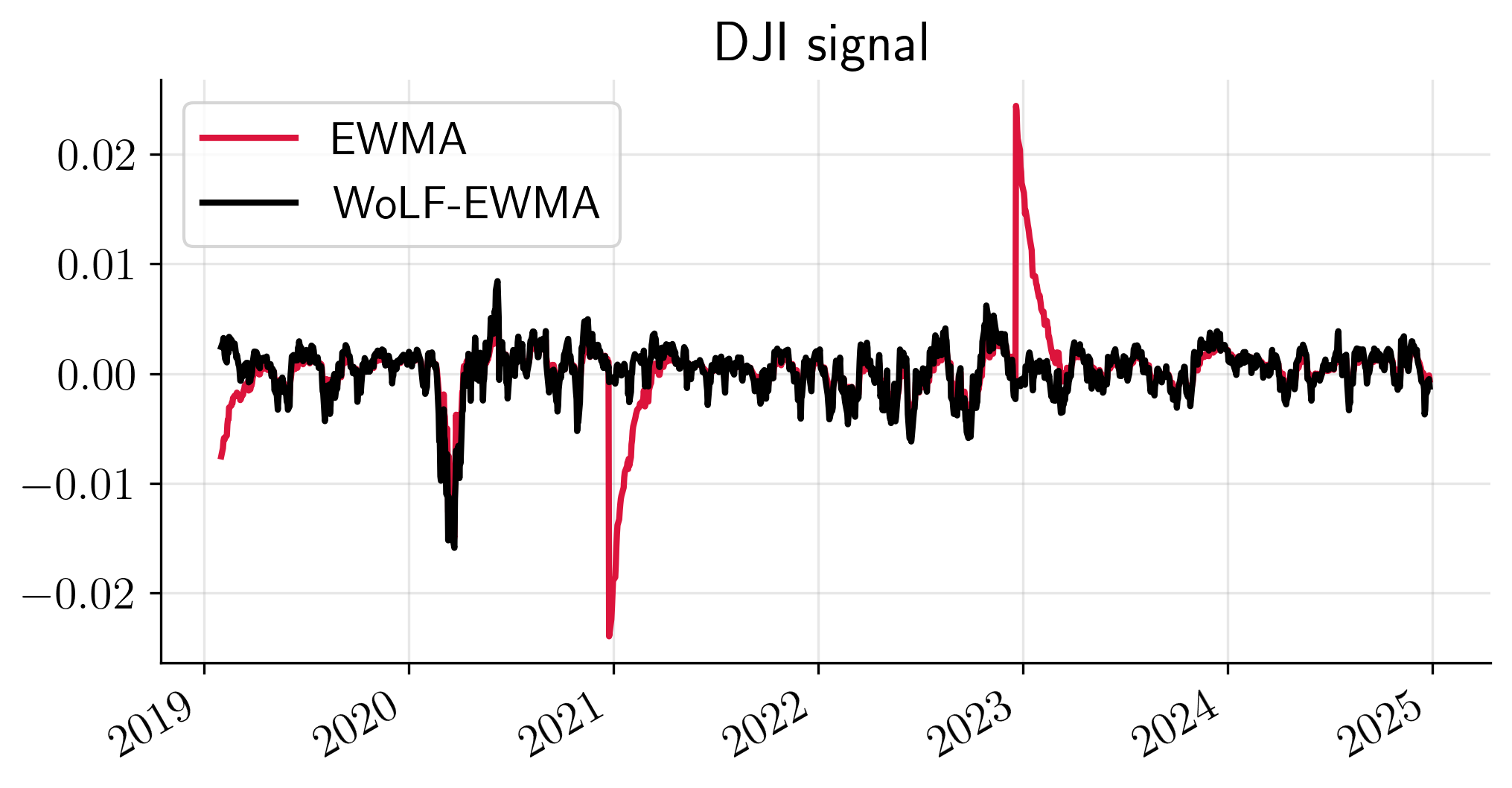}
    \caption{
        EWMA and WolF-EWMA estimates over
        DJI log-returns from 2019 to 2024.
    }
    \label{fig:dji-ewma-methods}
\end{figure}
We observe that under the standard EWMA, an outlier event significantly biases the posterior estimate of the signal $z_t$.
However, WoLF-EWMA ignores the outliers and provides a more robust estimate of the DJI signal.
Furthermore, the WoLF-EWMA estimate closely resembles that of EWMA outside \textit{outlier events}.
Finally, Figure \ref{fig:dji-learning-rate} shows the smoothing factors $\beta_t$ for the EWMA and the $k_t$.
\begin{figure}[htb]
    \centering
    \includegraphics[width=0.9\linewidth]{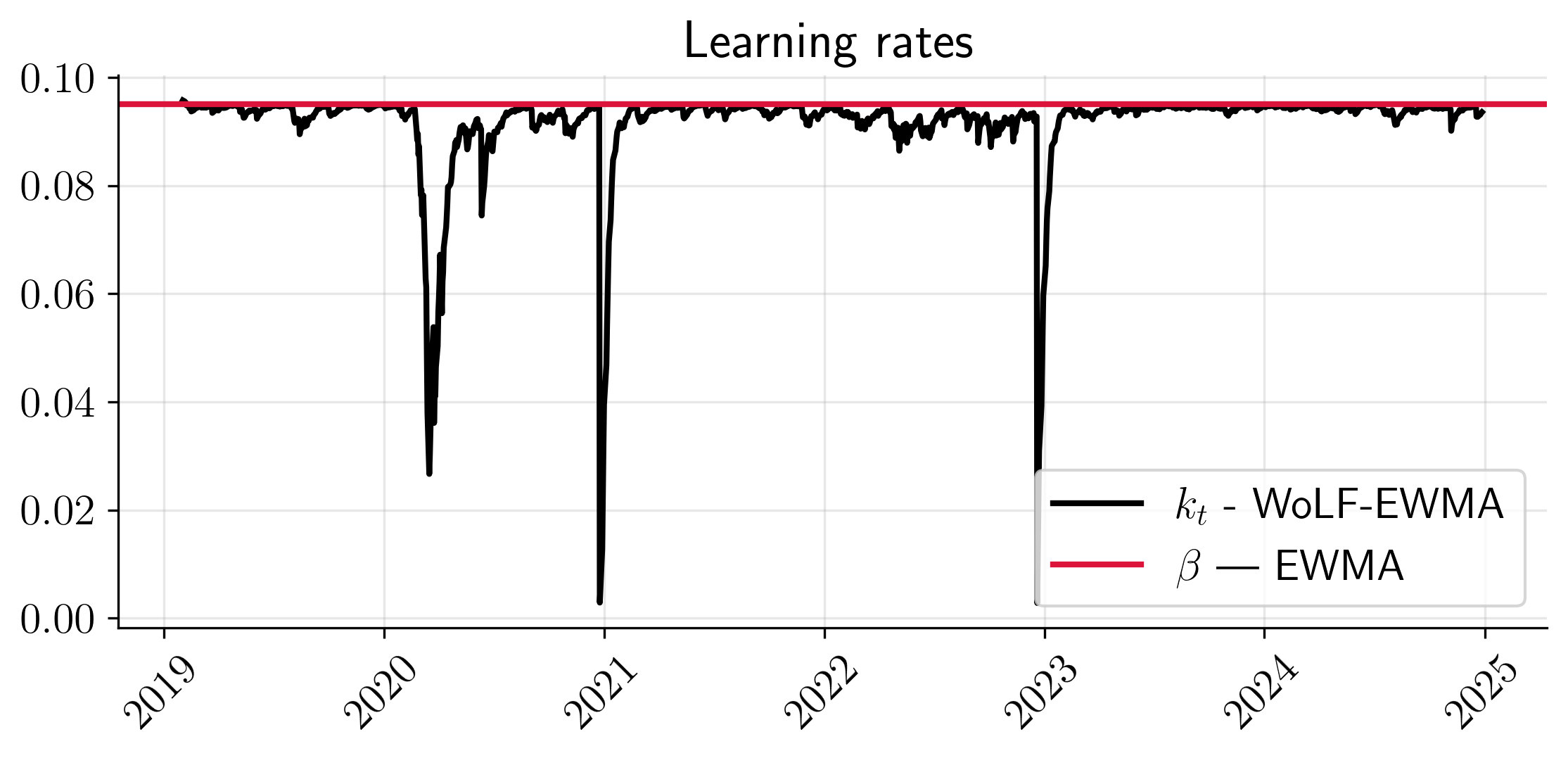}
    \caption{
        Smoothing factors for
        EWMA and WolF-EWMA for the 
        DJI log-returns from 2019 to 2024.
    }
    \label{fig:dji-learning-rate}
\end{figure}
We observe that the smoothing factor for WolF-EWMA $k_t$ is fairly stable and closely follows the fixed value $\beta$ for EWMA.
However, it decreases when the observation $y_t$ is is unusually large.
This occurs at the outlier events.

\section{Conclusion}
We introduced a provably robust filtering algorithm based on generalised Bayes
which we call the weighted observation likelihood  filter or WoLF.
Our algorithm is as fast as the KF, has closed-form update equations, and is
straightforward to apply to various filtering methods.
The superior performance of the WoLF is shown on a wide range of filtering problems.
In contrast, alternative robust methods either have higher computational complexity than the WoLF, or
similar computational complexity but not higher performance.

\chapter{Scalability}
\label{ch:scalability}


Throughout this thesis, we have assumed that the posterior density over model parameters is
Gaussian with a full-rank covariance matrix.
This assumption, while simple and powerful,
results in significant computational and memory demands.
Specifically,
most of the methods introduced in Chapter \ref{ch:recursive-bayes} scale at rates of at least $O(D^2)$ in memory and $O(D^3)$ in computation time.
This limitation makes such methods impractical for online learning in high-dimensional parameter spaces,
as is common with deep neural networks --- the $O(D^2)$ memory requirement alone makes them infeasible for training even moderately sized neural networks.

For example,  
consider a multi-layer perceptron (MLP) with $28^2$ input units, three hidden layers, and 100 units in each hidden layer.
Such a network contains approximately 99,000 parameters.  
Representing these parameters as 32-bit floats would require around 40GB of memory  
for a single step of the Extended Kalman Filter (EKF) algorithm (see Section \ref{sec:extended-kalman-filter}).  
This makes it infeasible to run on standard GPU devices and impractical for real-time applications.  

To address these limitations, this chapter focuses on the challenge of recursively training deep neural networks
using filtering techniques.
We propose three scalable algorithms that build on the methods introduced in Chapter \ref{ch:recursive-bayes}.  
Each algorithm leverages different structural assumptions about the state space  
and applies approximations to reduce both the computational and memory requirements.  

First, in Section \ref{sec:subspace}, we present the method introduced in \cite{duranmartin2022-subspace-bandits},
which projects the weights of the neural network onto a lower-dimensional affine subspace
and performs EKF-like updates on the reduced space.
Next, Section \ref{sec:PULSE} introduces the method proposed in \cite{cartea2023sharpbayes},
which builds upon the subspace approach and leverages ideas from the ``last-layer'' literature in Bayesian neural networks.
Specifically, this method projects the hidden layers of a neural network onto a lower-dimensional affine space and works with a full-covariance matrix for the last layer.
Then,
Section \ref{sec:lofi} describes the low-rank Kalman filter (LoFi) method
introduced in \cite{chang2023lofi}.
This method imposes a diagonal-plus-low-rank structure on the posterior precision matrix.
This results in EKF-like updates.
In contrast to the previous two methods,
LoFi updates all the parameters of the neural network but reduces the computational cost
by maintaining a diagonal plus low-rank posterior precision matrix.

Finally, Section \ref{sec:scalability-summary} provides an overview of the methods
presented in this chapter and
Section \ref{sec:experiments-scalability}
presents an empirical study of the performance of
the proposed algorithms on an online classification task using the Fashion MNIST dataset \citep{xiao2017fashionmnist}.

\section{Subspace parameters}\label{sec:subspace}

There is a growing literature showing that the number of parameters (or degrees of freedom)
required to fit a neural network is much smaller than the total number of parameters contained
in the neural network architecture.
These \textit{active} parameters of a neural network can either be found as a subset of nodes
\citep{frankle2023lotteryhypothesis},
or contained within a lower-dimensional linear subspace \citep{li2018subspace,larsen2022subspaceNets}.
This observation is called the lottery-ticket hypothesis.

Research on the lottery-ticket hypothesis exploits the over-parametrisation of neural networks
in the sense that
``a randomly-initialised dense neural network contains subnetworks (or linear subspaces) that,
when trained in isolation,
reach test accuracy comparable to that of the original network'' \citep{frankle2023lotteryhypothesis}.
The subnetworks or linear subspaces that satisfy the lottery ticket hypothesis are called \textit{winning tickets}.

The work in \cite{duranmartin2022-subspace-bandits}
exploits the lottery-ticket hypothesis by projecting the weights of the neural network
to a lower-dimensional affine subspace.
Then, they make use of the EKF to perform sequential updates
over the lower-dimensional space.
We provide details below.

\subsection{The SSM for the subspace EKF and update step}

Consider a \textit{projection} matrix $\vA\in\real^{\dimstate\times\dimstatesub}$,
with $d \ll D$.
Write the parameters of the neural network as a linear mapping between the projection matrix $\vA$
and a lower-dimensional vector $\phiddensub\in\real^\dimstatesub$,
plus an \textit{offset} vector $\vtheta_* \in \real^\dimstate$:
\begin{equation}\label{eq:subspace-params}
    \vtheta(\phiddensub) = \vA\,\phiddensub + \vtheta_*.
\end{equation}
The lower-dimensional vector $\phiddensub$ is assumed to be time-dependent
and evolving according to the equation $\phiddensub_t = f(\phiddensub_{t-1}) + \vu_t$,
with $\vu_t$ a zero-mean random variable with known covariance matrix.\footnote{
The term \eqref{eq:subspace-params} is different from techniques such as
LoRA \citep{hu2021lora} that reduce the parameters in each layer using a linear projection.
Here, we project all of the parameters of the network onto a single lower-dimensional space.
}
Under assumption \eqref{eq:subspace-params},
the SSM over model parameters and observations becomes
\begin{equation}\label{eq:subspace-ssm}
\begin{aligned}
    \phiddensub_t &= f(\phiddensub_{t-1}) + \vu_t,\\
    \vtheta_t &:= \vtheta(\phiddensub_t) = {\bf A}\,\phiddensub_t + \vtheta_*,\\
    \vy_t &= h(\vtheta_t, \vx_t) + \ve_t.
\end{aligned}
\end{equation}

The SSM \eqref{eq:subspace-ssm} corresponds
to that presented in Section \ref{sec:extended-kalman-filter} with measurement function
$h(\vA\,\phiddensub_t + \vtheta_*, \vx_t)$.
Assuming Gaussian prior for the subspace ${\cal N}(\phiddensub_0 \cond \vmu_0, \vSigma_0)$,
application of the EKF is straightforward.
Algorithm \ref{algo:subspace-extended-kalman-filter}
shows a step of the EKF under the subspace assumption.
For simplicity, we assume that $f(\phiddensub) = \phiddensub$.

\begin{algorithm}[htb]
\begin{algorithmic}[1]
    \REQUIRE $\data_t = (\vx_t, \vy_t)$ // datapoint
    \REQUIRE $(\vmu_{t-1}, \vSigma_{t-1})$ // previous $\dimstatesub$-dimensional mean and covariance
    \REQUIRE $\vA \in \real^{\dimstate\times\dimstatesub}$ // projection matrix
    \REQUIRE $\vtheta_*\in \real^\dimstate $ // offset vector
    \STATE  // predict step
    \STATE $\vF_t \gets \nabla_\vtheta f(\vmu_{t-1})$
    \STATE $\vmu_{t|t-1} \gets \vF_t\vmu_{t-1}$
    \STATE $\vSigma_{t|t-1} \gets \vF_t\vSigma_{t|t-1}\vF_{t}^\intercal + \vQ_t$
    \STATE //update step
    \STATE $\vH_t \gets \nabla_\vtheta h(\vA\,\vmu_{t|t-1} + \vtheta_*, \vx_t)$
    \STATE $\vS_t = \vH_t\,\vA\,\vSigma_{t-1}\,\vA^\intercal\,\vH_t^\intercal + \vR_t$
    \STATE $\vK_t = \vSigma_{t|t-1}\,\vA^\intercal\,\vH_t^\intercal\,\vS_t^{-1}$
    \STATE $\vmu_t \gets \vmu_{t-1} + \vK_t(\vy_t - \vH_t\,\vA\,\vmu_{t|t-1})$
    \STATE $\vSigma_t \gets \vSigma_{t|t-1} - \vK_t\,\vA^\intercal\,\vH_t^\intercal\vSigma_{t|t-1}$
    \RETURN $(\vmu_t, \vSigma_t)$
\end{algorithmic}
\caption{
    predict and update steps for the subspace extended Kalman filter
    with subspace for $t \geq 1$.
}
\label{algo:subspace-extended-kalman-filter}
\end{algorithm}

\subsection{Warmup phase: estimating the projection matrix and the offset term}
\label{subsec:warm-up-subspace}
Two important components of the method are the projection matrix $\vA$ and the offset vector $\vtheta_*$.
The offset can either be initialised from a random process, as in \cite{li2018subspace},
or determined during a \textit{warmup} period, as described in \cite{larsen2022subspaceNets}.
In this section, we focus on the approach proposed in \cite{larsen2022subspaceNets},
where the projection matrix $\vA$ is constructed by performing a singular value decomposition (SVD)
on iterates of batch stochastic gradient descent (SGD),
and the offset $\vtheta_*$ is taken as the final step of the warmup procedure.

Suppose we are given a warm-up dataset $\dwarmup$
with $N_{\rm warmup} \geq 1$ datapoints,
which we divide into into $B$ 
non-intersecting random batches
$\data_{(1)}, \ldots, \data_{(B)}$ such that
$$
\bigcup_{b=1}^B \data_{(b)} = \dwarmup.
$$

Consider the negative log-likelihood
\begin{equation}\label{eq:neg-log-likelihood}
    -\log p(\data_{\rm warmup} \cond {\bm\theta}) = -\sum_{n=1}^{N_{\rm warmup}} \log p(\vy_n\cond{\bm\theta}, \vx_n).
\end{equation}
And train for $E$ epochs using randomised batch SGD.
At the end of the $E$ epochs, we obtain ${\bm\theta}^{(E)}$.
Then, the offset vector $\vtheta_*$ is given by
\begin{equation*}
    \vtheta_* = \vtheta^{(E)}.
\end{equation*}
Next, to estimate the projection matrix $\vA$,
we take the iterates found during the SGD optimisation procedure.
To avoid redundancy, we
skip the first $n$ iterations and store the iterates every $k$ steps.
Let
\begin{equation*}
{\cal E} = \begin{bmatrix}
    \horzbar & \vtheta^{(n)} & \horzbar\\
    \horzbar & \vtheta^{(n + k)} & \horzbar\\
    \horzbar & \vtheta^{(n + 2k)} &\horzbar\\
    & \vdots & \\
    \horzbar & \vtheta^{(E)} & \horzbar\\
\end{bmatrix} \in \real^{\hat{E}\times D}\,,
\end{equation*}
where $\hat{E} = \lfloor (E - n) / k \rfloor +1$.
With the SVD decomposition ${\cal E} = {\bf U}\,\vSigma\,{\bf V}$ and the first $d$ columns of the matrix ${\bf V}$,
the projection matrix is
\begin{equation*}
    {\bf A} = \begin{bmatrix}
    \vertbar & \vertbar &  & \vertbar \\
    {\bf V}_{:,1} & {\bf V}_{:,\dimstatesub} & \ldots & {\bf V}_{:,\dimstatesub} \\
    \vertbar & \vertbar & & \vertbar \\
    \end{bmatrix},
\end{equation*}
where ${\bf V}_{:,\,k}$ denotes the $k$-th column of ${\bf V}$. 
Algorithm \ref{algo:MAP-SGD}
shows the warmup procedure to find $\vA$ and $\vtheta_*$ using 
a warmup dataset $\data_{\rm warmup}$.
In Algorithm \ref{algo:MAP-SGD},
the function $\bm\kappa: \real^M \to \real^M$ is the per-step transformation of the Adam algorithm;
see \cite{kingma2014adam}.
\begin{algorithm}[H]
\begin{algorithmic}[1]
    \REQUIRE $\vtheta^{(0)}$ // initial set of parameters
    \REQUIRE $E \geq 1$ // number of epochs
    \REQUIRE $\data_{\rm warmup}$ // warmup dataset
    \REQUIRE ${\cal E} = []$ // empty matrix
    \REQUIRE $n \geq 0$, $k \geq 1$ // skip and stride numbers
    \FOR{epoch $e=1,\ldots,E$}
        \FOR{batch $m=1, \ldots, M$}
            \STATE $G_e \gets -\nabla_{\vtheta}\log p(\data_{(m)} \cond \vtheta)$
            \STATE ${\vtheta}^{(e)} \gets {\vtheta}^{(e-1)} - \vkappa(G_T)$
        \ENDFOR
        \IF{$(e \geq n)$ \AND $(e \mod k = 0)$}
            \STATE ${\cal E} \gets \left[{\cal E}^\intercal, \left(\vtheta^{(e)}\right)^{\intercal}\right]^\intercal$ // stack vertically
        \ENDIF
    \ENDFOR
    \STATE ${\bf U}\,\vSigma\,{\bf V} \gets {\rm SVD}({\cal E})$ // SVD decomposition of stacked SGD iterates
    \STATE $\vA \gets \begin{bmatrix} \vV_{:,\,1} & \ldots & \vV_{:,\,d} \end{bmatrix}$ // define projection matrix
    \STATE $\vtheta_* \gets \vtheta^{(E)}$ // define offset parameters
\end{algorithmic}
\caption{
    Initialisation of the offset vector and projection matrix
    via batch SGD.
}
\label{algo:MAP-SGD}
\end{algorithm}

\section{Subspace and last-layer parameters}\label{sec:PULSE}
In the previous section,
we leveraged the lottery ticket hypothesis to make online estimation of the posterior density of high-dimensional model parameters computationally tractable.
However, the performance of this method heavily depends on the choice of the projection matrix ${\bf A}$,
which projects the parameters of the hidden and output layers onto a single linear subspace.
This reliance on ${\bf A}$ can limit the method’s flexibility and effectiveness in practice.

An alternative approach involves treating the last layer and the hidden layers of a neural network as separate components.
Commonly referred to as Bayesian last-layer methods \citep{harrison2024variationalbayesianlayers} or neural-linear methods,
these approaches improve the predictive power of neural networks by optimising the weights of the hidden layers using standard techniques, such as Adam,
and then placing a posterior density over the parameters of the last layer \citep[Section 17.3.5]{murphy2023-pmlbook2}.
This separation is particularly useful in applications such as Bayesian neural contextual bandit problems \citep{riquelme2018deepbanditshowdown}.

Because the hidden layers are typically much more computationally expensive to train than the output layers,
these methods often update the hidden layers intermittently, while the output layer is updated more frequently.
This division helps to reduce the computational cost of training while maintaining predictive accuracy.

Building on these ideas, the work in \cite{cartea2023sharpbayes} proposes a fully online Bayesian version of last-layer methods,
called PULSE (projection-based unification of last-layer and subspace estimation).
PULSE combines the strengths of the lottery ticket hypothesis with the power of neural-linear approaches by applying subspace projections only to the hidden layers,
while maintaining a full-covariance Gaussian posterior over the parameters of the last layer.\footnote{
In \cite{cartea2023sharpbayes}, PULSE is employed to detect the so-called \textit{toxicity} of trades being sent to a broker.
This is different from other models \citep{cartea2022brokers,  bergault2024mean, cartea2024nash, aqsha2024strategiclearningtradingbrokermediated}
because we do not assume that traders are 
informed or uninformed, instead, we think of each trade separately 
and we predict whether  the trade is informed or uninformed.
}
This approach allows for a fully-online treatment of online learning of high-dimensional neural networks,
while maintaining a reduced computational cost.

Specifically, PULSE provides a Bayesian framework for sequentially updating the parameters of a neural network.
It uses subspace projections to represent the hidden-layer parameters compactly
while maintaining a detailed posterior density for the last layer.
Unlike previous approaches, which either update all parameters or focus solely on the last layer during online learning,
PULSE strikes a balance by combining both strategies: projecting the hidden-layer parameters and fully updating the last-layer parameters.
Whereas the original PULSE method decouples the update for the hidden layer and the subspace entirely,
here, we introduce a modification that solves a coupled system of equations.

\subsection{The SSM for PULSE}
We decompose the model parameters for a neural network $\vtheta\in\real^\dimstate$
between the hidden-layer parameters $\phidden\in\real^{D_{\rm hidden}}$ and
the last-layer parameters $\plast  \in \real^\dimstatelast$. That is, $\vtheta = (\phidden, \plast)$.
Here, $D = D_{\rm hidden} + \dimstatelast$.
Next, write the hidden layer parameters $\phidden$ as an affine projection of the form
\begin{equation}\label{eq:hidden-layer-decomposition}
    \phidden = \vA\,\phiddensub + \phidden_*.
\end{equation}
Similar to \eqref{eq:subspace-params}, ${\bf A}$ is a $({D_{\rm hidden}\times \dimstatehiddensub})$ fixed projection matrix  and
$\phiddensub_t\in\real^\dimstatehiddensub$ are the projected (subspace) parameters such that $\dimstatehiddensub \ll 
D_{\rm hidden}$, 
and $\phidden_*\in\real^{D_{\rm hidden}}$ is the offset term. The term $\phidden_*$ is initialised as in Section \ref{subsec:warm-up-subspace} but considering only the last layer parameters.

Similar to Section \ref{sec:subspace},
the lower-dimensional vector $\phiddensub$ is assumed to be time-dependent
and evolving according to the equation $\phiddensub_t = \phiddensub_{t-1} + \vu_t^{\rm hidden}$,
with $\vu_t^{\rm hidden}$ a zero-mean random vector with known covariance matrix.
Next, 
the parameters of the last layer are assumed to be time-dependent and evolve according
to $\plast_t = \plast_{t-1} + \vu_t^{\rm last}$ with $\vu_t^{\rm last}$ a zero-mean random vector with known covariance matrix.
Here, we take ${\rm Cov}(\vu_s^{\rm hidden}, \vu_t^{r\rm last}) = \bf{0}$ for all $s,t$.
Under these assumptions,
the SSM over model parameters and observations take the form
\begin{equation}\label{eq:subspace-last-layer-ssm}
\begin{aligned}
    \phiddensub_t &= \phiddensub_{t-1} + \vu_{t}^{\rm hidden},\\
    \plast_t &= \plast_{t-1} +\vu_t^{\rm last},\\
    \vtheta_t &= ({\bf A}\,\phiddensub_t + \phidden_*, \plast_t),\\
    \vy_t &= h(\vtheta_t, \vx_t) + \ve_t,
\end{aligned}
\end{equation}
where, as before, $\ve_t$ is a zero mean random vector with known covariance matrix $\vR_t$,
$\vy_t\in\real^\dimobs$ are the measurements and
$h: \real^\dimstate \times \real^\dimin \to \real^\dimobs$
is a differentiable function w.r.t. the first entry, e.g., a neural network.

\subsection{Objective}
We introduce Gaussian priors for both $\plast$ and ${\phiddensub}$ at the beginning of the deploy stage.
Let $t=1$ be the first timestamp of the deploy dataset.
Denote the prior densities for $\plast$ and ${\phiddensub}$ by
\begin{align*}
     &\normdist{\plast}{\plast^{(E)}}{\sigma^2_\plast\,{\bf I}_\dimstatelast},\\
     &\normdist{{\phiddensub}}{ {\bf 0} } {\sigma^2_\phiddensub\,{\bf I}_\dimstatehiddensub},
\end{align*}
where $\plast^{(E)}$ is the last iterates from the warmup stage, and
$\sigma^2_\plast,\,\sigma^2_\phiddensub$ are the coefficients of the prior covariance matrix.

At each timestep during the deploy stage, we approximate the posterior densities of the last-layer parameters $\plast$, and the subspace hidden-layer parameters $\phiddensub$, as disjoint multivariate Gaussians.

To compute the posterior mean and covariance of the subspace hidden-layer parameters $(\mhidden_t, \covhidden_t)$,
as well as the posterior mean and covariance of the last-layer parameters $(\mlast_t, \covlast_t)$,
we solve a recursive variational inference (VI) optimisation problem.
This approach resembles the R-VGA method introduced in Section \ref{sec:rvga},
so that
$\cov(\vu_t^{\rm hidden}) = {\bf 0}$ and $\cov(\vu_t^{\rm last}) = {\bf 0}$.
Specifically, at each timestep, we optimise the block mean-field objective
\begin{equation}\label{eq:subspace-last-rvga}
    \vTheta_t
    = \argmin_{\mhidden, \mlast, \covhidden, \covlast}
    \KL{\normdist{\plast}{\mlast}{\covlast}\normdist{\phiddensub}{\mhidden}{\covhidden}}
    {p(\plast, \phiddensub \cond \data_{1:T})},
\end{equation}
where $\vTheta_t = (\mhidden_t, \mlast_t, \covhidden_t, \covlast_t)$ represents the set of optimised posterior parameters, and $p(\plast, \phiddensub \cond \data_{1:t})$ denotes the reference posterior density at time $t$, given by
\begin{equation}
\begin{aligned}
    p(\plast, \phiddensub \cond \data_{1:t})
    &\propto q_{t-1}(\plast, \phiddensub \cond \data_{1:t-1})\,
    p(\vy_t \cond \vtheta, \vx_t) \\
    &= \normdist{\plast}{\mlast_{t-1}}{\covlast_{t-1}} \,
       \normdist{\phiddensub}{\mhidden_{t-1}}{\covhidden_{t-1}} \,
       p(\vy_t \cond \vtheta, \vx_t).
\end{aligned}
\end{equation}
Here,
$\vtheta = ({\bf A}\,\phiddensub + \phidden_*, \plast)$ and
$q_{t-1}(\plast, \phiddensub \cond \data_{1:t-1})$ is the posterior density from the previous timestep, which we take to be the prior when computing the posterior at time $t$. In general, the density $p(\vy_t \cond (\phiddensub, \plast), \vx_t) $ might not be linear-Gaussian ($h$ not linear in $\vtheta$).
In such cases, we employ a linearised  moment-matched Gaussian likelihood
$\hat{p}(\vy_t \cond \vtheta, \vx_t)$---we return to this point in the next subsection. 

Then, the approximated posterior at time $t$ is 
\begin{equation}
q_{t}(\phiddensub, \plast)
= \normdist{\phiddensub}{\mhidden_t}{\covhidden_t}\,{\normdist{\plast}{\mlast_t}{\covlast_t}}.
\end{equation}

\subsection{Update step}
To obtain closed-form updates,
we consider a modified likelihood for the measurement $\vy_t$,
which we take as the one introduced in the exponential-family EKF 
of Section \ref{sec:exponential-family-ekf}.
Let $h(\vtheta, \vx)$ be the mean of the measurement model $p(\vy \cond \vtheta, \vx)$.
A first-order approximation of $h(\vtheta, \vx_t)$ around $(\mhidden_{t-1}, \mlast_{t-1})$ yields
\begin{equation}
    \bar{h}_t(\phiddensub, \plast)
    = \kappa_t
    + \bar{\vZ_t}\,(\phiddensub - \mhidden_{t-1})
    + \bar{\vW_t}\,(\plast - \mlast_{t-1})
\end{equation}
with
\begin{equation}
\begin{aligned}
    \kappa_t &= h\big((\mhidden_{t-1}, \mlast_{t-1}), \vx_t\big),\\
    \bar{\vZ_t} &= \nabla_\phiddensub h\big((\phiddensub, \mlast_{t-1}), \vx_t\big)\vert_{\phiddensub = \mhidden_{t-1}},\\
    \bar{\vW_t} &= \nabla_\plast h\big((\mhidden_{t-1}, \plast), \vx_t\big)\vert_{\plast = \mlast_{t-1}}.
\end{aligned}
\end{equation}
The density for the measurements is taken to be
\begin{equation}\label{eq:linearised-model-pulse}
\begin{aligned}
    \hat{p}(\vy_t \cond( \phiddensub, \plast), \vx_t)
    = {\cal N}(\vy_t \cond \bar{h}_t( \phiddensub, \plast), \bar{\vR}_t)
    \propto \exp\left(-\tfrac{1}{2}\|\bar{\vR}_t^{-1/2}\,\left(\vy_t - \bar{h}_t(\phiddensub, \plast)\right)\|_2^2\right),
\end{aligned}
\end{equation}
where $\bar{\vR}_t$
is the moment-matched observation variance.
Next, our reference posterior is modified as
\begin{equation}
\begin{aligned}
    p(\phiddensub, \plast \cond \data_{1:t})
    &\propto q_{t-1}(\phiddensub, \plast \cond \data_{1:t-1})\,
    \hat{p}(\vy_t \cond (\phiddensub, \plast), \vx_t)\\
    &=
    \normdist{\phiddensub}{\mhidden_{t-1}}{\covhidden_{t-1}}\,
    \normdist{\plast}{\mlast_{t-1}}{\covlast_{t-1}}\,
    \hat{p}(\vy_t \cond (\phiddensub, \plast), \vx_t).
\end{aligned}
\end{equation}

\begin{proposition}\label{prop:fixed-points-pulse}
Optimising    the objective \eqref{eq:subspace-last-rvga}
    under the linearised measurement model \eqref{eq:linearised-model-pulse}
    yields the fixed-point equations
    \begin{equation}
    \begin{aligned}
        \mlast_t &= \mlast_{t-1} - \covlast_{t-1}\nabla_{\mlast_t}{\cal E}_t,\\
        \covlast_t^{-1} &= \covlast_{t-1}^{-1} + 2\, \nabla_{{\covlast_t}}{\cal E}_t,\\
        \mhidden_t &= \mhidden_{t-1} - \covhidden_{t-1}\nabla_{\mhidden_t}{\cal E}_t,\\
        \covhidden_t^{-1} &= \covhidden_{t-1}^{-1} + 2\, \nabla_{{\covhidden_t}}{\cal E}_t,
    \end{aligned}
    \end{equation}
    where
    ${\cal E}_t := \mathbb{E}_{\normdisthidden \normdistlast}[\log \hat{p}(\vy_t \cond \phiddensub, \plast, \vx_t)]$.
\end{proposition}

\begin{proof}
We begin by rewriting the objective function \eqref{eq:subspace-last-rvga}
using the linearised measurement model \eqref{eq:linearised-model-pulse}.
To simplify notation, let
$\hat{p}(\vy_t) := p(\vy_t \,\vert\, \phiddensub, \plast; \vx_t)$,
$\vdhidden[t] := \normdist{\phiddensub}{\mhidden_t}{\covhidden_t}$, and
$\vdlast[t] := \normdist{\plast}{\mlast_t}{\covlast_t}$.
\begin{equation}
\begin{aligned}
    {\cal K}_t
    &= \KL{\normdistlast\,\normdisthidden}{\vdlast[t-1]\,\vdhidden[t-1]\, \hat{p}(\vy_t)}\\
    &= \iint \normdisthidden\,\normdistlast \log\left( \frac{\normdisthidden\,\normdistlast}
        { \vdhidden[t-1]\, \vdlast[t-1]\, p(\vy_t)} \right)
        \d\phiddensub\,\d\plast\\
    &= \iint \normdisthidden\,\normdistlast \left[
        \log\left(\frac{\normdisthidden}{\vdhidden[t-1]}\right)
        + \log\left(\frac{\normdisthidden}{\vdlast[t-1]} \right)
        - \log \hat{p}(\vy_t)
    \right] \d\phiddensub\, \d\plast\\
    &= \int \normdisthidden\, \log\left(\frac{\normdisthidden}{\vdhidden[t-1]}\right) \d\phiddensub + \int \normdistlast \log\left(\frac{\normdistlast}{\vdlast[t-1]}\right) \d\plast  \\ 
    &\quad  + \iint \normdisthidden\,\normdistlast\, \log\hat{p}(\vy_t)\, \d\plast \,\d\phiddensub \\
    &= \mathbb{E}_{\normdisthidden}\left[
        \log\left(\frac{\normdisthidden}{\vdhidden[t-1]}\right)
    \right] + \mathbb{E}_{\normdistlast}\left[
        \log\left(\frac{\normdistlast}{\vdlast[t-1]}\right)
    \right] \\
    &\quad  + \mathbb{E}_{\normdisthidden\, \normdistlast}[\log\hat{p}(\vy_t)]\,.
\end{aligned}
\end{equation}

Finally, we obtain
\begin{equation}\label{eq:rvga-fsll-rewrite}
    {\cal K}_t
    = \KL{\normdistlast}{\vdlast[t-1]}
    + \KL{\normdisthidden}{\vdhidden[t-1]}
     + {\cal E}_t
\end{equation}
where ${\cal E}_t := \mathbb{E}_{\normdisthidden \normdistlast}[\log \hat{p}(\vy_t)]$.

The first and second terms in \eqref{eq:rvga-fsll-rewrite} correspond to a Kullback--Leibler
divergence between two multivariate Gaussians.
The last term corresponds to the posterior-predictive marginal log-likelihood for the $t$-th observation.
To minimise \eqref{eq:rvga-fsll-rewrite} with respect to $\vTheta_t$,
we use the Kullback-Leibler divergence between two multivariate Gaussian
derived in Proposition \ref{prop:KL-divergence-gaussians}.

The derivative of ${\cal K}_t$ with respect to $\mlast_t$ is 
\begin{align}
    \nabla_{\mlast_t} {\cal K}_t
    &= \nabla_{\mlast_t} \left(
    \KL{\normdistlast}{\vdlast[t-1]}
    + {\cal E}_t
    \right) \nonumber\\ 
    &= \nabla_{\mlast_t}\left(
        \frac{1}{2}\mlast_t^\intercal\covlast_{t-1}^{-1}\mlast_t - \mlast_t^\intercal\covlast_{t-1}\mlast_{t-1}
         + \nabla_{\mlast_t}{\cal E}_t
    \right)\nonumber \\
    &= \covlast_{t-1}^{-1}\mlast_t - \covlast_{t-1}^{-1}\mlast_{t-1} +\nabla_{\mlast_t} {\cal E}_t \nonumber \\
    &= \covlast_{t-1}^{-1}\left( \mlast_t - \mlast_{t-1} - \covlast_{t-1}\nabla_{\mlast_t}{\cal E}_t \right). \label{eq:dK-mlast}
\end{align}
Set \eqref{eq:dK-mlast} to zero and solve for
\begin{equation*}
    {\mlast_t} = \mlast_{t-1} - \covlast_{t-1}\nabla_{\mlast_t}{\cal E}_t.
\end{equation*}
Next, we estimate the condition
for $\covlast_t$. Use \eqref{eq:rvga-fsll-rewrite} to obtain
\begin{align}
    \nabla_{\covlast_t} {\cal K}_t &=
    \nabla_{\covlast_t}\left(
        -\frac{1}{2}\,\log|{\covlast_t}| + \frac{1}{2}\,{\rm Tr}\left({\covlast_t}\,\covlast_{t-1}^{-1}\right) + {\cal E}_t
    \right)\nonumber \\
    &= -\frac{1}{2}\,{\covlast_t}^{-1} + \frac{1}{2}\,\covlast_{t-1}^{-1} + \nabla_{\covlast}{\cal E}_t. \label{eq:dK-covlast}
\end{align}
The fixed-point solution for \eqref{eq:dK-covlast} satisfies
\begin{equation*}
    {\covlast_t}^{-1} = \covlast_{t-1}^{-1} + 2\, \nabla_{\covlast_t}{\cal E}_t.
\end{equation*}
The fixed-point conditions for $\mhidden_t$ and $\covhidden_t$ are derived similarly.
\end{proof}

Proposition \ref{prop:fixed-points-pulse}  
yields the fixed-point equations that the terms in $\vTheta_t$  
must satisfy.  
Furthermore,  
these equations are expressed in terms of the gradients of the expected linearised likelihood model  
\eqref{eq:linearised-model-pulse}.  
In the following proposition, we  
provide explicit expressions for the gradients  
of the linearised likelihood model with respect to each term in $\vTheta$.

\begin{proposition}\label{prop:gradients-wrt-linearised-log-likelihood-pulse}
Let
${\cal E}_t := \mathbb{E}_{\normdisthidden \normdistlast}[\log \hat{p}(\vy_t \cond (\phiddensub, \plast), \vx_t)]$
be the expected linearised log-likelihood.
The derivative of ${\cal E}_t$ w.r.t. $\mlast_t$ and $\covlast_t$ take the form
\begin{align}
    \nabla_{\mlast_t}\,{\cal E}_t&=
    -\bar{\vW_t}^\intercal\,\bar{\vR}_t^{-1}\,\left(\vy_t - \bar{h}_t(\mhidden_t, {\mlast_t})\right),\\
    \nabla_{\covlast_t}\,{\cal E}_t&=
    \tfrac{1}{2}\bar{\vW_t}^\intercal\,\bar{\vR}_t^{-1}\,\bar{\vW_t}.
\end{align}
Similarly, the derivative of ${\cal E}_t$ w.r.t. ${\mhidden_t}$ and ${\covhidden_t}$ take the form
\begin{align}
    \nabla_{\mhidden_t}\,{\cal E}_t &=
    -\bar{\vZ_t}^\intercal\,\bar{\vR}_t^{-1}\,\left(\vy_t - \bar{h}_t({\mhidden_t}, {\mlast_t})\right),\\
    \nabla_{\covhidden_t}\,{\cal E}_t&=
    \tfrac{1}{2}\bar{\vZ_t}^\intercal\,\bar{\vR}_t^{-1}\,\bar{\vZ_t}.
\end{align}
\end{proposition}

\begin{proof}
    From Bonnet's Theorem and Price's Theorem
    (See Theorem 3 and Theorem 4 in \cite{lin2019steinexpfam}),
    we obtain
    \begin{equation}
    \begin{aligned}
    \nabla_\mlast\mathbb{E}_{\normdisthidden \normdistlast}[\log \hat{p}(\vy_t\cond (\phiddensub, \plast), \vx_t)]
    &= \mathbb{E}_{\normdisthidden \normdistlast}[\nabla_\plast \log \hat{p}(\vy_t)],\\
    \nabla_\covlast\mathbb{E}_{\normdisthidden \normdistlast}[\log \hat{p}(\vy_t\cond (\phiddensub, \plast), \vx_t)]
    &= \tfrac{1}{2}\mathbb{E}_{\normdisthidden \normdistlast}[\nabla_\plast^2\log \hat{p}(\vy_t)].
    \end{aligned}
    \end{equation}
    
    The Jacobian and the Hessian and the log-likelihood w.r.t. $\plast$ are
    \begin{equation}
    \begin{aligned}
        \nabla_\plast\,\log\hat{p}(\vy_t\cond (\phiddensub, \plast), \vx_t)
        &=  -\bar{\vW_t}^\intercal\,\bar{\vR}_t^{-1}\,\left(\vy_t - \bar{h}_t(\phiddensub, \plast)\right),\\
        \nabla^2_\plast\,\log\hat{p}(\vy_t\cond (\phiddensub, \plast), \vx_t)
        &= \bar{\vW_t}^\intercal\,\bar{\vR}_t^{-1}\,\bar{\vW_t}.\\
    \end{aligned}
    \end{equation}
    Hence,
    \begin{equation}
    \begin{aligned}
        \mathbb{E}_{\normdisthidden \normdistlast}[\nabla_\plast\,\log\hat{p}(\vy_t\cond (\phiddensub, \plast), \vx_t)]
        &=  -\bar{\vW_t}^\intercal\,\bar{\vR}_t^{-1}\,\left(\vy_t - \bar{h}_t\right),\\
        \mathbb{E}_{\normdisthidden \normdistlast}[\nabla^2_\plast\,\log\hat{p}(\vy_t\cond (\phiddensub, \plast), \vx_t)]
        &= \bar{\vW_t}^\intercal\,\bar{\vR}_t^{-1}\,\bar{\vW_t}.\\
    \end{aligned}
    \end{equation}
    The result is derived similarly for the parameters of hidden subspace.
\end{proof}

Plugging the result of Proposition \ref{prop:gradients-wrt-linearised-log-likelihood-pulse}  
into the fixed-point equations given in Proposition \ref{prop:fixed-points-pulse}  
yields a pair of linear equations whose only unknowns are the terms in $\vTheta$.  
The next theorem derives the update equations for PULSE,  
which explicitly solve for each term in $\vTheta$,  
yielding $\vTheta_t$.  

\begin{theorem}[PULSE]\label{theorem:pulse}
    The approximated posterior that solves  objective \eqref{eq:subspace-last-rvga}
    under the linearised measurement model \eqref{eq:linearised-model-pulse} at time $t$ is
    \begin{equation}
        q_{t}(\phiddensub, \plast)
        = \normdist{\phiddensub}{\mhidden_t}{\covhidden_t}\,{\normdist{\plast}{\mlast_t}{\covlast_t}},
    \end{equation}
    where
    \begin{equation}
    \begin{aligned}
        \mlast_t &= \mlast_{t-1} + \hat{\vK}_{\plast,t}\,(\vy_t - h\big((\mhidden_{t-1}, \mlast_{t-1}), \vx_t\big),\\
        \mhidden_t &= \mhidden_{t-1} + \hat{\vK}_{\phiddensub,t}\,(\vy_t - h\big((\mhidden_{t-1}, \mlast_{t-1}), \vx_t\big),\\
        \covlast_t^{-1} &= \covlast_{t-1}^{-1} + \bar{\vW_t}^\intercal\,\bar{\vR}_t^{-1}\,\bar{\vW_t},\\
        \covhidden_t^{-1} &= \covhidden_{t-1}^{-1} + \bar{\vZ_t}^\intercal\,\bar{\vR}_t^{-1}\,\bar{\vZ_t},
    \end{aligned}
    \end{equation}
    and
    \begin{equation}
    \begin{aligned}
        \vK_{\phiddensub, t} &= \covhidden_t\,\bar{\vZ_t}^\intercal\,\bar{\vR}_t^{-1},\\
        \vK_{\plast, t} &= \covlast_t\,\bar{\vW_t}^\intercal\,\bar{\vR}_t^{-1},\\
        \hat{\vK}_{\phiddensub,t} &= \left(\vI_\dimstatehiddensub - \vK_{\phiddensub,t}\,\bar{\vW_t}\,\vK_{\plast,t}\,\bar{\vZ_t}\right)^{-1}\,
        \vK_\phiddensub\,(\vI_\dimobs - \bar{\vW_t}\,\vK_{\plast,t}), \\
        \hat{\vK}_{\plast,t} &= \left(\vI_\dimstatelast - \vK_{\plast,t}\,\bar{\vZ_t}\,\vK_{\phiddensub,t}\,\bar{\vW_t}\right)^{-1}\,
        \vK_\plast\,(\vI_\dimobs - \bar{\vZ_t}\,\vK_{\phiddensub,t}),
    \end{aligned}
    \end{equation}
    where $\vI_m$ is an $m\times m$ identity matrix.
\end{theorem}

\begin{proof}
    For the precision matrix of the last-layer parameters, we obtain
    \begin{equation}\label{eq:part-posterior-precision-last-layer}
    \begin{aligned}
        \covlast_t^{-1}
        &=  \covlast_{t-1}^{-1} + 2\, \nabla_{\covlast_t}{\cal E}_t\\
        &= \covlast_{t-1}^{-1} + \bar{\vW_t}^\intercal\,\bar{\vR}_t^{-1}\,\bar{\vW_t}.
    \end{aligned}
    \end{equation}
    Next, the precision matrix for the subspace hidden layer parameters take the form
    \begin{equation}
    \begin{aligned}
        {\covhidden_t}^{-1}
        &=  \covhidden_{t-1}^{-1} + 2\, \nabla_{{\covhidden_t}}{\cal E}_t\\
        &= \covhidden_{t-1}^{-1} + \bar{\vZ_t}^\intercal\,\bar{\vR}_t^{-1}\,\bar{\vZ_t}.
    \end{aligned}
    \end{equation}
    
    The posterior mean of the last layer of model parameters take the form
    \begin{equation}
    \begin{aligned}
        {\mhidden_t}
        &= \mhidden_{t-1} - \covhidden_{t-1}\nabla_{\mhidden_t}{\cal E}_t\\
        &= \mhidden_{t-1} + \covhidden_{t-1} 
        \bar{\vZ_t}^\intercal\,\bar{\vR}_t^{-1}
        \,\left(\vy_t - \bar{h}_t({\mhidden_t}, {\mlast_t})\right).
    \end{aligned}
    \end{equation}
    Similarly, the posterior mean of the subspace model parameters take the form
    \begin{equation}
    \begin{aligned}
        {\mlast_t}
        &= \mlast_{t-1} - \covlast_{t-1}\nabla_{\mlast_t}{\cal E}_t\\
        &= \mlast_{t-1} + \covlast_{t-1} \bar{\vW_t}^\intercal\,\bar{\vR}_t^{-1}\,\left(\vy_t - \bar{h}_t({\mhidden_t}, {\mlast_t})\right).
    \end{aligned}
    \end{equation}
    Hence, we need to solve the system of equations
    \begin{align}
        {\mlast_t} &= \mlast_{t-1} + \covlast_{t-1} \bar{\vW_t}^\intercal\,\bar{\vR}_t^{-1}\,\left(\vy_t - \bar{h}_t({\mhidden_t}, {\mlast_t})\right),\label{eq:part-fixed-point-mlast}\\
        {\mhidden_t} &= \mhidden_{t-1} + \covhidden_{t-1} \bar{\vZ_t}^\intercal\,\bar{\vR}_t^{-1}\,\left(\vy_t - \bar{h}_t({\mhidden_t}, {\mlast_t})\right).
    \end{align}
    
    In the rest of the proof, we show the solution for $\mlast_t$. Similar steps yield the solution for $\mhidden_t$.
    We begin by expanding the right hand side of \eqref{eq:part-fixed-point-mlast}. We obtain
    \begin{equation}
    \begin{aligned}
        {\mlast_t}
        &= \mlast_{t-1} + \covlast_{t-1} \bar{\vW_t}^\intercal\,\bar{\vR}_t^{-1}\,\left(
        \vy_t - \bar{h}_t({\mhidden_t}, {\mlast_t})
        \right)\\
        &= \mlast_{t-1} + \covlast_{t-1} \bar{\vW_t}^\intercal\,\bar{\vR}_t^{-1}\,\left(
        \vy_t - [\kappa_t + \bar{\vZ_t}\,(\phiddensub - \mhidden_{t-1}) + \bar{\vW_t}\,({\plast_t} - \mlast_{t-1})]
        )\right).
    \end{aligned}
    \end{equation}
    Expanding terms and grouping for ${\mlast_t}$ and $\mlast_{t-1}$, we obtain
    \begin{equation}\label{eq:part-mlast-implicit}
    \begin{aligned}
    {\mlast_t}
    &= \mlast_{t-1} + \left[\covlast_{t-1}^{-1} + \bar{\vW}_t^\intercal\,\bar{\vR}_t^{-1}\,\bar{\vW}_t\right]^{-1}
    \bar{\vW}_t^\intercal\,\bar{\vR}_t^{-1}\left(\vy_t - \kappa_t - \bar{\vZ}_t\,({\mhidden_t} - \mhidden_{t-1})\right) \\
    &= \mlast_{t-1} + \covlast_{t}\bar{\vW}_t^\intercal\,\bar{\vR}_t^{-1}\,\left(\vy_t - \kappa_t - \bar{\vZ}_t\,({\mhidden_t} - \mhidden_{t-1})\right)\\
    &= \mlast_{t-1} + {\vK}_{{\plast_t},t}\,\left(\vy_t - \kappa_t - \bar{\vZ}_t\,({\mhidden_t} - \mhidden_{t-1})\right),
    \end{aligned}
    \end{equation}
    where we defined $\vK_{{\plast_t},t} = \covlast_{t}\bar{\vW}_t^\intercal\,\bar{\vR}_t^{-1}$.

    Similarly, for the subspace hidden-layer parameters, we obtain
    \begin{equation}\label{eq:part-mhidden-implicit}
        {\mhidden_t} = \mhidden_{t-1} + \vK_{\phiddensub,t}\,\left(\vy_t - \kappa_t - \bar{\vW}\,({\mlast_t} - \mlast_{t-1})\right)
    \end{equation}
    with $\vK_{\phiddensub,t} = \covhidden_t\,\bar{\vZ}_t^\intercal\,\bar{\vR}_t^{-1}$.

    The terms in \eqref{eq:part-mlast-implicit} and \eqref{eq:part-mhidden-implicit}
    resemble a Kalman filter update,
    as described in Proposition \ref{prop:kf-update-step-precision}.
    However, these terms are interdependent,
    as each relies on the unknown variables ${\mhidden_t}$ and ${\mlast_t}$.
    In what follows, we derive a solution to this system of equations.

    From \eqref{eq:part-mhidden-implicit}, note that
    \begin{equation}
        {\mhidden_t} - \mhidden_{t-1}
        = \vK_{\phiddensub,t}\,\left(\vy_t - \kappa_t - \bar{\vW}\,({\mlast_t} - \mlast_{t-1})\right).
    \end{equation}
    Then, \eqref{eq:part-mlast-implicit} takes the form
    \begin{equation}
    {\mlast_t} = \mlast_{t-1} + {\vK}_{{\plast_t},t}\,
    \left(\vy_t - \kappa_t - \bar{\vZ}_t\,
    \left( \vK_{\phiddensub,t}\,\left(\vy_t - \kappa_t - \bar{\vW}\,({\mlast_t} - \mlast_{t-1})\right) \right)
    \right).
    \end{equation}
    Expanding and grouping the terms ${\mlast_t}$ and $\mlast_{t-1}$, we obtain
    \begin{equation}
        {\mlast_t} = \mlast_{t-1} +
        \left( \vI_\dimstatelast - \vK_{{\plast_t},t}\,\bar{\vZ}_t\,\vK_{\phiddensub}\,\bar{\vW}_t\right)^{-1}\,
        \vK_{{\plast_t},t}\,(\vI_\dimobs - \bar{\vW}_t\,\vK_{\phiddensub,t})\,\left(\vy_t - \kappa_t\right).
    \end{equation}
    A similar derivation yields $\mhidden_t$.
\end{proof}

Theorem \ref{theorem:pulse} demonstrates that the update equations for PULSE share a structural resemblance with those of the extended Kalman filter (EKF) and the exponential family EKF (expfamEKF), as introduced in Sections \ref{sec:extended-kalman-filter} and \ref{sec:exponential-family-ekf}. However, a key distinction lies in the modification of the gain matrix, which is adjusted to account for the dependencies between the last-layer parameters and the subspace hidden-layer parameters.

Algorithm \ref{algo:pulse-update} provides pseudocode for the update step for PULSE.
\begin{algorithm}[H]
\begin{algorithmic}[1]
    \REQUIRE $\vy_t$ // measurement at time $t$
    \REQUIRE $\vR_t$ // observation variance at time $t$
    \REQUIRE $(\mhidden_{t-1}, \covhidden_{t-1})$ // previous posterior mean and covariance for subspace parameters
    \REQUIRE $(\mlast_{t-1}, \covlast_{t-1})$ // previous posterior mean and covariance  for the last-layer parameters 
    \STATE $\hat{\vy}_t \gets h\big((\mhidden_{t-1}, \mlast_{t-1}), \vx_t\big)$
    \STATE  $\bar{\vZ_t} \gets \nabla_\phiddensub h\big((\phiddensub, \mlast_{t-1}), \vx_t\big)\vert_{\phiddensub = \mhidden_{t-1}}$
    \STATE $\bar{\vW_t} \gets \nabla_\plast h\big((\mhidden_{t-1}, \plast), \vx_t\big)\vert_{\plast = \mlast_{t-1}}$
    \STATE $\covhidden_t^{-1} = \covhidden_{t-1}^{-1} + \bar{\vZ_t}^\intercal\,\vR_t^{-1}\,\bar{\vZ_t}$
    \STATE $\covlast_t^{-1} = \covlast_{t-1}^{-1} + \bar{\vW_t}^\intercal\,\vR_t^{-1}\,\bar{\vW_t}$
    \STATE // posterior mean updates
    \STATE $\vK_{\phiddensub, t} = \covhidden_t\,\bar{\vZ_t}^\intercal\,\vR_t^{-1}$
    \STATE $\vK_{\plast, t} = \covlast_t\,\bar{\vW_t}^\intercal\,\vR_t^{-1}$
    \STATE $\hat{\vK}_{\phiddensub,t} = \left(\vI_\dimstatehiddensub - \vK_{\phiddensub,t}\,\bar{\vW_t}\,\vK_{\plast,t}\,\bar{\vZ_t}\right)^{-1}\,
    \vK_\phiddensub\,(\vI_\dimobs - \bar{\vW_t}\,\vK_{\plast,t})$
    \STATE $\hat{\vK}_{\plast,t} = \left(\vI_\dimstatelast - \vK_{\plast,t}\,\bar{\vZ_t}\,\vK_{\phiddensub,t}\,\bar{\vW_t}\right)^{-1}\,
    \vK_\plast\,(\vI_\dimobs - \bar{\vZ_t}\,\vK_{\phiddensub,t})$
    \STATE $\mlast_t \gets \mlast_{t-1} + \hat{\vK}_{\plast,t}\,(\vy_t - \hat{\vy}_t)$
    \STATE $\mhidden_t \gets \mhidden_{t-1} + \hat{\vK}_{\phiddensub,t}\,(\vy_t - \hat{\vy}_t)$
\end{algorithmic}
\caption{
    Update step for PULSE
}
\label{algo:pulse-update}
\end{algorithm}

\section{The low-rank extended Kalman filter}\label{sec:lofi}
The low-rank Kalman filter (LoFi) introduced in \cite{chang2023lofi}
approximates the the distribution over model parameters as Gaussian,
$q(\vtheta_t \cond \data_{1:t})=\normdist{\vtheta_t}{\vmu_t }{\vSigma_t}$.
Here, the posterior precision is diagonal plus low rank (DLR), i.e., it has the form
$\vSigma_t^{-1} = \vUpsilon_t + \vW_t \vW_t^\intercal$,
where $\vUpsilon_t$ is diagonal and $\vW_t$ is a $\dimstate \times d$ matrix.
Here, we seek to work with $\vUpsilon_t$ and $\vW_t$.
In this sense, we seek ``predict and update equations'' that
depend on these terms, rather than the covariance matrix itself.

Below, we show an efficient recursive form to estimate the terms that comprise the DLR posterior precision matrix,
as well as the posterior mean.
This has two main steps---a predict step and an update step.
The predict step takes $O(\dimstate\,d^2 + d^3)$ time,
And the update step takes $O(\dimstate (d + \dimobs)^2)$ time.
We define these steps below:
\begin{equation}\label{eq:lofi-ssm}
\begin{aligned}
    p(\vtheta_t \cond \vtheta_{t-1}) &= {\cal N}(\vtheta_t \cond \vtheta_{t-1}, \hat{q}\,\vI),\\
    p(\vy_t \cond \vtheta_{t-1}) &= {\cal N}(\vy_t \cond \vH_t\vtheta_t, \vR_t),\\
\end{aligned}
\end{equation}
with $\hat{q} \geq 0$.

\subsection{Predict step}
Here, we derive a predict step that makes use of the DLR structure of the covariance matrix.
\begin{proposition}[Posterior predictive covariance matrix]\label{prop:lofi-predict-covariance}
    Suppose $\vSigma_{t-1}^{-1} = \vUpsilon_{t-1} + \vW_{t-1}\,\vW_{t-1}^\intercal$.
    Then, under \eqref{eq:lofi-ssm}, the predict step $p(\vtheta_t \cond \data_{1:t-1})$ is
    of Gaussian form with mean $\vmu_{t-1}$ and covariance
    \begin{equation}
        \vSigma_{t|t-1} = \vUpsilon_{t|t-1}^{-1} - \vUpsilon_{t-1}^{-1}\,\vW_{t-1}\vB_{t|t-1}\,\vW_{t-1}^\intercal\,\vUpsilon_{t-1}^{-1},
    \end{equation}
    where
    \begin{align}
        \vUpsilon_{t|t-1}^{-1} &= \vUpsilon_{t-1}^{-1} + \hat{q}\,\vI_{\dimstate},\label{eq-part:upsilon-pred}\\
        \vB_{t|t-1} &= \left(\vI_{d}  + \vW_{t-1}^\intercal\,\vUpsilon_{t-1}^{-1}\,\vW_{t-1} \right)^{-1}.\label{eq-part:B-pred}
    \end{align}
\end{proposition}
\begin{proof}
    Following Proposition \eqref{prop:kf-equations}, we obtain 
    $\vmu_{t|t-1} = \vmu_{t-1}$.
    Next,
    \begin{equation}
    \begin{aligned}
        \vSigma_{t|t-1}
        &= \vSigma_{t-1} + \hat{q}\,\vI_D\\
        &= \left(\vUpsilon_{t-1} + \vW_{t-1}\,\vW_{t-1}^\intercal\right)^{-1} + \hat{q}\,\vI_D\\
        &= (\vUpsilon_{t-1}^{-1} + \hat{q}\,\vI_D)
        - \vUpsilon_{t-1}^{-1}\vW_{t-1}\,\left(\vW_{t-1}^\intercal\,\vUpsilon_{t-1}^{-1}\,\vW_{t-1}
        + \vI_d\right)^{-1}\vW_{t-1}\,\vUpsilon_{t-1}^{-1}\\
        &= \vUpsilon_{t|t-1}^{-1} - \vUpsilon_{t-1}^{-1}\,\vW_{t-1}\vB_{t|t-1}\,\vW_{t-1}^\intercal\,\vUpsilon_{t-1}^{-1},
    \end{aligned}
    \end{equation}
    where $\vUpsilon_{t|t-1}$ and $\vB_{t|t-1}$ are defined in \eqref{eq-part:upsilon-pred} and \eqref{eq-part:B-pred} respectively.
\end{proof}

Next, we provide predict-step equations that depend on $\vUpsilon_{t-1}$ and $\vW_{t-1}$ only.
To do this, we work with the precision matrix.

\begin{lemma}
    \label{lemma:part-matrix-form-C}
    For $\vB_{t|t-1}$ defined in \eqref{eq-part:B-pred},
    $\vUpsilon_{t-1}$ and $\vUpsilon_{t|t-1}$ $\dimstate$-dimensional diagonal covariance matrices, and
    $\vW_{t-1}$ a $\dimstate\times d$ matrix.
    The following identity holds
    \begin{equation}
    \begin{aligned}
        \vC_t^{-1}
        &:= \vB_{t|t-1}^{-1} - {\vW}_{t-1}^\intercal\,\vUpsilon_{t-1}^{-1}\,\vUpsilon_{t|t-1}\,\vUpsilon_{t-1}^{-1}\,\vW_{t-1}\\
        &= \vI_d + \vW_{t-1}^\intercal\,\left(
        \vUpsilon_{t-1}^{-1} - \vUpsilon_{t-1}^{-1}\,\vUpsilon_{t|t-1}\,\vUpsilon_{t-1}^{-1}
        \right)\,\vW_{t-1}.
    \end{aligned}
    \end{equation}
\end{lemma}
\begin{proof}
    The proof follows through algebraic manipulation:
    \begin{equation}
    \begin{aligned}
        \vC_t
    ^{-1}    &= \vB_{t|t-1}^{-1} - {\vW}_{t-1}^\intercal\,\vUpsilon_{t-1}^{-1}\,\vUpsilon_{t|t-1}\,\vUpsilon_{t-1}^{-1}\,\vW_{t-1}\\
        &= \left( \vI_d + \vW_{t-1}^\intercal\,\vUpsilon_{t-1}^{-1}\,\vW_{t-1}\right)
        - {\vW}_{t-1}^\intercal\,\vUpsilon_{t-1}^{-1}\,\vUpsilon_{t|t-1}\,\vUpsilon_{t-1}^{-1}\,\vW_{t-1}\\
        &= \vI_d + \vW_{t-1}^\intercal\,\left(
        \vUpsilon_{t-1}^{-1} - \vUpsilon_{t-1}^{-1}\,\vUpsilon_{t|t-1}\,\vUpsilon_{t-1}^{-1}
        \right)\,\vW_{t-1}.
    \end{aligned}
    \end{equation}
\end{proof}

\begin{proposition}[Posterior predictive precision matrix]\label{prop:lofi-predict-precision}
    The predicted precision matrix derived from Proposition \ref{prop:lofi-predict-covariance} takes the form
    \begin{equation}
        \vSigma_{t|t-1}^{-1} = \vUpsilon_{t|t-1} + \vW_{t|t-1}\,\vW_{t|t-1}^\intercal,
    \end{equation}
    where
    \begin{align}
        \vUpsilon_{t|t-1} &= \left(\vUpsilon_{t-1}^{-1} + \hat{q}\,\vI_{\dimstate}\right)^{-1}, \label{eq:lofi-predict-diagonal}\\
        \vW_{t|t-1} &= \vUpsilon_{t|t-1}\,\vUpsilon_{t-1}^{-1}\,\vW_{t-1}\,\vC_t^{1/2} \label{eq:lofi-predict-low-rank}.
    \end{align}
    Here ${\bf A}^{1/2}$ refers to the Cholesky decomposition of a positive definite matrix ${\bf A}$.
\end{proposition}
\begin{proof}
    Consider the posterior predictive covariance matrix derived in Proposition \ref{prop:lofi-predict-covariance}.
    Then, the posterior predictive precision matrix takes the form
    \begin{equation}
    \begin{aligned}
        &\vSigma_{t|t-1}^{-1}\\
        &= \left(\vUpsilon_{t|t-1}^{-1} - \vUpsilon_{t-1}^{-1}\,\vW_{t-1}\vB_{t|t-1}\,\vW_{t-1}^\intercal\,\vUpsilon_{t-1}^{-1}\right)^{-1}\\
        &= \vUpsilon_{t|t-1} +
        \vUpsilon_{t|t-1}\,\vUpsilon_{t-1}^{-1}\,\vW_{t-1}
        \left(
        \vB_{t|t-1}^{-1} -
        {\vW}_{t-1}^\intercal\,\vUpsilon_{t-1}^{-1}\,\vUpsilon_{t|t-1}\,\vUpsilon_{t-1}^{-1}\,\vW_{t-1}
        \right)^{-1}\,\vW_{t-1}\,\vUpsilon_{t-1}^{-1}\vUpsilon_{t|t-1}\\
        &= \vUpsilon_{t|t-1} +
        \vUpsilon_{t|t-1}\,\vUpsilon_{t-1}^{-1}\,\vW_{t-1}\,
        \vC_t^{-1}\,\vW_{t-1}\,\vUpsilon_{t-1}^{-1}\vUpsilon_{t|t-1}\\
        &= \vUpsilon_{t|t-1} + \vW_{t|t-1}\,\vW_{t|t-1}^\intercal. 
    \end{aligned}
    \end{equation}
    Here, $\vC_t$ is derived in Lemma \ref{lemma:part-matrix-form-C} and $\vW_{t|t-1}$ is defined in \eqref{eq:lofi-predict-low-rank}.
\end{proof}

Proposition \ref{prop:lofi-predict-precision} derives predict steps in terms of $\vUpsilon_{t-1}$ and $\vW_{t-1}$,
so that dependence in $\vSigma_{t|t-1}$ is implicit.
The computational cost of the predict step is $O(\dimstate\,d + d^3)$.
Algorithm \ref{algo:lofi-predict} provides pseudocode for the predict step.

\begin{algorithm}[H]
\begin{algorithmic}[1]
    \REQUIRE $\vmu_{t-1}$ // previous mean
    \REQUIRE $(\vUpsilon_{t-1}, \vW_{t-1})$ // previous diagonal and low-rank parts
    \REQUIRE $q_t$ // scalar dynamics covariance
    \STATE $\vmu_{t|t-1} \gets \vmu_{t-1}$
    \STATE $\vUpsilon_{t|t-1} \gets \left(\vUpsilon_{t-1}^{-1} + q_t\,\vI_{\dimstate}\right)^{-1}$
    \STATE $\vC_t^{-1} \gets \vI_d + \vW_{t-1}^\intercal\,\left( \vUpsilon_{t-1}^{-1} - \vUpsilon_{t-1}^{-1}\,\vUpsilon_{t|t-1}\,\vUpsilon_{t-1}^{-1} \right)\,\vW_{t-1}$
    \STATE $\vW_{t|t-1} \gets \vUpsilon_{t|t-1}\,\vUpsilon_{t-1}^{-1}\,\vW_{t-1}\,\vC_t^{1/2}$
    \STATE $\hat{\vy}_t \gets h(\vmu_{t|t-1}, \vx_t)$ // one-step-ahead forecast
\end{algorithmic}
\caption{
    Predict step for LoFi
}
\label{algo:lofi-predict}
\end{algorithm}

\subsection{Update step}
The LoFi update step makes use of the DLR form of the precision matrix.
\begin{proposition}\label{prop:lofi-update}
    The form of the posterior mean and posterior covariance after an update step takes the form
    \begin{equation}
    \begin{aligned}
        \ve_t &= \vy_t - \hat{\vy}_t\,,\\
        \vmu_t &= \vmu_{t|t-1} +
        \left[
        \vUpsilon_{t|t-1}^{-1} -
        \vUpsilon_{t|t-1}^{-1}\,\tilde{\vW_t} \left( \vI_{d+\dimobs} + \tilde{\vW}_t^\intercal\,\vUpsilon_{t|t-1}^{-1}\,\tilde{\vW}_t \right)^{-1}\,\tilde{\vW}_t^\intercal\,\vUpsilon_{t|t-1}^{-1}
        \right]\vH_t^\intercal\,\vR_t^{-1}\,\ve_t,\\
        \vSigma_{t}^{-1} &= \vUpsilon_{t|t-1} + \tilde{\vW}_t\,\tilde{\vW}_t^\intercal,
    \end{aligned}
    \end{equation}
    where
    \begin{equation}\label{eq:part-tilde-W-lofi}
        \tilde{\vW}_{t|t-1} =
        \begin{bmatrix}
        \vW_{t|t-1}& \vH_t\,\vR_t^{1/2}
        \end{bmatrix},
    \end{equation}
    $\vW_{t|t-1}$ is given by \eqref{eq:lofi-predict-low-rank} and
    $\vUpsilon_{t|t-1}$ is given by \eqref{eq:lofi-predict-diagonal}.
\end{proposition}
\begin{proof}
    Following the Kalman filter update for the precision matrix
    shown in Proposition \ref{prop:kf-update-step-precision}, we obtain
    \begin{equation}
    \begin{aligned}
        \vSigma_t^{-1}
        &= \vSigma_{t|t-1}^{-1} + \vH_t\,\vR_{t}^{-1}\,\vH_t\\
        &= \vUpsilon_{t|t-1} + \vW_{t|t-1}\,\vW_{t|t-1}^\intercal +
        \vH_t^\intercal\,\vR_t^{-\intercal/2}\,\vR_{t}^{-1/2}\,\vH_t\\
        &= \vUpsilon_{t|t-1} +
        \begin{bmatrix}
            \vW_{t|t-1} & \vH_t\,\vR_{t-1}^\intercal
        \end{bmatrix}
        \begin{bmatrix}
            \vW_{t|t-1}^\intercal\\[10pt]
            \vR_{t}^{-1/2}\,\vH_t
        \end{bmatrix}\\
        &= \vUpsilon_{t|t-1} + \tilde{\vW}_{t}\,\tilde{\vW}_{t}^\intercal,
    \end{aligned}
    \end{equation}
    where $\tilde{\vW}_t$ is given by \eqref{eq:part-tilde-W-lofi}.

    Next, using the Woodbury identity matrix, $\vSigma_t$ takes the form
    \begin{equation}\label{eq:part-posterior-covariance-lofi-rewrite}
    \begin{aligned}
       \vSigma_t
       &= \left(\vUpsilon_{t|t-1} + \tilde{\vW}_{t}\,\tilde{\vW}_{t}^\intercal\right)^{-1}\\
       &= \vUpsilon_{t|t-1}^{-1} -
       \vUpsilon_{t|t-1}^{-1}\,\tilde{\vW}_t
       \left(
        \vI_{d+o} + \tilde{\vW}_t^\intercal\vUpsilon_{t|t-1}^{-1}\tilde{\vW}_t
       \right)^{-1}
       \tilde{\vW_{t}}^\intercal\,\vUpsilon_{t|t-1}^{-1}.
    \end{aligned}
    \end{equation}
    The proof concludes from the update step of the Kalman filter shown in Proposition \ref{prop:kf-update-step-precision}
    with covariance given by \eqref{eq:part-posterior-covariance-lofi-rewrite}.
\end{proof}

After the update step of Proposition \eqref{prop:lofi-update}, the low-rank component $\tilde{\vW}$
is a $\dimstate\times (d+o)$ matrix.
To maintain a $\dimstate\times d$ low-rank matrix,
LoFi performs a singular value decomposition (SVD) over the $\dimstate\times(d + o)$ low rank matrix $\tilde{\vW}_t$
and maintains the top $d$ singular components.

Algorithm \ref{algo:lofi-update} provides pseudocode for the update step and subsequent low-rank
projection of the matrix $\tilde{\vW}_t$.
\begin{algorithm}[H]
\begin{algorithmic}[1]
    \REQUIRE $\vy_t$ // measurement at time $t$
    \REQUIRE $\vmu_{t|t-1}$ // predicted mean
    \REQUIRE $(\vUpsilon_{t|t-1}, \vW_{t|t-1})$ // predicted diagonal and low-rank components
    \STATE $\tilde{\vW}_{t|t-1} \gets
        \begin{bmatrix} \vW_{t|t-1}& \vH_t\,\vR_t^{1/2} \end{bmatrix}$
    \STATE $\tilde{\vP}_t \gets \vUpsilon_{t|t-1}^{-1} -
        \vUpsilon_{t|t-1}^{-1}\,\tilde{\vW_t} \left( \vI_{d+\dimobs} + \tilde{\vW}_t^\intercal\,\vUpsilon_{t|t-1}^{-1}\,\tilde{\vW}_t \right)^{-1}\,\tilde{\vW}_t^\intercal\,\vUpsilon_{t|t-1}^{-1}$
    \STATE $\tilde{\vK}_t \gets \tilde{\vP}_t\,\vH_t^\intercal\,\vR^{-1}$
    \STATE $
        \vmu_t \gets \vmu_{t|t-1}  + \tilde{\vK}_t\,(\vy_t - \hat{\vy}_t)
    $
    \STATE // build low-rank component (using reduced-rank SVD)
    \STATE $\tilde{\vW}_t^\intercal\,\tilde{\vW}_t \gets \tilde{\vV}_t\,\tilde{\vS}_t^2\,\tilde{\vV}_t^\intercal$ // right-singular vectors
    \STATE $\tilde{\vU}_t \gets \tilde{\vW}\,\tilde{\vV}\,\tilde{\vS}_t^{-1}$ // left singular vectors
    \STATE $\vW_t \gets \tilde{\vU}_{:, :d}\,(\tilde{\vS}_t)_{:d,:}$ // low-rank construction
    \STATE $ (\vD_t)_i \gets \sum_{j=d}^{d+o} (\tilde{\vU})_{i,j}\,(\tilde{\vS}_t)_j\,(\tilde{\vU})_{i,j}\,(\tilde{\vS}_t)_j$ // dropped variance
    \STATE $\vUpsilon_t \gets \vUpsilon_{t|t-1} + \vD_t$
\end{algorithmic}
\caption{
    Update step for LoFi
}
\label{algo:lofi-update}
\end{algorithm}
In Algorithm \ref{algo:lofi-update}, Line 7 is computed for all $i=1,\ldots,D$ and $j=1,\ldots,d$;
and Line 8 is computed for all $i=1,\ldots,D$.
Finally
$\vA_{:, d:} = \begin{bmatrix} \vA_{:, 1} & \ldots & \vA_{:, d} \end{bmatrix}$
and similarly for $\vA_{d:, :}$.

\section{Summary of methods}\label{sec:scalability-summary}
In this section, we summarise the methods introduced in this chapter, namely,
the subspace method, the PULSE method, and the LoFi method.

\begin{table}[htb]
    \footnotesize
    \centering
    \begin{tabular}{l|ccc}
    \toprule
        Method & Assumption & Time complexity & Memory complexity\\
    \midrule
        subspace & $\vtheta_t = \vA\,\phiddensub_t + \phidden_*$ & $O(d\,\dimstate + d^3)$ & $O(d\,\dimstate + d^2)$\\
        LoFi & $\vSigma_{t}^{-1} = \vW_t\,\vW_t^\intercal + \vUpsilon_t$ & $O(\dimstate\,(d+o)^2 + (d+o)^3)$ & $O(d + d\,\dimstate)$\\
        PULSE & $\vtheta_t = ({\bf A}\,\phiddensub_t + \phidden_*, \plast_t)$ & $O(d_{\rm hidden}^3 + d_{\rm last}^3)$ & $O(\dimstate\,d_{\rm hidden} + d_{\rm hidden}^2 + d_{\rm last}^2)$\\
    \bottomrule
    \end{tabular}
    \caption{
        Time and memory complexity of the update step for various methods.
        The row \textit{Assumption} denotes the change from the assumptions in used in the EKF algorithm.
    }
    \label{tab:complexity-scalability}
\end{table}

\section{Experiments}\label{sec:experiments-scalability}

In this section, we present empirical results in which we evaluate the performance and speed (time to run)
of the methods presented in this chapter.
We also study the effects of various
hyper-parameters of our algorithm, such as how we choose the subspace.

\subsection{Online classification with a convolutional neural network}
In this experiment, we consider consider the problem of one-step-ahead classification
of a stream of images coming from the FashionMNIST dataset \citep{xiao2017fashionmnist}.
We test the Subspace method, the PULSE method, and the LoFi method
to train a modified LeNet5 architecture with ReLU activation;
we add an additional 20-unit dense layer
\citep{lecun1998lenet5}.

For each of the methods presented in this chapter, we consider
dimensions of sizes $d \in \{1, 5, 10, 25, 50, 70, 100\}$,
which corresponds to the subspace parameters for the Subspace method (Section \ref{sec:subspace}),
the hidden-layer subspace for PULSE (Section \ref{sec:PULSE}), and
the low-rank component for LoFi (Section \ref{sec:lofi}).

For the subspace and PULSE methods,
we estimate the projection matrix $\vA$ using the procedure outlined in Section \ref{subsec:warm-up-subspace}
using a warmup dataset of 2000 samples separate from the training data.

In the following figures we show the one-step-ahead classification accuracy for each of the methods
as a function of $d$,
using an exponentially-weighted moving average (EWMA) with a span value of 100 observations.

\paragraph{Subspace:}
Figure \ref{fig:fashion-mnist-subspace-rank} shows the result of the Fashion MNIST task for the subspace method.
We consider the additional dimensions $\{150, 200, 250\}$ to compare to the total number of parameters in PULSE (see below).
\begin{figure}[htb]
    \centering
    \includegraphics[width=0.9\linewidth]{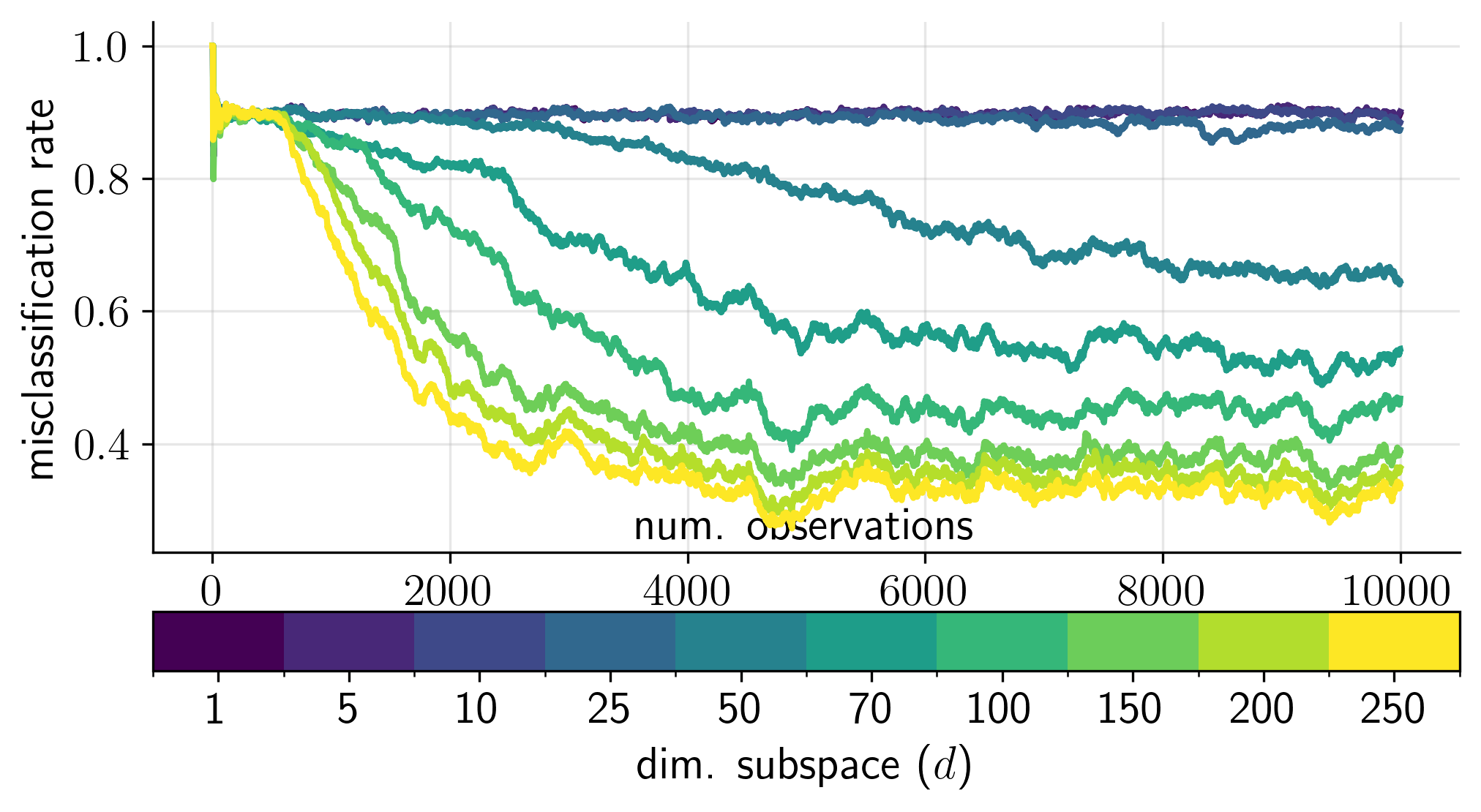}
    \caption{
    Comparison of the prequential accuracy on the Fashion MNIST dataset for the subspace method.
    The $y$-axis shows the EWMA accuracy using a span of 100 observations.
    }
    \label{fig:fashion-mnist-subspace-rank}
\end{figure}
We observe that for $d \leq 10$, the method performs no better than a random classifier, which corresponds
to a misclassification rate of $0.9$.
Then, for $d > 10$, the method improves its performance as $d$ increases.
This is because we allow the algorithm to consider higher degrees of freedom.

\paragraph{PULSE:}
Next, Figure \ref{fig:fashion-mnist-pulse-rank} shows the results for the PULSE method.
Here $d$ is the dimension of the subspace of the parameters in the hidden-layers.
\begin{figure}[htb]
    \centering
    \includegraphics[width=0.9\linewidth]{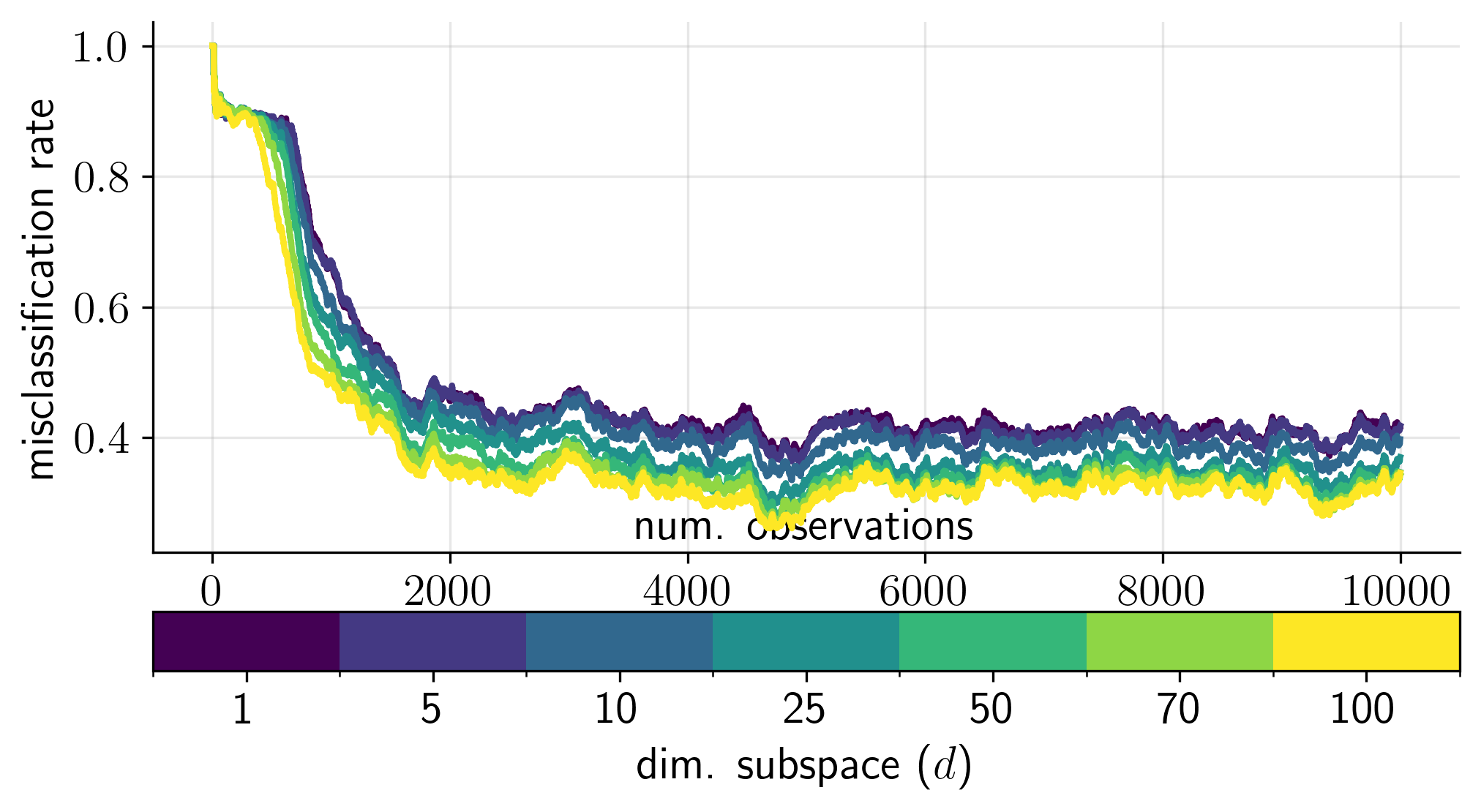}
    \caption{
    Comparison of the prequential accuracy on the Fashion MNIST dataset for the PULSE method.
    The $y$-axis shows the EWMA accuracy using a span of 100 observations.
    }
    \label{fig:fashion-mnist-pulse-rank}
\end{figure}
Because the output is $10$-dimensional and the second-to-last layer has $20$ units,
the total number of parameters that get updated is $200 + d$.
As a consequence, the performance of the model places strong emphasis on the last layer parameters.
We observe that $d=1$ provides a much better result than the subspace counterpart.
However, the performance of the method as we increase $d$ is not as stark as in the subspace case
shown in Figure \ref{fig:fashion-mnist-subspace-rank}.

\paragraph{LoFi:}
Finally, Figure \ref{fig:fashion-mnist-lofi-rank} the results for the LoFi method.
Here $d$ is the rank of the DLR matrix that characterises the posterior precision matrix.
For LoFi, all $\dimstate$ model parameters get updated, however,
the total number of meta-parameters required to perform the update step is $(D + D\,d)$:
$D$ diagonal terms and $D\,d$ low-rank terms.
\begin{figure}[htb]
    \centering
    \includegraphics[width=0.9\linewidth]{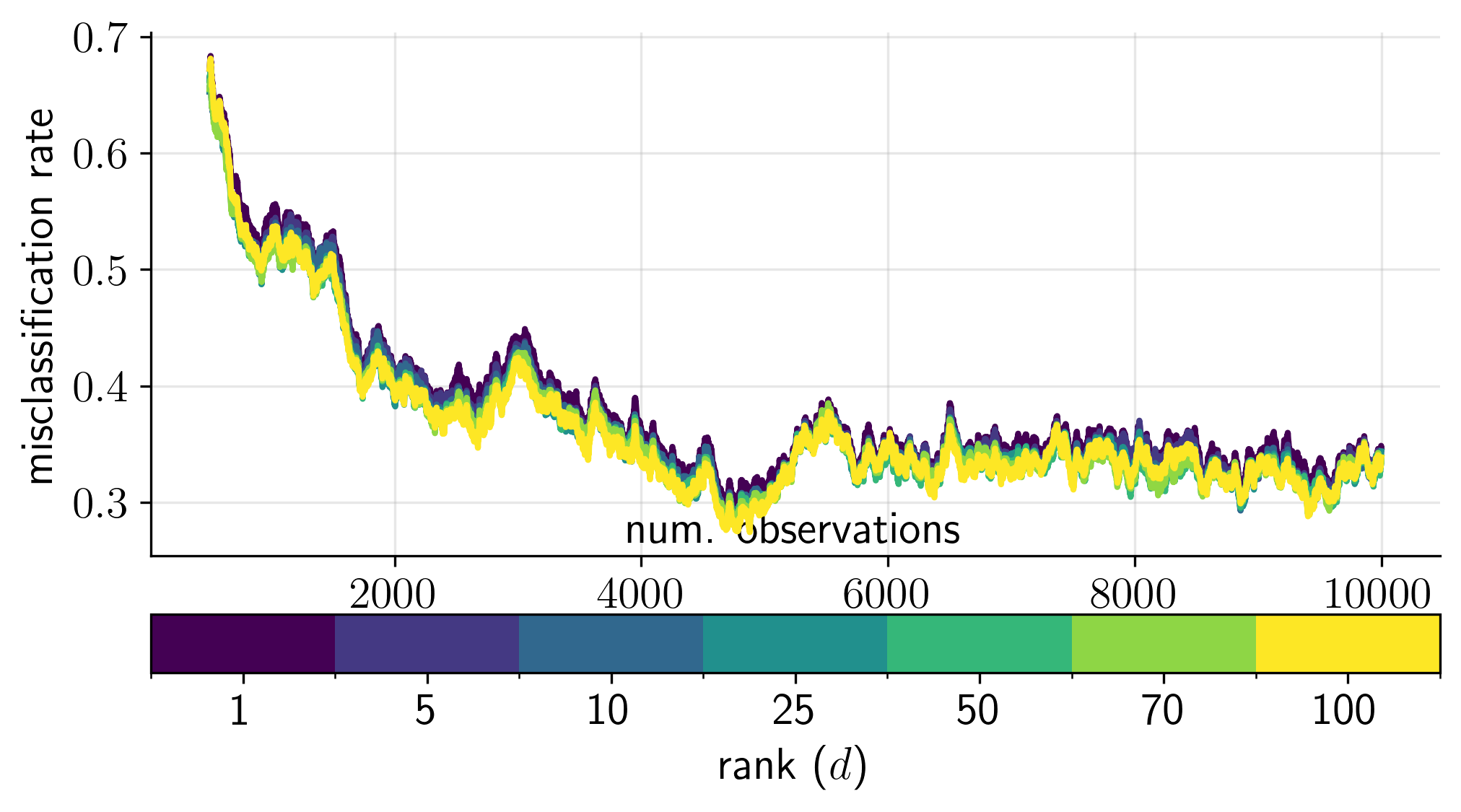}
    \caption{
    Comparison of the prequential accuracy on the Fashion MNIST dataset.
    The $y$-axis shows the EWMA accuracy using a span of 100 observations.
    }
    \label{fig:fashion-mnist-lofi-rank}
\end{figure}
We observe that for this dataset, the total rank does not have a strong influence on the performance of the method
as the rank increases, relative to the Subspace and the PULSE method.

\paragraph{Showdown:}
Figure \ref{fig:showdown-mnist-lofi-rank} compares the subspace, PULSE, and LoFi method on the Fashion MNIST task.
The $x$-axis shows the running time in seconds and the $y$-axis shows the final one-step-ahead misclassification rate.
Each marker corresponds to the final misclassification rate.
The lines correspond to the median performance across choices of $d$. 
\begin{figure}[htb]
    \centering
    \includegraphics[width=0.9\linewidth]{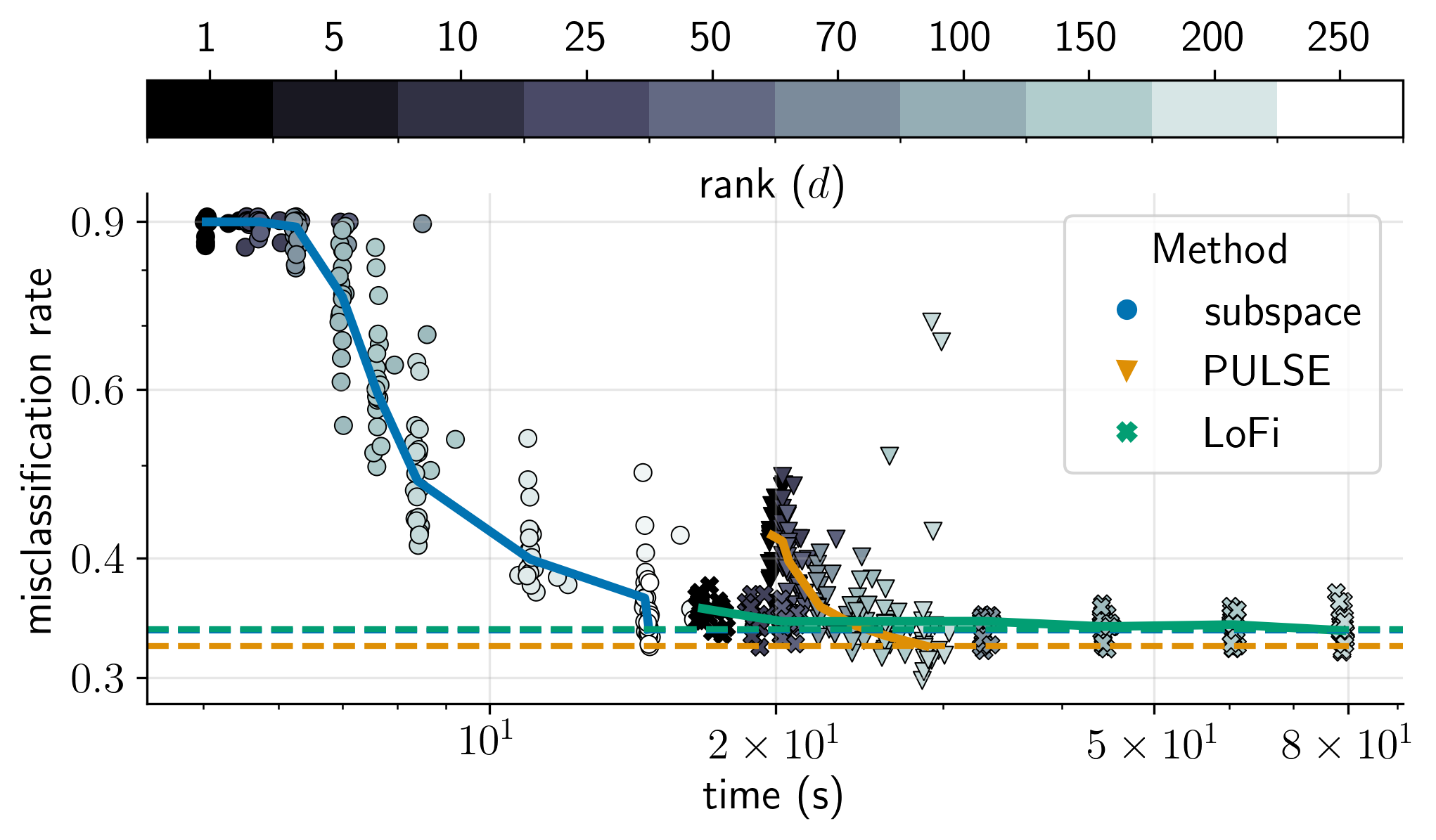}
    \caption{
    Comparison of the prequential misclassification rate on the Fashion MNIST dataset on 
    the last $8000$ observations.
    The dashed lines denote the best-performing configuration for each method.
    }
    \label{fig:showdown-mnist-lofi-rank}
\end{figure}
We observe that
LoFi is the method that maintains a consistent misclassification rate as a function of $d$.
However, the running time increases significantly as we increase $d$.
Next, PULSE starts a higher misclassification rate than LoFi. However,
at $d=70$ PULSE matches the performance of LoFi at around half the running time of Lofi.
Finally, the subspace method has the highest variability among the competing methods.
However, its performance significantly improves for $d \geq 100$
and matches the performance of LoFi with less running time.

\section{Conclusion}
In this chapter, we presented three methods for scalable filtering in high-dimensional parametric models,  
such as those found in neural networks.  
Specifically, we introduced:  
(i) the subspace method, which projects the weights of a neural network into a lower-dimensional subspace  
and performs filtering within this reduced space;  
(ii) the PULSE method, which extends the subspace method by defining a density function  
over the subspace and the last-layer parameters found using variational optimisation; and  
(iii) the LoFi method, which tracks model parameters using a low-rank plus diagonal representation  
of the posterior precision matrix.  

To evaluate these methods, we conducted experiments on an online classification problem  
using the Fashion MNIST dataset.  
The results demonstrate the applicability of these approaches to neural networks,  
highlighting their potential for scalable and efficient Bayesian filtering in high-dimensional settings.  
\chapter{Final remarks}
\label{ch:conclusion}

This thesis has proposed Bayesian filtering as a principled framework for adaptive, robust,  
and scalable online learning in the presence of non-stationary environments, misspecified models,  
and high-dimensional parameters.  

The primary contribution of this thesis has been the development of novel filtering methods,  
demonstrating their effectiveness in addressing sequential problems in machine learning.  
Specifically, we have introduced:  
(i) a unified framework for adapting to non-stationary environments,  
(ii) a novel, lightweight filtering method that is provably robust and a straightforward extension of the Kalman filter, and  
(iii) a suite of methods designed to reduce the time and memory complexity of classical filters  
when applied to high-dimensional parametric models.  

Despite these contributions, several limitations remain.
In particular, in future work, we would like to evaluate the proposed methods on novel architectures  
such as the
Transformer architecture
\citep{vaswani2017transformers,moreno2024rough} or
Graph Neural Networks \citep{scarselli2008graphneuralnet,arroyo2025vanishing} in temporal settings.  
Furthermore, future work will address (i) reinforcement learning problems, where agents can experience non-stationarity
even in static environments \citep{zhang2022catastrophicrl,waldon2024dare} and
(ii) sequential decision making problems in finance \citep{scalzo2021nonstationary,cartea2023optimal,cartea2024decentralized,drissi2022solvability,arroyo2024deep}.

These limitations present avenues for future research,  
including fully-online reinforcement learning and temporal problems involving non-stationarity on graph-structured data.  

In conclusion, this thesis demonstrates the versatility and potential of Bayesian filtering as a framework  
for parametric online learning and lays the foundation for further advancements in adaptive, robust,  
and scalable algorithms. By bridging the gap between traditional Bayesian approaches and the demands  
of modern machine learning, this work contributes to the development of tools that will become increasingly essential  
in real-world applications characterised by uncertainty, non-stationarity, and high-dimensionality.

\bibliographystyle{plainnat}
\bibliography{refs}

\end{document}